\documentclass[final,times,3p, authoryear]{elsarticle}
\pdfoutput=1

\journal{Computational Statistics \& Data Analysis}

\usepackage{graphicx}
  \graphicspath{{./figures/}}
  \DeclareGraphicsExtensions{.pdf,.png}
 
\usepackage[caption=false,font=footnotesize]{subfig}

\usepackage[english]{babel}
\usepackage{microtype}
\usepackage[utf8]{inputenc}
\usepackage{amsmath, amssymb, amsthm} 

\newtheorem{proposition}{Proposition}
\newtheorem{corollary}{Corollary}
\newtheorem{theorem}{Theorem}
\theoremstyle{definition} 
\newtheorem{definition}{Definition}
\newtheorem{remark}{Note}

\usepackage{placeins}

\usepackage{url}
\usepackage{xspace}
\usepackage[ruled]{algorithm}
\usepackage{algorithmicx}
\usepackage{algpseudocode} 

\usepackage{siunitx} 

\newcommand{\cNML}{c_{~\textnormal{NML}}}
\newcommand{\cMODL}{c_{~\textnormal{Enum}}}
\newcommand{\cGMODL}{c_{~\textnormal{G-Enum}}}

\newcommand{\cNMLindex}{\log  \binom{E}{K-1}}
\newcommand{\cNMLCOMP}{\log  ~{\mathcal{R}}^n_{\mathcal{M}}}
\newcommand{\cNMLlikelihoodpartition}{\log  \frac{n^n}{{h_1}^{h_1} \dotsm {h_K}^{h_K} }}

\newcommand{\cMODLindex}{\log ^* K}
\newcommand{\cMODLcutpointchoice}{\log  \binom{E+K-1}{K-1}}
\newcommand{\cMODLpartitionchoice}{\log  \binom{n+K-1}{K-1}}
\newcommand{\cMODLlikelihoodpartition}{\log  \frac{n!}{h_1!... h_K!}}
\newcommand{\clikelihooddistrib}{\sum^K_{k=1} h_k \log  E_k}

\newcommand{\cGMODLcutpointchoice}{\log^* G + \log  {{G+K-1}\choose{K-1}}}

\newcommand{\NMLname}{\texttt{NML}\xspace}
\newcommand{\ENUMname}{\texttt{Enum}\xspace}
\newcommand{\GENUMname}{\texttt{G-Enum}\xspace}

\newif\iffastdraft
\fastdraftfalse

\begin{document}

\begin{frontmatter}

\title{Fast and fully-automated histograms for large-scale data sets}

\author[univ1,company]{Valentina Zelaya Mendizábal}
\ead{valentina.zelaya@etu.univ-paris1.fr}

\author[company]{Marc Boullé}
\ead{marc.boulle@orange.com}

\author[univ2]{Fabrice Rossi\corref{mycorrespondingauthor}}
\cortext[mycorrespondingauthor]{Corresponding author}
\ead{Fabrice.Rossi@dauphine.psl.eu}

\address[univ1]{SAMM EA 4543, Université Paris 1 Panthéon-Sorbonne}
\address[univ2]{CEREMADE, CNRS, UMR 7534, Université Paris-Dauphine,
PSL University, 75016 Paris, France}
\address[company]{Orange Labs}

\begin{abstract}
  G-Enum histograms are a new fast and fully automated method for
  irregular histogram construction. By framing histogram construction
  as a density estimation problem and its automation as a model
  selection task, these histograms leverage the \emph{Minimum
    Description Length principle (MDL)} to derive two different model
  selection criteria. Several proven theoretical results about these
  criteria give insights about their asymptotic behavior and are used 
  to speed up their optimisation. These insights, combined to
  a greedy search heuristic, are used to construct histograms in
  linearithmic time rather than the polynomial time incurred by
  previous works. The capabilities of the proposed MDL density
  estimation method are illustrated with reference to other fully
  automated methods in the literature, both on synthetic and large
  real-world data sets.
\end{abstract}


\end{frontmatter}


\section{Introduction}
Histograms are popular non-parametric density estimators available in all
statistical computing packages. They are particularly adapted for univariate
visualisation (see for instance \cite{Zubiaga2016GraphicalPO}) and can serve as the
starting point for more complex analyses. Histograms are also quite useful
because they provide a compressed representation of the
distribution of a random variable. This property has been used for instance in
database query optimization, as illustrated in \cite{10.1093/comjnl/45.5.494,IOANNIDIS200319}. 

In theory, histograms require very few parameters and can adapt to
any kind of distribution given enough bins. In practice however, the choice of the
binning can have an unforeseeable effect on the accuracy of data estimation,
on the readability of the visualisation and on the size of the representation. 

Numerous strategies exist to automatically select
the number of bins in a regular histogram with equal width bins (for instance the
pioneering Sturges' formula \citep{SturgesRule1926} and the Freedman-Diaconis'
rule \citep{Freedman1981}). A major problem with regular histograms is that regions with higher or lower densities are treated the same, which gives histograms their limited reputation. For complex distributions, irregular histograms with bins of different widths are more adapted \citep{Rissanen1992-densityestimation}, but fully automated construction methods are scarce or limited, as pointed out in
\cite{Davies2009,RMG2010}: many methods rely on user adjustable
parameters, while others have high computational loads that hinder their use on large scale data. Scalable and fully automated approaches that specify the location, number and widths of histogram bins, based only on the observed properties of the data, are still rare overall.

Although alternative methods of density estimation, such as kernel density
estimators, are usually recommended as a fitting solution to this drawback, we
argue that histograms can stay relevant for density estimation if the choice
of binning is done properly, especially because of the ease with which they
are interpreted \citep{Zubiaga2016GraphicalPO}. Most fully automated methods in the literature
view histogram construction as a model selection task and implement it via
some form of penalized quality criterion that balances the quality of the
histogram as a representation of the data with its complexity. Among these,
the \emph{Minimum Description Length principle} (MDL
\citep{Rissanen78,Grunwald07}) provides a sound general framework to implement
model selection. The key idea of MDL is that any regularity in the data can be
used to ``compress the data'', i.e. to describe it in a shorter manner. More
formally, the best choice for a model and its parameters is the one that
minimizes the coding length of the model parameters and of the data given the
model.

MDL has been applied to histogram construction, using different derivations of
the coding length. For instance, the Bayesian Mixture criterion and a uniform
prior were used to construct irregular histograms in
\cite{Rissanen1992-densityestimation}. Another work formalized a MDL criterion
via the {\em Normalized Maximum likelihood (NML)} distribution - which, unlike
the Bayesian Mixture criterion, has several desirable optimality properties
\citep{pmlr-v2-kontkanen07a}. Notice however that the NML criterion for
histograms relies on a single user-defined parameter, the \emph{accuracy} at
which the data is to be approximated, $\epsilon$. In addition, while the
experimental results reported on Gaussian distributions are very good, this
approach has some scalability limitations. These issues are consequences of the
computational complexities of the NML criteria evaluation and of the search
for the optimal model in the space of all histograms.

We introduce a new MDL based histogram construction method that
aims to reduce the computational burden of previous solutions and to enable
automatic tuning of the accuracy parameter $\epsilon$. Our formulation is
based on a Bayesian maximum \emph{a posteriori} interpretation of MDL. The
criterion derived from this formulation does not contain the NML parametric
complexity term from \cite{pmlr-v2-kontkanen07a} which enables both a simpler
analysis and a faster evaluation. Leveraging the theoretical properties of the
proposed criterion, we derive a simple search heuristic over the space of
histograms with a  linearithmic time
complexity ($\mathcal{O}(n \log  n)$) rather than a polynomial one. As previous MDL based criteria, our
criterion uses an accuracy parameter. We study the effect of this parameter on
the histograms obtained by optimising our criterion. Then we introduce a new
\emph{granulated} criterion, derived from the base one, that automates the
choice of $\epsilon$.

The rest of this paper is organized as follows. Section \ref{sec:related-work}
provide a short overview of other fully automated histogram construction
methods. Section \ref{sec:formulation} specifies in details the
histogram construction problem. Section \ref{sec:nml} describes the NML approach
from \cite{pmlr-v2-kontkanen07a} and discusses its limitations. Section
\ref{sec:gmodl} introduces two enumerative criteria for histogram
construction, while their theoretical properties and consistency are discussed
in section \ref{sec:properties}. An efficient optimisation algorithm that
benefits from these properties is then presented in section
\ref{sec:optimisation}. In section \ref{sec:experiments}, experiments
demonstrate the performance of all three criteria and other state-of-the-art methods on synthetic and real large-scale datasets. Concluding remarks are given in section
\ref{sec:conclusion}.

\section{Related work}\label{sec:related-work}
As pointed out in \cite{Birge2006} and \cite{Davies2009}, among others, almost all non-heuristic
automatic histogram construction is based on a notion of risk. The
goal is to minimise
\begin{equation}\label{eq:risk}
R_n(f, \hat{f}_{\theta}, l)=\mathbb{E}_f\left[l\left(f,\hat{f}_{\theta}(x_1,\ldots,x_n) \right)\right],
\end{equation}
where $f$ is the true unknown density of the data, $x_1,\ldots,$ $x_n$ a 
sample from $f$, $l$ a loss function, and $\hat{f}_{\theta}$ the estimation
procedure with its parameters $\theta$ (e.g. the number of bins in a regular
histogram). An ideal choice of $\theta$ would minimise the risk. 

Most statistical computing software still include simple
methods, such as the Sturges rule \citep{SturgesRule1926}, which divides the
data domain into $K^*=1 + \log _2 n$ intervals. Other more
principled approaches derived from asymptotic analyses of the risk are also included. For
instance, Scott's rule \citep{Scott1979OnOA} is derived from the
asymptotic behaviour of the risk for the squared L2-loss $l(f,
g)=\|f-g\|_2^2$. Notice that this type of analysis relies on smoothness
assumptions on the true density $f$, which are used to derive the optimal bin
width from characteristics of this unknown density. Those characteristics are
in turn estimated from the data either using the Normal distribution as a
reference as in Scott's rule, using heuristic consideration as in Freedman-Diaconis'
rule \citep{Freedman1981} or based on plug-in methods as in
\cite{Wand1997plugin}. See \cite{Sulewksi2020} and the references therein
for other examples of such approaches. Notice however that none of these approaches
can be used to build {\em irregular} histograms. As pointed out earlier, being able to have intervals of varying widths is advantageous to describe denser and narrower ranges of values.  

An alternative to asymptotic analyses is to leverage the
cross-validation (CV) principle to directly estimate the risk. For instance
\cite{Rudemo1982EmpiricalChoice} proposes to use leave-one-out estimates of
the risk. Those can be computed in closed-form for the squared L2-loss and the
Kullback-Leibler loss (see \cite{hall1990akaike} for the latter). Other
versions of the CV principle have been applied to risk estimation and model
selection, in particular the leave-p-out cross-validation which generalizes
leave-one-out and the V-fold cross-validation. While leave-p-out CV is
generally too computationally intensive as it needs evaluating $n \choose p$
models, Celisse and Robin provide in \cite{CelisseRobin2008} an efficient
closed-form formula for the squared L2-loss of histograms, including irregular
ones. The same paper shows that the V-fold CV has a larger variance than the
equivalent leave-p-out CV (i.e. when $p=\frac{n}{V}$) and is therefore less
reliable. Additional results on leave-p-out CV are proved in
\cite{Celisse2014Optimal}. It is shown that the leave-one-out CV is optimal in
terms of risk estimation (for the squared L2-loss) but not in terms of model
selection, i.e. when the risk estimate is used to select $\theta$ in
\eqref{eq:risk}. In this latter task a leave-p-out CV is preferable and p
should be adapted both to the size of the data and to the complexity of the
model collection. Unfortunately, no rule for choosing an optimal value is currently known and experiments on simulated data confirm a strong link between
$\frac{p}{n}$ and the quality of the associated histogram, see \cite{Celisse2014Optimal}. 

Another way to address limitations of the simple techniques consists in using
a penalized likelihood approach, a very common approach in model selection
problems. The classical penalties are Akaike's information
criterion AIC \citep{Akaike1998} (used by \cite{Taylor1987AIC} for histograms) and
Schwarz' Bayesian information criterion BIC \citep{Schwarz1978BIC}. As shown in
\cite{RMG2010} AIC systematically underpenalises complex histograms (even
regular ones) which makes it unsuitable for model section in this particular
task (BIC does not suffer from this issue). Like simpler rules, AIC and BIC are
based on asymptotic considerations only. This motivated Birgé and Rozenholc to use
non asymptotic results on penalized maximum likelihood from
\cite{Castellan1999} to derive a modified AIC
criterion for regular histograms \citep{Birge2006}. An even stronger penalty is
needed for irregular histograms, as shown in \cite{RMG2010}. 

Penalized estimators can also be constructed using Rissanen's approaches to
minimal complexity. Hall and Hannan derive in \cite{HallHannan1988Stochastic}
two such criteria: one is based on Rissanen's stochastic complexity ideas
\citep{Rissanen1986} and the other one on Rissanen's minimum description length
(MDL, \cite{Rissanen78}). An extension of the stochastic complexity based
model to irregular histograms is proposed in
\citep{Rissanen1992-densityestimation}. As pointed out in the introduction, another MDL
approach is proposed \citep{pmlr-v2-kontkanen07a}, using Normalized Maximum
Likelihood (NML). We will describe in more detail this approach in Section
\ref{sec:nml}. 

Penalized likelihood approaches can generally be seen from a Bayesian point of
view as instances of the maximum a posteriori (MAP) principle. Several authors have
argued the advantages of using a MAP approach for histogram construction. For instance,
the Bayesian block method \citep{Scargle2013} addresses optimal segmentation
with a MAP approach and proposes to apply its general framework to irregular
histogram construction. A similar solution that uses a different likelihood and a
different prior on parameters is proposed in \cite{knuth2019histograms} for
regular histograms. 

Most of other methods introduced are for regular histograms (see for
instance references in \cite{Birge2006,Davies2009,RMG2010}). Among methods
adapted to irregular histograms, the taut string approach
\citep{Davies2004,Davies2009} is interesting as it provides one of the few
fully automated construction methods. It has been shown to give
good results in extensive benchmarks, especially in terms of identifying the
modes of a distribution \citep{Davies2009,RMG2010}. The approach
is based on constructing a piecewise linear spline of minimal length, the taut
string, that lies in a tube around the empirical cumulative distribution
function. A histogram is constructed from the derivative of the taut string,
see \cite{Davies2004,Davies2009} for details. While the method is fully
automated, it leverages the so-called $\kappa$-order Kuiper metrics whose
parameters are set to some default values based on the data size via a
heuristic. As stated by the authors, the performance of the taut string
approach depends on those parameters. In addition, histograms
built by this approach are not maximum likelihood histograms:
densities associated to intervals are not directly derived from data
frequencies in those intervals.

Another approach that does not fit directly into the penalized likelihood
framework is the essential histogram proposed by
\cite{LiEtAl2020EssentialHistogram}. The main idea of the method is to
build a set of histograms that are close enough to the empirical
distribution with respect to a collection of statistical tests. The
essential histogram is the simplest of the collection. While
appealing, the method uses a user defined confidence region specified
by a unique parameter whose influence on the result cannot be
discarded. The method is therefore not fully automated. 

In the present section, we focused on the criteria and heuristics proposed to
automatically select a possibly irregular histogram from the data. We have
deliberately postponed the discussion of the algorithmic considerations to
the next section.

\section{Irregular histograms and dynamic programming}\label{sec:formulation}
A histogram on an interval $[a, b]$ is associated to a partition of $[a, b]$
into $K$ intervals $\{ [a, c_1],]c_1, c_2], ..., ]c_{K-1},b]\}$. Optimising a
quality criterion for a histogram implies to optimise over a subset of such
partitions. The case of regular histograms is very simple, as there is only one
partition associated to a number of sub-intervals $K$. On the contrary,
irregular histograms are associated to an infinite set of partitions, which is
not tractable in general.

\subsection{Restricted irregular histograms}\label{sec:restr-irreg-hist}
There are essentially two ways to address this problem. The data driven
approach consists in restricting the sub-interval endpoints $c_k$ to be
observations, as in \cite{Davies2009,RMG2010}. Additional regularisation can
be implemented by forbidding too short intervals but experiments in
\cite{RMG2010} show this has a limited and mostly negative impact on performance. The second approach consists in using a regular grid from which
the endpoints are selected
\citep{CelisseRobin2008,pmlr-v2-kontkanen07a,Rissanen1992-densityestimation}. This
can be seen as enforcing an approximation accuracy on the observations, as
detailed below. In both cases, the optimisation problem is now restricted to a
finite set of partitions.

In this paper we use the fixed grid approach. We assume (for now) a given
approximation accuracy $\epsilon$. An observation $x$ in the interval $[x_{\min},
x_{\max}]$  is approximated by the closest value $\widetilde{x}$
in $\mathcal{X}=\{ x_{\min} + t\epsilon ; t = 0,... , E-1\}$, $E =
1+\dfrac{L}{\epsilon}$ and  $L = x_{\max} - x_{\min}$ is
the `{\em domain length}' of the data. We restrict the possible accuracies
such that $E \in \mathbb{N}$. Following \cite{pmlr-v2-kontkanen07a}, we define
the set of possible endpoints for sub-intervals as
\begin{equation*}
\mathcal{C} = \left\{x_{\min} - \frac{\epsilon}{2} + t\epsilon ; t = 0,... , E \right\}.
\end{equation*}
These endpoints define $E$ {\em
  elementary bins}  of length $\epsilon$, which we call $\epsilon$-bins (see
figure \ref{fig:interval-choices}).
\begin{figure}[htbp]
\centering
\includegraphics[width=\linewidth]{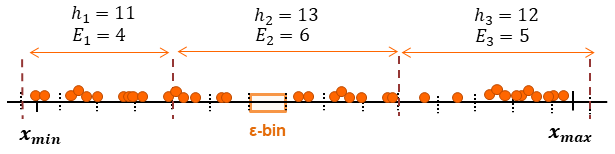}
\caption{A possible choice of intervals, with their respective data-counts
  $(h_k)_{1\leq k\leq K}$ and lengths expressed in number of $\epsilon$-bins $(E_k)_{1\leq k\leq
  K}$.}
\label{fig:interval-choices}
\end{figure}
They are the building blocks of histogram intervals: possible grouping of consecutive $\epsilon$-bins into $K$ intervals, with $K$ ranging from 1 to $E$, defines a
histogram. A histogram $\mathcal{M}$ is fully specified by $\mathcal{M} = (K,
C, (h_k)_{1 \leq k \leq K})$ where $K$ is the number of intervals,
$C=(c_k)_{0\leq k\leq K}\subset
\mathcal{C}$ are the endpoints and $h_k$ are the counts in each bin (we have
$c_0=x_{\min} - \epsilon/2$ and $c_K=x_{\max}+\epsilon/2$). Notice
that such a histogram is not bound to any given data set, despite the use of the
term ``counts'' for the $h_k$. Most of the methods studied in this
paper, including our proposals, follow a maximum likelihood principle
and therefore given $K$ and $\mathcal{C}$, the $h_k$ will be directly
computed from the observations. Considering arbitrary histograms is
nevertheless useful to include non maximum likelihood methods such as
the taut string approach from \cite{Davies2004,Davies2009}, and to
define easily prior distributions on histograms (see Section 
\ref{sec:prior-distribution}). 

The piecewise constant density associated to $\mathcal{M}= (K,
C, (h_k)_{1 \leq k \leq K})$ is given by
\begin{align}
  \label{eq:density}
  f_{\mathcal{M}}(x)=&\frac{1}{n}\sum_{k=1}^K\frac{h_k}{c_k-c_{k-1}}\mathbb{I}_{]c_{k-1},
                       c_k]}(x),\\
  =&\frac{1}{n\epsilon}\sum_{k=1}^K\frac{h_k}{E_k}\mathbb{I}_{]c_{k-1},
                       c_k]}(x),\nonumber
\end{align}
where $E_k$ is the number of $\epsilon$-bins in interval $[c_{k-1}, c_k]$,
$n=\sum_{k=1}^Kh_k$ and $\mathbb{I}_A$ is the indicator function of the 
set $A$. Notice that the intervals are exclusive of their lower bound, even
for the first interval $]c_0, c_1]$. This is by virtue of the definition of
$c_0$ which is outside of the range of the data: no observation can be
missed despite the exclusion of $c_0$ from the formal definition of the
density. This enables us to keep a simple notation but it has no effect
on the rest of the analysis.  

\subsection{Dynamic programming}\label{sec:dynamic-programming}
While the two classical restrictions over the partitions we highlighted (data-driven grids and fixed grids) make the optimisation space finite, it remains large. There are $E\choose K-1$ different partitions of $E$ elementary bins in $K$ intervals, and $E$ may be big.
However, most criteria
used in the literature are additive and can therefore be optimised using the
dynamic programming principle \citep{Bellman1961Function}. This was proposed
originally in \cite{Kanazawa198OptimalVariable} for irregular histograms with
endpoints constrained to observations. Since then, dynamic programming has
been used systematically for all irregular histogram estimation techniques we
are aware of.

Provided the chosen criterion can be computed efficiently, dynamic programming
has a computational complexity of $\mathcal{O}(E^2 \cdot K_{max})$ when the
histograms have at most $K_{max}$ bins (and thus a maximal complexity of
$\mathcal{O}(E^3)$).

\subsection{Data driven grid versus fixed grid}\label{sec:data-driven-grid}
Both finite grids have advantages and drawbacks. Data driven grids do not
introduce an additional precision parameter but this comes with a potential
large computational cost. Dynamic programming has a computational
cost of $\mathcal{O}(n^2 \cdot K_{max})$ for a data driven grid with $n$
observations and a limit of $K_{max}$ bins in the histograms. If one sets $K_{max}=n$
to avoid introducing yet another parameter, data driven grids lead to a cubic
cost $\mathcal{O}(n^3)$ and thus do not scale to large data sets. To
circumvent this problem \citep{RMG2010} use a greedy clustering of the data
driven grid into $\max(n^{\frac{1}{3}},100)$ bins prior to applying a dynamic
programming approach. This blurs the distinction between data driven grids and
fixed ones, and introduces a new parameter (the cut-off value of 100 used in
\cite{RMG2010}). 

Another limitation of data driven grids is that they assume a perfect
representation of real numbers. This is acceptable for small data sets but for
large scale data sets, the probability of getting identical observations
increases and the limits of real number representation start to play a
role. This problem is discussed in
\cite{knuth2019histograms,KnuthEtAl2006Rounded} and in Section
\ref{sec:role-appr-accur}.

Finally in a Bayesian perspective, endpoints prior distributions have
to be specified. With a data driven grid, continuous densities are
needed while discrete distributions are needed when using a fixed
grid. As a consequence, specifying priors that favor simple histograms is easier in
the discrete case than in the continuous one. 

Fixed grids bypass the above issues, at the cost of introducing a new
accuracy parameter. We propose in Section \ref{sec:g-enum-criterion} a way to automate the choice of this parameter. This mitigates the drawbacks of the fixed grid, while keeping all its advantages.

\section{An NML criterion for histogram density estimation} \label{sec:nml}
Among minimum description length approaches, Kontkanen and Myllymäki's solution
\citep{pmlr-v2-kontkanen07a} derived using the Normalized Maximum
Likelihood (NML) is the closest to our proposal. This approach consists in
minimising a criterion $\cNML$ over all possible histograms defined on a fixed
grid as in Section \ref{sec:restr-irreg-hist}. The criterion is evaluated for
a data set $D=(x_i)_{1\leq i\leq n}$ of $n$ observations. 

To ease the comparison of
this criterion to our own, we report the following simplified expression for
$\cNML$, using notations of Section \ref{sec:restr-irreg-hist}
(the derivation is provided in \ref{apx:KM-rewrite}):
\begin{align}\label{eq:cNML}
\cNML(\mathcal{M}|D)   = &\cNMLindex + \cNMLCOMP\\ 
& + \cNMLlikelihoodpartition +  \clikelihooddistrib,   \nonumber
\end{align} 
where $\cNMLCOMP$ is called the {\em NML parametric complexity}. Notice that
the $h_k$ are computed from the data set $D$ based on the histogram
specification (see Section \ref{sec:likelihood} for a discussion on this aspect). 

The NML parametric complexity can be interpreted as the logarithm of the number of
different distributions for a given model class \citep{Grunwald07}. The term
$\cNMLindex$ represents the code length of the choice of $K-1$ cut-points
among the $E$ possible candidates. Notice that this term is not
directly derived from the NML distribution but was added by the authors as a
way to index the different model choices, as recommended in \cite{Grunwald07}.
The last two terms represent the log-likelihood of the partition of $n$
entries into $K$ parts of size $(h_k)_{1 \leq k \leq K}$ and the log-likelihood of
data distribution within each interval respectively.

The NML density has several important theoretical optimality properties: it is
the {\em minimax optimal universal model} \citep{Shtarkov1987}. Additionally,
codes that are derived from it minimize the expected regret when choosing the
worst-case generating density model for the data \citep{Rissanen2001}. In
\cite{pmlr-v2-kontkanen07a}, the authors show how to
find the NML-optimal cut points via a dynamic programming in a polynomial time
with respect to $E$, the total number of cut-points available for
selection. Furthermore, their experiments showed that the minimization
of the \NMLname criterion produced good quality histograms of several Normal
distributions of $n=10,000$ samples, all for a fixed approximation accuracy
($\epsilon = 0.1$).

We argue however that the \NMLname approach has the following limitations.
\begin{description}
\item[\textbf{NML computation:}] the NML approach results in a
  criterion whose exact computation remains costly. The cost of computing the NML parametric
  complexity was reduced from $\mathcal{O}(n^2)$
  \citep{KontkanenBMRT03} to  $\mathcal{O}(n+K)$ algorithm in
  \cite{Kontkanen2007}. 
          
Approximations of the NML parametric complexity have been introduced to reduce its computation time. For instance a $\mathcal{O}(\sqrt{dn}+K)$ algorithm (where $d$
is the precision in digits) is given in \cite{Mononen2008}. Szpankowski
proposes in \cite{Szpankowski98} an approximate computation in  $O(1)$ time
but with an asymptotical error in $\mathcal{O}(n^{-3/2})$. The
trade-off of using these faster methods is either a loss in precision or an
error that is hard to evaluate non asymptotically w.r.t. n and K.

  Notice that the main issue of the computational cost of the NML complexity is
  induced by the need to evaluate it at least $\mathcal{O}(n)$ times to build an
  optimal histogram, leading to minimal total cost of $\mathcal{O}(n^2)$ for an
  exact calculation. 

	\smallskip
  \item[\textbf{Choice of $\epsilon$:}] Although the authors argue that the
    sole user-defined parameter $\epsilon$ does not play a role in the model selection
    process, we show on the contrary in Section \ref{sec:role-appr-accur} that its impact 
    cannot be overlooked. An automated choice is therefore needed. 
  \end{description}
Notice also that \citep{pmlr-v2-kontkanen07a} uses a dynamic programming
approach to optimise $\cNML$ which has a $\mathcal{O}(E^2 \cdot K_{max})$ or a
$\mathcal{O}(E^3)$ computational cost, as already explained in Section
\ref{sec:dynamic-programming}. 

We improve over the NML approach on those three issues by introducing in
Section \ref{sec:gmodl} an enumerative criterion that avoids the NML computation
cost, in Section \ref{sec:g-enum-criterion} a granulated version of the our
criterion that automates the choice of the precision parameter and finally in
Section \ref{sec:optimisation} greedy search heuristics that together reduce the
global computational cost from $\mathcal{O}(E^3)$ to $\mathcal{O}(n\log n)$. 

\section{G-Enum: a granulated enumerative criterion}\label{sec:gmodl}
This paper's first contribution is the introduction of a  granulated enumerative two-part code for histogram density estimation, which we call \GENUMname. It builds on a simpler enumerative approach, for which we introduce granularity later.  

\subsection{Enumerative criterion}
Our enumerative criterion for histogram model selection exploits a maximum a
posteriori Bayesian interpretation of the MDL principle which has the same
optimality properties as the NML distribution \citep{BoulleEtAlArxiv16}. In
this framework, the best model is the one that maximises the probability of
the model given the data, $p(\mathcal{M} |D)$ where $D=(x_i)_{1\leq i\leq n}$
is a data set of $n$ observations (we allow non unique values in $D$). This
formulation is well known to be equivalent to a penalised likelihood
approach. Indeed we have
\begin{equation*}
\log p(\mathcal{M} |D)=\log p(\mathcal{M})+\log p(D|\mathcal{M})-\log p(D).
\end{equation*}
As $\log p(D)$ does not depend on $\mathcal{M}$, maximising $p(\mathcal{M}
|D)$ is equivalent to maximising a compromise between a large likelihood $\log
p(D|\mathcal{M})$ and a small prior probability $\log p(\mathcal{M})$ for a
complex model and the reverse for a simple one. In practice, we minimise an enumerative
criterion 
\begin{equation}
\cMODL(\mathcal{M}|D)= -\log p(\mathcal{M}) - \log p(D|\mathcal{M}).
\end{equation}

\subsubsection{Prior distribution}\label{sec:prior-distribution}
Using notations from Section \ref{sec:restr-irreg-hist}, where $K$ is the number of intervals, $C$ is the set of endpoints of the histograms invertals and $\{h_k\}$ are the counts of values in each of these intervals,
we factor $p(\mathcal{M})$ as 
\begin{align*}
 p(\mathcal{M})= p(K) \cdot~ p(C|K) \cdot~ p &((h_k)_{1 \leq k \leq K} | K, C),
\end{align*}
without loss of generality. Then, we use a hierarchical prior on the
parameters and we assume conditional independence of $C$ and $\{h_k\}_{1 \leq
  k \leq K} $ given $K$. Thus, we have
\begin{align*}
 p(\mathcal{M})= p(K) \cdot~ p(C|K) \cdot~ p &((h_k)_{1 \leq k \leq K} | K),
\end{align*}
with the following components:
\begin{itemize}
\item \textbf{Number of intervals}~

  For the number of intervals $K$, we could use a uniform prior distribution
  leading to $p(K)=\frac{1}{E}$, but Rissanen's universal prior for integers \citep{rissanen1983}
  is more adapted as it favors small integers, i.e. simpler histograms. We
  have
  \begin{equation}
p(K)=\exp(-\cMODLindex),
\end{equation}
with 
$\log^* K=  \log _{2}\left(\kappa_{0}\right)+\sum_{j>1} \max \left(\log _{2}^{(j)}(K), 0\right)$,
where $\kappa_{0}=\sum_{p>1} 2^{-\log _{2}^{*}(p)}=2.865$ and $\log _{2}^{(j)}(K)$
is the $j$-th composition of $\log _{2}$, i.e. $\log _{2}^{(1)}(K)=\log
_{2}(K)$ and $\log _{2}^{(j)}(K)=\log_2(\log _{2}^{(j-1)}(K))$.

\medskip

For the rest of the parameters, we use uniform priors.

\medskip

\item \textbf{Interval composition}~

  Specifying interval endpoints $C$ is equivalent to specifying the number of
  elementary bins gathered into each interval. We use a uniform distribution
  over all the ordered partitions of the $E$ elementary bins into $K$
  contiguous subsets of size $E_1, E_2, \dots E_K$ such that $E_1 + E_2 + \dots + E_K=E$. We
  allow empty subsets, which leads to
\begin{equation}
p(C|K) = \frac{1}{\binom{E+K-1}{K-1}}.
\end{equation}
Notice that we could forbid empty subsets (i.e. identical endpoints) which
would lead to $p(C)=\frac{1}{\binom{E}{K-1}}$ but this later solution is non
monotone with respect to $K$ and favors values around $\frac{E}{2}$ which is
not a desirable property.

\item \textbf{Prior on $(h_k)_{1 \leq k \leq K}$}~

  For a given number of intervals $K$, the $(h_k)_{1 \leq k \leq K}$ specify
  the number of observations in each interval. A value for $\{h_k\}_{1 \leq k
    \leq K}$ is therefore a vector of $K$ non negative integers which sum to
  $n$. There are of $\binom{n+K-1}{K-1}$ such vectors and we use a uniform distribution
  over them, leading to 
\begin{equation}
p ((h_k)_{1 \leq k \leq K} | K) = \frac{1}{\binom{n+K-1}{K-1}}
\end{equation}
\end{itemize}

\subsubsection{Likelihood}\label{sec:likelihood}
The likelihood $p(D|\mathcal{M})$ is obtained using a generative model for
histograms specified by $\mathcal{M}$ which is also based on uniform
distributions. Importantly, we do not assume
the observations to be independent, contrarily to most of the histogram models. 

Notice also that the $\{h_k\}_{1 \leq k\leq K}$ in $\mathcal{M}$ are not
computed from the data but parameters. This means that $p(D|\mathcal{M})$ is
non zero if and only if $\mathcal{M}$ and $D$ are \emph{compatible}.
\begin{definition}
A data set $D=(x_i)_{1\leq i\leq n}$ and a histogram $\mathcal{M}= (K,
C, (h_k)_{1 \leq k \leq K})$ are \emph{compatible} if and only if 
\begin{equation*}
  \forall k\in\{1,\ldots, K\},
  \left|\left\{i\in\{1,\ldots, n\}\mid x_i\in ]c_{k-1},c_k] \right\}\right|=h_k,  
\end{equation*}
where $|S|$ denotes the cardinal of set $S$.
\end{definition}
Notice that in practice the data ``set'' $D$ can contain multiple times the
same value and thus we take into account the index set when counting observations in
this definition (see Section \ref{sec:role-appr-accur} for additional details
on this aspect).

Given $\mathcal{M}$, a data set $D$ is generated by a hierarchical
distribution based on uniform distributions.
\begin{enumerate}
\item Observations are generated by first choosing which interval is
  responsible for generating each individual observation. This corresponds to
  building a mapping $I_{map}$ from $1,\ldots,n$ to $1,\ldots, K$ specifying
  that $x_i$ will be generated in interval $I_{map}(i)$. $I_{map}$ is compatible
  with $\mathcal{M}$ if
  \begin{equation*}
\forall k \in\{1,\ldots, K\}, \left|\left\{i\in \{1,\ldots, n\}\mid
    I_{map}(i)=k\right\}\right|=h_k. 
\end{equation*}
There are $\frac{n!}{h_1!... h_K!}$ such mappings and we assume a uniform
distribution on them, that is
\begin{equation}
p(I_{map}|\mathcal{M})=\frac{h_1!... h_K!}{n!}
\end{equation}
\item Then given $I_{map}$ we generate observations assigned independently to distinct intervals, i.e.
  \begin{equation}
p(D|I_{map},\mathcal{M})=\prod_{k=1}^Kp((x_i)_{I_{map}(i)=k}|\mathcal{M}).
\end{equation}
\item In a given interval, observations are generated by choosing their elementary
  bins independently and uniformly. That is we assume
  \begin{equation}
p((x_i)_{I_{map}(i)=k}|\mathcal{M})=\left(\frac{1}{E_k}\right)^{h_k},
\end{equation}
as there are $E_k$ elementary bins in interval $k$ and we have $h_k$ data
points to generate. Notice than if $h_k=0$, then we can have $E_k=0$ as
in this case $\forall i,\ I_{map(i)}\neq k$ by compatibility. 
\end{enumerate}
Then when $\mathcal{M}$ and $D$ are compatible, we have
\begin{equation}
p(D| \mathcal{M}) = \frac{h_1!... h_K!}{n!} \cdot \prod^K_{k=1} \left( \frac{1}{E_k} \right)^{h_k}.
\end{equation}
The simple enumerative criterion for histogram density estimation is thus
given by
\begin{align}\label{eq:modl}
  \cMODL(\mathcal{M}|D)=& \cMODLindex+ \cMODLcutpointchoice \nonumber\\
&  + \cMODLpartitionchoice \nonumber\\
& + \cMODLlikelihoodpartition + \clikelihooddistrib.
\end{align}
We use the convention that $0\times \log 0=0$ to avoid the case
where $h_k=0$. Notice that equation \eqref{eq:modl} is valid only when $\mathcal{M}$ and
$D$ are compatible. When $\mathcal{M}$ and $D$ are not compatible,
$p(D|\mathcal{M})=0$ and we set accordingly
$\cMODL(\mathcal{M}|D)=+\infty$. As a consequence, histograms that are
not compatible with the data are excluded from the solution space of
our approach. This ensure that histograms obtained by minimizing the
proposed criterion use a classical maximum likelihood estimation of
the density in each of their intervals. 

\subsubsection{Comparison to the NML criterion}

A side by side comparison of the \ENUMname and \NMLname criteria terms is presented in table \ref{tab:comp}. 
\renewcommand{\arraystretch}{2.6}
\begin{table*}[htbp]
\centering
\caption{Term comparison of the \NMLname   and enumerative criteria}
\begin{tabular}{p{2.5cm}p{3.5cm}p{4.8cm}p{2.8cm}}
\hline
Crit. & Indexing terms & Multinomial terms &  Bin index terms\\\hline\hline
\NMLname \newline \cite{pmlr-v2-kontkanen07a}& $ \displaystyle \cNMLindex$ &$\displaystyle \cNMLCOMP +  \cNMLlikelihoodpartition$ & $ \displaystyle \clikelihooddistrib$ \\\hline
\ENUMname &  $\displaystyle \cMODLindex + \cMODLcutpointchoice$& $\displaystyle  \cMODLpartitionchoice + \cMODLlikelihoodpartition$ & $\displaystyle \clikelihooddistrib$\\ \hline
\end{tabular}
\label{tab:comp}
\end{table*}
\renewcommand{\arraystretch}{1.5}

The terms that represent the indexing of $\epsilon$-bins within each interval are exactly the same for both criteria. The differences lie in model indexing and the encoding of the multinomial distributions. 

\paragraph*{Model indexing terms}~~
As in the \NMLname   criterion, we too have a term that serves as an index for model families. In the \ENUMname criterion, model indexing is done through a $\cMODLcutpointchoice$ term. This term is monotone with $K$: it penalizes models with many intervals. The corresponding term in the \NMLname  criterion is not monotone and might thus favour choices with many cut-points (especially when $K\sim E$). 
 
\paragraph*{Multinomial terms}~~
These terms encode the multinomial distributions of $n$ observations
into $K$ intervals. Both versions are universal codes
\citep{Grunwald07} and have been proved to have the same optimality
properties \citep{BoulleEtAlArxiv16}. What sets them apart is that the
enumerative version is easier to compute and to analyze.

The classical \NMLname  parametric complexity ($\cNMLCOMP$) is replaced by a simpler
term in the \ENUMname approach ($\cMODLpartitionchoice$). As shown in
\cite{pmlr-v2-kontkanen07a}, the \NMLname  complexity can be computed exactly in
$\mathcal{O}(n+K)$ time. In stark contrast, the equivalent term in our
enumerative criterion can be computed in $\mathcal{O}(1)$ time. 

As we can see, both criteria are very similar, though the \ENUMname criterion is clearly simpler to compute and to analyse theoretically, using its closed form expression (see Section \ref{sec:properties}).

\subsection{G-Enum criterion}\label{sec:g-enum-criterion}
The \ENUMname criterion solves the computational issues induced by the NML
parametric complexity without introducing approximations, but we still have to
choose the approximation accuracy $\epsilon$ associated to the fixed grid. We
introduce \emph{granularity} to mitigate this issue. 

\subsubsection*{Granularity and $\epsilon$}
The main issue with $\epsilon$ is that it plays two roles. It serves as a way
to set a precision limit for the representation of real numbers, as discussed
in Section \ref{sec:data-driven-grid}. It also plays a crucial role in setting
the resolution of the histograms themselves through $E$. The first role is
more related to the data collection process and to the way the fixed grid is specified than to modelling, while the
second appears directly in the model selection process via $E$.  

We propose to explicitly separate those roles by introducing an intermediate parameter $G$,
the \emph{granularity}. For a given $E$, we assume that the numerical domain
can be split into $G$ bins ($1 \leq G \leq E$) of equal width. Each of these
new elementary bins, that we call {\em $g$-bins}, is composed of $g = E/G$
$\epsilon$-bins. Each of the intervals of any constructed histogram has a
length that is a multiple of these $g$-bins. In other words, each interval is
no longer composed of a multiple of $\epsilon$-bins but rather composed of
$G_k$ $g$-bins. To better grasp this new setting, Figure
\ref{fig:setting-granularity} shows the case where three intervals were made
from $g$-bins. In this illustration, we have $E=20$ and $G=5$, so that each
$g$-bin is composed of $4$ $\epsilon$-bins. 

\begin{figure}[htbp]
\centering
\includegraphics[width=\linewidth]{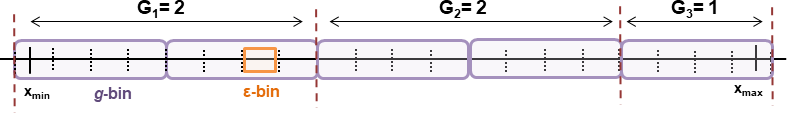}
\caption{Replacing $\epsilon$-bins by $g$-bins in the model building}
\label{fig:setting-granularity}
\end{figure}

By keeping $\epsilon$, we still have a fixed grid and user defined
resolution for both the data and the grid itself, while the introduction of
$G$ enables to chose the overall maximal complexity of the histogram. 

\subsubsection*{Granulated criterion}
Each granulated bin can either contain a whole number or a fraction of $\epsilon$-bins. To simplify our reasoning and keep this paper self-contained, we assume that the number of elementary $\epsilon$-bins E is a multiple of G, i.e. $g = E/G \in \mathbb{N}^*$.

Using this new parameter, the criterion changes both in model indexing terms
and in the bin indexing terms: 
\begin{itemize}
\item {\bf change in model indexing terms}~ If we assume that $G$ is distributed according to the universal prior for integers, the prior term for choosing the length of the intervals turns from
$ \displaystyle \cMODLcutpointchoice$ into $ \displaystyle \cGMODLcutpointchoice$.
\smallskip 
\item {\bf change in bin indexing terms}~ Interval lengths are no longer described by $E_k$ $\epsilon$-bins but rather by $G_k$ $g$-bins. Given that each $g$-bin is actually an aggregation of $g = E/G$ $\epsilon$-bins, the likelihood term of the data distribution within an interval is now  given by: 
\begin{align}
\sum^K_{k=1} h_k \log  \left(G_k \cdot\frac{E}{G}\right) & = \sum^K_{k=1} h_k \log  G_k  + \sum^K_{k=1} h_k  \log  \frac{E}{G} \nonumber\\
& = \sum^K_{k=1} h_k \log  G_k  + n \log  \frac{E}{G}
\end{align}
\end{itemize}

This new criterion, which we call {\bf \GENUMname} is very similar to the
original criterion introduced in this paper, as shown in table
\ref{tab:comp-GMODL}. The accuracy $\epsilon$ remains under the control of the
user and can be fixed using standard considerations on real number computer
representation (see Section \ref{sec:role-appr-accur}), all the while $G$ is optimised
as part of model selection. 

\renewcommand{\arraystretch}{3.3}
\begin{table*}[htbp]
\centering
\caption{Term comparison of the \ENUMname and \GENUMname criteria}
\begin{tabular}{p{1.3cm}p{4.9cm}p{4.8cm}p{3.7cm}}
\hline
Crit. & Indexing terms & Multinomial terms & Bin indexing terms\\\hline\hline
\ENUMname &  $ \displaystyle \log ^* K + \log  \binom{E+K-1}{K-1}$ & $\displaystyle \log  \binom{n+K-1}{K-1}+  \log  \frac{n!}{h_1!... h_K!}$ & $\sum^K_{k=1} h_k \log  E_k$\\ \hline
\GENUMname &  $ \displaystyle  \log ^* K +\log^* G + \log  {{G+K-1}\choose{K-1}}$ & $\displaystyle  \log  \binom{n+K-1}{K-1}+  \log  \frac{n!}{h_1!... h_K!}$ &$\sum^K_{k=1} h_k \log  G_k + n \log  \frac{E}{G}$\\\hline
\end{tabular}

\label{tab:comp-GMODL}
\end{table*}

\section{Theoretical properties}\label{sec:properties}
An analytical study of both the \ENUMname and \GENUMname criteria provides insights
into how the enumerative MDL criterion works. Most of the propositions that
follow cannot be extended to the \NMLname  criterion because of the non-monotone
nature of some of its terms and because of the lack of closed-form expression for the
parametric complexity. We focus solely on \ENUMname and
\GENUMname histograms in what follows, as only these criteria are fit for analysis.

\subsection{Compelling properties of enumerative histograms}
We study the behaviour of the criterion in certain configurations in order to
identify the most important factors that determine the overall shape of an
optimal histogram (proofs of these properties are provided in
\ref{apx:proofs}). Notice that although our analysis 
focuses on the encoding vision of the criterion, these properties can also be 
understood as statements of the most likely outcomes in terms of
probability. This is due to the natural link between a coding length (the
criterion value) and a posteriori probability at the basis of information
theory: a more probable model will have a shorter coding length. Note  also that
in the properties, models are always assumed to be compatible with the data. 

We show first an elementary property of optimal histograms.
\begin{proposition}\label{proposition:no-null-length}
Let $\mathcal{M}=(K, (c_k)_{0\leq k\leq K}, (h_k)_{1 \leq k \leq K})$ be
an optimal histogram for the data set $D$. Then
\begin{equation*}
\forall k, 1\leq k\leq K\ c_{k-1}<c_{k}.
\end{equation*}
In other words, an optimal histogram cannot contain zero-length intervals. 
\end{proposition}
Then we show that too complex histograms will not be selected by our criterion.
\begin{proposition}\label{proposition:single-bin-vs-singleton}
  Let $D$ be a data set with $n$ observations. Let us denote
  $\mathcal{M}_{K=n}$ a histogram compatible with $D$ such that there is one
  observation per interval and $\mathcal{M}_{K=1}$ a histogram compatible with $D$ with only one interval. Then 
  the coding length of $\mathcal{M}_{K=1}$ is shorter than
 the one of $\mathcal{M}_{K=n}$:
\begin{equation*}
\cMODL(\mathcal{M}_{K=n}|D) > \cMODL(\mathcal{M}_{K=1}|D).
\end{equation*}
\end{proposition}
This can be extended to more complex solutions with empty intervals.
\begin{proposition}\label{proposition:single-bin-vs-singleton-empty}
  Let $D$ be a data set with $n$ observations. Let us denote
  $\mathcal{M}_{K>n}$ a histogram compatible with $D$ consisting of either
  singleton or empty intervals, one interval for each observation and empty
  intervals in-between, and let $\mathcal{M}_{K=1}$ be as in Proposition
  \ref{proposition:single-bin-vs-singleton}. Then the coding length of
  $\mathcal{M}_{K=1}$ is shorter than the one of $\mathcal{M}_{K>n}$:
\begin{equation*}
\cMODL(\mathcal{M}_{K > n}|D) > \cMODL(\mathcal{M}_{K=1}|D).
\end{equation*}
\end{proposition}
In addition, optimal histograms will never exhibit two consecutive empty
intervals. 
\begin{proposition} \label{proposition:non-consecutive-empty}
  For any data set $D$, the coding length of a histogram with two
  adjacent empty intervals is always longer than the coding length of a
  histogram  with no consecutive empty intervals.
\begin{equation*}
 \cMODL(\mathcal{M}_{K, (h_A, h_B)=0}|D) > \cMODL(\mathcal{M}_{K-1, h_{A\cup B}=0}|D)
\end{equation*}
\end{proposition}
Despite those general rules, local optimisation remains possible in the sense
that empty intervals or intervals with only a single observation can be
present in optimal histograms.
\begin{proposition}\label{proposition:singleton-bins}
  There exist data sets such that the optimal histogram has at least one
  interval which contains only a single observation. 
\end{proposition}
\begin{proposition} \label{proposition:empty-bins} 
There exist data sets such that the optimal histogram has at least one
interval that does not contain any observation. 
\end{proposition}
We can also characterise the structure of optimal histograms in different ways.
\begin{proposition}\label{proposition:Kbound} 
An optimal histogram has at most $2n-2$ intervals ($K^* \leq 2n-2$).
\end{proposition}

\begin{proposition}\label{proposition:opt-cutpoint} 
In an optimal histogram, each interval endpoint is at most $\epsilon$ away from one of
the values of the data set.
\end{proposition}
Propositions \ref{proposition:non-consecutive-empty}, \ref{proposition:Kbound}
and \ref{proposition:opt-cutpoint} will serve as the basis for further
improvements on the search of the best histogram. For instance, as Proposition
\ref{proposition:non-consecutive-empty} states that an optimal histogram will
not have consecutive empty intervals, our search heuristic will systematically
favour the merge of two consecutive empty bins. See Section
\ref{sec:optimisation} for additional details. 

\subsection{Role of the approximation accuracy $\epsilon$}\label{sec:role-appr-accur}
The authors of the \NMLname  criterion argue that the effect of the approximation
accuracy $\epsilon$ can be ignored during the model selection process
\citep{pmlr-v2-kontkanen07a} and set to the accuracy at which the data have been
recorded. In their experiments, $\epsilon$ is fixed to a rather large value
($0.1$) which is inconsistent with current recording practices. 

In this section, we conduct a theoretical investigation of the behaviour of the \ENUMname criterion when
$\epsilon \rightarrow 0$ (or equivalently, when $E \rightarrow
\infty$). Experimental results illustrate this behaviour in
Section \ref{sec:role-appr-accur-1}. 

\subsubsection{Behaviour when $\epsilon \rightarrow 0$}
To study the behaviour of our enumerative criterion as $\epsilon$ tends to 0, we
first introduce the definition of \emph{singular} intervals.

\begin{definition}\label{def:singular-intervals}
  Let $\mathcal{M}=(K, C, (h_k)_{1 \leq k \leq K})$ be a histogram
  compatible with a data set $D$ and built on a $\epsilon$ grid. 
  An interval $I_k=]c_{k-1},c_{k}]$ of $\mathcal{M}$ is singular if the following
  conditions are met
  \begin{enumerate}
  \item $c_k-c_{k-1}=\epsilon$;
  \item $h_k>0$;
  \item all values of $D$ that belong to $I_k$ are identical.
  \end{enumerate}
  We denote by $S(\mathcal{M},D)$ the number of data points from $D$ that
  belong to singular intervals. In other words
  \begin{equation*}
S(\mathcal{M},D)=\sum_{I_k\text{ is singular}}h_k.
  \end{equation*}
\end{definition}
Then we have the following theorem (see
\ref{apx:limit-histogram-proof} for the proof).
\begin{theorem}\label{th:limit-histogram}
  Let $D$ be a data set with $n$ observations. There
  exists two positive values $C(D)$ and $E(D)$ that depends only on $D$ such
  that for all $\epsilon\leq E(D)$ for any
  \emph{optimal} histogram $\mathcal{M}^\star$ have
\begin{equation}
  \label{eq:enum:limit}
  \begin{split}
    \left| \cMODL(\mathcal{M}^\star|D)
      -\left\{K-1+n-S(\mathcal{M}^\star,D)\right\}\log\frac{1}{\epsilon}\right|\\\leq
    C(D).
  \end{split}
\end{equation}
\end{theorem}
As $\cMODL(\mathcal{M}^\star|D)$ is dominated by the $\log\frac{1}{\epsilon}$ term when
$\epsilon\rightarrow 0$, equation \eqref{eq:enum:limit} shows that there is an
inherent trade off between $K$ and $S(\mathcal{M}^\star,D)$. With a higher
value of $K$, a histogram has more opportunities to have singular intervals and
thus to have a larger value of $S(\mathcal{M}^\star,D)$. If $D$ contains
multiple instances of the same value (i.e. there are $i$ and $j$ such that
$i\neq j$ and $x_i=x_j$), $S(\mathcal{M}^\star,D)$ can grow fast enough to
compensate the increase of $K$ and it is therefore not possible to
characterise further the asymptotic behaviour of
$\cMODL(\mathcal{M}^\star|D)$. 

However in the particular case where all values are distinct in $D$, we have
the following corollary (see \ref{apx:limit-histogram-proof} for the
proof).
\begin{corollary}\label{cor:limit-histogram}
Let $D$ be a data set with $n$ distinct observations. Then for $\epsilon$
sufficiently small, the optimal histogram build on $\epsilon$ bins for the
\ENUMname criterion is the trivial one with a single interval
\begin{equation*}
\mathcal{M}=(1,
\{x_{\min}-\frac{\epsilon}{2}, x_{\max}+\frac{\epsilon}{2}\},n).
\end{equation*}
\end{corollary}
Enumerative histograms have a detrimental asymptotic behaviour: if we strive
to be more precise in terms of cut point positions, the quality of the model
will be poorer.

\subsubsection{Illustration of the asymptotic
  behaviour}\label{sec:convergence-rate}
To complement the asymptotic analysis provided by Theorem
\ref{th:limit-histogram} and its corollary, we study a simple example in the
present section. Experiments on synthetic data are provided in
\ref{sec:role-appr-accur-1}. 

Let $\mathcal{U}[a, b]$ denote the uniform distribution on the interval
$[a,b]$. Let $\alpha$ and $\theta$ be two real numbers with
$\alpha\in]0,\frac{1}{2}]$ and $\theta\in]\frac{1}{2},1]$. Let us consider a
data set $D_n(\theta, \alpha)$ consisting of $n\theta$ independent
observations generated by $\mathcal{U}[0, \alpha]$ and $n(1-\theta)$ independent
observations generated by $\mathcal{U}[\alpha, 1]$.

An optimal histogram should contain two intervals with a cut point close to
$\alpha$, provided that the mixture density $\theta  \mathcal{U}[0, \alpha]+(1-\theta)\mathcal{U}[\alpha, 1]$,
is distinct enough from $\mathcal{U}[0, 1]$. However Corollary
\ref{cor:limit-histogram} applies and if the precision $\epsilon$ is small
enough, a single interval will be preferred. We can study on this simple
example the relationship between $\epsilon$, $n$ and the mixture density to
get a better understanding of the limitations proved by our results. 

We study the optimal histogram with a unique interval,
$\mathcal{M}^\star_{1,\epsilon}$ as well as the ideal histogram with two
intervals $\mathcal{M}^\star_{2, \epsilon}$ with a cut point in $\alpha$. Let
$\Delta(n,\epsilon, \alpha, \theta)$ be
\begin{equation*}
\Delta(n,\epsilon, \alpha, \theta)=  \cMODL(\mathcal{M}^\star_{2, \epsilon}|D_n(\theta, \alpha))- \cMODL(\mathcal{M}^\star_{1,
  \epsilon}|D_n(\theta, \alpha)).
\end{equation*}
We show in \ref{apx:convergence-rate} that
\begin{align*}
\Delta(n,\epsilon, \alpha, \theta)=& \log \left(\frac{1}{\epsilon}+1\right)- n
                                     D_{KL}(\theta \| \alpha) + O(\log n),
\end{align*}
where $D_{KL}(\theta\| \alpha)$ is Kullback-Leibler divergence between the Bernoulli distribution with parameter
$\theta$ ($\mathcal{B}(\theta)$) and the one with parameter $\alpha$ ($\mathcal{B}(\alpha)$). The histogram with two intervals
is  the optimal one if $\Delta(n,\epsilon, \alpha, \theta) < 0$.

For a fixed value of $\epsilon$, the histogram with two intervals is preferred
for larger values of $n$. The value of $n$ needed to induce this preference is
reduced by an increased $KL$ divergence between $\mathcal{B}(\theta)$ and
$\mathcal{B}(\alpha)$. An important aspect is that $E=\frac{1}{\epsilon}$
influences the criterion in a logarithmic way while the influence of $n$ is
linear. This means that, in this example, large values of $E$ can be compensated by relatively
small values of $n$.

To illustrate further the behaviour, we computed exactly $\Delta(n,\epsilon,
\alpha, \theta)$ for $\theta = 1/2$ and
$\alpha=1/10$, and for different values of $n$ and $E$.
Table~\ref{tableAsymptoticTransitions} summarizes those results by showing for
several $n$ the minimum value of $E$ above which the single interval solution
is preferred over the two interval ones.

\renewcommand{\arraystretch}{2}
\begin{table}[hbtp]
\caption{Transition from two intervals to one single interval.}
\label{tableAsymptoticTransitions}
\begin{center}
\begin{tabular}{cc}\hline
$n$  & $E$ \\\hline 
10 & $30$ \\
12 & $80$ \\
16 & $530$ \\
20 & $3700$ \\
30 & $5.05\times 10^5$ \\
40 & $7.28 \times 10^7$ \\
50 & $1.08 \times 10^{10}$\\\hline
\end{tabular}
\end{center}
\end{table}
In this simple and extreme example, the transition appears only for very large
values of $E$. However, as shown in \ref{sec:role-appr-accur-1}, in more
realistic settings the effects of $E$ manifests in a more reasonable range of
values justifying the need for selecting it (or equivalently $\epsilon$)
carefully. 

\subsection{Behaviour of the G-Enum criterion}\label{sec:behaviour-g-enum}
While the decoupling between the precision of the grid and the resolution
of the histograms provided by the \GENUMname enables us to optimise the latter, it
does not change the properties of the criterion \emph{for a fixed value of
  $G$}. 

This can be seen by interpreting a histogram  $\mathcal{M}$ computed on $G$ g-bins
constructed from $\epsilon$-bins as if it was constructed
directly on g-bins. In other words, we can compute its quality according to
the \GENUMname criterion, taking into account the underlying $\epsilon$-bins or
according to the \ENUMname criterion in which $E$ is set to $G$. Then we have
\begin{equation*}
\cGMODL(\mathcal{M},D)_{G,E} = \cMODL(\mathcal{M},D)_{G} + \log^* G +  n \log
{\frac{E}{G}},
\end{equation*}
where the dependency to $G$ and $E$ has been made explicit using subscripts.

This means that all propositions established for the \ENUMname criterion still
hold for the \GENUMname criterion, for any fixed $G$ parameter. In addition, the
best histogram for \GENUMname when $G$ is fixed to a given value is exactly the
best histogram for \ENUMname with $E=G$. Thus, we can use any optimisation strategy
designed to find the best \ENUMname histogram to 
obtain the best \GENUMname histogram for any fixed value of $G$. 

Moreover, if we consider two different values of $G$, $G_1$ and $G_2$, both
divisors of $E$, we have
\begin{equation*}
  \cGMODL(\mathcal{M},D)_{G_1,E} - \cGMODL(\mathcal{M},D)_{G_2,E}= \cMODL(\mathcal{M},D)_{G_1}-  \cMODL(\mathcal{M},D)_{G_2}+\log^* G_1-\log^* G_2 + n \log\frac{G_2}{G_1}.
\end{equation*}
This shows that the optimal model for \GENUMname depends only in a limited way on
$E$. The choice of $E$ simply constrains the possible values of $G$ to
its divisors, which in turn constrains the space of possible histograms. Importantly, in this space, the choice of the optimal histogram does not depend on $E$.

Notice also that while the \GENUMname criterion does no longer suffer from the
increase of $E$, it will behave similarly as $G$ increases: for distinct data
points, when $G$ is larger than a data dependent bound $G_{\max}(D)$ the
optimal histogram for a fixed $G$ consists in a single interval.

As a counter measure, we propose to set $E$ to a very large value using the limits
of numerical representation precision as a guideline. Given that integer
values are encoded using four bytes and have their values between
$\textnormal{INT}_{\min} \approx -2.10^9$ and
$\textnormal{INT}_{\max} \approx 2. 10^9$, one can set $E$ up to
$E_{\max} = 10^9 \approx \textnormal{INT}_{\max}/2$.

Then we can select a set of values of $G$ among the divisors of $E$, compute
the optimal histogram for each of those values according to \ENUMname and select the best overall one
using \GENUMname. Details about this procedure are given in Section
\ref{sec:granular-optimisation}. 

\section{Efficient search for MDL-optimal histograms}\label{sec:optimisation}
As recalled in Section \ref{sec:dynamic-programming}, most irregular histogram
construction methods leverage the dynamic programming principle to obtain an
optimal histogram (with respect to the chosen criterion). This is in
particular the case of the \NMLname  approach which has as a consequence a
$\mathcal{O}(E^3)$ complexity \citep{pmlr-v2-kontkanen07a}.

We describe in this Section several techniques used to reduce this
complexity. First, we introduce a greedy approach that can be applied to any
additive criterion to obtain in $\mathcal{O}(E \log  E)$ time a sub-optimal
histogram of good quality. Then, we show how to reduce this complexity to
$\mathcal{O}(n \log n)$ leveraging properties of enumerative
histograms. Finally, we detail and discuss the particularities of the optimisation of granular models.

\subsection{Greedy search}
We propose a greedy search heuristic that combines a classic bottom-up
approach and post-optimisations. It is based on a similar greedy algorithm to
minimize additive criteria in the case of supervised discretisation
\citep{BoulleML06}. It applies to any additive criterion, defined as follows.
\begin{definition}\label{def:additive-criterion}
  Let $\mathcal{M}=(K, C, (h_k)_{1 \leq k \leq K})$ be a histogram
  compatible with a data set $D$ of size $n$. The histogram is based on the
  intervals $]c_{k-1},c_k]$ defined by $C=(c_k)_{0\leq k\leq K}$. A criterion
  that evaluate the quality of $\mathcal{M}$ with respect to $D$,
  $q(\mathcal{M}|D)$, is \emph{additive} if it can be written
\begin{equation*}
 q(\mathcal{M}|D)=q_1(n, K)+\sum_{k=1}^Kq_2(c_{k-1},c_k,h_k).
\end{equation*}
\end{definition}
The main consequence of using an additive criterion is it locality. For
instance when two adjacent intervals are merged, the value of
$q(\mathcal{M}|D)$ for the resulting histogram can be computed from the value
of the criterion of the original model considering only the two merged
intervals. We exploit this property in the greedy search approach described by
Algorithm \ref{algo:greedy-search-optimisation}. 

\alglanguage{pseudocode}
\begin{algorithm}[htbp]
\small
\caption{Optimisation with a greedy search}
\label{algo:greedy-search-optimisation}
\begin{algorithmic}[1]
\Require $\epsilon$, $D=(x_i)_{1\leq i\leq n}$, additive criterion $q$
\Ensure Histogram model $\mathcal{M} = (K,
C, (h_k)_{1 \leq k \leq K})$
\Statex 
\Statex \textbf{Initialisation:}
\State \Call{Sort}{$D$} \Comment{in $\mathcal{O}(n \log  n)$}
\State $L_{\textnormal{bins}}$ = \Call{Create\_grid}{$D$, $\epsilon$} \Comment{with $E$ bins}
\Statex 
\For{consecutive bins $A$ and $B$} \Comment{in $\mathcal{O}(E)$ }
\State $\Delta q \gets  q(\mathcal{M}_{K-1, h_A + h_B}|D) - q(\mathcal{M}_{K, (h_A, h_B)}|D)$ 
\State $L_{\textnormal{merges}} \gets \Delta q $
\EndFor 
\State \Call{Sort}{$L_{\textnormal{merges}}$} \Comment{in $\mathcal{O}(E \log  E)$}
\Statex 
\Statex \textbf{Optimisation:}
\State $\Delta \gets$ \Call{Head}{$L_{\textnormal{merges}}$}
\While{$\Delta < 0$} 
\State $C$ = \Call{Merge}{$A$, $B$}
\State \Call{Update}{$L_{\textnormal{bins}}$, $C$}
\For{each bin adjacent to $C$}
\State Compute new cost variation $\Delta q$ \Comment{in $\mathcal{O}(1)$} 
\State \Call{Update} {$L_{\textnormal{merges}}$} \Comment{in $\mathcal{O}(\log  E)$}
\EndFor 
\EndWhile
\end{algorithmic}
\end{algorithm}

The algorithm is rather simple: we start with the most refined histogram based
of $E$ $\epsilon$-bins. A priority queue is used to store the effects on the
quality criterion of merging two adjacent intervals (here
$\epsilon$-bins). Then we coarsen the histogram by implementing the best merge
until the criterion cannot be improved ($\Delta < 0$). The key point is that
the merge qualities can be updated efficiently as a consequence of the
additivity of the criterion. 

Like most regularized histogram quality criteria, the \ENUMname and \NMLname  criteria are
additive: the overall computation time of the search for the best model can
thus go from $\mathcal{O}(E ^3)$ to $\mathcal{O}(E \log  E)$. 

Although the heuristic has the advantage of being
time and memory efficient, it may fall into a local optimum. We use two
heuristics suggested in \cite{BoulleML06} in to improve the quality of the final
model. Rather than stopping the merge process as soon as $\Delta < 0$, we
merge intervals up to the final histogram with a single interval. Then the
best histogram is selected among all the histograms explored by this greedy
merging approach. We post-optimise this histogram using local modifications of
the intervals, as detailed in \cite{BoulleML06}. This consists in choosing a
set of simple operation on contiguous intervals (split, merge, merge and
split, etc.) and in applying those operations in a greedy way until no
improvement is possible. These operations as chosen in a way that do not
modify the overall complexity of the algorithm.

Both heuristics are an important addition: extensive experiments reported in
\cite{BoulleML06} show that the greedy search alone produces an optimal
solution only in roughly 50~\% of the test cases, while the success rate
increases to 95~\%  when heuristics are included. 

\subsection{Optimisation gains for enumerative
  histograms}\label{sec:optimisation-gains}
The reduction from $\mathcal{O}(E ^3)$ to $\mathcal{O}(E \log  E)$ is very
important but our recommendation for setting $E$ (see Section
\ref{sec:convergence-rate}) would yield unacceptable running time even with
the greedy algorithm. However theoretical properties of the \ENUMname criterion can
be leveraged to reduce further the complexity.

Proposition \ref{proposition:opt-cutpoint}, in particular, shows that interval
end points must be close to data points: regardless of the precision of the
grid, i.e. of $E$, we need only to consider $2n-1$ candidate splits which are
the $\epsilon$ approximations of the data points. Thus, regardless of $E$, the
exact optimisation of the \ENUMname criterion can be done in $\mathcal{O}(n^3)$
while the greedy search has a complexity of $\mathcal{O}(n \log n)$.

\subsection{Optimisation of granular models}\label{sec:granular-optimisation}

To optimise the \GENUMname criterion, we use a large precision parameter $E=2^{30} \approx 10^9$, as large as possible w.r.t. the limits of numerical accuracy on computers (see section~\ref{sec:behaviour-g-enum}).
We then exploit a loop on power of two granularities $G=2^i, 0 \leq i \leq 30$ to retrieve the best model per granularity. As $E$ is a multiple of $G$ for each optimised granularity, the exact criterion of Table~\ref{tab:comp-GMODL} holds.
We exploit only the power of two granularities as a trade-off to both keep the computation time tractable and to explore a wide set of granularities.\\
Assuming that $n \leq E$, each step has a O$(G_i \log{G_i})$ time complexity for $G_i \leq n$ and O$(n \log{n})$ otherwise, since at most $2n$ intervals need to be considered for the optimisation of \ENUMname histograms. Overall, the total number of operations $t(n,E)$ is defined as follows.

\begin{eqnarray*}
t(n,E) &=& \sum_{i=1}^{\log_2 n} {2^i \log{2^i}} + \sum_{i=1+\log_2 n}^{\log_2 E} {n \log{n}},\\
       &\leq& \sum_{i=1}^{\log_2 n} {2^i \log{n}} + (\log_2 E - \log_2 n) n \log n,\\
			 &\leq& 2 n \log{n} + (\log_2 E - \log_2 n)n \log n,\\
			 &\leq& (2 + \log_2 E - \log_2 n)n \log n.
\end{eqnarray*}
As $E=10^9$ is a constant, the time complexity w.r.t. $n$ is O$(n \log n)$.
For a given $G$ parameter, the optimisation of the \GENUMname criterion can thus be performed in O$(n \log n)$.

\subsection{Efficient search for other methods}\label{sec:effic-search-other}
As pointed out in Section \ref{sec:data-driven-grid}, methods that use data
driven grid are generally based on dynamic programming and have a
$\mathcal{O}(n^3)$ complexity. Thus, they do not face the 
computational issues associated to a very fine fixed grid with a cost of
$\mathcal{O}(E^3)$, contrarily to e.g. the \NMLname  criterion
\citep{pmlr-v2-kontkanen07a}. An additional heuristic for fixed grid methods
consists in restricting the cut points between intervals to $\epsilon$
approximations of the data points, i.e. to apply Proposition
\ref{proposition:opt-cutpoint} even when its applicability has not been
proved. 

In the irregular histogram context, the applicability of dynamic programming is directly linked to the additivity of the criterion. Therefore, most of the methods reviewed earlier could benefit
from the greedy search algorithm proposed above. Note however that its complexity is
tied to the fact that the criterion can be updated in $\mathcal{O}(1)$. This is
the case for instance for the penalized likelihood variants studied in \cite{RMG2010}
but not for the NML criterion \citep{pmlr-v2-kontkanen07a}. As recalled
in Section \ref{sec:nml}, the parametric complexity term must be evaluated at
least $n$ times with an overall $\mathcal{O}(n^2)$ cost for an exact
calculation, which would mitigate any advantage the heuristic could provide.

\section{Experimental evaluation}\label{sec:experiments}
The experimental evaluation is divided in three parts. First, we compare the
performance of the three MDL criteria analysed in this paper. Secondly, we
compare our MDL methods to the state-of-the-art algorithms identified as
fully automated approaches to histogram building in the related work
section. Both of these benchmarking comparisons are done on synthetic data sets
of varying sizes, for which the performance of the different estimators can be
objectively measured. Finally, we showcase the performance and practical relevance of
our method for exploring real-world data of large scale.

A binary standalone implementation of the \GENUMname based histogram
construction is available here: \url{http://marc-boulle.fr/genum/}. 

\subsection{Experimental protocol}
The methods are evaluated on a collection of data sets. For artificial data
set, we use samples of increasing size from $n=10$ to
$n=100,000$ or $n=1,000,000$ samples when possible. 

For each evaluation metric, we report the mean and standard deviation of
results obtained over $10$ independent samples (for each distribution and
sample size). 

\subsubsection{Implementation}
All experiments are carried out on a Windows 10 machine equipped with an AMD
Ryzen 2-core processor and 6 GB of RAM memory. 

The MDL criteria presented in this paper as well as the optimisation
algorithms are implemented in C++.
As proposed in Section \ref{sec:effic-search-other}, we restrict the search of
the cut points of for  \NMLname  criterion to $\epsilon$ approximations of the
data points, even if Proposition \ref{proposition:opt-cutpoint} does not
apply. This brings some computational efficiency to this method even when $E$
is large.

\subsubsection{Metrics} 
The comparison between histograms models is based on the number of intervals
each histogram has, as well as the time it took to compute it. Additionally,
the relevance of each histogram is evaluated by computing the Hellinger
distance to the original model density. 

The Hellinger Distance (HD) $H(p,q)$ for $p,q$ being probability density functions, is defined as 
$$H(p,q)=\frac{1}{\sqrt{2}}\sqrt{\int(\sqrt{p(x)}-\sqrt{q(x)})^2} dx$$\\

A HD close to 0 indicates a strong similarity between probability distributions. It is worth noting that most authors used the {\bf squared} Hellinger distance, which may produce an impression of a better convergence of the estimations \citep{pmlr-v2-kontkanen07a,Davies2009,Luosto2012}. The Hellinger distances presented in this paper {\bf are not squared}.
HD measures reported are obtained via numerical integration to estimate the probability distributions of the model density and of the histogram that models it. 

\subsubsection{Reference distributions}\label{sec:exp-data sets}
For the benchmarking experiments, we compare all methods over the 6 distributions described in table \ref{tab:densities}.

\renewcommand{\arraystretch}{1.6}
\begin{table}[!h]
\centering
\caption{All the densities used for benchmarking \label{tab:densities}}
\begin{tabular}{p{2.5cm}p{4cm}}
\hline
 Normal &  $\mathcal{N}(0, 1)$\\
 Cauchy & $\mathcal{N}(0, 1)/ \mathcal{N}(0, 1)$\\
 Uniform & $\mathcal{U}([0;1])$\\
 Triangle as in \cite{Davies2009} & $\mathcal{T}(0.158)$\\
 Triangle mixture & $0.1 ~ \mathcal{T}(0.158) + 0.3 ~ \mathcal{T}(0.258) + 0.4 ~ \mathcal{T}(0.500) + 0.2 ~ \mathcal{T}(0.858)$\\
 Claw as in \cite{Davies2009} & $0.5 ~ \mathcal{N}(0.0, 1.0) + 0.1 ~ \mathcal{N}(-1.0, 0.1) + 0.1 ~ \mathcal{N}(-0.5, 0.1) + 0.1 ~ \mathcal{N}(0.0, 0.1) + 0.1 ~ \mathcal{N}(0.5, 0.1) + 0.1 ~ \mathcal{N}(1.0, 0.1)$ \\
\hline
\end{tabular}
\end{table}

\subsection{Comparison of MDL-based methods}
In this subset of experiments, we focus on illustrating the similarities and
differences between histograms produced by the \NMLname  , \ENUMname and
\GENUMname criteria. The \NMLname   and \ENUMname criteria were optimised
either using the optimal dynamic programming algorithm presented in
\cite{pmlr-v2-kontkanen07a} or with the search heuristic presented in this
paper.

We tested the \NMLname  , \ENUMname and \GENUMname criteria over the synthetic
data sets described in table \ref{tab:densities}. We set the approximation
accuracy to $\epsilon=0.01$ across all experiments on synthetic data for the \NMLname   and
\ENUMname methods. The accuracy of \GENUMname is optimised as part of the
estimation process. \\
Detailed
results are provided in the appendix for each distribution (figures
\ref{fig:comparison-MDL-normal} to \ref{fig:comparison-MDL-claw}). Section
\ref{sec:role-appr-accur-1} provides in addition a study of the impact of
$\epsilon$ on the results on a synthetic data set. In this section, we provide only the main insights we could gather from this benchmark.

\begin{table*}\scriptsize
\centering
\caption{Hellinger distance values over different distributions of sample size $n=10^4$, for MDL methods} \label{tab:MDL-hd-n10-4}
\begin{tabular}{p{2.5cm}|cccccc}
\hline
Distribution & \NMLname   + optimal & \NMLname   + heuristic & \ENUMname + optimal & \ENUMname + heuristic & \GENUMname     \\
\hline
Normal & $ 0.046 \pm 0.002$ & $ 0.047 \pm 0.002$ & $ 0.046 \pm 0.002$ & $ 0.047 \pm 0.002$ & $ 0.045 \pm 8 \cdot 10^{-4} $   \\ 
Cauchy & $ 0.074 \pm 0.003$ & $ 0.073 \pm 0.003$ & $ 0.075 \pm 0.004 $ & $ 0.075 \pm 0.004$ & $ 0.060 \pm 0.004 $  \\ 
Uniform & $ 0.051 \pm 0.005$ & $ 0.051 \pm 0.005 $ & $ 0.052 \pm 0.006$ & $ 0.052 \pm 0.006$ & $ 0.024 \pm 0.001$  \\ 
Triangle & $ 0.038 \pm 0.002$ & $ 0.038 \pm 0.002$ & $ 0.038 \pm 0.002$ & $ 0.039 \pm 0.002$ & $ 0.039 \pm 0.002$  \\ 
Triangle mixture & $ 0.038 \pm 0.003$ & $ 0.038 \pm 0.003 $ & $ 0.039 \pm 0.003 $ & $ 0.038 \pm 0.003$ & $ 0.039 \pm 0.002 $ \\ 
Gaussian mixture & $ 0.059 \pm 0.002$ & $ 0.059 \pm 0.002 $ & $ 0.058 \pm 0.002 $ & $ 0.058 \pm 0.002 $ & $ 0.057 \pm 0.002 $   \\ 
 \hline
\end{tabular}
\end{table*}

\paragraph*{Intervals}  ~~ Across all data sets and all sample sizes, \NMLname  , \ENUMname  and \GENUMname  histograms have the same number of intervals. This confirms that the criteria are very similar. Moreover,  we can see that choosing either the dynamic programming algorithm or the search heuristic for the optimisation results in almost identical interval counts.

\paragraph*{Hellinger distance} ~~ The same can be said for HD values: \NMLname  , \ENUMname  and \GENUMname  histograms are very similar estimators. A steady decrease in HD values can be observed for the 5 variants, which means that all methods produce better estimations as the sample size increases. We highlight however that across all data sets, the \GENUMname  method has slightly better HD values, especially for large sample sizes (see table \ref{tab:MDL-hd-n10-4} for an overview).

\paragraph*{Computation time}~~ The main difference between the methods lies in the computation time. The results obtained across the different distributions showcase the advantage of using the greedy search heuristic to optimise MDL criteria. A particularly striking difference is observed for the Cauchy distribution (figure \ref{fig:comparison-MDL-cauchy}), where \NMLname   histograms are computed up to 10 times faster with the heuristic than with the optimal algorithm for sample sizes from $n=100$ to $n=10^4$. From $n=10^5$, the optimal algorithm took over an hour long. Additionally, we observed that for the uniform, triangle and triangle mixture distributions, the same amount of time is needed to compute \NMLname histograms with either algorithm as the sample sizes increase. This shows that, for some data sets, the cost of computing the \NMLname   criterion itself can outweigh the benefits of using the search heuristic.\\
For the \ENUMname  histograms the gap in computation time is not as easily
closed. The \ENUMname  histograms computed with the optimal algorithm take
slightly less time than the \NMLname   histograms for the Normal, Cauchy and Gaussian
mixture distributions. For the uniform, triangle and triangle mixtures
however, using the \ENUMname  even with the optimal algorithm can cut
computation time by a factor of $100$ as sample sizes increase. Enumerative
histograms produced with the heuristic remain the fastest of all MDL-based
methods throughout all experiments.

It is worth noting that, although the \GENUMname  method takes slightly longer, it remains consistently at a mid-point between the methods.

\paragraph*{Overall conclusion} ~~ \NMLname   and \ENUMname  histograms are
interchangeable in terms of interval count and Hellinger distance, and this
regardless of what algorithm is chosen to optimise the criteria. In terms of
computation time however, there is a clear advantage in preferring the search
heuristic and the simpler enumerative criteria. Finally, the use of 
\GENUMname  histograms increases slightly the computational time and improves
also slightly the estimation quality in most cases in terms of Hellinger distance.

\subsection{Comparison with other fully automated histogram methods}
We compare in this section our MDL based methods with state-of-the-art fully
automated methods. More precisely, the comparison includes:
\begin{itemize}
\item \texttt{NML + heuristic}, the \NMLname   criterion optimised with our search heuristic;
\item \GENUMname , our fully automated enumerative criterion;
\item \texttt{Taut string} histograms \citep{Davies2004,Davies2009},
  as implemented in the \texttt{ftnonpar} R package \citep{ftnonpar};

\item \texttt{RMG} histograms \citep{RMG2010}, as implemented in the
  \texttt{histogram} R package \citep{mildenberger2019package};

\item Sturges rule histograms, as implemented in Python's \texttt{numpy}
  library \citep{harris2020array};

\item \texttt{Freedman-Diaconis} rule histograms \citep{Freedman1981}, as implemented in Python's \texttt{numpy} library;

\item \texttt{Bayesian blocks} \citep{Scargle2013}, as implemented in
  Python's \texttt{AstroPy} library \citep{astropy:2022}.
\end{itemize}

A very complete evaluation provided in \cite{RMG2010} concluded
that cross-validation based estimators were not among the best performing
solutions and we excluded them from the comparison. In addition \citep{RMG2010}
showed that \texttt{RMG} histograms tend to provide the best overall
performances. 

For each studied distribution, we provide a visual example of histograms
produced by the 7 methods, as well as the variations in interval counts,
computation times and HD values as sample sizes increase. Detailed results are
provided in \ref{apx:exp-others}.

Given the extensive nature of these experiments, we selected a subset of the
results to provide an 
overview for a $n=10^4$ sample size in tables \ref{tab:hd-n10-4},
\ref{tab:computation-n10-4} and \ref{tab:nbK-n10-4}. 

\begin{table*}\scriptsize
\centering
\caption{Hellinger distance values over different distributions of sample size $n=10^4$} \label{tab:hd-n10-4}
\begin{tabular}{p{2cm}|p{1.5cm}p{1.5cm}p{1.5cm}p{1.5cm}p{1.5cm}|p{1.5cm}p{1.5cm}}
\hline
Distribution & \GENUMname  & \NMLname  \newline  \cite{pmlr-v2-kontkanen07a} & BB \newline  \cite{Scargle2013} & TS \newline  \cite{Davies2009} & RMG \newline  \cite{RMG2010} & FD \newline  \cite{Freedman1981} & Sturges    \\
\hline
Normal & $ 0.045 \pm 6 {.} 10^{-4} $ & $ 0.046 \pm 0.002 $ & $ 0.047 \pm 0.002 $ & $ 0.040 \pm 0.002 $ & $ 0.034 \pm 0.002 $ & $ 0.033 \pm 0.002 $ & $ 0.055 \pm 0.002 $  \\ 
Cauchy & $ 0.061 \pm 0.004 $ & $ 0.074 \pm 0.003 $ & $ 0.064 \pm 0.002 $ & $ 0.045 \pm 0.005 $ & $ 0.064 \pm 0.001 $ & $ 0.138 \pm 0.002 $ & $ 0.862 \pm 0.036 $  \\ 
Uniform & $ 0.024 \pm 0.001 $ & $ 0.050 \pm 0.005 $ & $ 0.025 \pm 0.004 $ & $ 0.031 \pm 0.015 $ & $ 0.029 \pm 0.011$ & $ 0.082 \pm 0.011 $ & $ 0.028 \pm 0.002 $  \\ 
Triangle & $ 0.039 \pm 0.002 $ & $ 0.038 \pm 0.0025 $ & $ 0.039 \pm 0.001$ & $ 0.084 \pm 0.024 $ & $ 0.084 \pm 0.029 $ & $ 0.032 \pm 0.002 $ & $ 0.049 \pm 9 {.} 10^{-4} $  \\ 
Triangle mixture & $ 0.037 \pm 0.002 $ & $ 0.038 \pm 0.003 $ & $ 0.040 \pm 0.003 $ & $ 0.078 \pm 0.029 $ & $ 0.069 \pm 0.026 $ & $ 0.032 \pm 0.002 $ & $ 0.043 \pm 4 {.} 10^{-4} $  \\ 
Gaussian mixture & $ 0.057 \pm 0.002 $ & $ 0.059 \pm 0.002 $ & $ 0.060 \pm 0.002 $ & $ 0.040 \pm 0.001 $ & $ 0.052 \pm 0.002 $ & $ 0.060 \pm 0.001 $ & $ 0.142 \pm 0.013 $  \\ 
 \hline
\end{tabular}
\end{table*}

\begin{table*}\scriptsize
\centering
\caption{Computation times (in seconds) over different distributions of sample size $n=10^4$} \label{tab:computation-n10-4}
\begin{tabular}{p{2cm}|p{1.5cm}p{1.5cm}p{1.5cm}p{1.5cm}p{1.5cm}|p{1.5cm}p{1.5cm}}
\hline
Distribution & \GENUMname  & \NMLname  \newline  \cite{pmlr-v2-kontkanen07a} & BB \newline \cite{Scargle2013} & TS \newline \cite{Davies2009} & RMG \newline \cite{RMG2010} & FD \newline \cite{Freedman1981} & Sturges    \\
\hline
Normal & $ 0.014 \pm 0.003 $ & $ 2.724 \pm 0.283 $ & $ 5.785 \pm 0.479 $ & $ 0.014 \pm 0.002 $ & $ 1.239 \pm 0.085$ & $ 0.002 \pm 2. 10^{-4} $ & $ 6. 10^{-4} \pm 2. 10^{-4}$  \\ 
Cauchy & $ 0.028 \pm 0.006 $ & $ 121.60 \pm 4.99 $ & $ 3.250 \pm 0.112 $ & $ 0.116 \pm 0.205 $ & $ 0.906 \pm 0.107 $ & $ 0.009 \pm 0.014 $ & $ 0.001 \pm 0.003 $  \\ 
Uniform & $ 0.015 \pm 0.002 $ & $ 0.168 \pm 0.012$ & $ 5.989 \pm 0.167 $ & $ 0.011 \pm 0.002 $ & $ 1.387 \pm 0.139 $ & $ 0.002 \pm 3. 10^{-4} $ & $ 6. 10^{-4} \pm 2. 10^{-4}$   \\ 
Triangle & $ 0.014 \pm 0.005 $ & $ 0.169 \pm 0.005 $ & $ 5.962 \pm 0.113 $ & $ 0.015 \pm 0.002 $ & $ 1.291 \pm 0.091 $ & $ 0.002 \pm 2. 10^{-4} $ & $ 6. 10^{-4} \pm 2. 10^{-4}$   \\ 
Triangle mixture & $ 0.012 \pm 0.006 $ & $ 0.103 \pm 0.027 $ & $ 3.004 \pm 0.245 $ & $ 0.013 \pm 0.006 $ & $ 0.954 \pm 0.138 $ & $ 0.002 \pm 0.005 $ & $ 0.0 \pm 0.0$  \\ 
Gaussian mixture & $ 0.017 \pm 0.002 $ & $ 1.91 \pm 0.085 $ & $ 4.165 \pm 0.369 $ & $ 0.048 \pm 0.005 $ & $ 1.056 \pm 0.077 $ & $ 0.006 \pm 0.008 $ & $ 0.002 \pm 0.005 $  \\ 
 \hline
\end{tabular}
\end{table*}

\begin{table*}\scriptsize
\centering
\caption{Optimal bin counts over different distributions of sample size $n=10^4$} \label{tab:nbK-n10-4}
\begin{tabular}{p{2cm}|p{1.5cm}p{1.5cm}p{1.5cm}p{1.5cm}p{1.5cm}|p{1.5cm}p{1.5cm}}
\hline
Distribution & \GENUMname  & \NMLname  \newline  \cite{pmlr-v2-kontkanen07a} & BB \newline  \cite{Scargle2013} & TS \newline \cite{Davies2009} & RMG \newline \cite{RMG2010} & FD \newline  \cite{Freedman1981} & Sturges    \\
\hline
Normal & $ 16.30 \pm 0.46 $ & $ 15.60 \pm 1.02 $ & $ 15.90 \pm 1.04 $ & $ 72.50 \pm 4.41 $ & $ 39.70 \pm 7.34$ & $ 62.40 \pm 2.29 $ & $ 15.0 \pm 0.0 $  \\ 
Cauchy & $ 30.90 \pm 2.43 $ & $ 23.60 \pm 1.02 $ & $ 29.40 \pm 1.56 $ & $ 144.90 \pm 9.26 $ & $ 29.50 \pm 1.57 $ & $ 110711.90 \pm 132580.43 $ & $ 15.0 \pm 0.0 $ \\ 
Uniform & $ 1.0 \pm 0.0$ & $ 2.80 \pm 0.60 $ & $ 1.30 \pm 0.90 $ & $ 3.70 \pm 5.44 $ & $ 1.70 \pm 1.80 $ & $ 22.0 \pm 0.0 $ & $ 15.0 \pm 0.0 $  \\ 
Triangle & $ 12.50 \pm 0.92 $ & $ 13.60 \pm 0.66 $ & $ 12.60 \pm 0.66 $ & $ 48.0 \pm 5.85 $ & $ 33.70 \pm 8.25 $ & $ 32.10 \pm 0.30 $ & $ 15.0 \pm 0.0 $  \\ 
Triangle mixture & $ 11.20 \pm 0.75 $ & $ 12.00 \pm 0.77 $ & $ 10.90 \pm 0.70 $ & $ 42.30 \pm 4.90 $ & $ 27.60 \pm 8.0 $ & $ 30.80 \pm 0.40 $ & $ 15.0 \pm 0.0 $  \\ 
Gaussian mixture & $ 28.90 \pm 1.22 $ & $ 27.30 \pm 1.49 $ & $ 27.00 \pm 2.00 $ & $ 134.90\pm 9.37 $ & $ 100.40 \pm 11.60 $ & $ 66.40 \pm 2.42 $ & $ 15.0 \pm 0.0 $  \\ 
 \hline
\end{tabular}
\end{table*}

Both the tables and the more detailed plots show that the other fully automated methods compared here work their best in the specific cases they were designed for. Although rarely the best for each distribution type, \GENUMname  histograms are consistently among the best estimators, and this without the high variability of the other methods. Focusing on irregular histograms, \GENUMname is certainly among the most parsimonious in number of intervals. For exploratory analysis, this is an important quality because it makes the interpretation of the results easier and more reliable. \GENUMname is also by far the fastest of irregular methods, making it suitable to large data sets. 

\subsection{Illustration on a large-scale real-world data set}
We focus now on the performance and practical relevance of the \GENUMname
method for modelling an unknown distribution from a real-world large scale
data set.

The {\em Lunar Crater Database}
\footnote{\url{https://astrogeology.usgs.gov/search/map/Moon/Research/Craters/lunar_crater_database_robbins_2018}}
contains 1.3 million entries on lunar impact craters larger than 1 to 2 km in
diameter \citep{Robbins2018}. Craters were manually identified and measured
from images of NASA's Lunar Reconnaissance Orbiter (LRO), taken from 2011
until 2018. The run time and histogram sizes obtained on this data set are
given in Table \ref{tab:lunar}.

\begin{table*}\scriptsize
\centering
\caption{Results for the Lunar data set} \label{tab:lunar}
\begin{tabular}{lp{1.5cm}p{1.5cm}p{1.5cm}p{1.5cm}p{1.5cm}p{1.5cm}}
\hline
 & \GENUMname  & \NMLname \newline   \cite{pmlr-v2-kontkanen07a} & BB \newline  \cite{Scargle2013}& TS \newline  \cite{Davies2009} & RMG \newline  \cite{RMG2010}   \\
\hline
Number of intervals & 75 & 65 & 86 &478 &73\\ 
Computation time (seconds) & 1.3 & 1025 & 4489.26 & 53.59 & 37.02 \\ 
\hline
\end{tabular}
\end{table*}

Although we ignore which law governs craters' diameter distribution, the
granulated criterion produces a smooth-looking histogram with only  $K^* = 75$
intervals for the 1.3 million entries (Figure \ref{fig:lunar-gmodl}, with both
axis in log scale). This fairly compact representation is easy to
interpret. We notice for example, that the first 10 intervals are the most
dense ones; they account for about 40\% of all entries. Intervals become less
dense for higher crater diameters. The last interval spans over about 1500 km
of crater diameters and only accounts for 3 data entries.

\begin{figure}
\centering
\includegraphics[width=\linewidth]{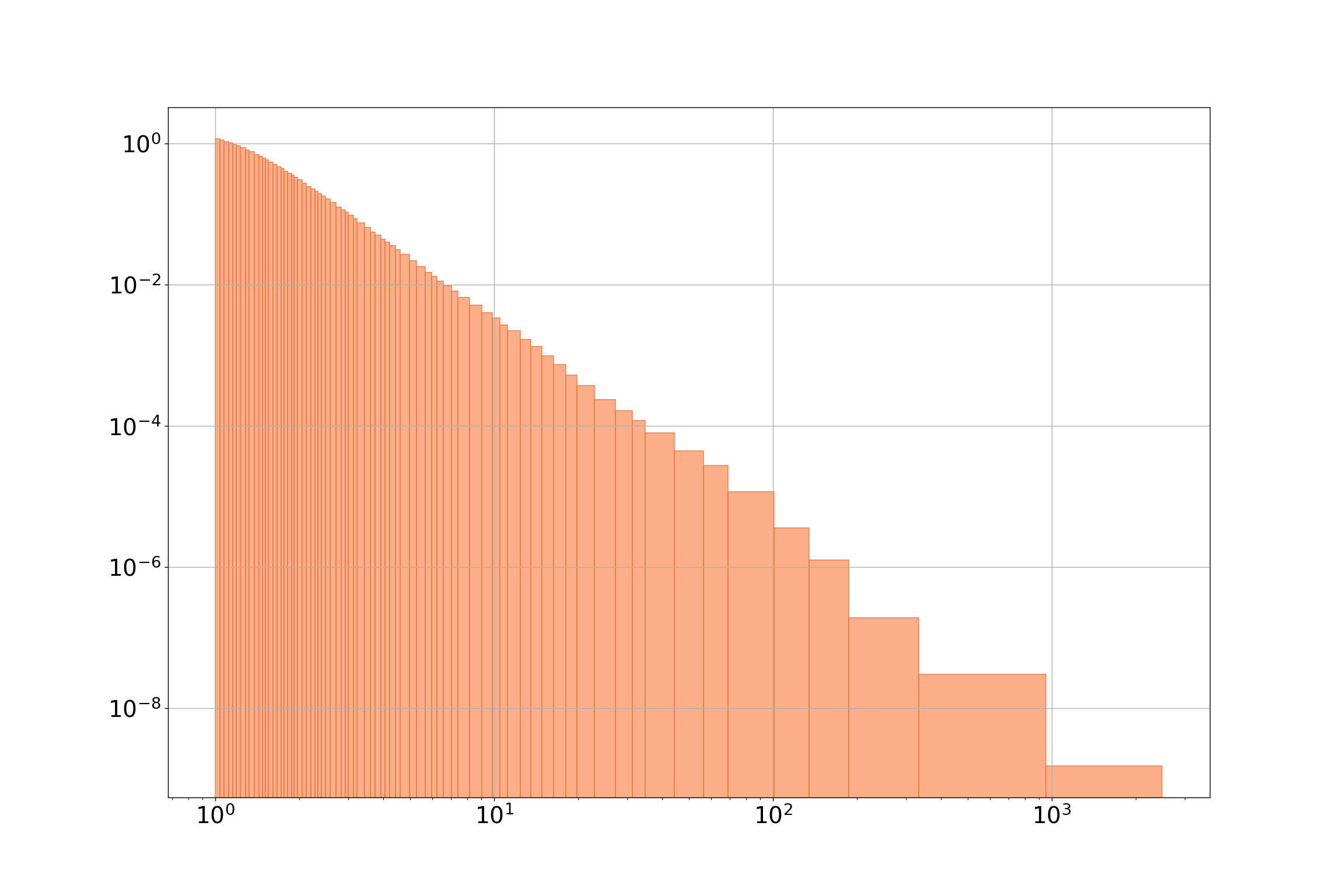}
\caption{\GENUMname histogram of the Lunar crater data set in a log scale ($K^* = 75$)}
\label{fig:lunar-gmodl}
\end{figure}

Overall, the shape of our irregular histogram also shows that our approach can
capture a power law decrease of the densities. This is in line with
astrophysics literature: power or multiple power laws are often used to fit
the crater size distribution \citep{Wang2016,MINTON201963}. The same
information would be hard to convey with an equal-width histogram and an
arbitrary number of bins. Our irregular histogram reveals interesting patterns
and only requires a reasonable number of intervals to do so.

The density estimation experiment on this data set also shows the scalability
of our method: the granular histogram was computed in about $1.3$ seconds. The
taut-string method and the RMG approach remain usable with respectively about
54 seconds and 37 seconds, but both \NMLname   histogram ($\approx 17$ minutes of
calculation) and the Bayesian Blocks method (almost 1 hour and 15 minutes of
run time) use an unreasonable quantity of computational resources.

\begin{figure}
  \centering  
  \caption{Different histograms obtained for the Lunar data set\label{fig:lunar-hists}}
  \setkeys{Gin}{width=0.25\textwidth}
	\subfloat[NML + heuristic ($K^*=65$)]
  {\includegraphics{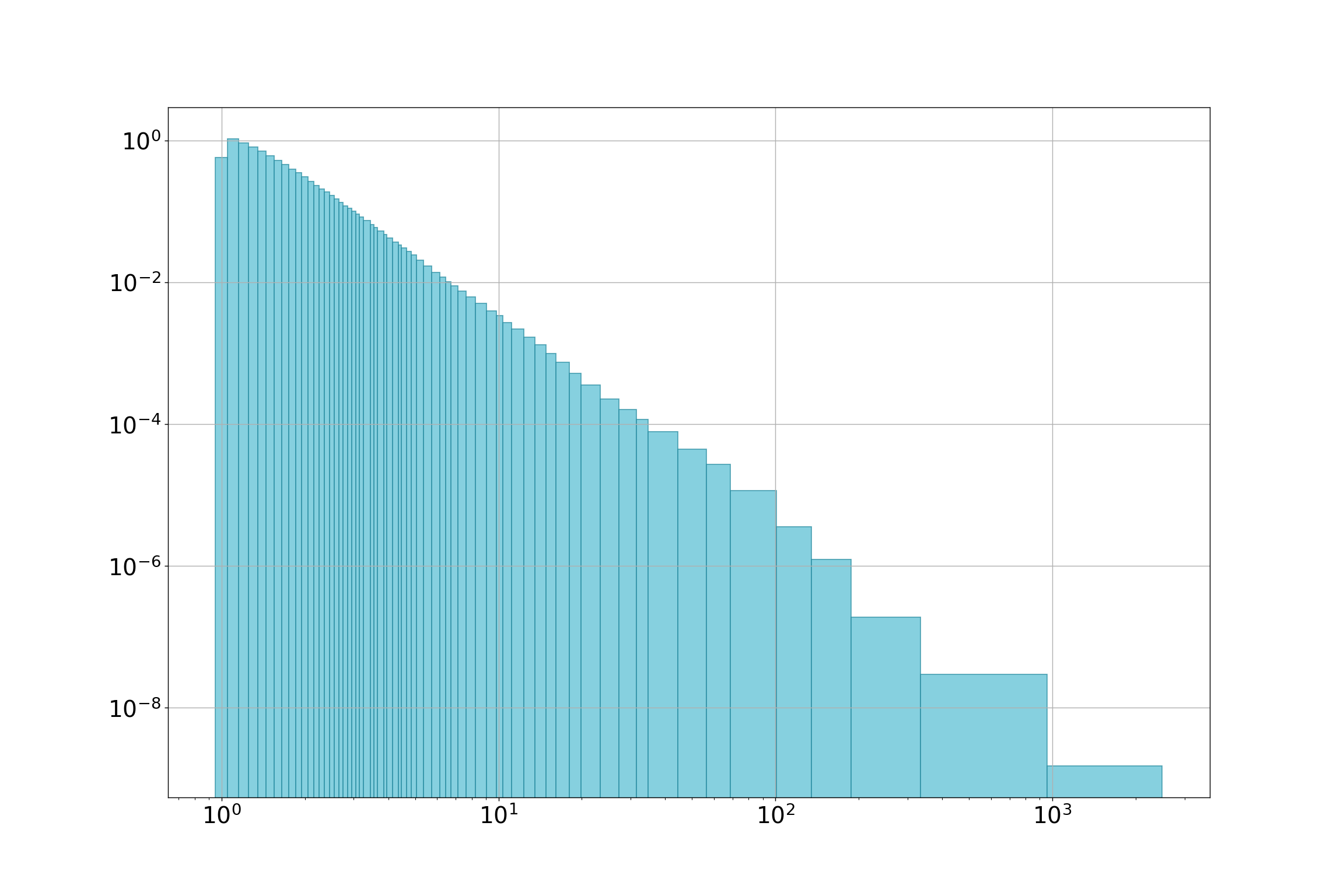}}%
	\hfill
	\subfloat[Taut string ($K^*= 478$)]
  {\includegraphics{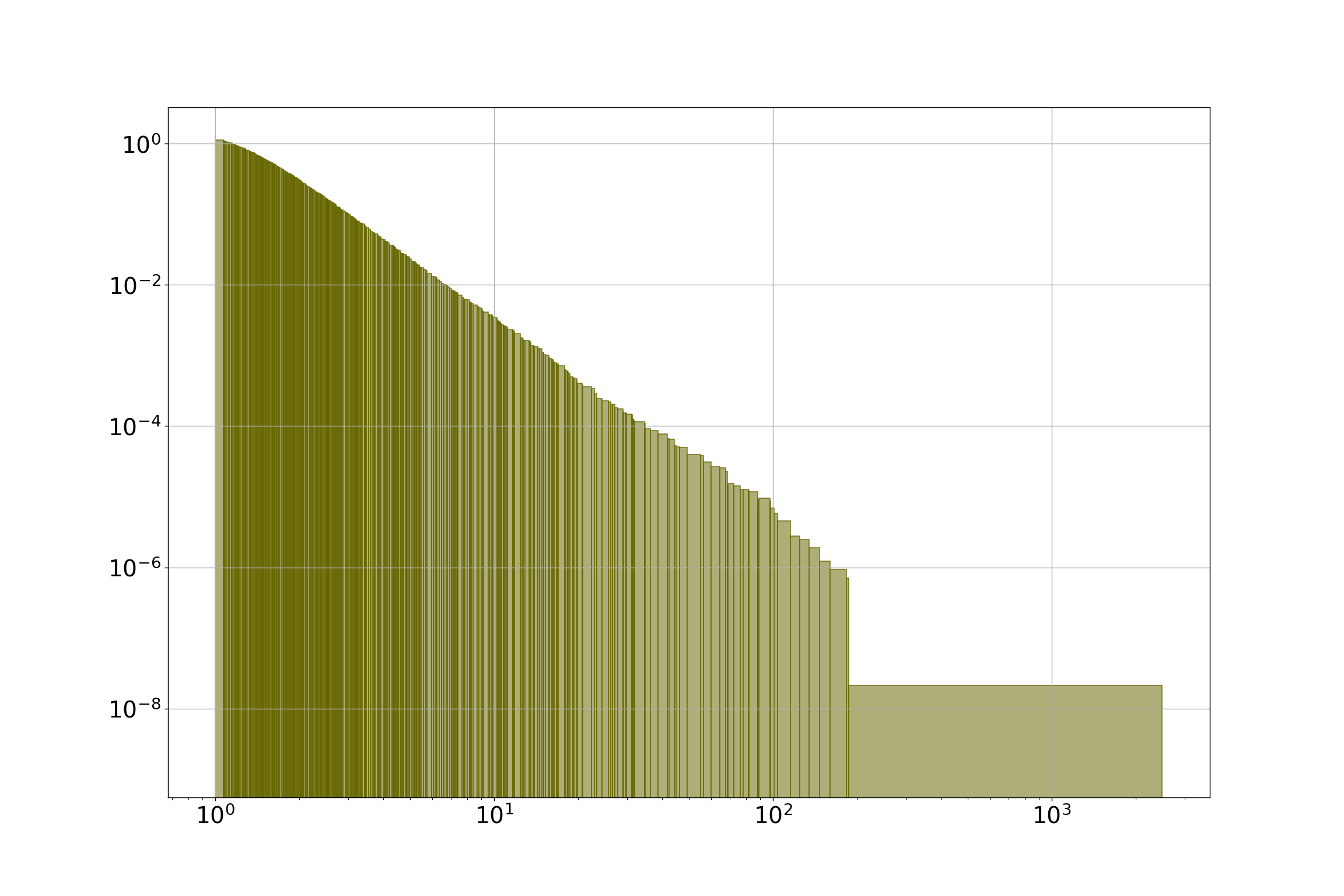}}%
	\hfill
	\subfloat[RMG ($K^*=73$)]
        {\includegraphics{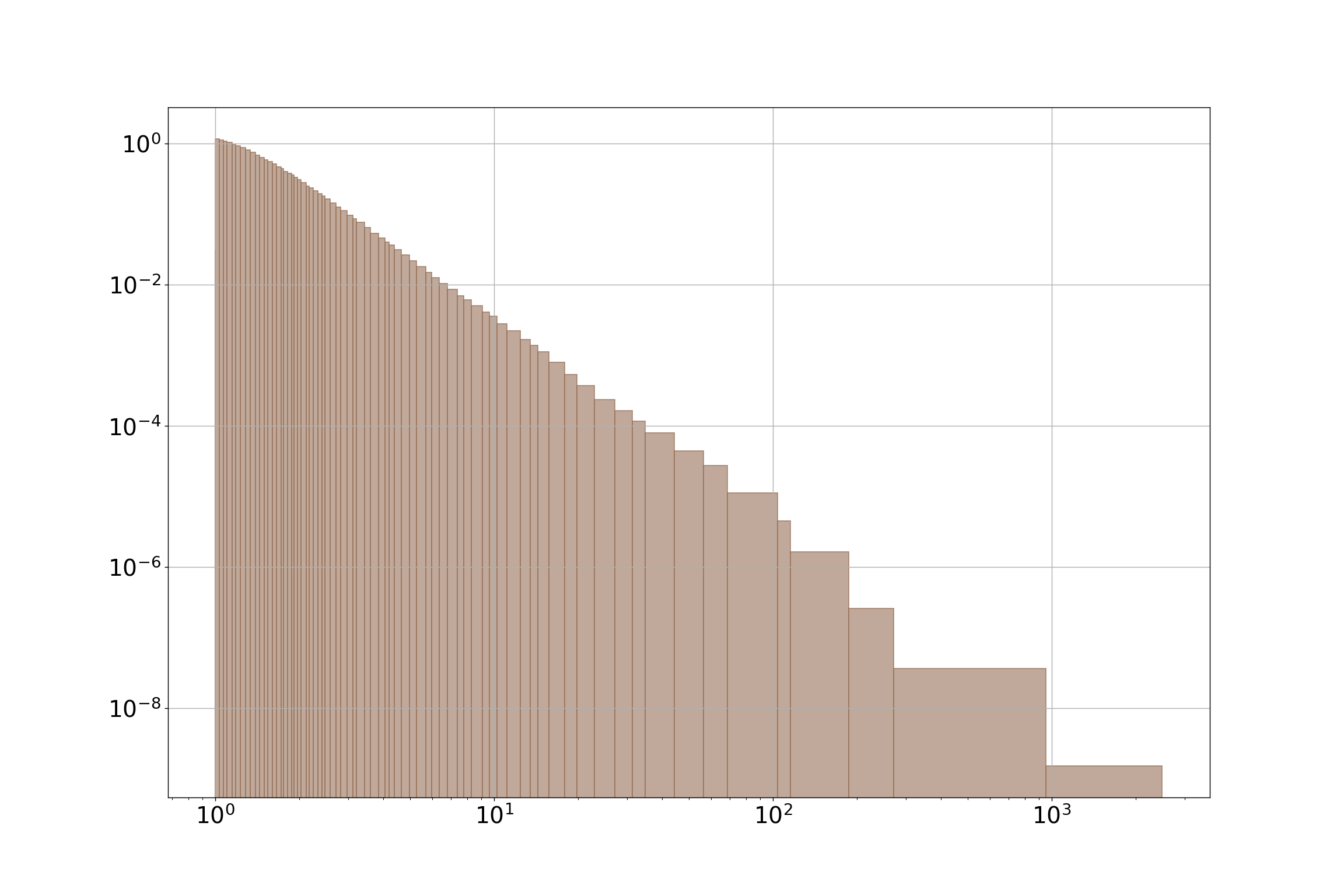}}%
        \hfill
	\subfloat[Bayesian Blocks ($K^*=86$)]
  {\includegraphics{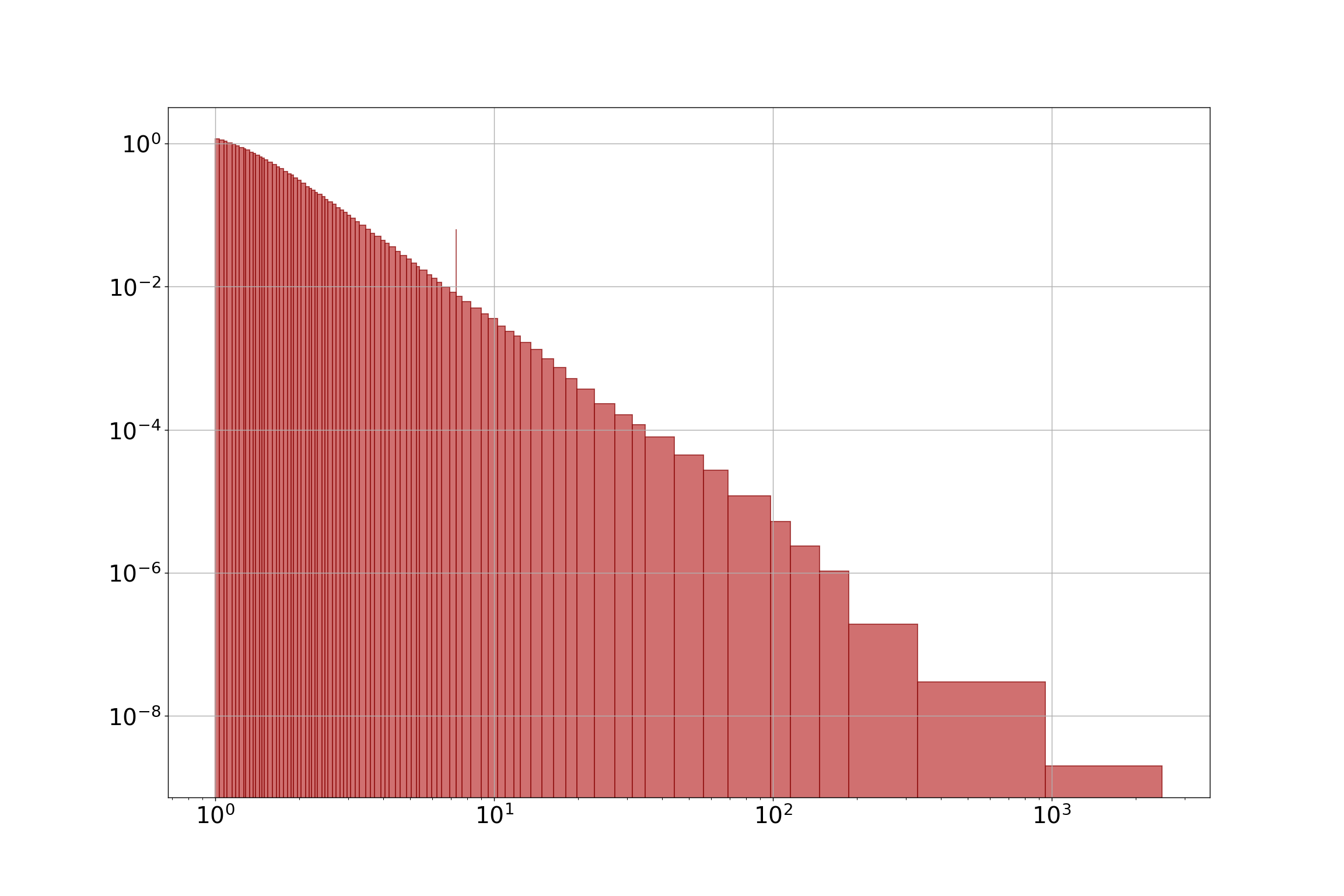}}%
\end{figure}

Figure \ref{fig:lunar-hists} displays the histograms obtained by those methods
while Table \ref{tab:lunar:hd} gives the Hellinger distances between the
densities estimated by the different methods. The \NMLname   histogram is the most
unique one. The log log display used on the figures hides to some extent the
main source of disagreement between the histograms which is the shorter first
interval used by the \NMLname   histogram compared to the others. All other
histograms are fairly close to one another. This emphasizes the non-parsimonious nature of the Taut string method which uses 478 intervals to
produce an histogram that is fairly close to the one obtained by the
Bayesian Blocks method with only 86 intervals. In addition, while this does
not play a major role in the Hellinger distance because of its quite low
support, the estimation of the tail of the distribution by the Taut string
histogram seems to be more crude than other solutions.

\begin{table}[htbp]
  \centering
  \begin{tabular}{c|ccccc}
    \hline
    Method & \GENUMname   & \NMLname   & RMG & TS & BB\\ \hline
    \GENUMname   & - & 0.131 & 0.0118 & 0.0103 & 0.0112\\
    \NMLname   & - & - & 0.131 & 0.123 & 0.130\\
    RMG & - & - & - & 0.0106 & 0.0114\\
    TS & - & - & - & - & 0.00846\\ \hline
  \end{tabular}
  
  \caption{Hellinger distances between the densities estimated by the methods
    on the Lunar Crater data set.}
  \label{tab:lunar:hd}
\end{table}

Overall, this experiment confirms the computational efficiency of \GENUMname
and its ability to produce compact representations of complex and unknown distributions. Importantly, the resulting models are better or comparable to the ones obtained by less efficient solutions. G-Enum histograms provide interesting insights on the previously unknown distribution laws of large data sets without needing much computation time.

\section{Conclusion}\label{sec:conclusion} 

We presented a simple yet robust and very efficient enumerative criterion for histogram model
selection that produces indistinguishable results to those obtained with much more compute-costly \NMLname  approach presented in a previous work.

By pairing our criterion with a search heuristic rather than the optimal but costly original optimisation algorithm, we achieve substantial gains in computation time.
By introducing granularity to alleviate our dependency to the sole user parameter of this problem, the approximation accuracy $\epsilon$, we achieve substantial gains in robustness.

With our theoretical and  experimental evaluation of these criteria we
show that our granulated MDL criterion fills a gap in the current histogram model
selection landscape : it's a resilient, efficient and fully automated approach to histogram density estimation that can scale to explore known or unknown distribution laws in large data sets.

\section*{Acknowledgments}
  We thanks the two anonymous reviewers and the associated editor for
  their insightful comments which help us to improve the quality of
  the manuscript. 

%
%

\newpage
\bibliographystyle{elsarticle-harv}
\bibliography{MDLhistspaper}   

\appendix
\newpage
\setcounter{proposition}{0}
\setcounter{theorem}{0}
\setcounter{corollary}{0}
\section{Rewrite of K\&M's NML criterion}\label{apx:KM-rewrite} 
The final form of K\&M's NML criterion in their paper is as follows. 
\begin{align*}
B(x^n | E, K, C) & = \mathcal{SC}(x^n | \mathcal{M} ) + \log  \binom{E}{K-1}\\
& = \sum^K_{k=1} -h_k (\log (\epsilon \cdot h_k) - \log  (L_k \cdot n)) \nonumber\\
& + \log  ~{\mathcal{R}}^n_{\mathcal{M}}  + \log  \binom{E}{K-1} \nonumber 
\end{align*}

To ease the comparison of both criteria, we simplified the likelihood term as described heareafter.
\begin{align*}
 \sum^K_{k=1} -h_k (\log (\epsilon \cdot h_k) - \log  (L_k \cdot n))&= \sum^K_{k=1} -h_k \left[\log  \epsilon -\log  L_k \right] \\
 &+ \sum^K_{k=1} -h_k \left[\log  h_k - \log  n \right]
\end{align*}
Then, we have 
\begin{align*}
\sum^K_{k=1} -h_k \left[\log  h_k - \log  n \right] &=  -\log  \left( \prod^K_{k=1} {h_k}^{h_k}  \right) +\log  n^n \\
&=  \log  \frac{n^n}{{h_1}^{h_1} \dotsm {h_K}^{h_K}}
\end{align*}
And
\begin{equation*}
\sum^K_{k=1} -h_k \left[\log  \epsilon -\log  L_k \right]= \sum^K_{k=1} h_k \log  \frac{L_k}{\epsilon}=\sum^K_{k=1} h_k \log  E_k 
\end{equation*}

\section{Proofs}\label{apx:proofs}
Proofs are not given in the same order as in the main text. Indeed
proofs of propositions \ref{proposition:singleton-bins} and
\ref{proposition:empty-bins} are based on proposition
\ref{proposition:opt-cutpoint} but for the clarity of exposure, we preferred
to present this latter proposition before the former ones. Of course, the
proof of proposition \ref{proposition:opt-cutpoint} uses neither proposition
\ref{proposition:singleton-bins} nor proposition \ref{proposition:empty-bins},
which is emphasized by the natural mathematical order used in the present
section. 

\subsection{Proof of proposition \ref{proposition:no-null-length}}
\begin{proposition}
Let $\mathcal{M}=(K, (c_k)_{0\leq k\leq K}, (h_k)_{1 \leq k \leq K})$ be
an optimal histogram for the data set $D$. Then
\begin{equation*}
\forall k, 1\leq k\leq K\ c_{k-1}<c_{k}.
\end{equation*}
In other words, an optimal histogram cannot contain zero-length intervals. 
\end{proposition}
\begin{proof}
  Let $\mathcal{M}=(K, (c_k)_{0\leq k\leq K}, (h_k)_{1 \leq k \leq K})$ be a
  histogram compatible with $D$. 
  Let us assume that $ c_{l-1}=c_{l}$ for some $l>0$ and let
  $\mathcal{M'}=(K-1, (c'_k)_{0\leq k\leq K-1}, (h_k')_{1 \leq k > K})$ be the
  histogram derived from $\mathcal{M}$ by removing the zero-length
  interval. It is defined formally by
\begin{equation*}
  c'_k=  \begin{cases}
    c_k&\text{when } k<l\\
    c_{k+1}&\text{when } k\geq l,
  \end{cases}
\end{equation*}
and
\begin{equation*}
  h'_k=  \begin{cases}
    h_k&\text{when } k<l\\
    h_{k+1}&\text{when } k\geq l.
  \end{cases}
\end{equation*}
Obviously, $\mathcal{M'}$ defines the same density as $\mathcal{M}$.
In addition, as $h_l=0$ 
  \begin{align*}
    \sum^K_{k=1} h_k \log E_k&=  \sum^K_{k=1,k\neq l} h_k \log E_k,  \\
    &=  \sum^K_{k=1,k\neq l} h_k \log \frac{c_k-c_{k-1}}{\epsilon},\\
    &=\sum^{K-1}_{k=1} h'_k \log \frac{c'_k-c'_{k-1}}{\epsilon},
    &=\sum^{K-1}_{k=1} h'_k \log E'_k.,
  \end{align*}
and   
  \begin{align*}
    \log \frac{n!}{\prod_{k=1}^Kh_k!}&=\log \frac{n!}{\prod_{k=1,k\neq
                                       l}^Kh_k!},\\
    &=\log \frac{n!}{\prod_{k=1}^{K-1}h'_k!}.
  \end{align*}
  Therefore $\Delta=\cMODL(\mathcal{M}|D)-\cMODL(\mathcal{M'}|D)$ is given by
  \begin{align*}
\Delta    =&\log^* K-\log^* (K-1)\\
                                                 &+\log\binom{E+K-1}{K-1}-\log\binom{E+K-2}{K-2}\\
                                                 &+\log\binom{n+K-1}{K-1}-\log\binom{n+K-2}{K-2},\\
  \end{align*}
which shows that $\Delta>0$ and thus that $\mathcal{M}$ cannot be optimal. 
\end{proof}

\subsection{Proof of proposition \ref{proposition:single-bin-vs-singleton}}\label{apx:th-single-bin-vs-singleton-proof}
\begin{proposition} 
  Let $D$ be a data set with $n$ observations. Let us denote
  $\mathcal{M}_{K=n}$ a histogram compatible with $D$ such that there is one
  observation per interval and $\mathcal{M}_{K=1}$ a histogram compatible with $D$ with only one interval. Then 
  the coding length of $\mathcal{M}_{K=1}$ is shorter than
 the one of $\mathcal{M}_{K=n}$:
\begin{equation*}
\cMODL(\mathcal{M}_{K=n}|D) > \cMODL(\mathcal{M}_{K=1}|D).
\end{equation*} 
\end{proposition}
\begin{proof}
  If $\mathcal{M}_{K=n}$ is compatible with $D$ and there is one observation
  per interval, then $\forall k, h_k=1$ and obviously, $E>n$. Also $\forall k, 
  E_k > 0$. 

Let us define $$\delta \cMODL(n, 1) = \cMODL(\mathcal{M}_{K=n}) -
\cMODL(\mathcal{M}_{K=1}).$$

Then we have 
\begin{eqnarray*}
\delta \cMODL(n, 1) &=& \log  \binom{E+n-1}{n-1} + \log  \binom{2n-1}{n-1} \\
            && + \log  n! + \sum^n_{k=1} \log  E_k \\
						&& + \log ^* n -\log ^* 1- n \log  E.
\end{eqnarray*}
Since $ \log  \binom{E+n-1}{n-1} = \sum^{n-1}_{k=1}{\log  (E+k)} - \log  (n-1)!$, we get
\begin{eqnarray*}
\delta \cMODL(n, 1) &=& \sum^{n-1}_{k=1}{\log  (E+k)} + \log  \binom{2n-1}{n-1} \\
            && + \log  n + \sum^n_{k=1} \log  E_k \\
						&& + \log ^* n -\log ^* 1- n \log  E.
\end{eqnarray*}
As function $f(x)=\log (1+x)$ is strictly concave and $f(0) = 0$, it is sub-additive on $[0, +\infty[$.
We obtain
\begin{eqnarray*}
\sum^n_{k=1} \log  E_k &=& \sum^n_{k=1} 	f(E_k-1) \\
            &\geq& f\left(\sum^n_{k=1} (E_k-1)\right) \\
						&\geq& \log  (E-n+1).
\end{eqnarray*}

Back to $\delta \cMODL(n, 1)$, we get
\begin{eqnarray*}
\delta \cMODL(n, 1) &\geq& \sum^{n-1}_{k=1}{\log  E(1+\frac{k}{E})} +  \log  \binom{2n-1}{n-1} \\
            && \log  n  + \log  E(1-\frac{n-1}{E})\\
						&&  + \log ^* n -\log ^* 1- n \log  E \\
            &\geq& \log ^* n  -\log ^* 1+ \log  n  + \log  \binom{2n-1}{n-1} \\
            && +\sum^{n-1}_{k=1}{\log  (1+\frac{k}{E})} + \log  (1-\frac{n-1}{E}).
\end{eqnarray*}

As $n \leq E$, we get
\begin{eqnarray*}
\log  (1-\frac{n-1}{E}) + \log n &=& \log  (n-\frac{n}{E}(n-1))\\
																	&\geq& \log  (n-(n-1))\\
																	&\geq& 0.
\end{eqnarray*}

Finally, $\delta \cMODL(n, 1) >0$, which proves the claim that the histogram with one single interval is always more probable than the one with $n$ singleton intervals.
\end{proof}

\subsection{Proof of proposition \ref{proposition:single-bin-vs-singleton-empty}}\label{apx:th-single-bin-vs-singleton-empty-proof}
\begin{proposition}
  Let $D$ be a data set with $n$ observations. Let us denote
  $\mathcal{M}_{K>n}$ a histogram compatible with $D$ consisting of either
  singleton or empty intervals, one interval for each observation and empty
  intervals in-between, and let $\mathcal{M}_{K=1}$ be as in Proposition
  \ref{proposition:single-bin-vs-singleton}. Then the coding length of
  $\mathcal{M}_{K=1}$ is shorter than the one of $\mathcal{M}_{K>n}$:
\begin{equation*}
\cMODL(\mathcal{M}_{K > n}|D) > \cMODL(\mathcal{M}_{K=1}|D).
\end{equation*}
\end{proposition}

\begin{proof}
Let us consider a histogram consists with $K > n$ intervals composed of $E_k$
$\epsilon$-bins, with either $h_k=1$ for the singleton intervals and $h_k=0$
for the empty intervals. 
\begin{eqnarray*}
\delta \cMODL(K, 1) &=& \log  \binom{E+K-1}{K-1} + \log  \binom{n+K-1}{K-1} \\
            && +\log ^* K - \log ^* 1 + \log  n! \\
						&& + \sum^n_{k=1} \log  E_k - n \log  E.
\end{eqnarray*}

We have
\begin{eqnarray*}
\log  \binom{E+K-1}{K-1} &=& \sum^{K-1}_{k=1}{\frac {\log (E+k)}{\log k}}\\
												  &=& \sum^{n}_{k=1}{\frac {\log (E+k)}{\log k}}\\
                          & & +\sum^{K-1}_{k=n+1}{\frac {\log (E+k)}{\log k}}			  
\end{eqnarray*}
 Given that $\log  (E+k) > \log  E$, we have $ \sum^{n}_{k=1}{\frac {\log (E+k)}{\log k}} > n \log E - \log n! $. Likewise, we know that $\log  (E + k) > \log  k$, so the sum $\sum^{K-1}_{k=n+1}{\frac {\log (E+k)}{\log k}}$ is greater than $\sum^{K-1}_{k=n+1}{1} = K - n - 1$. Given these two lower bounds, we can write that 
 \begin{equation*}
\log  \binom{E+K-1}{K-1} >  n \log E - \log n! +  (K - n - 1) 
\end{equation*}

We then obtain
\begin{eqnarray*}
\delta \cMODL(K, 1) &>& (K - n - 1) \\
            && +\log  \binom{n+K-1}{K-1} \\
            && +\log ^* K - \log ^* 1 + \sum^n_{k=1} \log  E_k.
\end{eqnarray*}
Finally, $\delta \cMODL(K, 1) >0$, which proves the claim that the histogram with one single interval is more probable that any histogram where entries are in singleton intervals.

\end{proof}
\begin{remark}
 Similar results to  Proposition~\ref{proposition:single-bin-vs-singleton} and
 Proposition~\ref{proposition:single-bin-vs-singleton-empty} could not be
 obtained for the NML criterion given by equation \eqref{eq:cNML} for two reasons. First, the analysis of the criterion for histograms with $K$ singleton intervals ($n \leq K \leq E$) is complex, since the prior term $\log  \binom{E}{K-1}$ is not monotonous and decreases for $K > E/2$.
Second, the parametric complexity $\log\mathcal{R}^n_{\mathcal{M}_K}$ term has no closed-form formula and its best known approximations \citep{Kontkanen2009,Szpankowski98} are given for fixed $K$ as $n$ goes to infinity. These approximations cannot be used when $n \leq K \leq E$.
\end{remark}

\subsection{Proof of proposition \ref{proposition:non-consecutive-empty} }\label{apx:th-non-consecutive-empty-proof}
We begin by evaluating the expression of the cost variation $\Delta c$ after
the merge of two intervals. Let A and B be two adjacent intervals composed of
$E_A$, $E_B$   $\epsilon$-length elementary bins and of data counts $h_A$,
$h_B$ in a $K$-bin histogram. They are merged into a single interval composed
of $E_{A\cup B}= E_A + E_B$ $\epsilon$-length elementary bins with $h_{A\cup
  B}=h_A + h_B$ elements, creating a histogram with $K-1$ bins (we
consider only histograms compatible with the data). 

For this fusion, the variation for the \ENUMname criterion, $\Delta c$,  is given by
\begin{align}\label{eq:deltaMODL}
{\Delta \cMODL} =&   \cMODL(\mathcal{M}_{K-1}) - \cMODL(\mathcal{M}_{K}) \\
=& \log  \frac{K-1}{E+K}  + \log  \frac{K-1}{n+K-1} \nonumber\\
  &+\log ^* (K-1)- \log ^* K\nonumber\\ 
& + \log  \frac {h_A! h_B!}{(h_A + h_B)!} +  \log  \frac{(E_A+E_B)^{(h_A+h_B)}}{{E_A}^{h_A}\cdot {E_B}^{h_B}} \nonumber
\end{align}
\smallskip

The fusion of two intervals is advantageous if the code length for the histogram
with $K-1$ intervals is smaller than the code length for the histogram with $K$
intervals, that is if $\Delta \cMODL <0$.  Proposition
\ref{proposition:non-consecutive-empty} corresponds to an interesting
particular case. 
\begin{proposition}
The coding length of a histogram with two adjacent empty intervals is always longer than the coding length of a histogram with no consecutive empty intervals.
\end{proposition}
  \begin{proof}
We analyse the cost variation when two adjacent intervals are empty for the
\ENUMname criterion  ($h_A = 0 $, $h_B = 0$, so $h_A + h_B = 0$). From equation
\ref{eq:deltaMODL} we have
\begin{align*}
{\Delta \cMODL}&_{(h_A=0, h_B=0)}=\\
& \log  \frac{K-1}{E+K-1}  + \log  \frac{K-1}{n+K-1} \\
&+\log ^* (K-1) - \log ^* K
\end{align*}
\smallskip

Given that $n+K-1 > K-1$, we have ~$\log  \frac{K-1}{n+K-1}<0$. Similarly, $E+K-1>K-1$, so ~$\log  \frac{K-1}{E+K-1}<0$. Rissanen's universal code for integers ($\log ^* K$) is a monotone function, so  $\log ^* K-1 - \log ^* K<0$.

We have thus shown that, for the \ENUMname criterion, the cost variation is {\bf always strictly negative} after the merge of two adjacent empty intervals. This means that keeping a model with two adjacent empty bins is always more costly in terms of code length, so our search algorithm will systematically favor the merge of two consecutive empty bins.
\end{proof}

\subsection{Proof of proposition \ref{proposition:Kbound}}\label{apx:th-Kbound-proof}

\begin{proposition}
An optimal histogram has at most $2n-2$ intervals ($K^* \leq 2n-2$).
\end{proposition}

\begin{proof}
  The maximal number of non empty intervals is obviously bounded by the number
  of observations $n$. Provided the data set has $n$ distinct observations and
  that $\epsilon$ is small enough, we can consider a histogram with $n$
  intervals containing each a single observation, separated by  empty
  intervals. According to proposition \ref{proposition:empty-MODL} an optimal
  histogram cannot have consecutive empty intervals and thus the maximal
  number of separating intervals is $n-1$. Therefore an optimal histogram
  cannot have more than $2n-1$ intervals.
  
  But proposition \ref{proposition:single-bin-vs-singleton-empty} shows that a
  histogram with a single interval is preferred by the criterion over a
  histogram consisting only of empty intervals and intervals containing each a
  single observation. Therefore the maximal model with $2n-1$ intervals cannot
  be optimal. Then the optimal number of intervals $K^*$ is bounded by $2n-2$.
\end{proof}
 
\setcounter{proposition}{7}
\subsection{Proof of proposition \ref{proposition:opt-cutpoint}}\label{apx:th-opt-cutpoint-proof}
\begin{proposition}
In an optimal histogram, each endpoint is at most $\epsilon$ away from one of the values of the data set.
\end{proposition}

\begin{proof}
Let $\mathcal{M}^\star=(K, \{c_k\}_{0 \leq k \leq K}, \{h_k\}_{1 \leq k \leq
  K})$ be an optimal histogram built from a data set $D$.

As by definition, $c_0=x_{min}-\epsilon/2$ and $c_K = x_{max}+\epsilon/2$, the
property is valid for $c_0$ and $c_K$.

Let us now assume that $K > 1$ and focus on endpoints $c_k, 0 < k < K$, that
is on cut-points between adjacent intervals. 

Let $i, j=i+1$ be the index of two such adjacent intervals. We study the location
of the cut-point between intervals $i$ and $j$, $c_i$, while considering the
exterior boundaries of intervals $i$ and $j$ ($c_{i-1}$ and $c_j$) to be
fixed. We denote $L_{i,j}=c_j-c_{i-1}$. We consider also the data counts of
the intervals $h_i$ and $h_j$ to be fixed. In other words, we study to what
extent the cut-point between intervals $i$ and $j$ can be freely set at
grid positions in the empty space between the last data point in $i$ and the
first data point in $j$, as illustrated by Figure~\ref{fig:consecutive-vals-cut}.

Let $D_i$ (resp $D_j$) be the data point in interval $i$ (resp. $j$). We
define
\begin{equation*}
L_i=\min\{t\epsilon\mid \max D_i\leq c_{i-1}+t\epsilon\},  
\end{equation*}
and
\begin{equation*}
L_j=\min\{t\epsilon\mid \min D_j> c_{j}-t\epsilon\},
\end{equation*}
i.e. $L_i$ and $L_j$ are the minimum lengths of intervals that contain
respectively $D_i$ and $D_j$. 

In terms of elementary $\epsilon$-length bins, let $E_i=L_i/\epsilon,
E_j=L_j/\epsilon, \widehat{E}=L_{i,j}/\epsilon$.  
We note $x = l/\epsilon \in [0, \widehat{E}-E_i-E_j]$ the number  of
$\epsilon$-length bins at which the cut-point $c_i$ is placed.   
\begin{figure}[htbp]
\centering
\includegraphics[width=\linewidth]{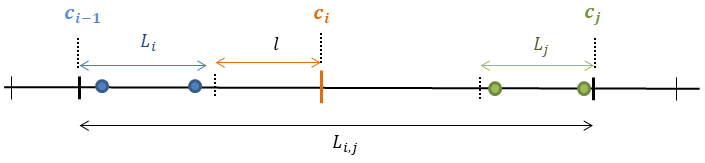}
\caption{Illustration and notations of the optimal cut-point between two
  values problem}
\label{fig:consecutive-vals-cut}
\end{figure}

In this setting, each different value of $x$ defines a different $K$-bin
histogram, $\mathcal{M}_K(x)$. Using the \ENUMname criterion, each model has the following cost : 
\begin{align*}
\cMODL(\mathcal{M}_K(x)|D)=& \log  \binom{E+K-1}{K-1} + \log  \binom{n+K-1}{K-1} \\
&+\log ^* K + \log  \frac{n!}{h_1!\ldots h_i!h_j!\ldots h_K!} \\
&+ \sum_{k < i} {h_k \log E_k} + h_i \log  (E_i+x) \\
&+ h_j \log  (\widehat{E}-(E_i+x)) + \sum_{k > j} {h_k \log E_k}\\
\end{align*}
Seeing as we focus on the frontier between just two intervals, the model index
term does not change. Since the data counts of each interval are preserved
regardless of the position of the interval's endpoints, the corresponding part
of the likelihood term will not change. In this setting, the sole term
responsible for the cost variation between models is:
\begin{equation*}
f_{h_i, h_j}(x) = h_i \log  (E_i+x) + h_j \log  (\widehat{E}-(E_i+x))      
\end{equation*}
From Proposition~\ref{proposition:non-consecutive-empty}, an optimal histogram
cannot contain two consecutive empty intervals, therefore at least one of the
two intervals is non empty. We study the three possible cases.

\begin{enumerate}
\item Let us first consider the case where both intervals are non empty
  ($h_i>0$ and $h_j>0$).  As $f_{h_i, h_j}$ is a concave function on the
  closed interval $[0, \widehat{E}-E_i-E_j]$, its minimum values are only
  reached at its extremities, that is for $x=0$ and $x= \widehat{E}-E_i-E_j$.

  We have thus shown that the best cut-point between the two intervals is
  placed at at most $\epsilon$, either from last value of the first interval
  or from the first value of the second interval.

\item  Let us now consider the case where the first interval is non empty,
  that is $h_i > 0$ and $h_j=0$. Then 
  $f_{h_i, h_j}(x) = h_i \log (E_i+x)$. In this case, the minimum value of
  $f_{h_i, h_j}(x)$ is obtained for $x=0$, that is with a cut-point placed at
  at most $\epsilon$ from last value of the first interval.

\item  Similarly, if the second interval is non empty, the cut-point is placed at
  at most $\epsilon$ from the first value of the second interval. Let us
  notice that in both cases, the empty interval is of maximal length.
\end{enumerate}
Overall, the endpoints of the optimal histogram are always at at most
$\epsilon$ from one value of the data set. 
\end{proof}

\begin{remark}
  Intuitively, the choice of which one of these two placements is best will
  also depend on the data count of each interval. As we have highlighted
  before, a part of the likelihood term will favour shorter dense intervals
  over large ones. See the proof of proposition
  \ref{proposition:singleton-bins} for an illustration in a simple case. 
\end{remark}

\begin{remark}\label{remark:proof-prop-refpr}
Notice that for any value $x(c_k)$ of the data set that is the closest one
from an endpoint $c_k$, $c_k$ is as  close as possible from $x(c_k)$. In
particular, if $x(c_k)$ is the last point in the interval $]c_{k-1}, c_k]$, we
have $c_k - x(c_k) < \epsilon$ with a strict inequality. Otherwise, the
interval $]c_{k-1}, c_k-\epsilon]$ still contains $x(c_k)$, with an endpoint
closer from $x(c_k)$. 
\end{remark}

\begin{remark}
Empty intervals are always as long as possible, with their two endpoints close
to values of the data set. 
\end{remark}

\setcounter{proposition}{4}

\subsection{Proof of proposition \ref{proposition:singleton-bins}}\label{apx:th-singleton-bins-proof}

  \begin{figure}[btp]
    \centering
\includegraphics[width=\linewidth]{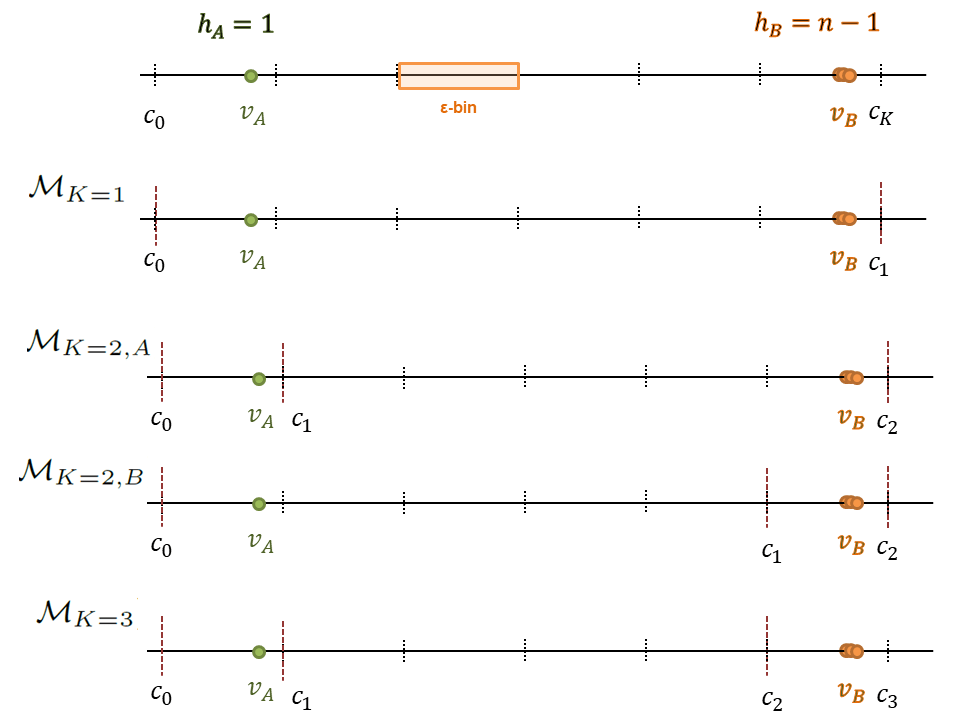}
    \caption{Graphical representation of the different cases involved in
      proposition \ref{proposition:singleton-bins}.}\label{fig:proposition:singleton-bins}
  \end{figure}

\begin{proposition}
There exist data sets such that the optimal histogram has at least one
  interval which contains only a single observation. 
\end{proposition}

\begin{proof}
  Let us consider a data set with $n$ observations and only two distinct values
  $v_A, v_B$ of frequencies $h_A=1, h_B=n-1$. We choose $v_A$ and $v_B$ such
  that $v_B-v_A\geq 3\epsilon$. According to proposition
  \ref{proposition:non-consecutive-empty} an optimal histogram cannot contain
  successive empty intervals and thus we need only to consider three cases:
  histograms with one, two or three intervals. Those cases are illustrated on
  Figure \ref{fig:proposition:singleton-bins}.

For $K=1$, we have 
\begin{eqnarray*}
\cMODL(\mathcal{M}_{K=1})&=& n \log  E + \log ^* 1.
\end{eqnarray*}

For $K=2$, the optimal split is necessarily within $\epsilon$ of $v_A$ or
$v_B$ according to proposition \ref{proposition:opt-cutpoint}. If the split is
close to $v_A$, we have $E_A=1, E_B=E-1$ and
\begin{eqnarray*}
\cMODL(\mathcal{M}_{K=2, A})&=& \log ^* 2 + \log(E+1) + \log(n+1)\\
  && + \log n + (n-1)\log (E-1)\\
	&=& n\log E + o(n).
\end{eqnarray*}
If the split is close to $v_B$,  we have $E_A=E-1, E_B=1$ and
\begin{eqnarray*}
\cMODL(\mathcal{M}_{K=2,B})&=& \log ^* 2 + \log(E+1) + \log(n+1)\\
  && + \log n + \log (E-1)\\
	&=& 2 \log E + 2 \log n + O(1).
\end{eqnarray*}

For $K=3$, the optimal split is necessarily right next to each value and we
thus have two non-empty intervals composed $E_A=1, E_B=1$ $\epsilon$-bins,
surrounding an empty interval of $E-2$ $\epsilon$-bins (this is again a
consequence of proposition \ref{proposition:opt-cutpoint}). We get

\begin{eqnarray*}
\cMODL(\mathcal{M}_{K=3})&=& \log ^* 3 + \log \frac{(E+1)(E+2)}{2}\\
  && + \log \frac{(n+1)(n+2)}{2} + \log n\\
	&=& 2 \log E + 3 \log n + O(1).
\end{eqnarray*}

Therefore, even for rather small $n$ and $E$, the model cost is minimal for
$K=2$ with the split within $\epsilon$ from $v_B$, i.e. for the histogram
where the first interval contains the singleton value $v_A$. 
\end{proof}

\subsection{Proof of proposition \ref{proposition:empty-bins}}\label{apx:th-empty-bins-proof}

\begin{proposition}\label{proposition:empty-MODL}
There exist data sets such that the optimal histogram has at least one
interval that does not contain any observation. 
\end{proposition}
\begin{proof}
Let us consider a data set with $n$ observations and only two values $v_A,
v_B$ of frequencies $h_A= h_B=\frac{n}{2}$. We choose $v_A$ and $v_B$ such
  that $v_B-v_A\geq 3\epsilon$. For the same reasons as in the
proof of proposition \ref{proposition:singleton-bins} we need only to consider three cases:
  histograms with one, two or three intervals. 

For $K=1$, we have 
\begin{eqnarray*}
\cMODL(\mathcal{M}_{K=1})&=& n \log  E + \log ^* 1.
\end{eqnarray*}

For $K=2$, the optimal split is necessarily within $\epsilon$ of $v_A$ or
$v_B$ according to proposition \ref{proposition:opt-cutpoint}. Because of the
symmetry of the setting, we can set the split to be near $v_B$.  Thus we have
two intervals composed of $E_A=E-1, E_B=1$ $\epsilon$-bins.

We get
\begin{eqnarray*}
\cMODL(\mathcal{M}_{K=2})&=& \log ^* 2 + \log(E+1) + \log(n+1)\\
  && + \log \frac{n!}{h_A! h_B!} + h_A \log (E-1)
\end{eqnarray*}

To analyse this quantity, we use the approximation given in \cite{Grunwald07}
(formula~4.36) which states that for $\theta \in]0, 1[$
\begin{equation*}
\log {{n}\choose{\theta n}}=  n H(\theta) -\frac{1}{2} \log (2 \pi n \mathrm{var}(\theta)) + O\left(\frac{1}{n}\right),
\end{equation*}
with $H(\theta)=-\theta \log \theta - (1-\theta) \log (1-\theta)$ and 
$\mathrm{var}(\theta) = \theta (1-\theta)$.

Using $h_A=h_B=\frac{1}{2}$, we obtain
\begin{equation*}
 \log \frac{n!}{h_A! h_B!}= n \log 2 -\frac{1}{2} \log \left(\pi \frac{n}{2}\right)+O\left(\frac{1}{n}\right),
\end{equation*}
and thus
\begin{equation*}
  \cMODL(\mathcal{M}_{K=2})=n \log (2 \sqrt E) + \log E + \frac{1}{2} \log n +O(1).
\end{equation*}
Finally, as in the proof of proposition \ref{proposition:singleton-bins}, the
optimal histogram with $K=3$ intervals consists in  two non-empty intervals,
each with $E_A=1, E_B=1$ $\epsilon$-bins, surrounding an empty interval
composed of $E-2$ $\epsilon$-bins. Therefore 
\begin{eqnarray*}
\cMODL(\mathcal{M}_{K=3})&=& \log ^* 3 + \log \frac{(E+1)(E+2)}{2}\\
  && + \log \frac{(n+1)(n+2)}{2} \\
	&& + \log \frac{n!}{h_A! h_B!}\\
	&=& n \log 2 + 2 \log E + \frac{3}{2} \log n + O(1).
\end{eqnarray*}

Therefore, even for rather small $n$ and $E$, the model cost is minimal for $K=3$, for the histogram where the second interval is empty and has $(E-2) \cdot\epsilon$.
\end{proof}

\section{Asymptotic behaviour of the MODL criterion when $\epsilon \rightarrow
  0$}\label{apx:limit-histogram-proof}
\subsection{proof of Theorem \ref{th:limit-histogram}}
\begin{theorem}
  Let $D$ be a data set with $n$ observations. There
  exists two positive values $C(D)$ and $E(D)$ that depends only on $D$ such
  that for all $\epsilon\leq E(D)$ for any
  \emph{optimal} histogram $\mathcal{M}^\star$ have
\begin{equation}
  \begin{split}
    \left| \cMODL(\mathcal{M}^\star|D)
      -\left\{K-1+n-S(\mathcal{M}^\star,D)\right\}\log\frac{1}{\epsilon}\right|\\\leq
    C(D).
  \end{split}
\end{equation}

\end{theorem}
\begin{proof}
  Let $D$ be a fixed data set with $n$ observations. For any $\epsilon$ such
  that $\dfrac{L}{\epsilon}$ is an integer,
  $\mathcal{M}_{\epsilon}^\star=(K, (c_k)_{0\leq k\leq K}, (h_k)_{1 \leq k
    \leq K})$ be a optimal histogram constructed on the
  $E=\dfrac{L}{\epsilon}+1$ regular grid of $\epsilon$-bins.
The \ENUMname criterion for
$\mathcal{M}_{\epsilon}^\star$ is given by: 
\begin{eqnarray*}
\cMODL(\mathcal{M}_{\epsilon}^\star|D) &=&  \log ^* K + \log  \binom{n+K-1}{K-1} \\
&&+ \log  \binom{E+K-1}{K-1} \\
&&+ \log  \frac{n!}{h_1!... h_{K}!} + \sum^{K}_{k=1} h_k \log  E_{k,\epsilon},
\end{eqnarray*}
where the $E_{k,\epsilon}$ are the sizes of the intervals expressed in terms
of $\epsilon$-bins. 

More precisely, let $c_{0, \epsilon}, c_{1, \epsilon}, \ldots, c_{K,
  \epsilon}$ be the end 
points of the $K$  intervals of $\mathcal{M}_{\epsilon}^\star$. By definition of the
grid, we have
\[
L_{k, \epsilon} = c_{k, \epsilon} - c_{k-1, \epsilon} = E_{k, \epsilon} \times
\epsilon, 1 \leq k \leq K
\]
and
\[
L + \epsilon = c_{K_\epsilon, \epsilon} - c_{0, \epsilon} = E \times
\epsilon.
\]
Using 
\begin{equation*}
\log  \binom{E+K-1}{K-1}=\sum_{k=1}^{K-1}\log (E+k)-\log(K-1)!,
\end{equation*}
and the notations above, we have
\begin{eqnarray*}
\cMODL(\mathcal{M}_{\epsilon}^\star|D)
&=&  \log ^* K + \log  \binom{n+K-1}{K-1} \\
&&+ \log  \frac{n!}{h_1!... h_{K}!} - \log (K-1)!\\
&&+ \sum_{k=1}^{K-1} \log \left(\frac{L}{\epsilon} +k+1\right) \\
&&+ \sum^{K}_{k=1} {h_k \log  \frac{L_{k, \epsilon}}{\epsilon}}.
\end{eqnarray*}
Obviously, the main influence of $\epsilon$ on the criterion is through the
last two terms. Let us first notice that $\forall k,\ 1\leq k\leq K-1$ 
\begin{equation*}
\log\left(\frac{L}{\epsilon} +k+1\right)=\log\frac{1}{\epsilon}+ \log
(L+(k+1)\epsilon). 
\end{equation*}
and thus
\begin{multline}\label{eq:fulllength}
  \sum_{k=1}^{K-1}\log \left(\frac{L}{\epsilon} +k+1\right)=\\
  (K-1)\log\frac{1}{\epsilon}+\sum_{k=2}^{K}\log (L+k\epsilon).
\end{multline}
To analyse the last term $\sum^{K}_{k=1} {h_k \log  \frac{L_{k, \epsilon}}{\epsilon}}$, let us define $D_{k,\epsilon}$ by
\begin{equation*}
D_{k,\epsilon}=\arg\min_{x\in D}|c_{k,\epsilon}-x|-\arg\min_{x\in D}|c_{k-1,\epsilon}-x|,
\end{equation*}
that is the distance between the data points that are the closest to the boundaries
of interval $]c_{k-1,\epsilon},c_{k,\epsilon}]$. By definition, we have
\begin{equation*}
|L_{k, \epsilon}- D_{k,\epsilon}|\leq \min_{x\in
  D}|c_{k,\epsilon}-x|+\min_{x\in D}|c_{k-1,\epsilon}-x|. 
\end{equation*}
By proposition~\ref{proposition:opt-cutpoint} we know that 
\begin{equation*}
  \forall 0\leq k \leq K,\ \min_{x\in D}|c_{k,\epsilon}-x|\leq \epsilon, 
\end{equation*}
as $\mathcal{M}_{\epsilon}^\star$ is optimal. This shows that
\begin{equation*}
 |L_{k, \epsilon}- D_{k,\epsilon}|\leq 2\epsilon. 
\end{equation*}
Using $D_{k,\epsilon}$, we have
\begin{align}\label{eq:length}
 & \sum^{K}_{k=1} h_k \log \frac{L_{k, \epsilon}}{\epsilon}=\\
 & \sum_{D_{k,\epsilon}=0}h_k \log \frac{L_{k, \epsilon}}{\epsilon}+
  \sum_{D_{k,\epsilon}>0}h_k\left(
    \log\frac{1}{\epsilon}+\log\frac{L_{k,\epsilon}}{D_{k,\epsilon}}+\log D_{k,\epsilon}\right).\notag
\end{align}
Using equations \eqref{eq:fulllength} and \eqref{eq:length}, we can decompose
the \ENUMname criterion $\cMODL(\mathcal{M}_{\epsilon}^\star|D)$ into the sum of the following three
terms:
\begin{eqnarray*}
c_{1}(\mathcal{M}_{\epsilon}^\star|D) &=&  \log ^* K + \log  \binom{n+K-1}{K-1} \\
                                        &&+ \log  \frac{n!}{h_1!... h_{K}!} - \log (K-1)!\\
  &&+\sum_{D_{k,\epsilon}=0}h_k \log \frac{L_{k, \epsilon}}{\epsilon}\\
&&+ \sum_{D_{k, \epsilon}> 0} {h_k \log D_{k, \epsilon}}, \\
c_{2}(\mathcal{M}_{\epsilon}^\star|D) &=&  \sum_{k=2}^{K} {\log (L + k\epsilon)} 
  + \sum_{D_{k, \epsilon}> 0} {\log \frac{L_{k, \epsilon}}{D_{k, \epsilon}}}.\\
c_{3}(\mathcal{M}_{\epsilon}^\star|D) &=& (K-1) \log \frac{1}{\epsilon} 
  + \sum_{D_{k, \epsilon}> 0} {h_k \log \frac{1}{\epsilon}},
\end{eqnarray*}
As $\mathcal{M}_{\epsilon}^\star$ is optimal, $K\leq 2n-2$ (proposition
\ref{proposition:Kbound}) and as a consequence
$c_{1}(\mathcal{M}_{\epsilon}^\star|D)$ is upper and lower bounded by
constants that depend only on the data set.

This is obviously the case of the terms $\log^\star K$,
$\log  \binom{n+K-1}{K-1}$ and $- \log (K-1)!$. The term
$\log \frac{n!}{h_1!... h_{K}!}$ can take only a finite number of positive values for a
fixed $n$ and any $K\leq 2n-2$ and is therefore bounded by the largest of
those values. Moreover, when $D_{k,\epsilon}=0$,
$\epsilon\leq L_{k, \epsilon}\leq 2\epsilon$ and thus
\begin{equation*}
0\leq \sum_{D_{k,\epsilon}=0}h_k \log \frac{L_{k, \epsilon}}{\epsilon} \leq n\log 2.
\end{equation*}
It is also obvious that $D_{k,\epsilon}\leq L$ and thus 
\begin{equation*}
  \sum_{D_{k, \epsilon}> 0} {h_k \log D_{k, \epsilon}}\leq n\log L. 
\end{equation*}
To lower bound this quantity, we introduce
\begin{equation}\label{eq:dmin:def}
D_{\min}=\min\{|x_i-x_j|\mid x_i\in D, x_j\in D, x_i\neq x_j\},
\end{equation}
which depends only on the data set $D$ and is strictly positive (by hypothesis
on $D$). By definition
for all $k$ such that $D_{k, \epsilon}> 0$, $D_{k, \epsilon}\geq D_{\min}$,
and therefore
\begin{equation*}
  \sum_{D_{k, \epsilon}> 0} {h_k \log D_{k, \epsilon}}\geq \min(\log D_{\min},0). 
\end{equation*}

Thus there exists $C_1(D)\geq 0$ such that for all $\epsilon>0$,
$-C_1(D)\leq c_{1}(\mathcal{M}_{\epsilon}^\star|D)\leq C_1(D)$.

The second term $c_{2}(\mathcal{M}_{\epsilon}^\star|D)$ can also be upper and
lower bounded with a minimal condition on
$\epsilon$. Indeed we have
\begin{equation*}
(K-1)\log L\leq \sum_{k=2}^{K} {\log (L + k\epsilon)} \leq (K-1)\log(L+K\epsilon).
\end{equation*}
Assuming that $\epsilon\leq L$ and with $K\leq 2n-2$, we have
\begin{equation*}
(2n-3)\log L \leq \sum_{k=2}^{K} {\log (L + k\epsilon)} \leq (2n-3)\left(\log L+ \log (2n-1)\right).
\end{equation*}

We also have for $D_{k, \epsilon}>0$
\begin{equation*}
  \begin{array}{rcccl}
    -2\epsilon&\leq&L_{k, \epsilon}-D_{k, \epsilon}&\leq&2\epsilon\\
    -2\dfrac{\epsilon}{D_{k, \epsilon}}&\leq & \dfrac{L_{k, \epsilon}}{D_{k,
                                              \epsilon}}-1&\leq&
                                                                 2\dfrac{\epsilon}{D_{k, \epsilon}}\\
    1-2\dfrac{\epsilon}{D_{k, \epsilon}}&\leq& \dfrac{L_{k,
                                               \epsilon}}{D_{k,\epsilon}}&\leq&1+2\dfrac{\epsilon}{D_{k,
                                                                                \epsilon}}\\
    1-2\dfrac{\epsilon}{D_{\min}}&\leq& \dfrac{L_{k,\epsilon}}{D_{k,\epsilon}}&\leq&1+2\dfrac{\epsilon}{D_{\min}}
  \end{array}
\end{equation*}
Thus if we assume $\epsilon\leq \frac{D_{\min}}{4}<L$, we have
  \begin{equation*}
-\log 2\leq \log\left(\dfrac{L_{k,\epsilon}}{D_{k,\epsilon}}\right)\leq \log \frac{3}{2}.
  \end{equation*}
  Thus if   $\epsilon\leq \frac{D_{\min}}{4}$, we have
  \begin{equation*}
    \begin{array}{rcccl}
 -K\log 2&\leq&    \displaystyle \sum_{D_{k, \epsilon}> 0} {\log \frac{L_{k,
                 \epsilon}}{D_{k,    \epsilon}}}&\leq& K
                                                       \log\frac{3}{2}
      \\
  -(2n-2)\log 2&\leq&\displaystyle\sum_{D_{k, \epsilon}> 0} {\log \frac{L_{k,
                 \epsilon}}{D_{k,    \epsilon}}}
          &\leq& (2n-2)\log \frac{3}{2}
  \end{array}
  \end{equation*}
Thus there exists $C_2(D)\geq 0$ such that for all $K\leq 2|D|-2$ and for all
$0<\epsilon\leq \frac{D_{\min}}{4}$,
$-C_2(D)\leq c_{2}(\mathcal{M}_{\epsilon}^\star|D)\leq C_2(D)$.

The last term $c_3$ is the only one that depends on $\epsilon$ in a non
bounded way. It can be interpreted in a more direct way using
\begin{equation*}
 \sum_{D_{k, \epsilon}> 0} h_k=n-\sum_{D_{k, \epsilon}=0} h_k,
\end{equation*}
and characterising intervals such that $h_k>0$ and $D_{k, \epsilon}=0$. Let
$I_k$ be such an interval and let $x$ be a point from $D$ in $I_k$. When
$D_{k, \epsilon}=0$, $L_{k,\epsilon}=\epsilon$ or
$L_{k,\epsilon}=2\epsilon$. Combined with the assumption that $\epsilon\leq
\frac{D_{\min}}{4}$, all data points in $I_k$ take the same value $x$.

Using proposition \ref{proposition:opt-cutpoint}, we can then rule out the
case $L_{k,\epsilon}=2\epsilon$. Indeed we have $I_k= ]c_{k-1},
c_k]$. According to proposition \ref{proposition:opt-cutpoint} both $c_{k-1}$
and $c_k$ must be at most distant of $\epsilon$ from a data point. Considering
$D_{k, \epsilon}=0$ this must be the same value $x'$. If
$c_k-c_{k-1}=2\epsilon$, then the only possibility is that $c_k=x'+\epsilon$
and $c_{k-1}=x'-\epsilon$, and thus $x'=x$, the only value taken by data points
in $I_k$. However based on Remark \ref{remark:proof-prop-refpr} above, we
should then have $c_k=x$. Thus, $L_{k,\epsilon}=\epsilon$ and $I_k$ is a
singular interval as per definition \ref{def:singular-intervals}. Therefore
$\sum_{D_{k, \epsilon}=0} h_k=S(\mathcal{M}_{\epsilon}^*,D)$, which  
concludes the proof. 
\end{proof}

\subsection{proof of Corollary \ref{cor:limit-histogram}}
\begin{corollary}
Let $D$ be a data set with $n$ distinct observations. Then for $\epsilon$
sufficiently small, the optimal histogram build on $\epsilon$ bins for the
\ENUMname criterion is the trivial one with a single interval
\begin{equation*}
\mathcal{M}=(1,
\{x_{\min}-\frac{\epsilon}{2}, x_{\max}+\frac{\epsilon}{2}\},n).
\end{equation*}
\end{corollary}
\begin{proof}
Let us consider two optimal histograms $\mathcal{M}_{K,\epsilon}^\star$ and
$\mathcal{M}_{K',\epsilon}^\star$ with $K\neq K'$. We have
\begin{multline*}
    \bigl| \cMODL(\mathcal{M}_{K,\epsilon}^\star|D)
      -\cMODL(\mathcal{M}_{K',\epsilon}^\star|D)-\\\log\frac{1}{\epsilon}(K-K'+S(\mathcal{M}_{K',\epsilon}^*,D)-S(\mathcal{M}_{K,\epsilon}^*,D))\bigr|\leq
    2C(D).
\end{multline*}
As $K$ (resp. $K'$) and $S(\mathcal{M}_{K,\epsilon}^*,D)$
(resp. $S(\mathcal{M}_{K',\epsilon}^*,D)$) are integers, the minimum non zero
value of
\begin{equation*}
|K-K'+S(\mathcal{M}_{K',\epsilon}^*,D)-S(\mathcal{M}_{K,\epsilon}^*,D)|
\end{equation*}
is 1. Therefore when $\epsilon<e^{-2C(D)}$, the sign of
\begin{equation*}
\cMODL(\mathcal{M}_{K,\epsilon}^\star|D)  -\cMODL(\mathcal{M}_{K',\epsilon}^\star|D),
\end{equation*}
is the one of $K-K'+S(\mathcal{M}_{K',\epsilon}^*,D)-S(\mathcal{M}_{K,\epsilon}^*,D)$
as long as this quantity is not zero. In other words, histograms can be
compared based on this difference as long as $\epsilon$ is small
enough. Our goal is to show that $\mathcal{M}_{1,\epsilon}^\star$ is
optimal. This histogram is preferred over $\mathcal{M}_{K,\epsilon}^*$ when
\begin{align*}
  1-K+S(\mathcal{M}_{K,\epsilon}^*,D)-S(\mathcal{M}_{1,\epsilon}^*,D)&<0,\\
  S(\mathcal{M}_{K,\epsilon}^*,D)&<K-1,
\end{align*}
as $\mathcal{M}_{1,\epsilon}^\star$  does not contain singular intervals. 

Let us now focus on the assumption that the $n$ values in the data set are
distinct. This implies that a singular interval $I_k$ is
associated to $h_k=1$ and thus $S(\mathcal{M}_{K,\epsilon}^*,D)$ is exactly
the number of singular intervals. This number is controlled by the fact that
when $\epsilon$ is small enough, a histogram cannot have adjacent singular
intervals. Indeed singular intervals have a maximal width of $\epsilon$. Let
$I_k$ and $I_{k+1}$ be two such adjacent intervals, with the associated values
$x_i\in I_k$ and $x_j\in I_j$. Then
\begin{equation*}
D_{\min}\leq |x_j-x_i|\leq 2\epsilon,
\end{equation*}
with $D_{\min}$ has defined in equation \eqref{eq:dmin:def}. The lower bound
is induced by the fact that $I_k$ and $I_{k+1}$ contain only a single value
each (by definition) and are adjacent: there is no value from the data set in
$]x_i,x_j[$. The inequalities can be fulfilled simultaneously only if
$\epsilon\geq \frac{D_{\min}}{2}$ and thus when $\epsilon<\frac{D_{\min}}{2}$,
singular intervals cannot be adjacent. We assume in the rest of the proof that
$\epsilon<\min\left(\frac{D_{\min}}{2}, e^{-2C(D)}\right)$. Then 
\begin{equation*}
K\geq 2  S(\mathcal{M}_{K,\epsilon}^*,D)-1,
\end{equation*}
as we need ``in between'' non singular intervals to separate singular
ones. Therefore we have
\begin{equation*}
K-  S(\mathcal{M}_{K,\epsilon}^*,D)\geq S(\mathcal{M}_{K,\epsilon}^*,D)-1,
\end{equation*}
and if $S(\mathcal{M}_{K,\epsilon}^*,D) \geq 3$, 
\begin{align*}
S(\mathcal{M}_{K,\epsilon}^*,D)\leq K-2.
\end{align*}
This shows that $\mathcal{M}_{1,\epsilon}^*$ is always preferred to  by the
\ENUMname criterion to histograms with $S(\mathcal{M}_{K,\epsilon}^*,D) \geq 3$. 

We need now to discuss the remaining cases, i.e., situations when
$S(\mathcal{M}_{K,\epsilon}^*,D)\leq 2$. If $K\geq
S(\mathcal{M}_{K,\epsilon}^*,D)+2$, $\mathcal{M}_{1,\epsilon}^*$ is
preferred and thus we have two interesting particular cases to handle:
\begin{enumerate}
\item $K=2$ and $S(\mathcal{M}_{2,\epsilon}^*,D)=1$: this is a
  particular case with a single singular interval completed by a single large
  non singular one;
\item  $K=3$ and $S(\mathcal{M}_{3,\epsilon}^*,D)=2$: this is again a particular
  case where we have two singular intervals separated by a single large
  non singular interval.
\end{enumerate}
In both cases, the dominating term of the \ENUMname criterion is
$n\log\frac{1}{\epsilon}$, exactly as for $\mathcal{M}_{1,\epsilon}^*$ and we
need to compare with more precision the values of the criterion to choose the
optimal histogram.

For a single interval, we have
\begin{equation}\label{eq:one:interval:eps}
\cMODL(\mathcal{M}_{1,\epsilon}^\star|D)=  \log ^*1+n\log\left(\frac{L}{\epsilon}+1\right).
\end{equation}
We use the fact that if $\lambda>0$, $\rho>0$ and $\kappa$ do not depend on
$\epsilon$, for $\epsilon$ large enough such that
$\frac{\rho}{\epsilon}+\kappa>0$, 
\begin{equation}\label{eq:log}
\lambda\log\left(\frac{\rho}{\epsilon}+\kappa\right)=\lambda\log\rho+\lambda\log
\frac{1}{\epsilon}+o(\epsilon). 
\end{equation}
Thus we have
\begin{equation*}
  \cMODL(\mathcal{M}_{1,\epsilon}^\star|D)=  n\log\frac{1}{\epsilon}+\log ^*1+n\log L+o(\epsilon).                                   
\end{equation*}
For $K=2$ and $S(\mathcal{M}_{2,\epsilon}^*,D)=1$, we have
\begin{align*}
  \cMODL(\mathcal{M}_{2,\epsilon}^\star)=& \log ^*2+\log (n+1)+\log n+\log \left(\frac{L}{\epsilon}+2\right)\\
&  +(n-1)\log \frac{L}{\epsilon}.
\end{align*}
Using equation \eqref{eq:log} to handle the two logarithmic terms
with $\epsilon$, we have 
\begin{align*}
  \cMODL(\mathcal{M}_{2,\epsilon}^\star|D)=& n\log\frac{1}{\epsilon}\\
                                         &+\log ^*2+n\log L\\
                                         &+\log (n(n+1))+o(\epsilon).
\end{align*}
Finally we have for $K=3$ and $S(\mathcal{M}_{2,\epsilon}^*,D)=2$
\begin{align*}
  \cMODL(\mathcal{M}_{3,\epsilon}^\star|D)=& \log ^*3+\log\frac{(n+1)(n+2)}{2}+\\
                                     &+\log
                                       \left(\frac{L}{\epsilon}+2\right)+\log
                                       \left(\frac{L}{\epsilon}+3\right)
  -\log 2\\
&  +(n-2)\log \left(\frac{L}{\epsilon}-1\right)\\
  &+\log n(n-1),
\end{align*}
Using again equation \eqref{eq:log}, we have
\begin{align*}
  \cMODL(\mathcal{M}_{3,\epsilon}^\star|D)=&n\log\frac{1}{\epsilon}\\
                                     &+\log ^*3+n\log L\\
                                     &+\log \frac{(n-1)n(n+1)(n+2)}{4} 
  +o(\epsilon).
\end{align*}
When $\epsilon$ converges to 0,
$\cMODL(\mathcal{M}_{k,\epsilon}^\star|D)-n\log\frac{1}{\epsilon}-n\log L$
reduces to $\log^\star k$ plus some positive terms when $k>1$. As $\log^\star$
is an increasing function, when $\epsilon$ is mall enough,
$\cMODL(\mathcal{M}_{1,\epsilon}^\star|D)$ becomes smaller than
$\cMODL(\mathcal{M}_{k,\epsilon}^\star|D)$ for $k>1$ and thus
$\mathcal{M}_{1,\epsilon}^\star$ is the optimal histogram. 
\end{proof}

\section{Illustration of the asymptotic  behaviour} \label{apx:convergence-rate}
We have:
\begin{eqnarray*}
\cMODL(\mathcal{M}_{2, \epsilon}|D_n(\alpha,\theta)) &=&  \log ^* 2 + \log (n+1) + \log (E+1)\\ 
 &&+ \log  \frac{n!}{(n \theta)!(n (1-\theta))!} + n \theta \log \alpha E \\
 &&+ n (1-\theta) \log (1-\alpha) E,\\
&=&  \log ^* 2 + \log (n+1) + \log (E+1) \\
  && + n H(\theta) -\frac{1}{2} \log (2 \pi n \mathrm{var}(\theta)) + O(1/n) \\
  && + n \log E +  n \theta \log \alpha\\ 
  && + n (1-\theta) \log (1-\alpha),
\end{eqnarray*}
where we used the same approximation from \cite{Grunwald07} as in Section
\ref{apx:th-empty-bins-proof}. Using equation \eqref{eq:one:interval:eps}, we obtain
\begin{align*}
  \Delta(n,\epsilon, \alpha, \theta)=& \cMODL(\mathcal{M}^\star_{2, \epsilon}|D_n(\theta, \alpha))- \cMODL(\mathcal{M}^\star_{1,
                                       \epsilon}|D_n(\theta, \alpha)) \\
  =&\log ^* 2 - \log ^* 1 + \log (n+1) + \log (E+1) \\
  & + n H(\theta) -\frac{1}{2} \log (2 \pi n \mathrm{var}(\theta)) + O(1/n) \\
  & + n \theta \log \alpha + n (1-\theta) \log (1-\alpha).
\end{align*}
This can be simplified into
\begin{align*}
\Delta(n,\epsilon, \alpha, \theta)=& \log \left(E+1\right)\\
  		& + n \left( H(\theta) + \theta \log \alpha + (1-\theta) \log (1-\alpha)\right)\\
  		& + O(\log n).
\end{align*}
Finally we conclude using
\begin{align*}
 D_{KL}(\theta\| \alpha)&=\theta\log\frac{\theta}{\alpha} +(1-\theta)
                          \log\frac{1-\theta}{1-\alpha},\\
  &=-H(\theta)-\theta \log \alpha - (1-\theta) \log (1-\alpha).
\end{align*}

\section{Benchmarking and comparison of MDL methods}\label{apx:exp-MDL}
The first series of figures of the appendix compares the \NMLname  , \ENUMname
and \GENUMname  criteria. The three methods are evaluated on different sample
sizes of a Normal (figure \ref{fig:comparison-MDL-normal}, Cauchy (figure
\ref{fig:comparison-MDL-cauchy}), uniform (figure
\ref{fig:comparison-MDL-uniform}), triangle (figure
\ref{fig:comparison-MDL-triangle}), triangle mixture (figure
\ref{fig:comparison-MDL-tmix}) and Gaussia mixture (figure
\ref{fig:comparison-MDL-claw}) distributions. 

We highlight that all results are means over 10 experiments for each sample
size. All results are shown in log scale. The standard deviations of the
metrics are represented asymmetrically in the graphs to because of the log scale.

\begin{figure}
\centering
\caption{Comparison between MDL methods over a Normal distribution of different sample sizes \label{fig:comparison-MDL-normal}}
\setkeys{Gin}{width=0.3\textwidth}
\subfloat[Number of intervals,
          \label{fig:intervals-MDL-normal}]{\includegraphics{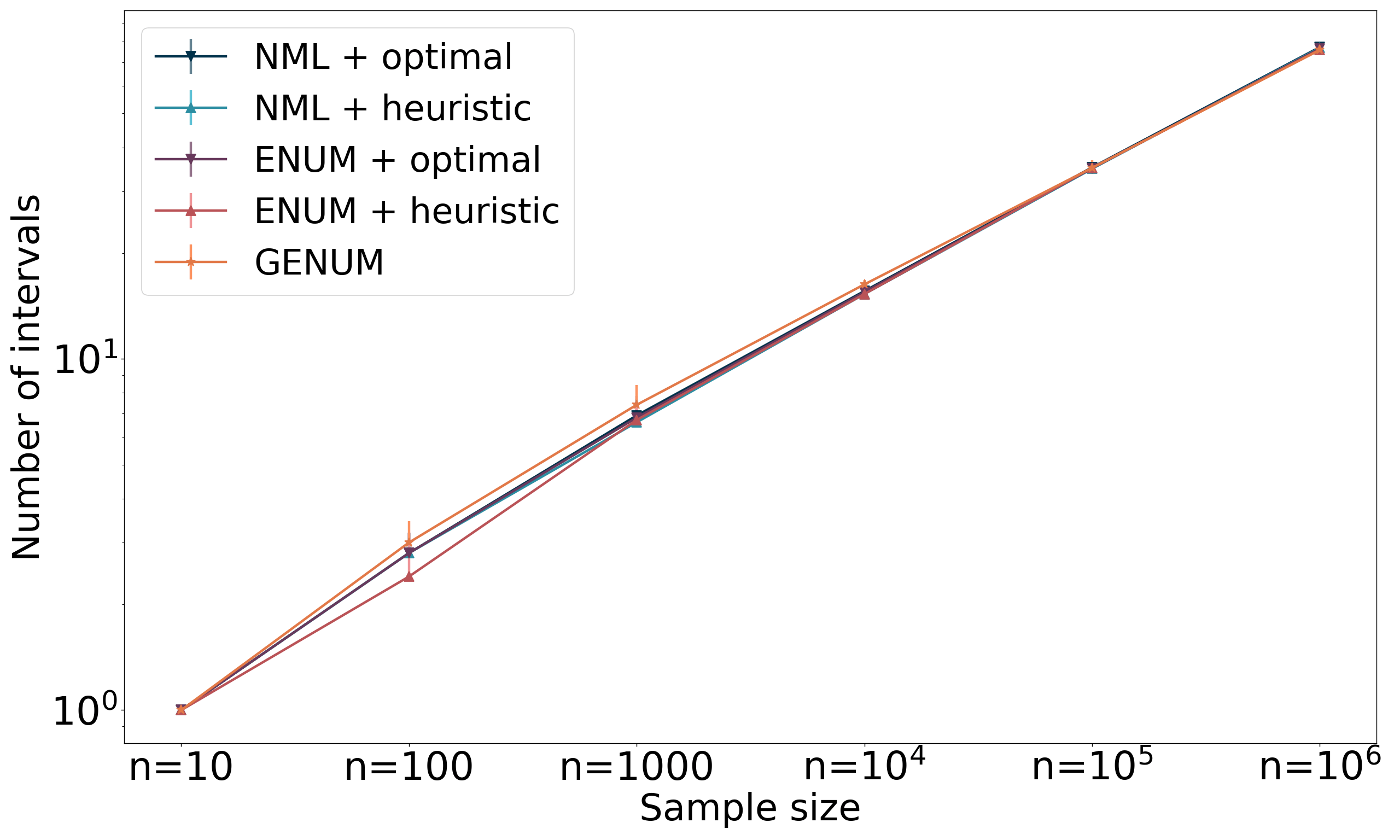}}
    \hfill
\subfloat[Computation time,
          \label{fig:time-MDL-normal}]{\includegraphics{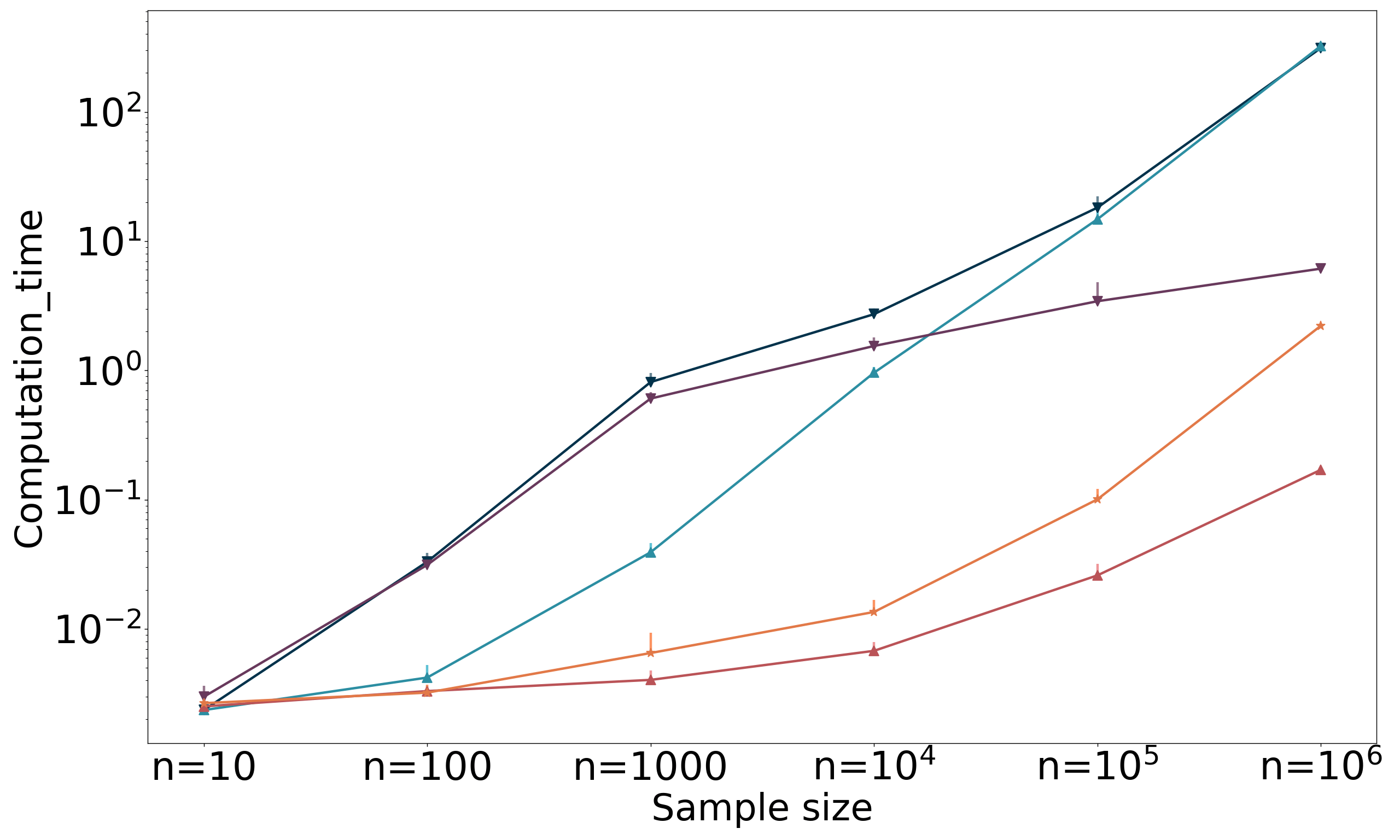}}
    \hfill
\subfloat[Hellinger distance,
          \label{fig:hd-MDL-normal}]{\includegraphics{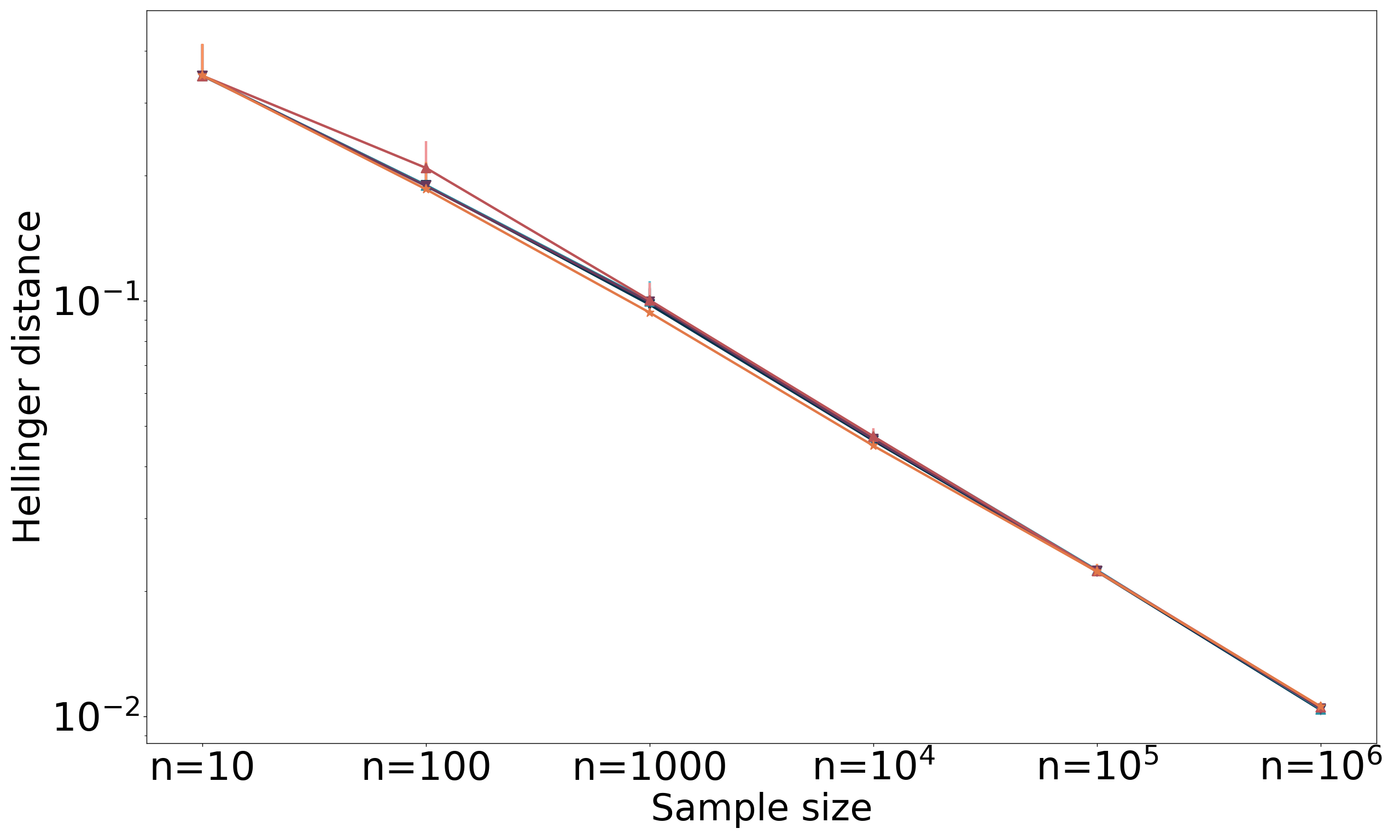}}

\end{figure}

\begin{figure}
\centering
\caption{Comparison between MDL methods over a Cauchy distribution of different sample size \label{fig:comparison-MDL-cauchy}}
\setkeys{Gin}{width=0.3\textwidth}
\subfloat[Number of intervals,
          \label{fig:intervals-MDL-cauchy}]{\includegraphics{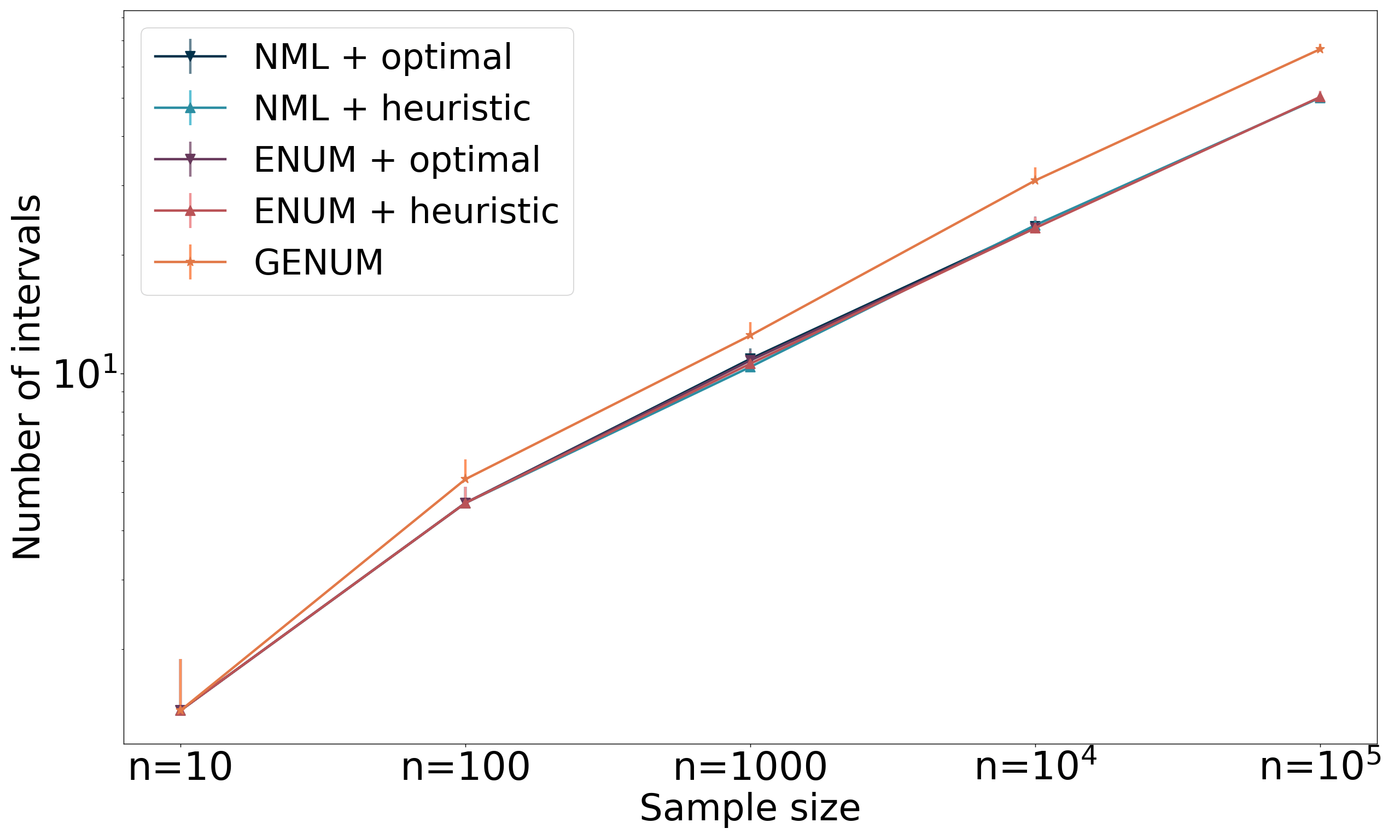}}
    \hfill
\subfloat[Computation time,
          \label{fig:time-MDL-cauchy}]{\includegraphics{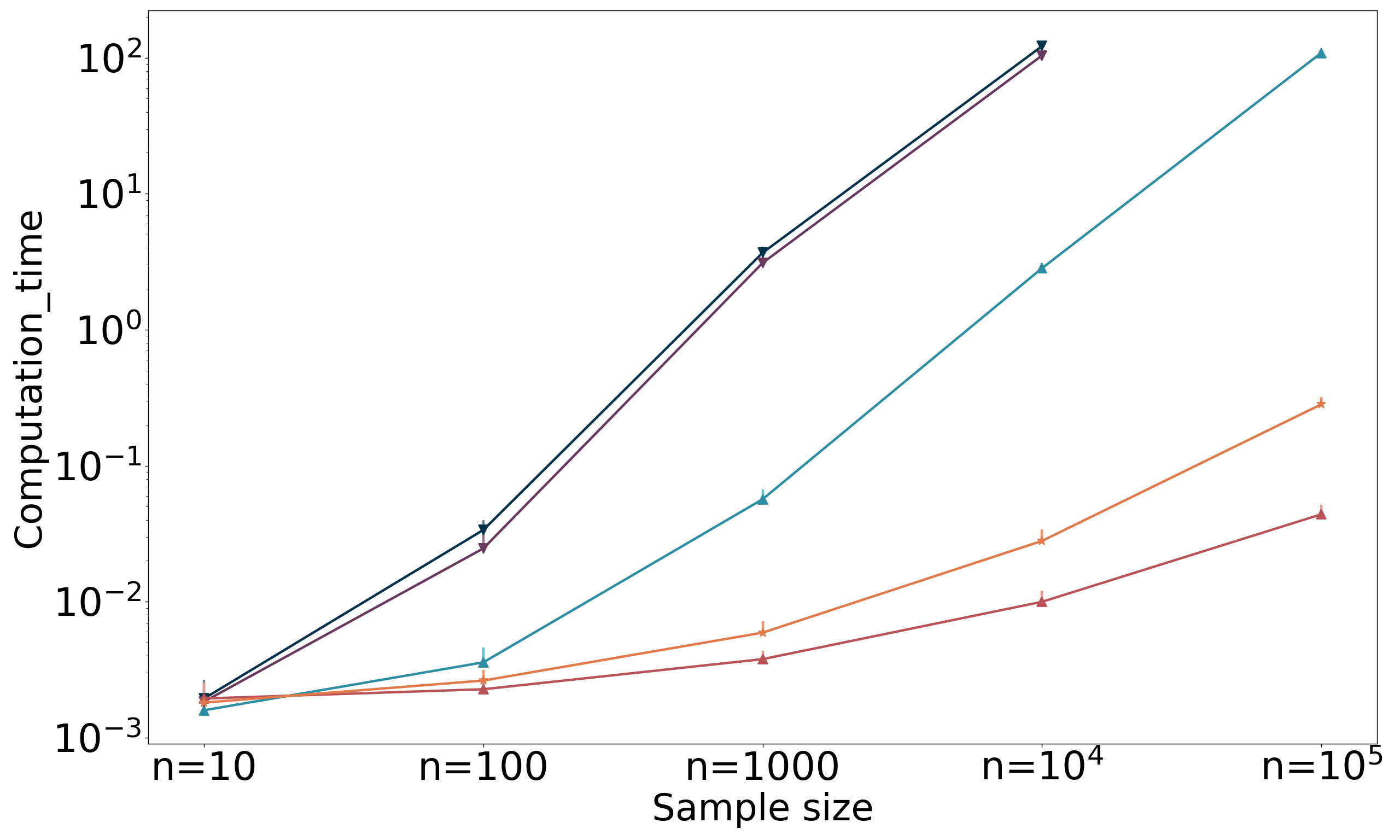}}
    \hfill
\subfloat[Hellinger distance,
          \label{fig:hd-MDL-cauchy}]{\includegraphics{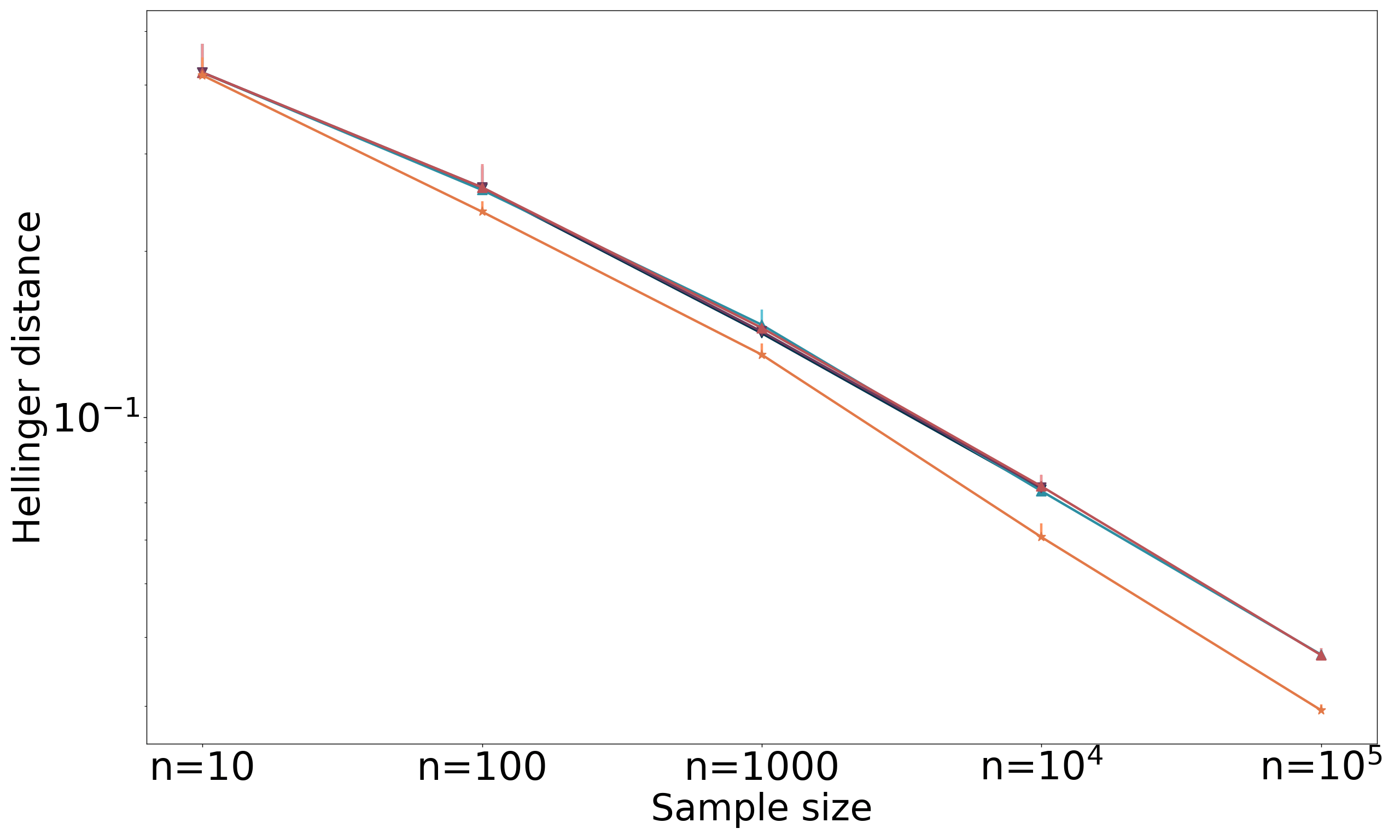}}
\end{figure}

\begin{figure}
\centering
\caption{Comparison between MDL methods over a uniform distribution of different sample size \label{fig:comparison-MDL-uniform}}
\setkeys{Gin}{width=0.3\textwidth}
\subfloat[Number of intervals,
          \label{fig:intervals-MDL-uniform}]{\includegraphics{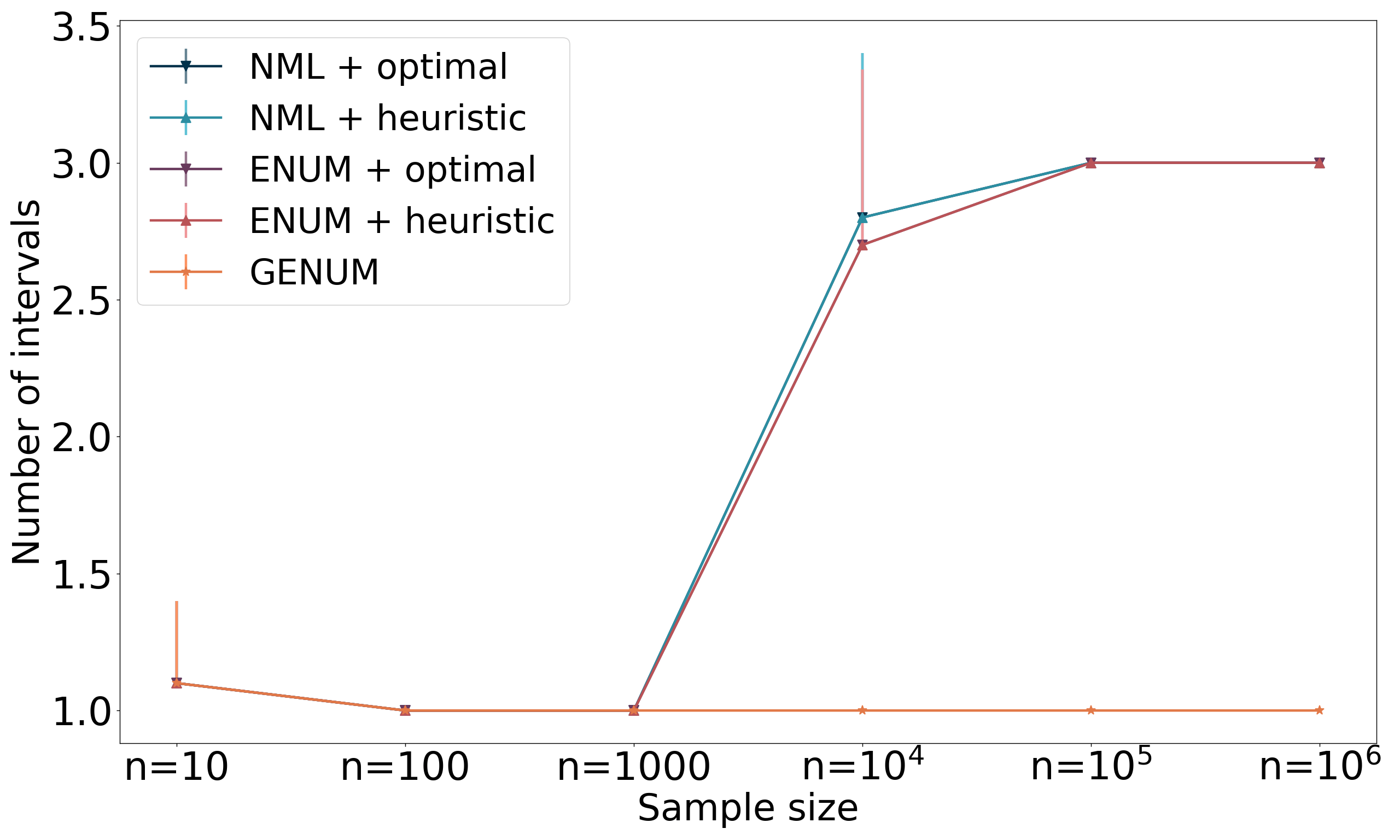}}
    \hfill
\subfloat[Computation time,
          \label{fig:time-MDL-uniform}]{\includegraphics{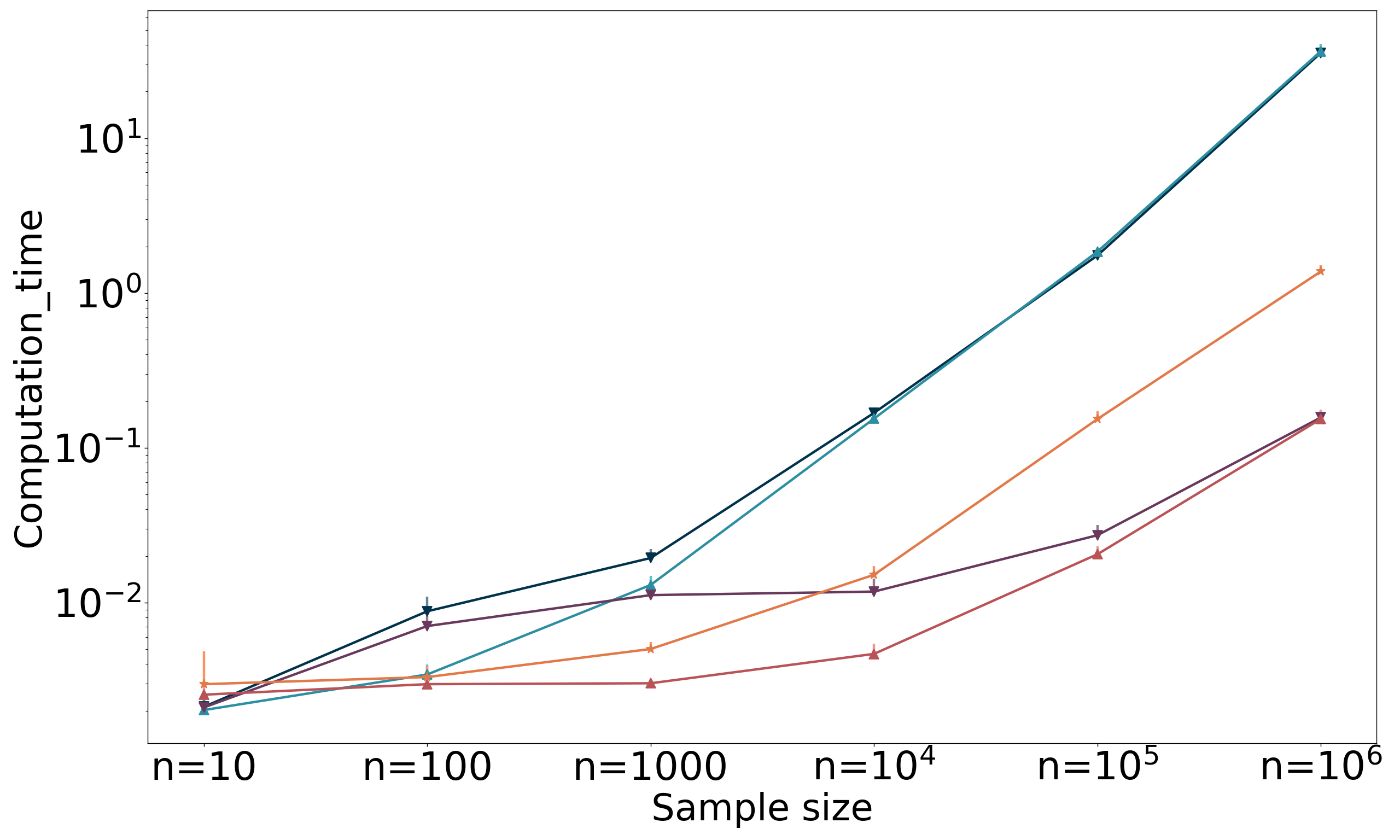}}
    \hfill
\subfloat[Hellinger distance,
          \label{fig:hd-MDL-uniform}]{\includegraphics{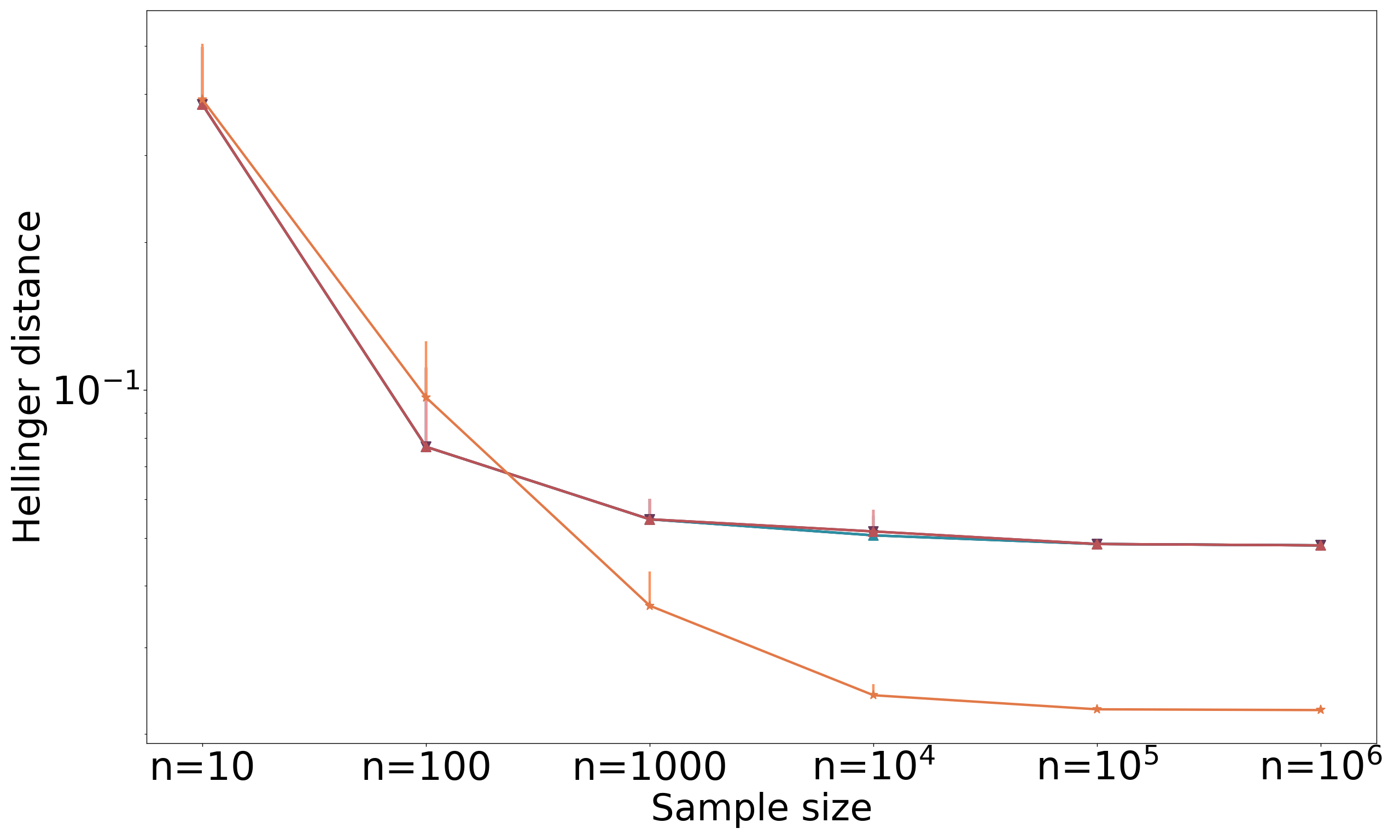}}
\end{figure}

\begin{figure}
\centering
\caption{Comparison between MDL methods over a Triangle distribution of different sample sizes \label{fig:comparison-MDL-triangle}}
\setkeys{Gin}{width=0.3\textwidth}
\subfloat[Number of intervals,
          \label{fig:intervals-MDL-triangle}]{\includegraphics{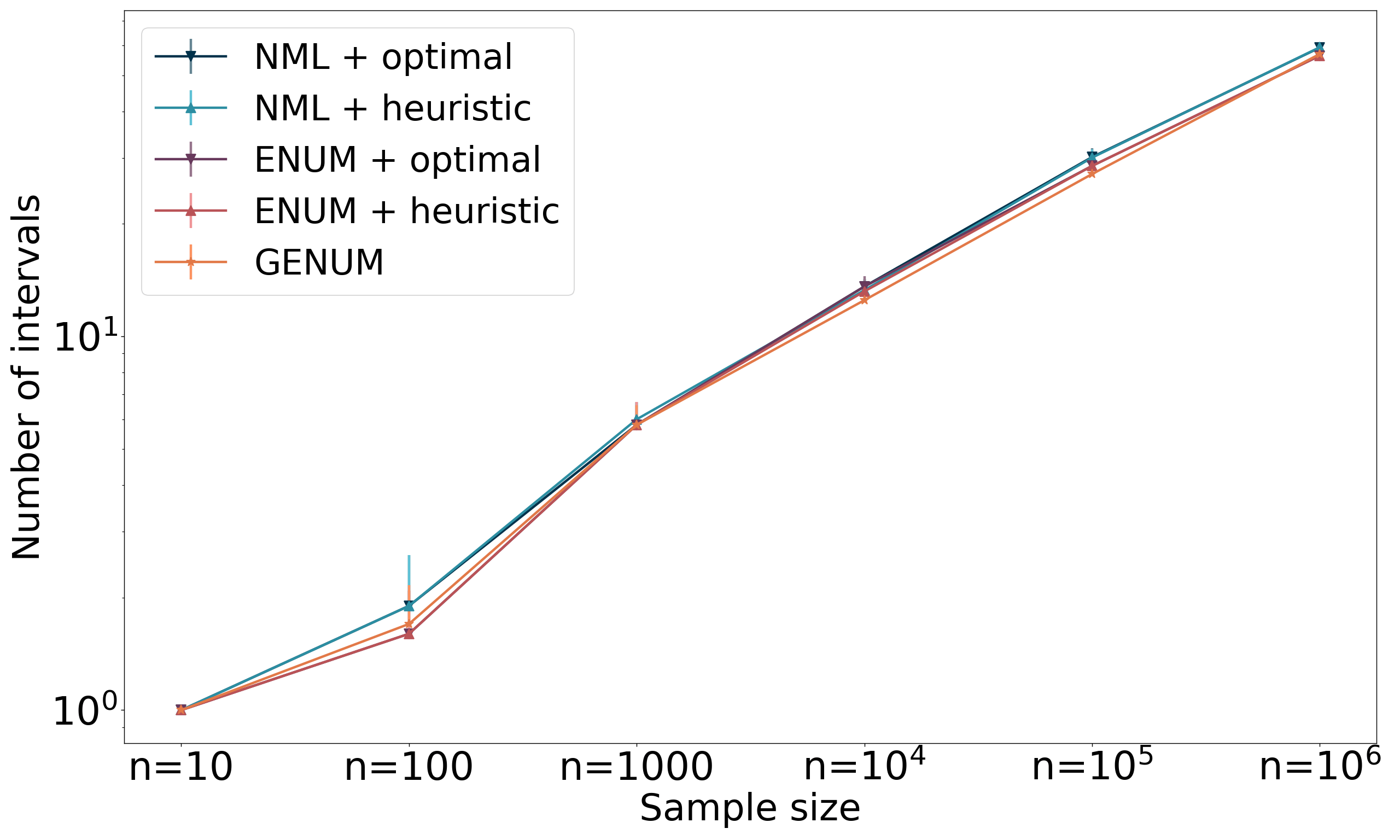}}
    \hfill
\subfloat[Computation time,
          \label{fig:time-MDL-triangle}]{\includegraphics{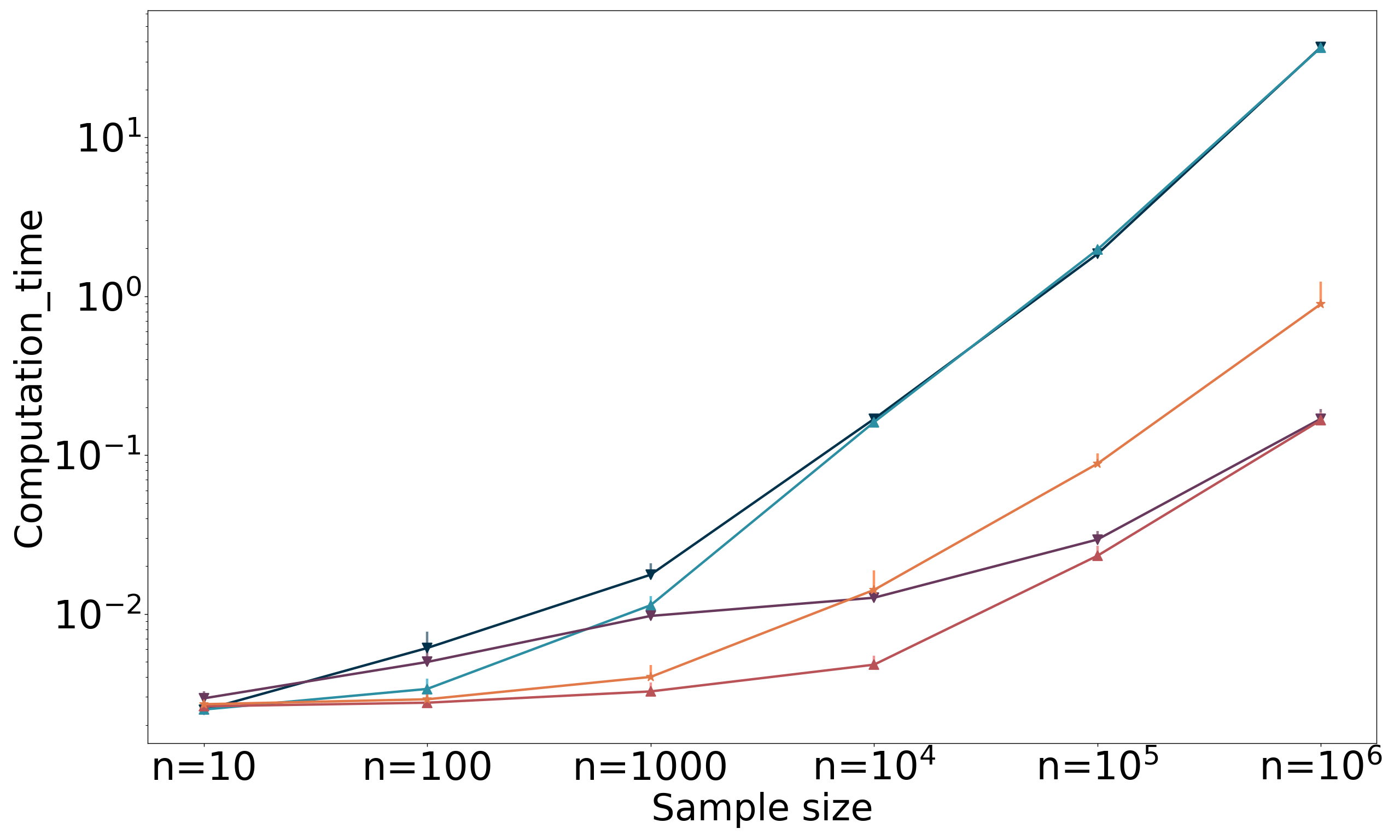}}
    \hfill
\subfloat[Hellinger distance,
          \label{fig:hd-MDL-triangle}]{\includegraphics{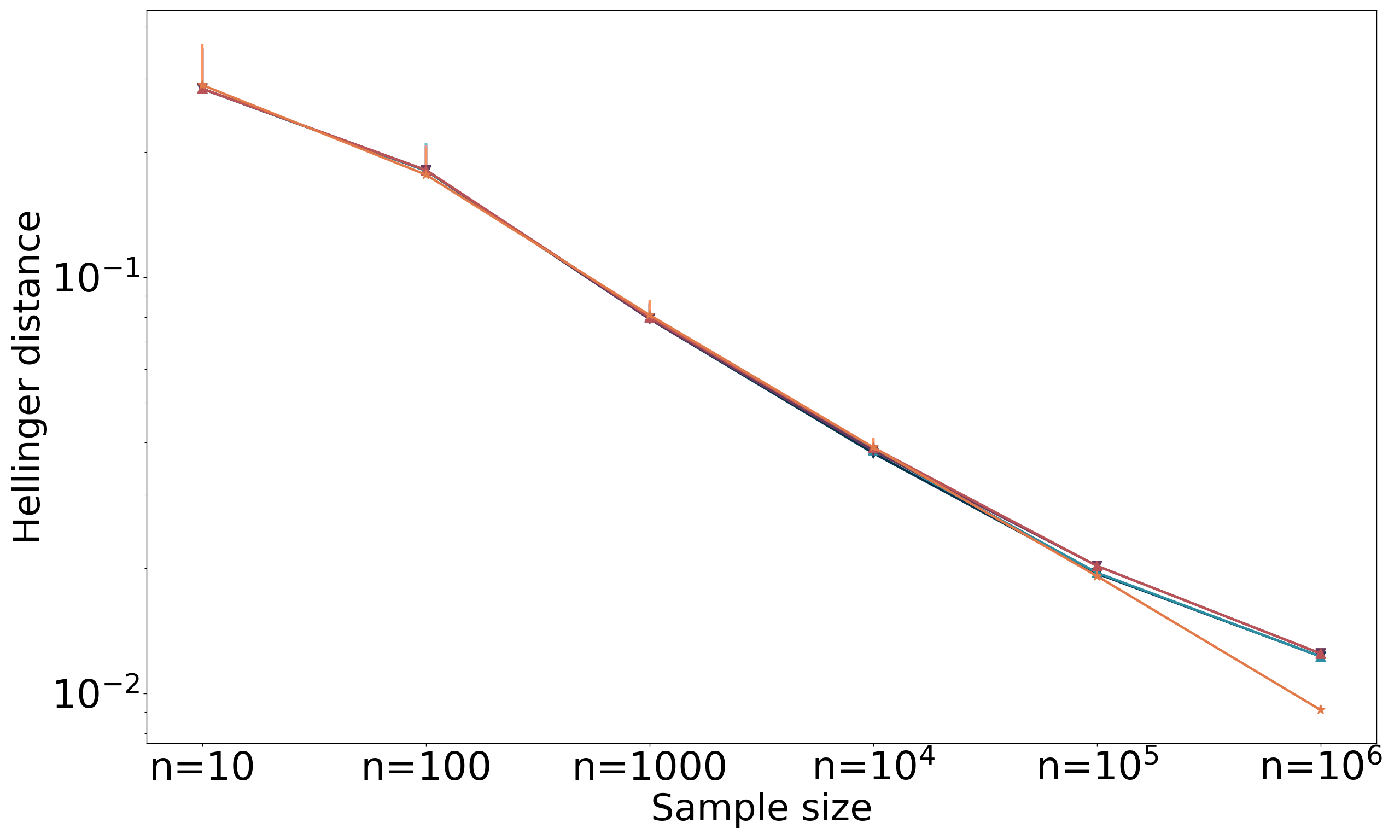}}
\end{figure}

\begin{figure}
\centering
\caption{Comparison between MDL methods over a mixture of 4 Triangle distributions, of different sample sizes \label{fig:comparison-MDL-tmix}}
\setkeys{Gin}{width=0.3\textwidth}
\subfloat[Number of intervals,
          \label{fig:intervals-MDL-tmix}]{\includegraphics{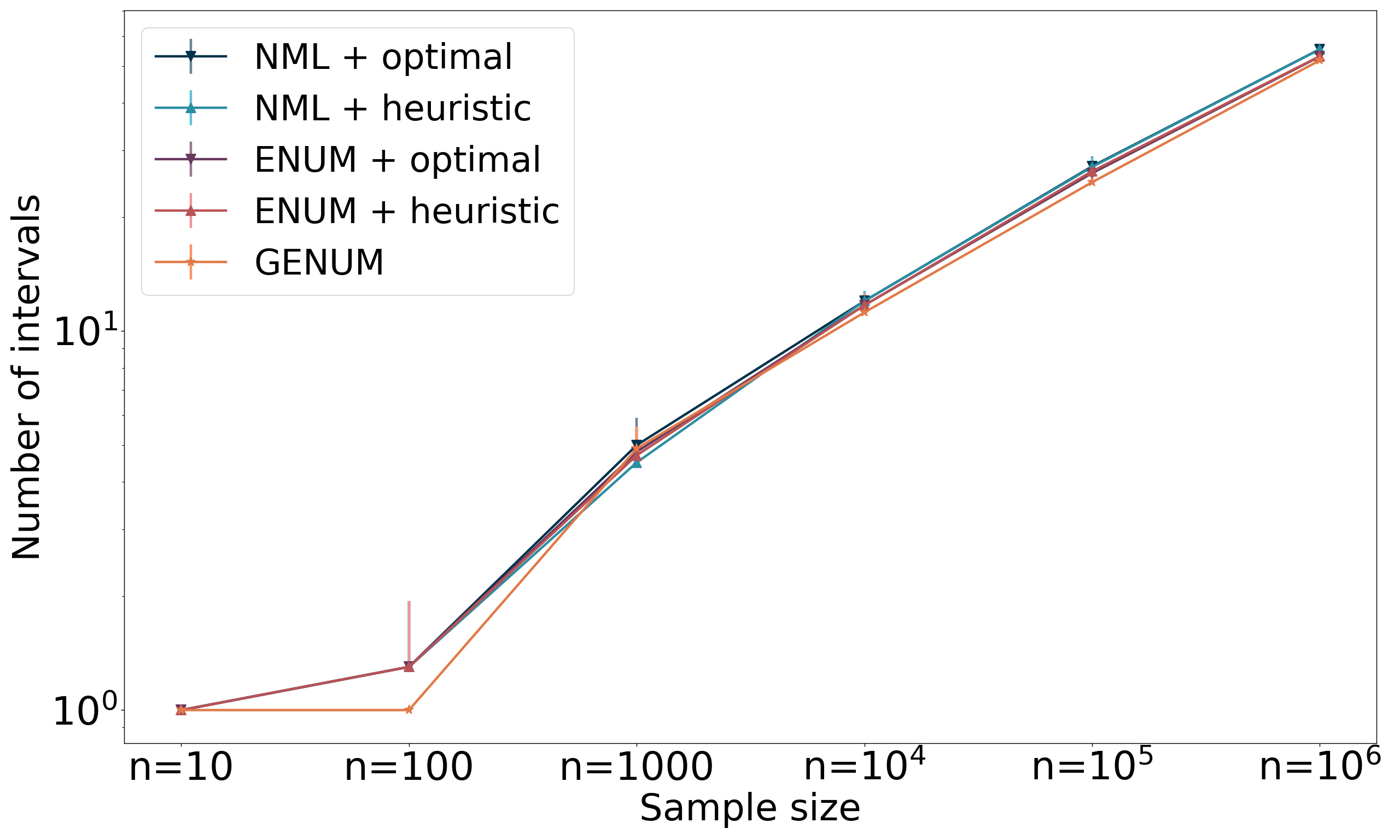}}
    \hfill
\subfloat[Computation time,
          \label{fig:time-MDL-tmix}]{\includegraphics{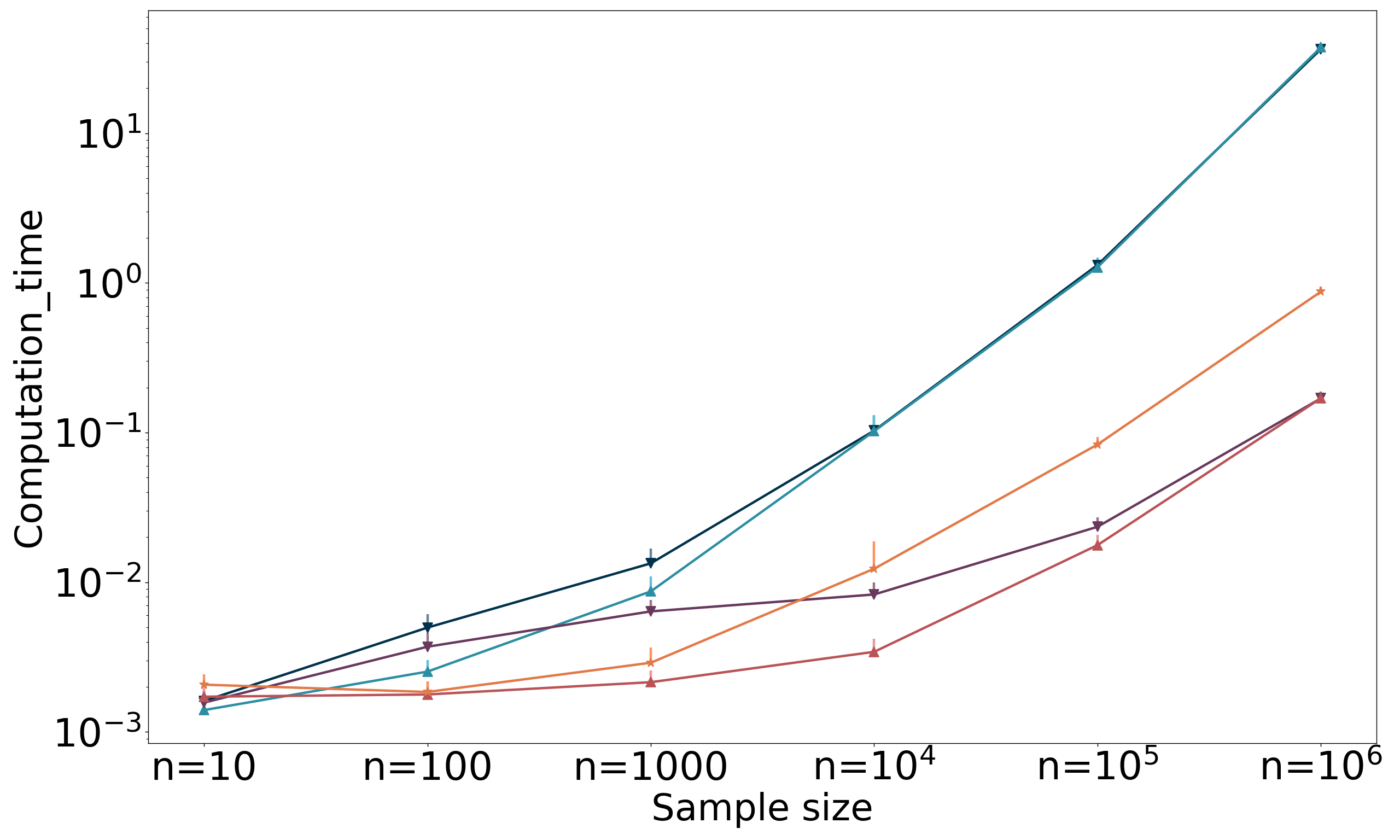}}
    \hfill
\subfloat[Hellinger distance,
          \label{fig:hd-MDL-tmix}]{\includegraphics{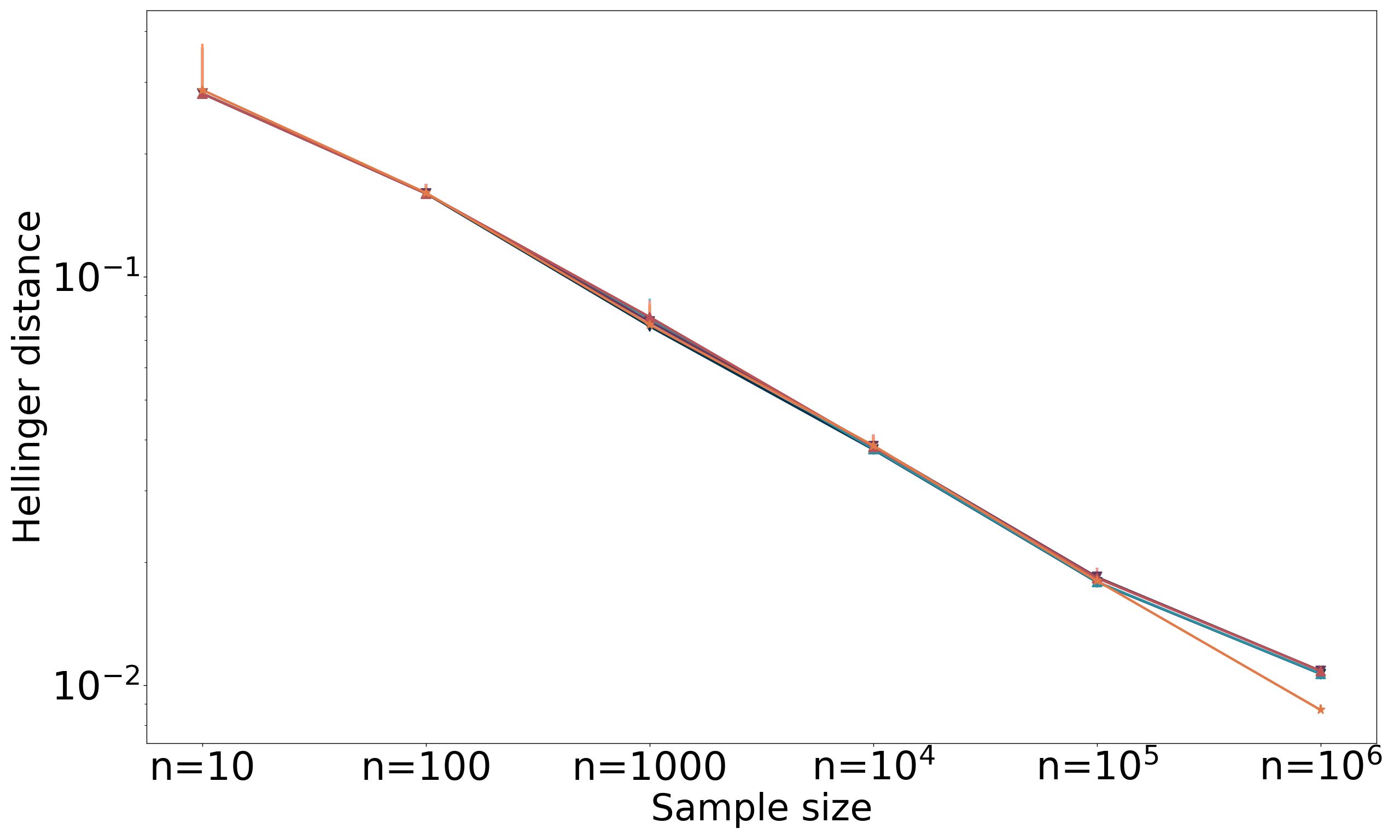}}
\end{figure}

\begin{figure}
\centering
\caption{Comparison between MDL methods over the Claw Gaussian mixture of different sample sizes \label{fig:comparison-MDL-claw}}
\setkeys{Gin}{width=0.3\textwidth}
\subfloat[Number of intervals,
          \label{fig:intervals-MDL-claw}]{\includegraphics{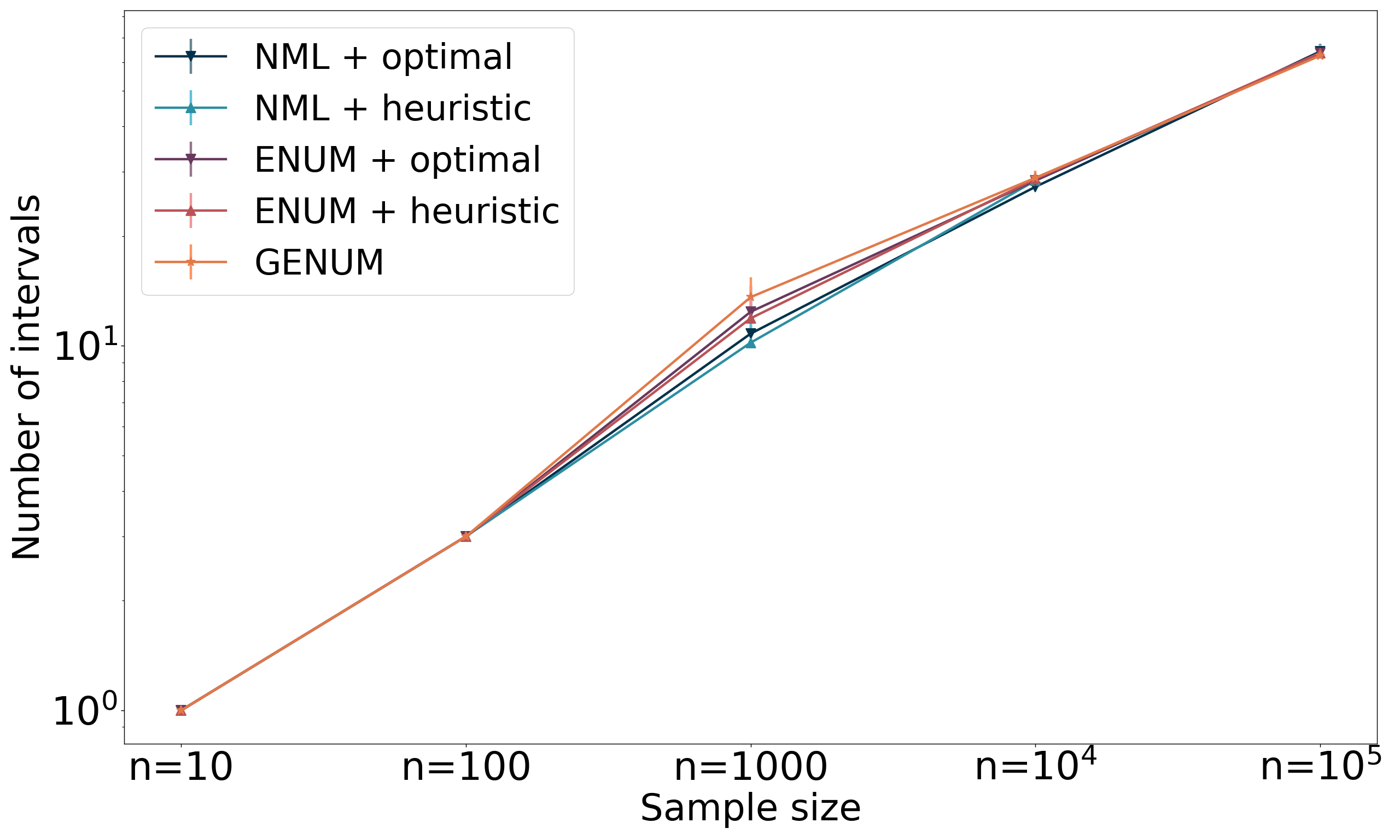}}
    \hfill
\subfloat[Computation time,
          \label{fig:time-MDL-claw}]{\includegraphics{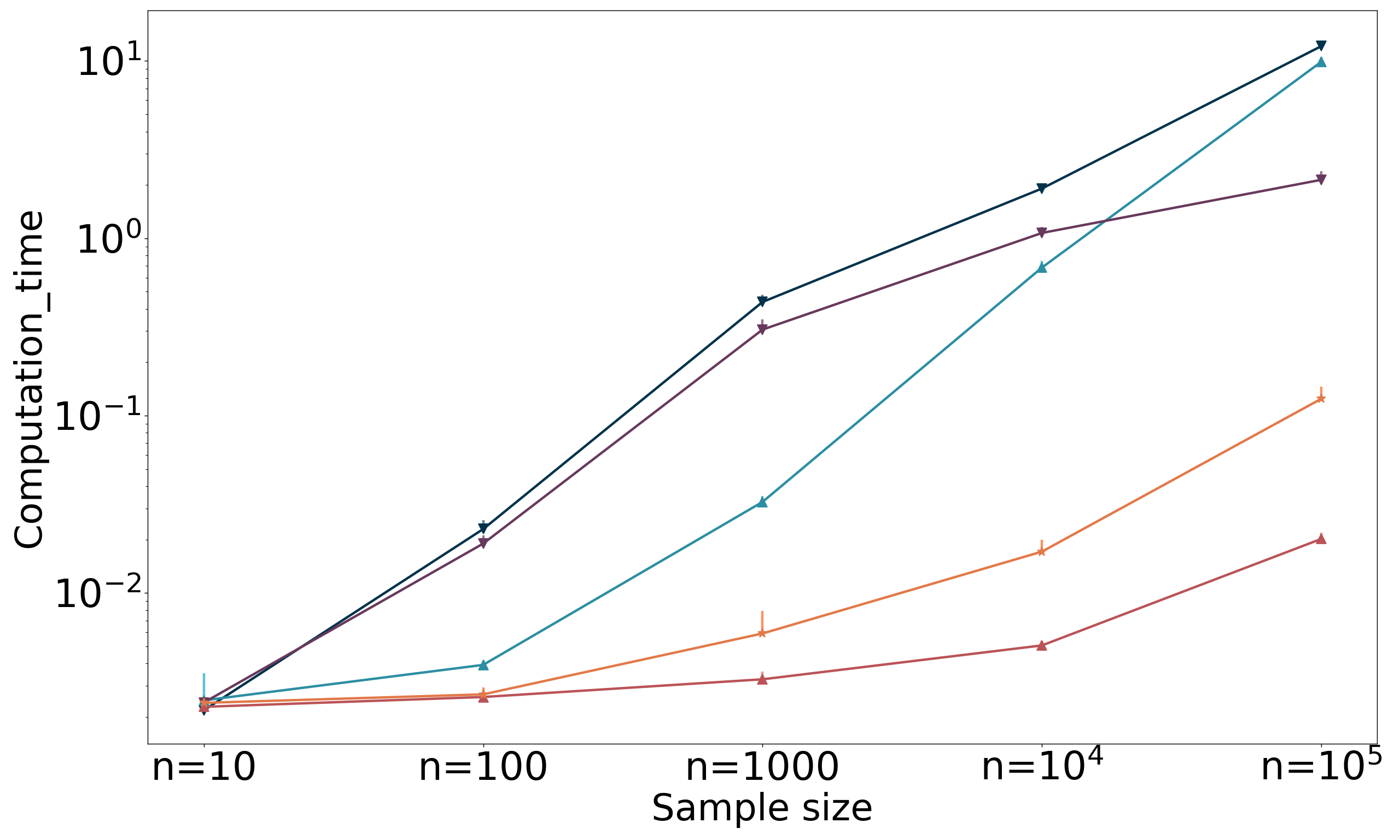}}
    \hfill
\subfloat[Hellinger distance,
          \label{fig:hd-MDL-claw}]{\includegraphics{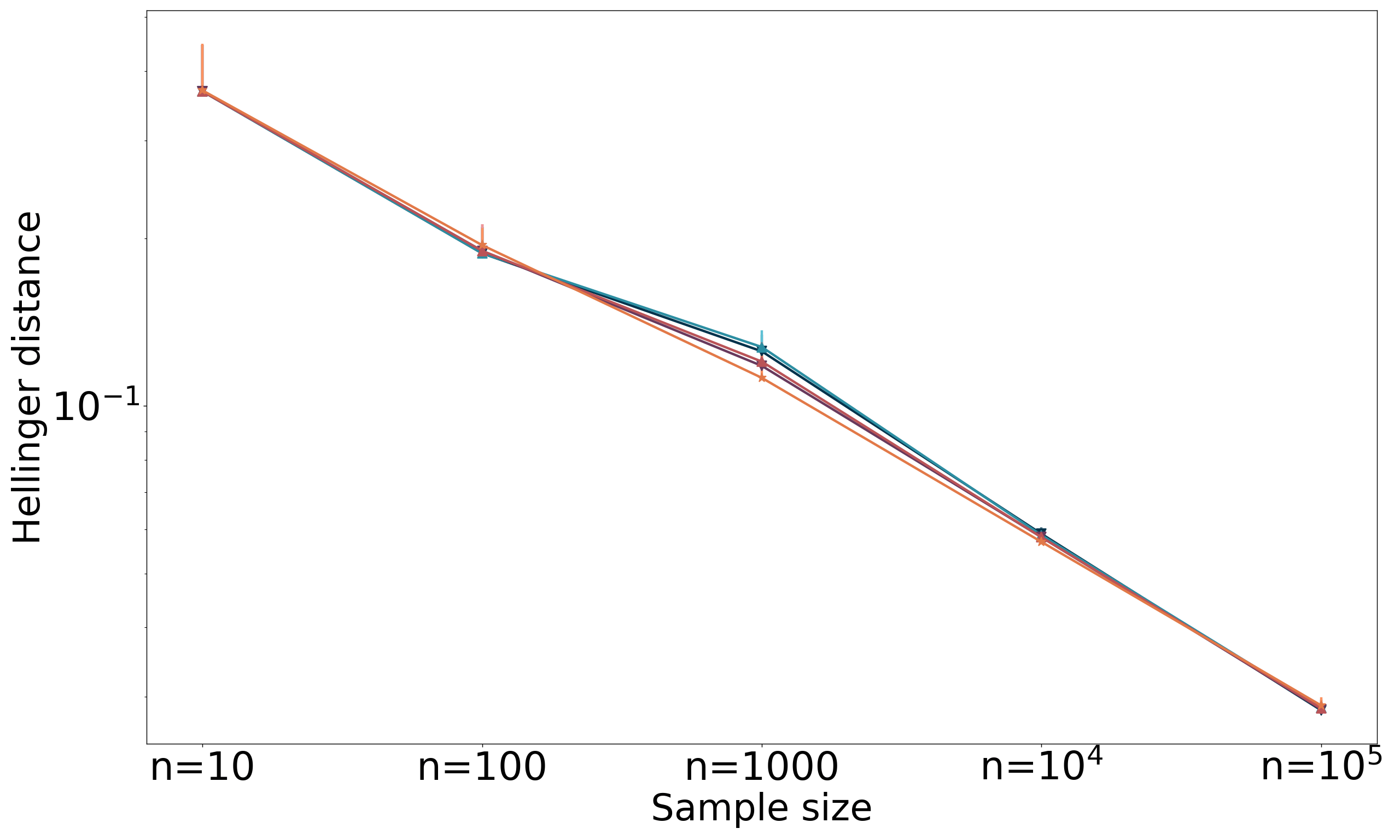}}
\end{figure}

\section{Role of the approximation accuracy,
  $\epsilon$}\label{sec:role-appr-accur-1}
As a complement to the theoretical analysis conducted in Section
\ref{sec:role-appr-accur}, we conducted a series of experiments to show the
effect in practice of using a too small value for $\epsilon$. We compare all
MDL methods over a Normal distribution of small size ($n=100$ and $n=1000$),
for $\epsilon$ ranging from $2.0$ to $10^{-8}$. 

Figures \ref{fig:sensitivity-normal-n100} and
\ref{fig:sensitivity-normal-n1000} summarise the results. For both sample
sizes we can observe that, as $\epsilon \rightarrow 0$, the number of
intervals in \NMLname   and \ENUMname  histograms decreases. Particularly, for
the $n=100$ distribution, the estimations models start at 3 intervals and
\NMLname   and \ENUMname  histograms slowly decrease to having a single
one. This observation is in line with Corollary \ref{cor:limit-histogram}.

In stark contrast, \GENUMname  remains at 3 intervals throughout. The same
consistency  can be observed in regards to computation times and Hellinger
distance. While the \NMLname   and \ENUMname  methods take more time and produce
lower quality results as the approximation accuracy increases, the \GENUMname
steadily produces quality results in less time.

Note also that the best value for $\epsilon$ is different for both sample
sizes. For a distribution of $n=100$ samples, having $\epsilon=0.1$ guarantees
the lowest Hellinger distance for the \NMLname   and \ENUMname
methods. For $n=1000$ samples, it is preferable to set $\epsilon=0.5$. When we
do not know anything about the nature of the data, there is no guarantee we
will make the right choice for \NMLname   and \ENUMname. On the other
hand, our granularised approach automatically selects $\epsilon^*\approx1.324$ and
$\epsilon^*\approx 0.178$ for samples sizes $n=100$ and $n=1000$
respectively. While these choices certainly do not achieve the lowest HD
technically possible, they have the merit of not being too far off. \GENUMname
histograms seem to select fair enough accuracy values for a fully automated
estimation.

These sensitivity experiments have shown that the value of $\epsilon$ cannot
be  overlooked as easily as thought: for some distributions and middle-range
sample sizes, the approximation accuracy can play an important role in
computation time and estimation quality. In an exploratory analysis context,
where little is known about the data, a truly fully automated approach such as
\GENUMname  is preferable than the other MDL methods. 

\begin{figure}
\centering
\caption{Comparison of MDL methods over a Normal distribution of size $n=100$ for different values of $\epsilon$ \label{fig:sensitivity-normal-n100}}
\setkeys{Gin}{width=0.3\textwidth}
\subfloat[Number of intervals,
          \label{fig:intervals-sensitivity-normal-n100}]{\includegraphics{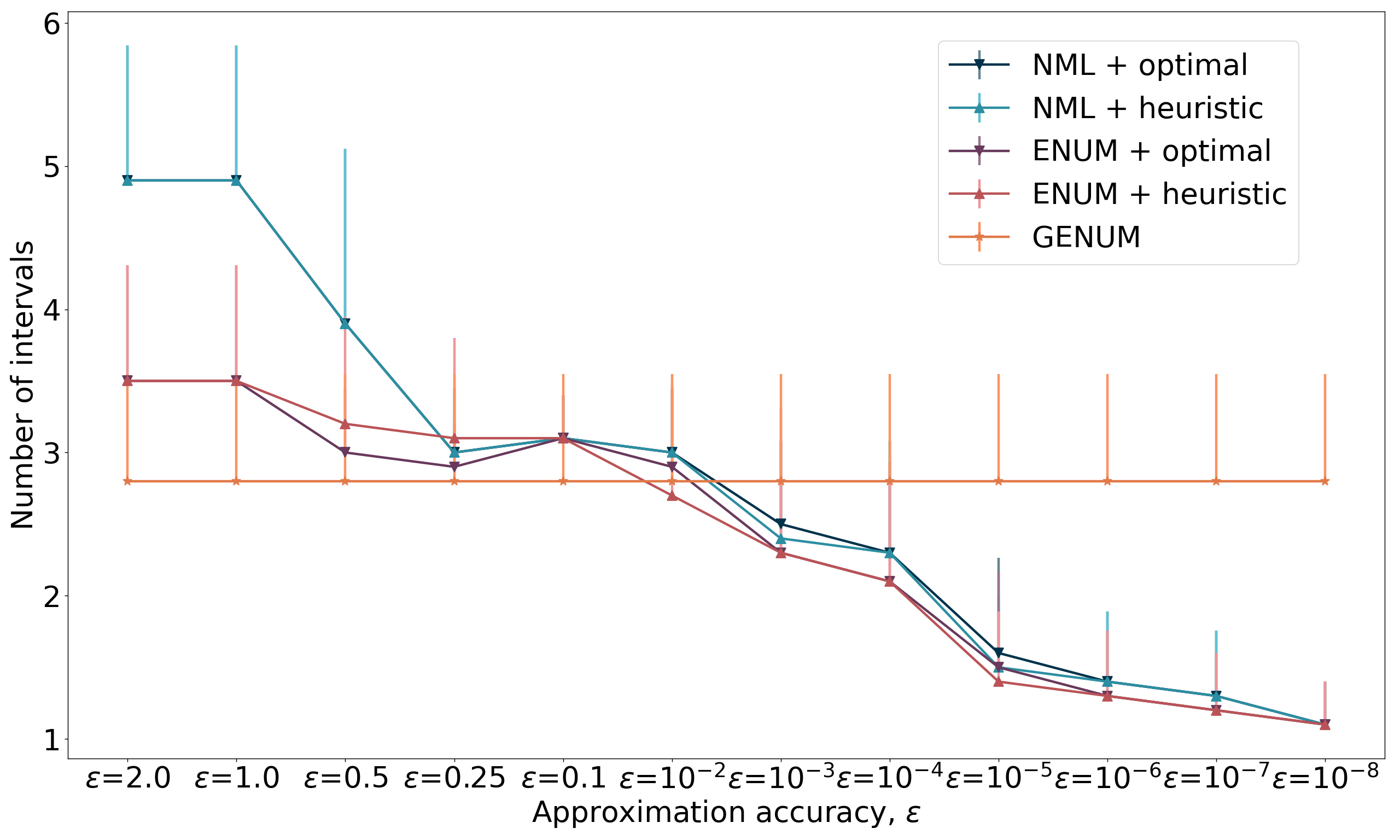}}
    \hfill
\subfloat[Computation time,
          \label{fig:time-sensitivity-normal-n100}]{\includegraphics{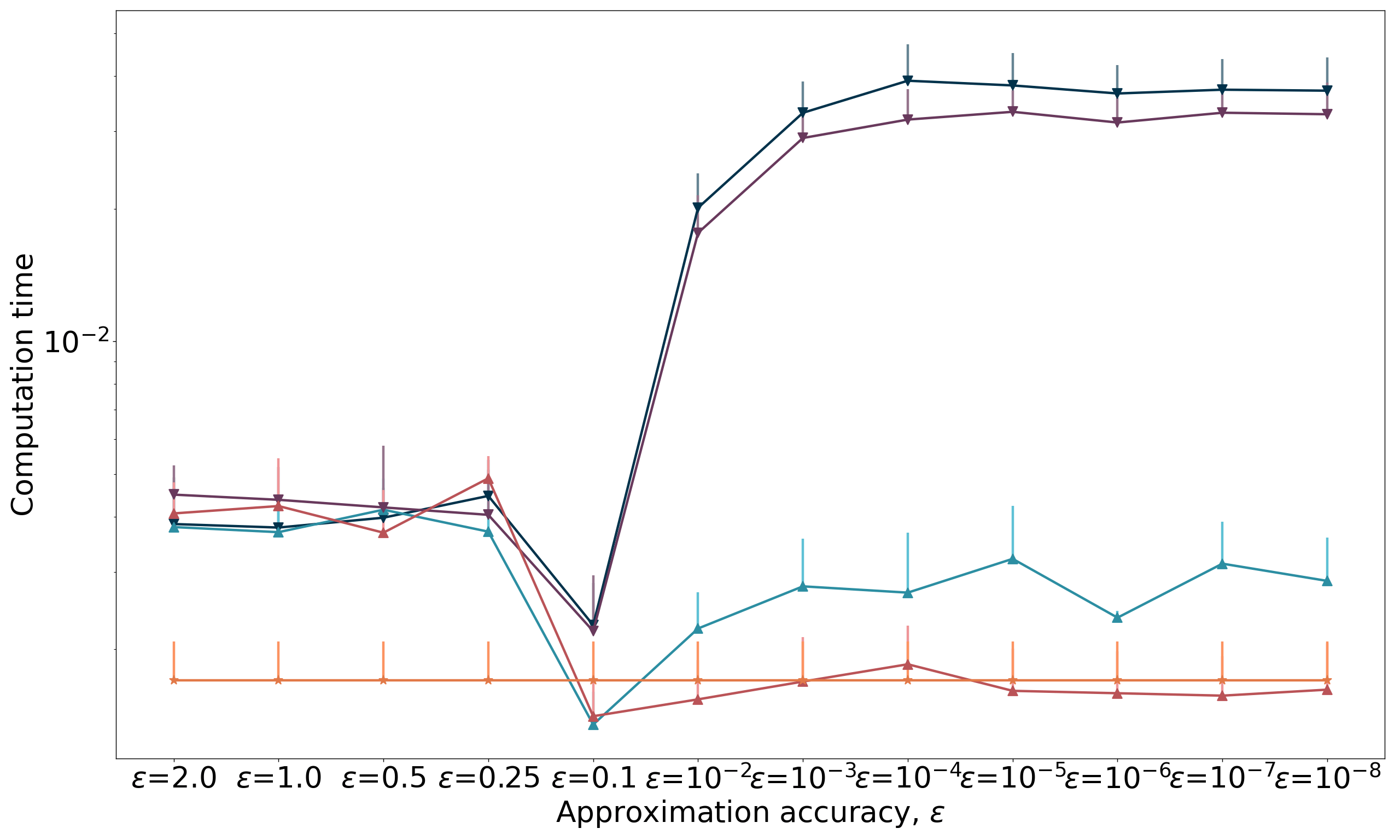}}
    \hfill
\subfloat[Hellinger distance,
          \label{fig:hd-sensitivity-normal-n100}]{\includegraphics{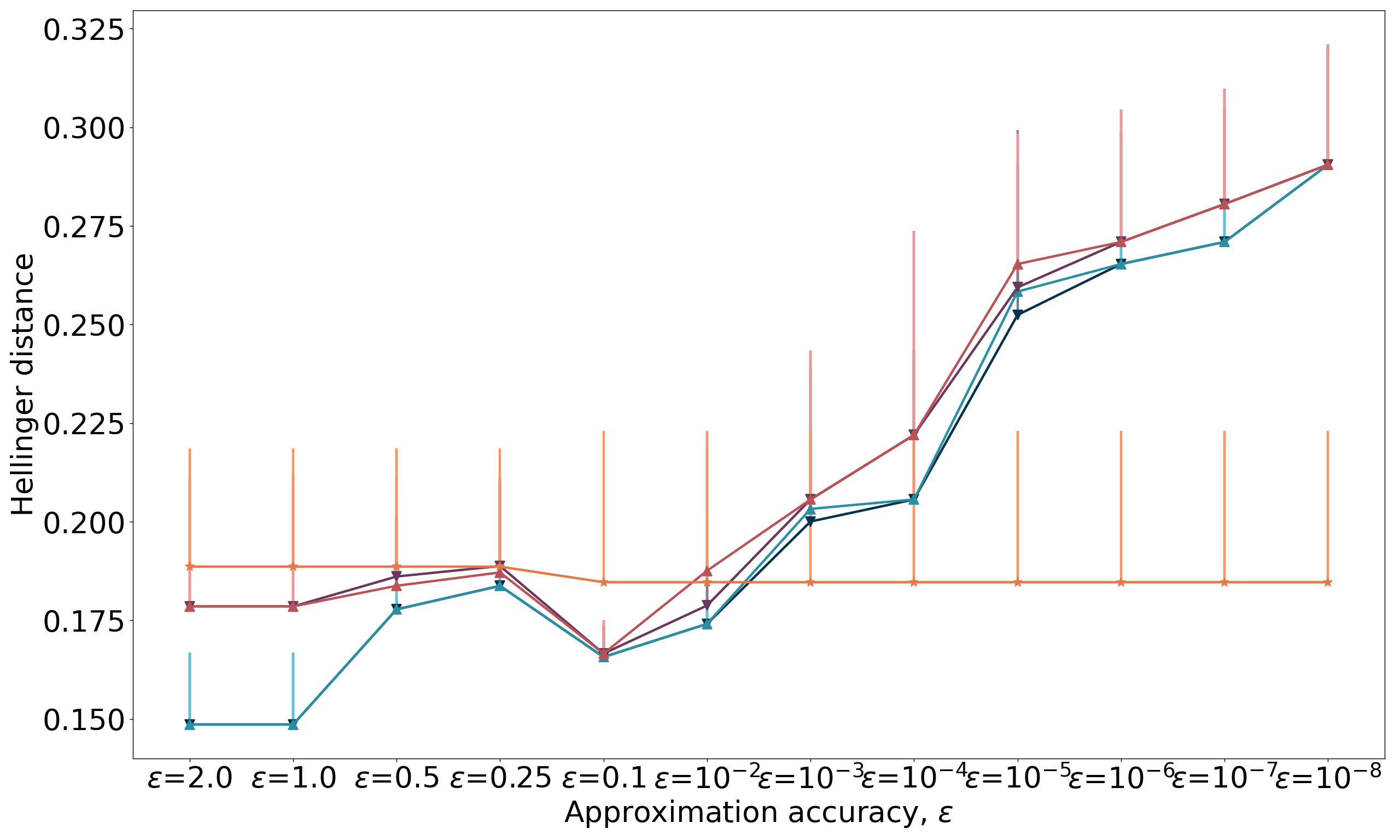}}

\end{figure}

\begin{figure}
\centering
\caption{Comparison of MDL methods over a Normal distribution of size $n=1000$ for different values of $\epsilon$ \label{fig:sensitivity-normal-n1000}}
\setkeys{Gin}{width=0.3\textwidth}
\subfloat[Number of intervals,
          \label{fig:intervals-sensitivity-normal-n1000}]{\includegraphics{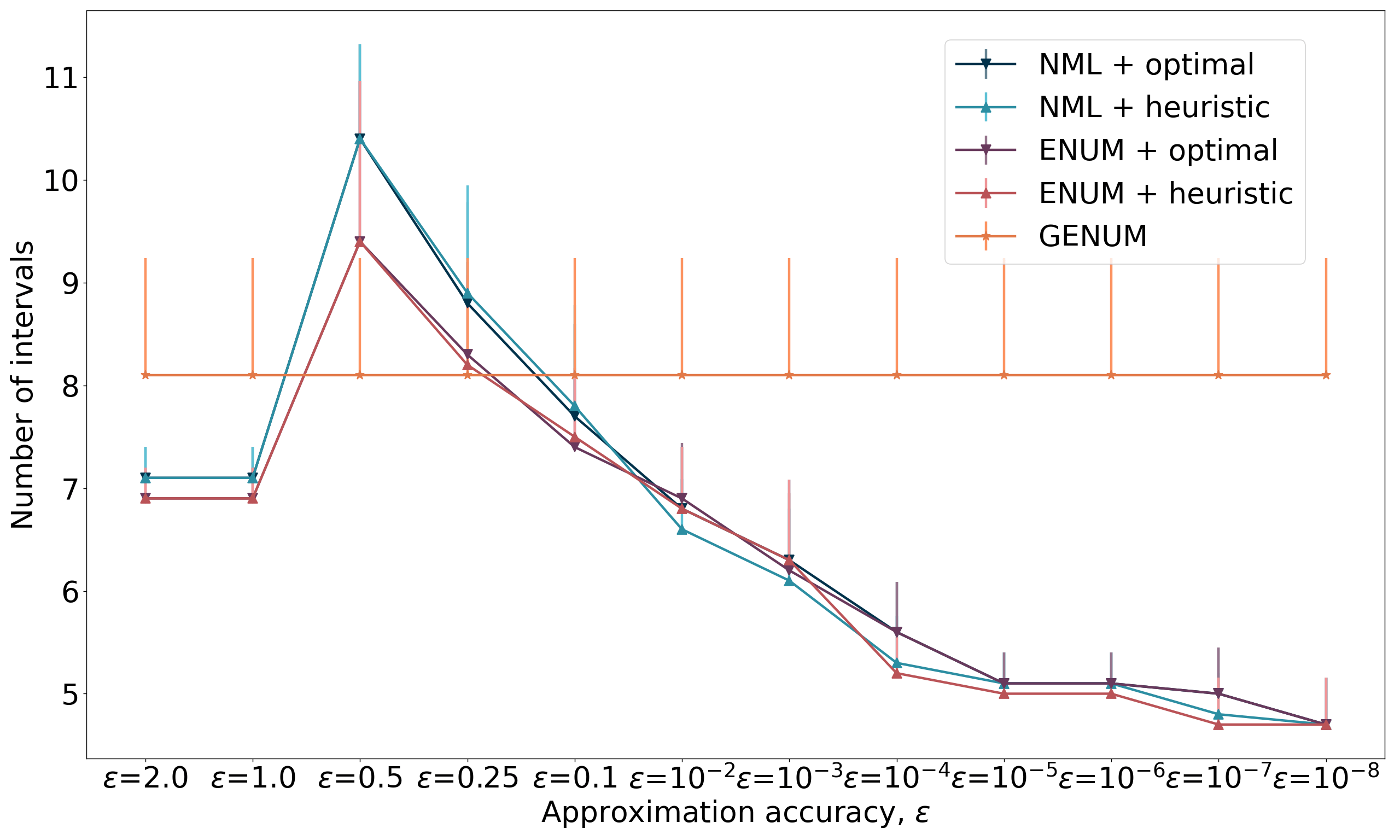}}
    \hfill
\subfloat[Computation time,
          \label{fig:time-sensitivity-normal-n1000}]{\includegraphics{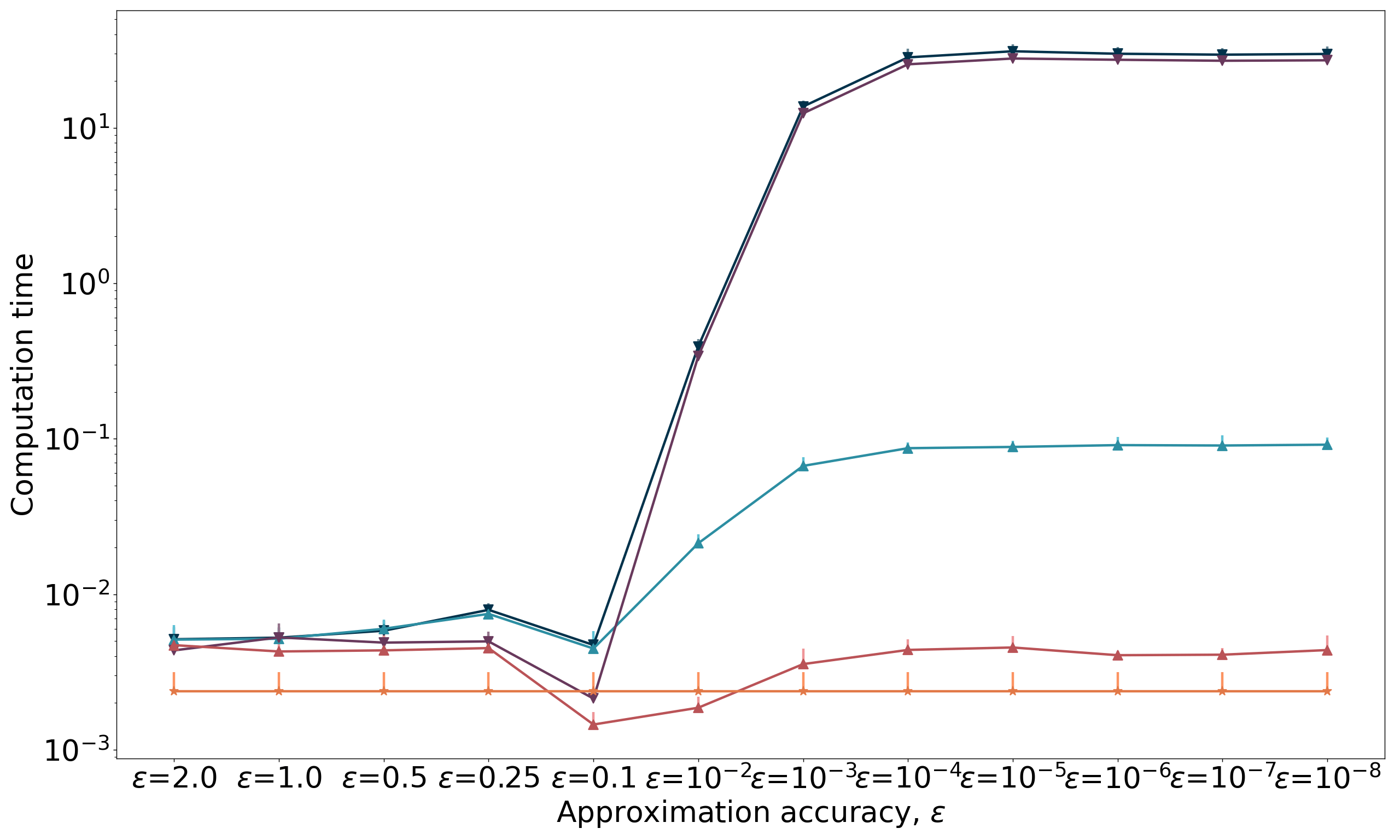}}
    \hfill
\subfloat[Hellinger distance,
          \label{fig:hd-sensitivity-normal-n1000}]{\includegraphics{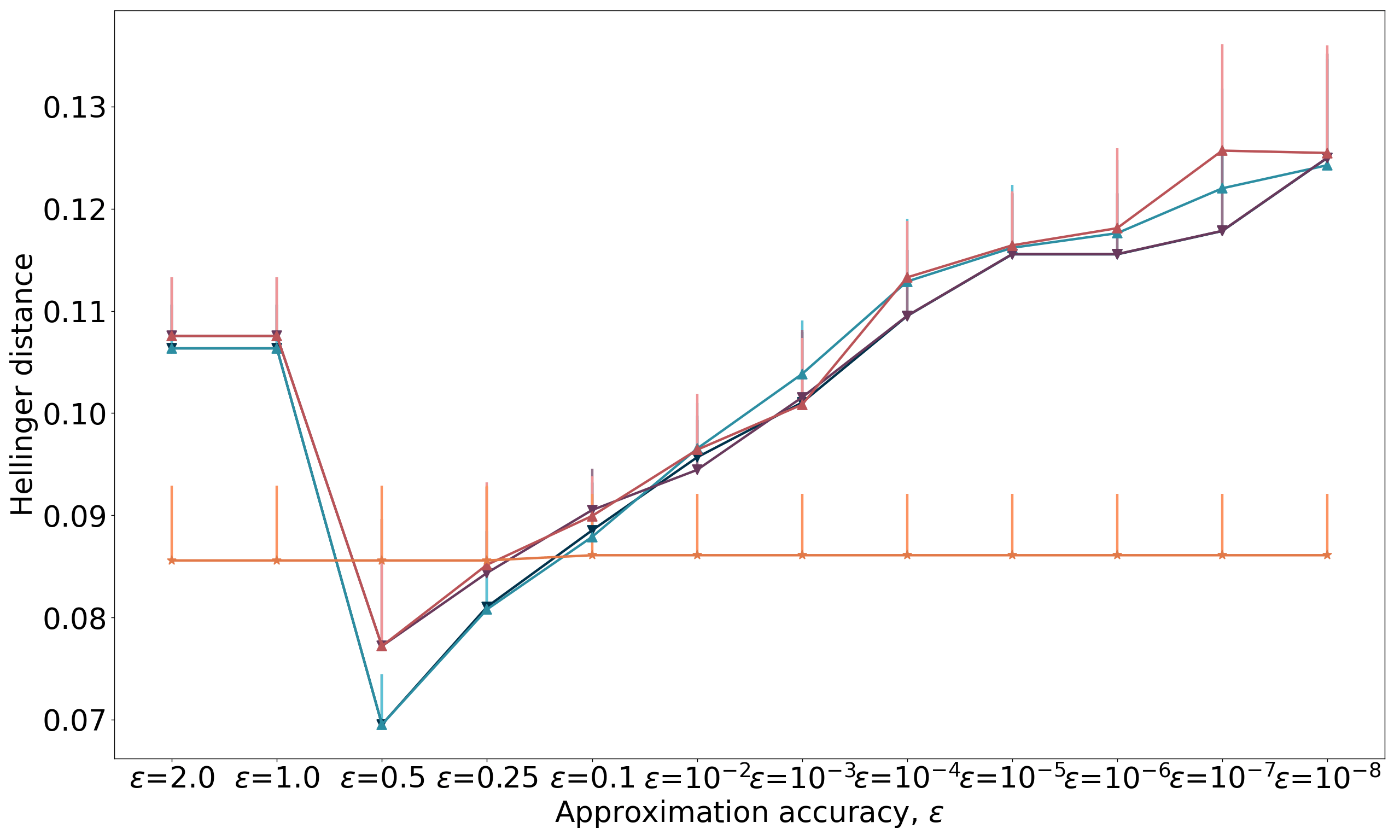}}

\end{figure}

\section{Benchmarking and comparison with other methods}\label{apx:exp-others}
The last series of figures compare the \NMLname and \GENUMname  criteria to
other 5 state-of-the-art methods for building histograms. 

 All 7 methods are evaluated on different sample sizes of a Normal (figures
 \ref{fig:normal-hists} and \ref{fig:comparison-others-normal}, Cauchy
 (figures \ref{fig:cauchy-hists} and \ref{fig:comparison-others-cauchy}),
 uniform (figures \ref{fig:uniform-hists} and
 \ref{fig:comparison-others-uniform}), triangle (figures
 \ref{fig:triangle-hists} and \ref{fig:comparison-others-triangle}), triangle
 mixture (figures \ref{fig:tmix-hists} and \ref{fig:comparison-others-tmix})
 and Gaussian mixture (figures \ref{fig:gmix-claw-hists} and
 \ref{fig:comparison-others-claw}) distributions.

We highlight that all results are means over 10 experiments for each sample
size. All results are shown in log scale. The standard deviations of the
metrics are represented asymmetrically in the graphs to because of the log scale.

\begin{figure}
  \centering  
  \caption{Different histograms obtained for a single Normal distribution of size $n=10^4$ \label{fig:normal-hists}}
  \setkeys{Gin}{width=0.25\textwidth}
  \subfloat[A Normal distribution ]
  {\includegraphics{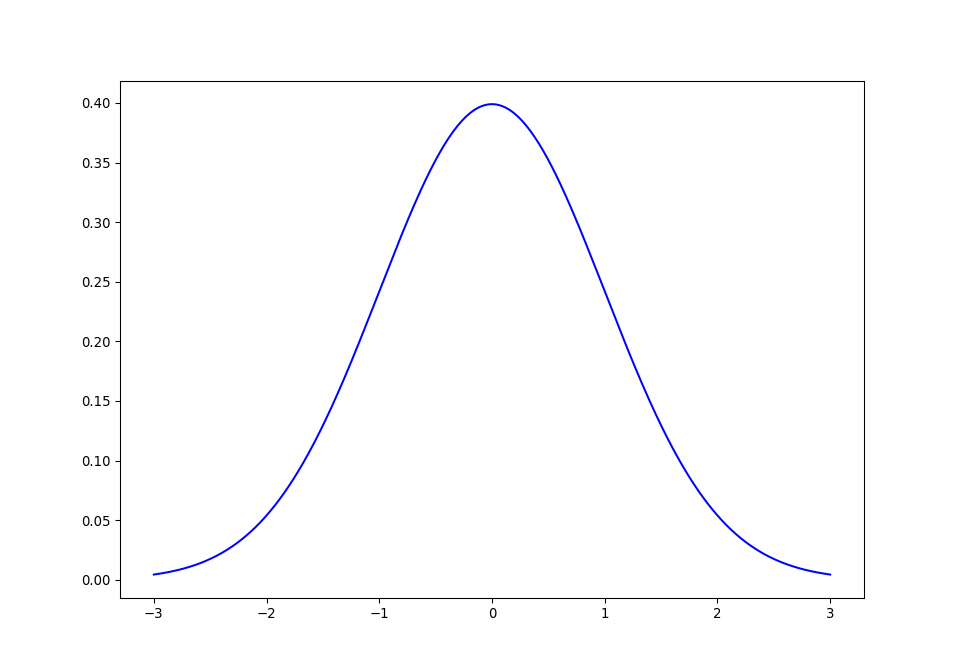}}%
	\hfill
	\subfloat[NML + heuristic ($K^*=16$)]
  {\includegraphics{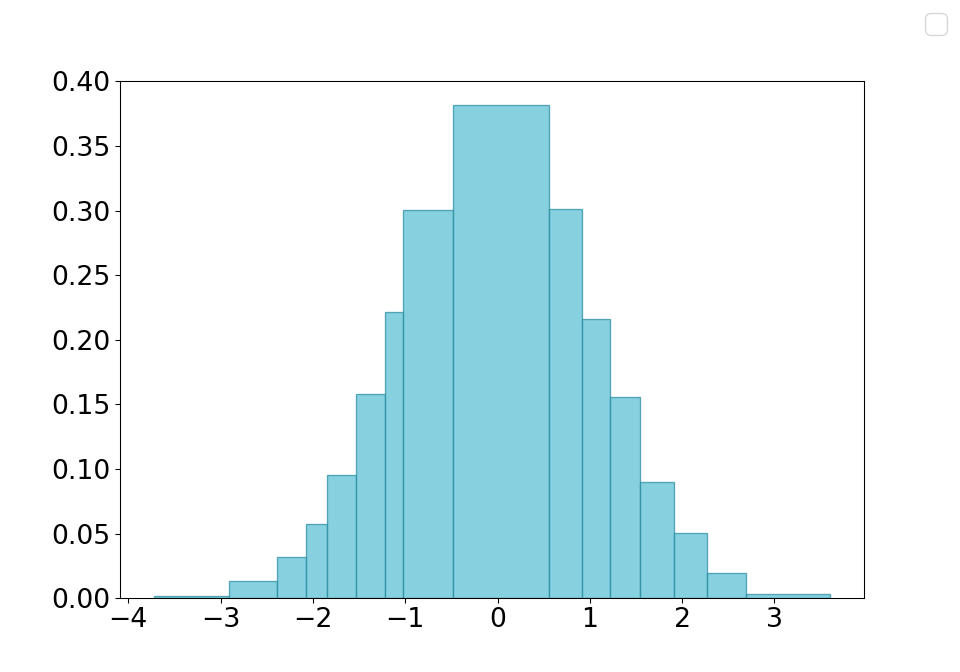}}%
	\hfill
	\subfloat[\GENUMname ($K^*=17$)]
  {\includegraphics{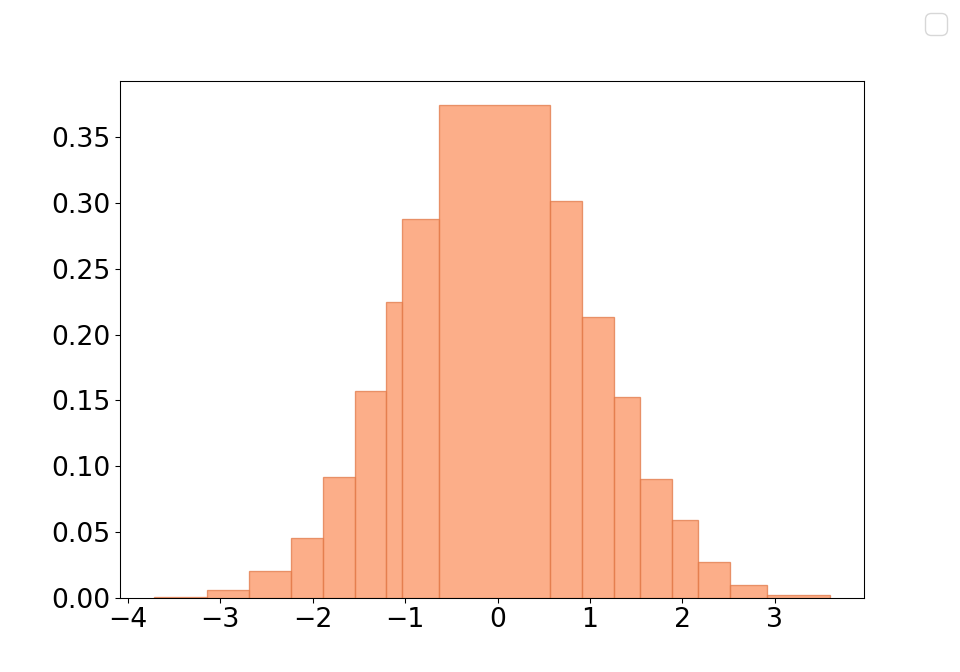}}%
	\hfill
	\subfloat[Bayesian blocks ($K^*=16$)]
  {\includegraphics{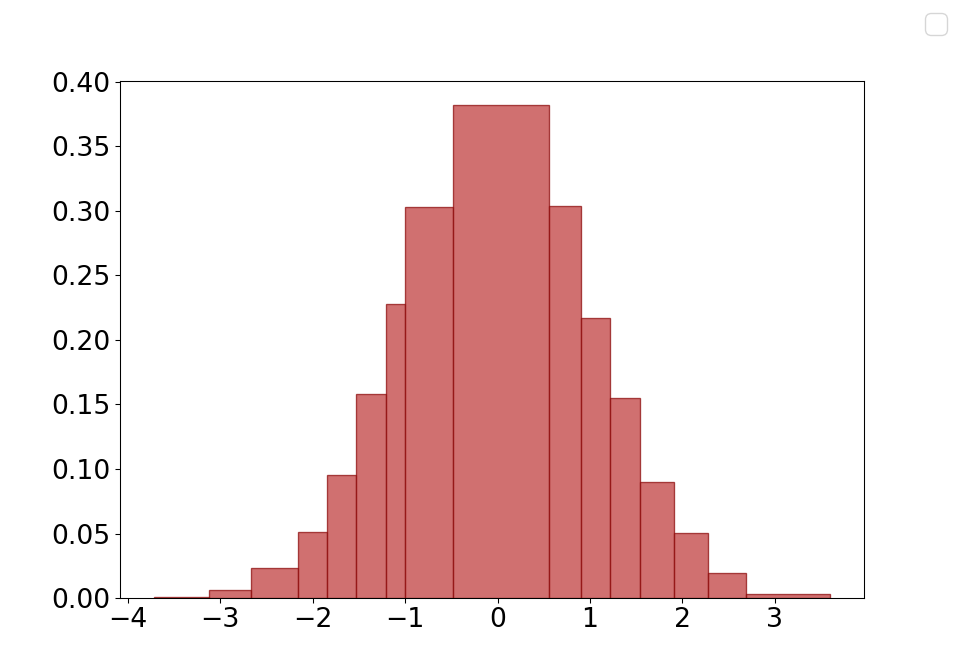}}%
	\hfill
	\subfloat[Taut string ($K^*= 68$)]
  {\includegraphics{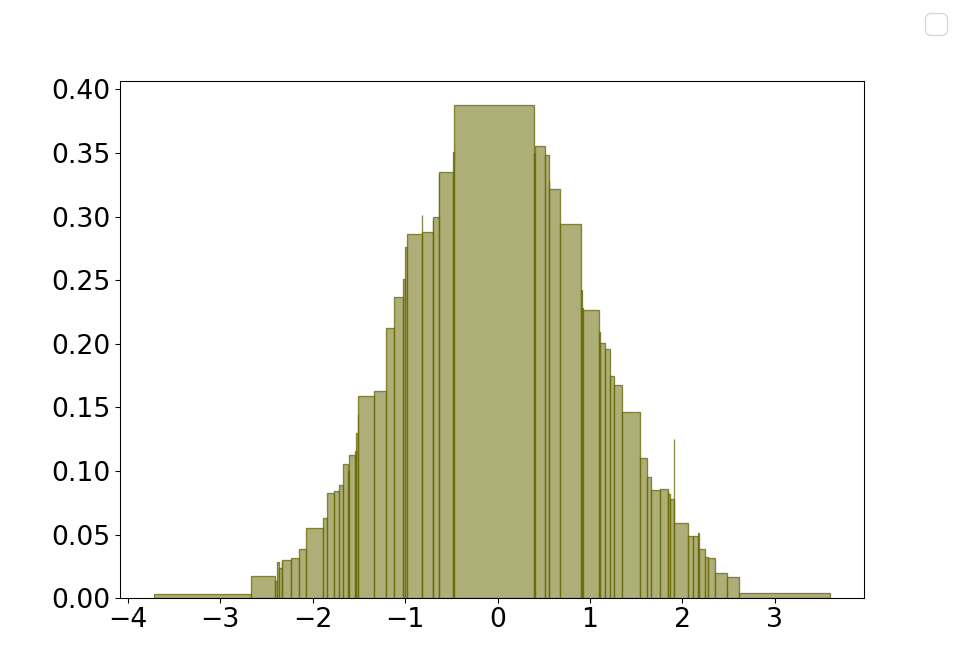}\label{fig:normal-TS}}%
	\hfill
	\subfloat[RMG ($K^*=35$)]
  {\includegraphics{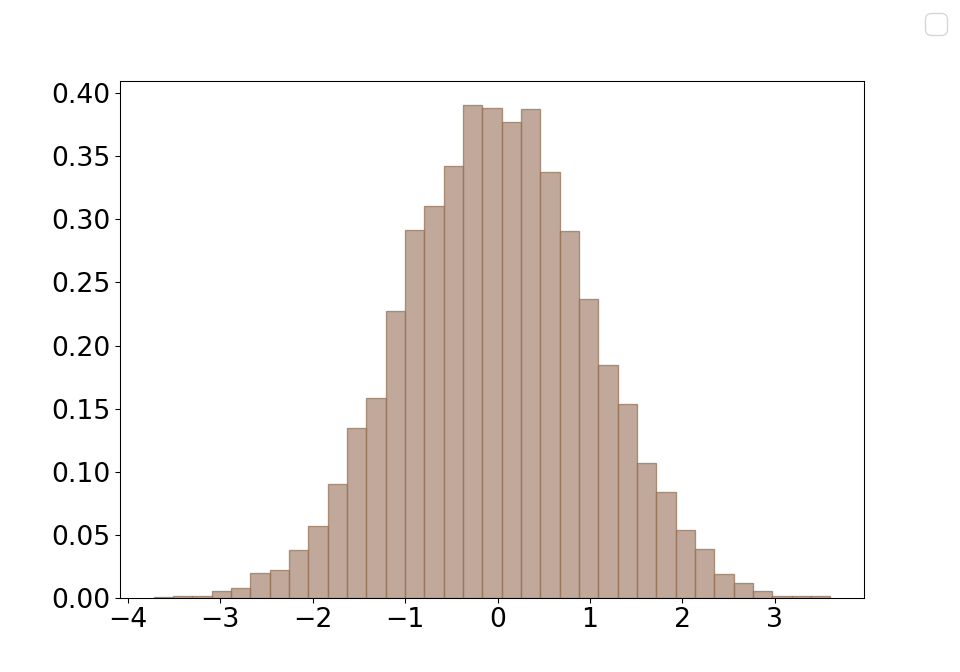}}%
	\hfill
	\subfloat[Sturges ($K^*=15$)]
  {\includegraphics{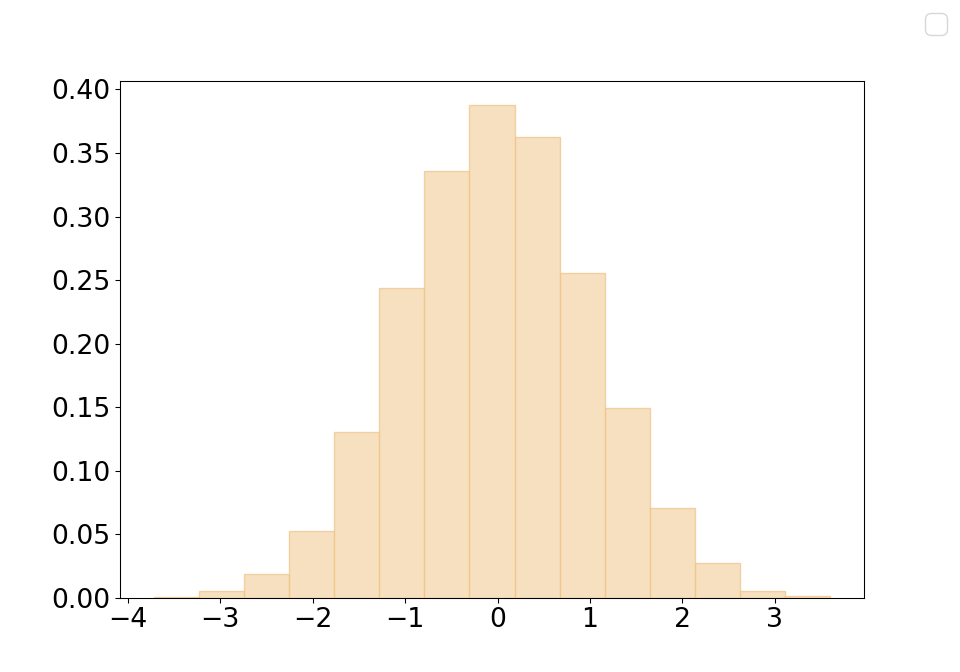}}%
	\hfill
	\subfloat[Freedman-Diaconis ($K^*= 59$)]
  {\includegraphics{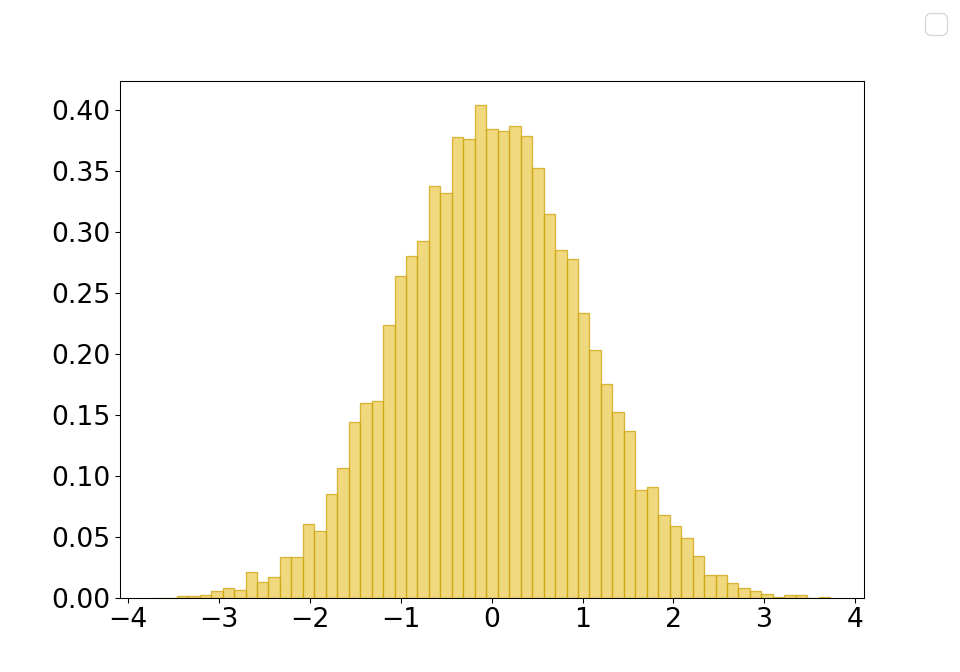}}
\end{figure}

\begin{figure}
\centering
\caption{Comparison with state-of-the-art methods over a Normal distribution of different sample sizes \label{fig:comparison-others-normal}}
\setkeys{Gin}{width=0.3\textwidth}
\subfloat[Number of intervals,
          \label{fig:intervals-others-normal}]{\includegraphics{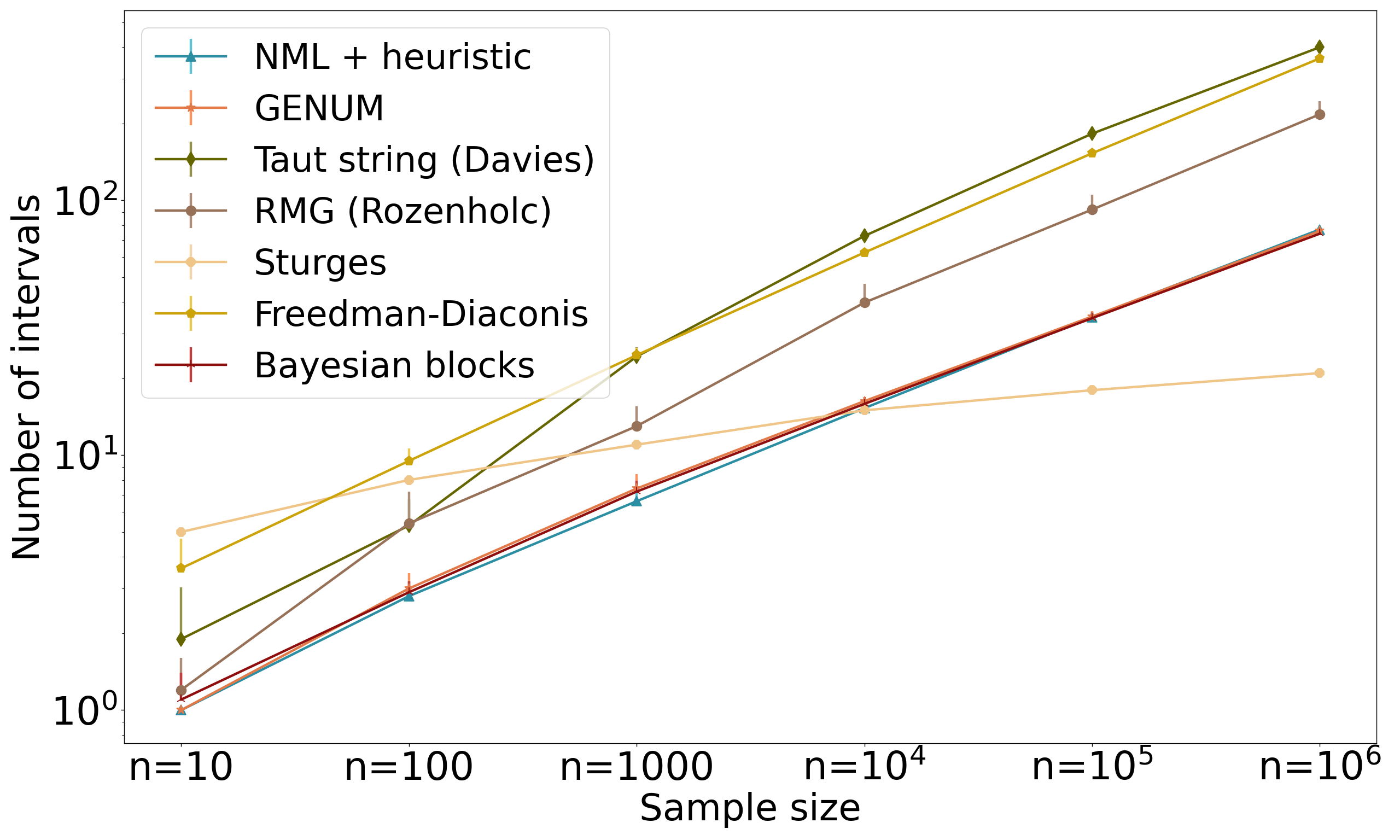}}
    \hfill
\subfloat[Computation time,
          \label{fig:time-others-normal}]{\includegraphics{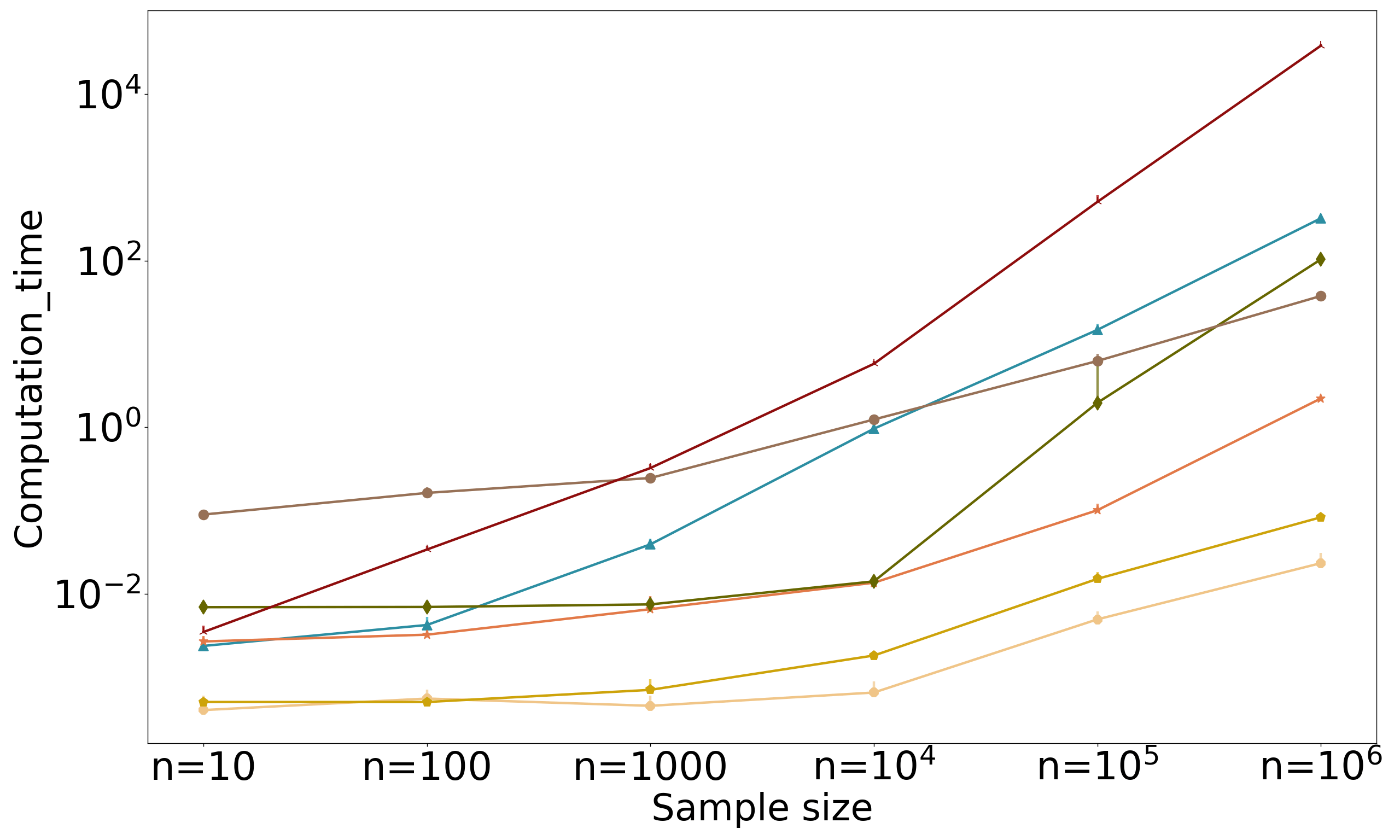}}
    \hfill
\subfloat[Hellinger distance,
          \label{fig:hd-others-normal}]{\includegraphics{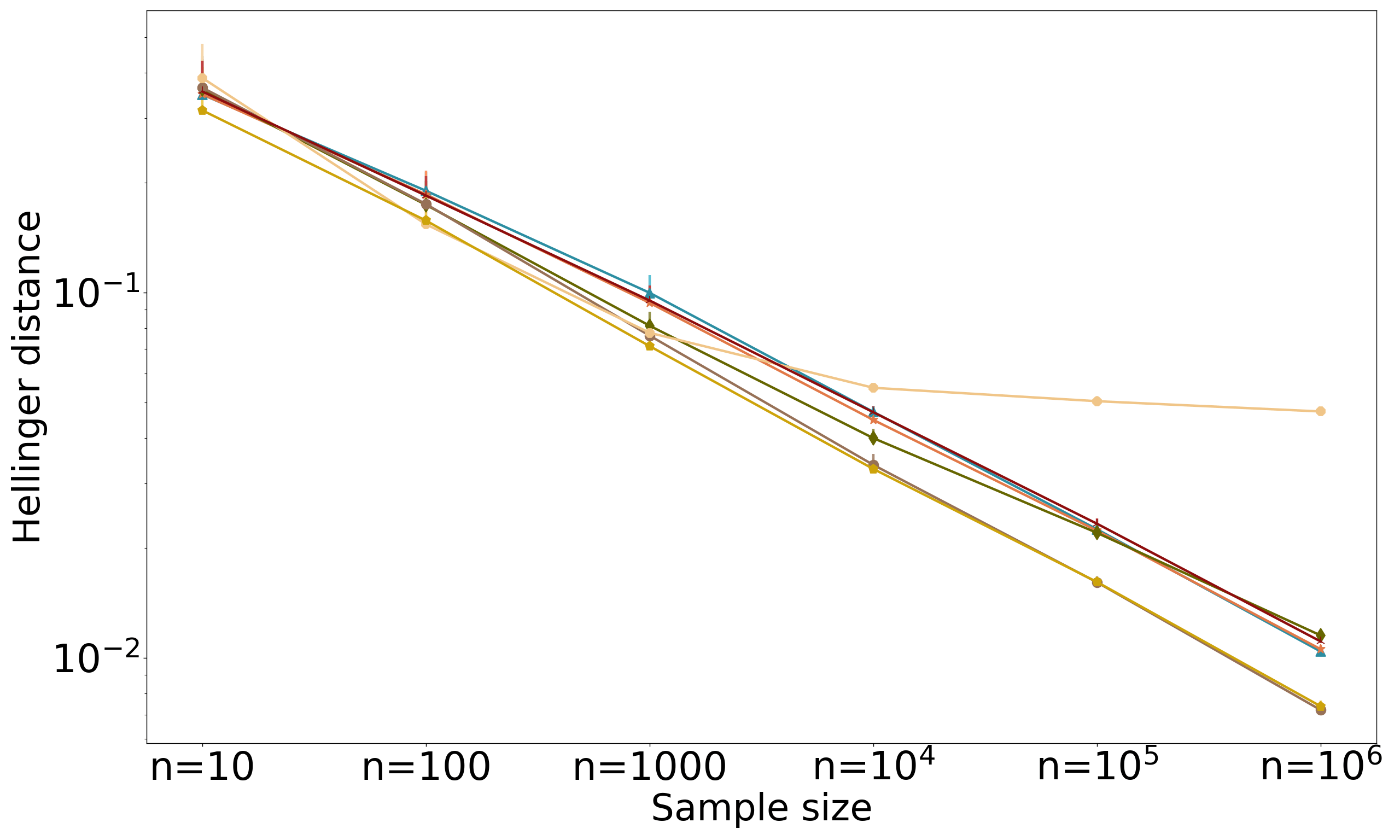}}

\end{figure}

\begin{figure}
  \centering  
  \caption{Different histograms obtained for a single  Cauchy distribution of size $n=10^4$ \label{fig:cauchy-hists}}
  \setkeys{Gin}{width=0.25\textwidth}
  \subfloat[A Cauchy distribution in $\log$ scale]
  {\includegraphics{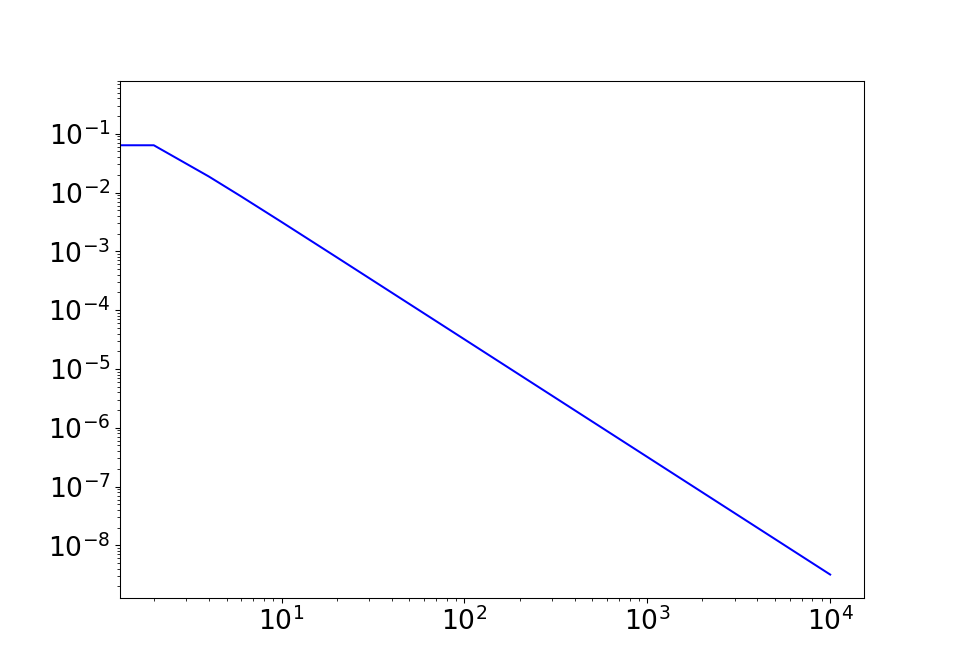}}%
	\hfill
	\subfloat[NML + heuristic ($K^*=24$)]
  {\includegraphics{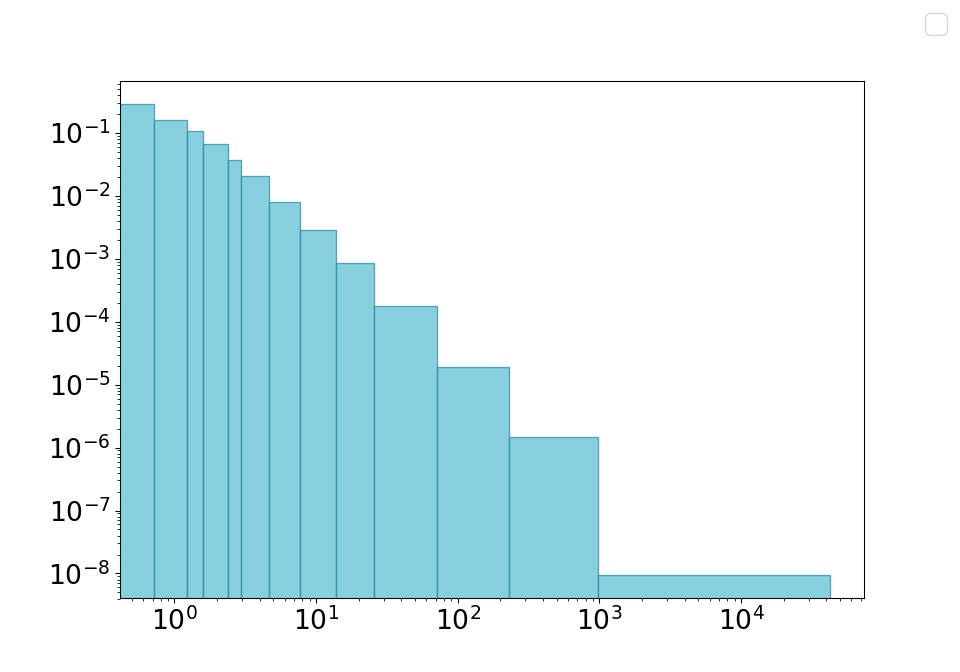}}%
	\hfill
	\subfloat[\GENUMname ($K^*=30$)]
  {\includegraphics{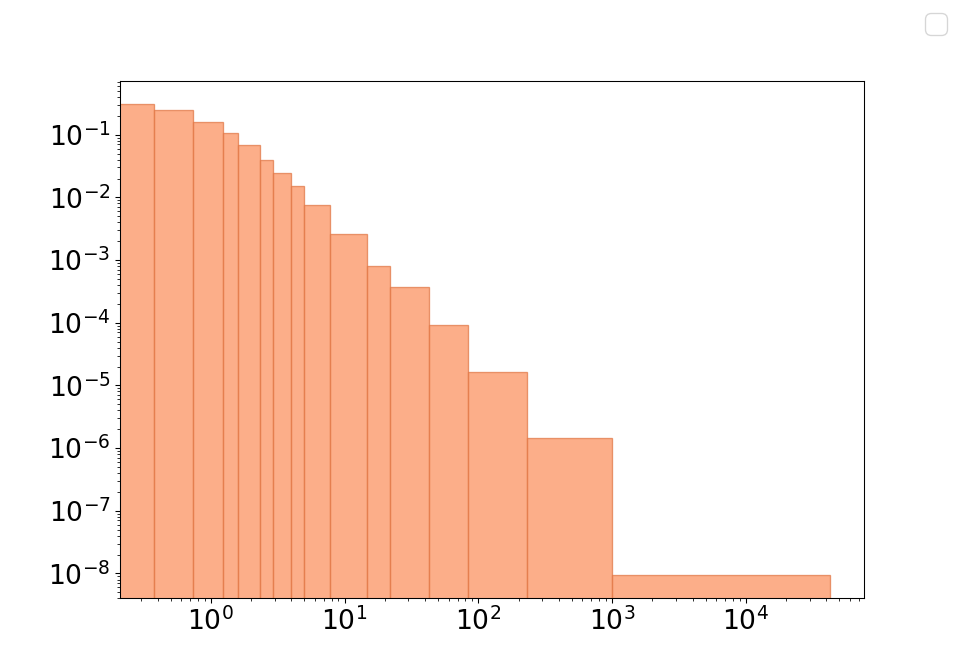}}%
	\hfill
	\subfloat[Bayesian blocks ($K^*=32$)]
  {\includegraphics{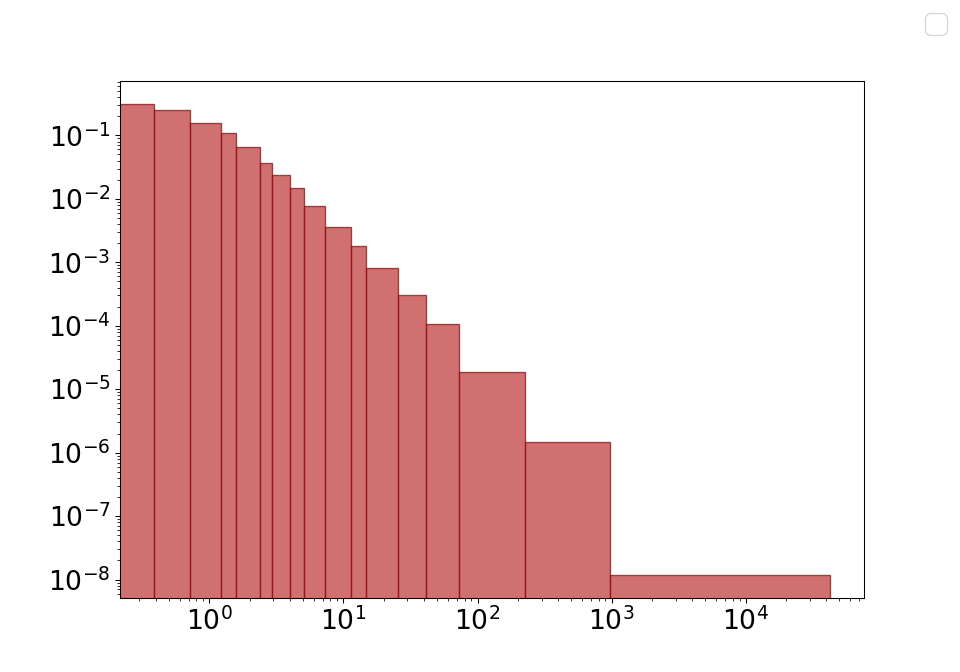}}%
	\hfill
	\subfloat[Taut string ($K^*= 141$)]
  {\includegraphics{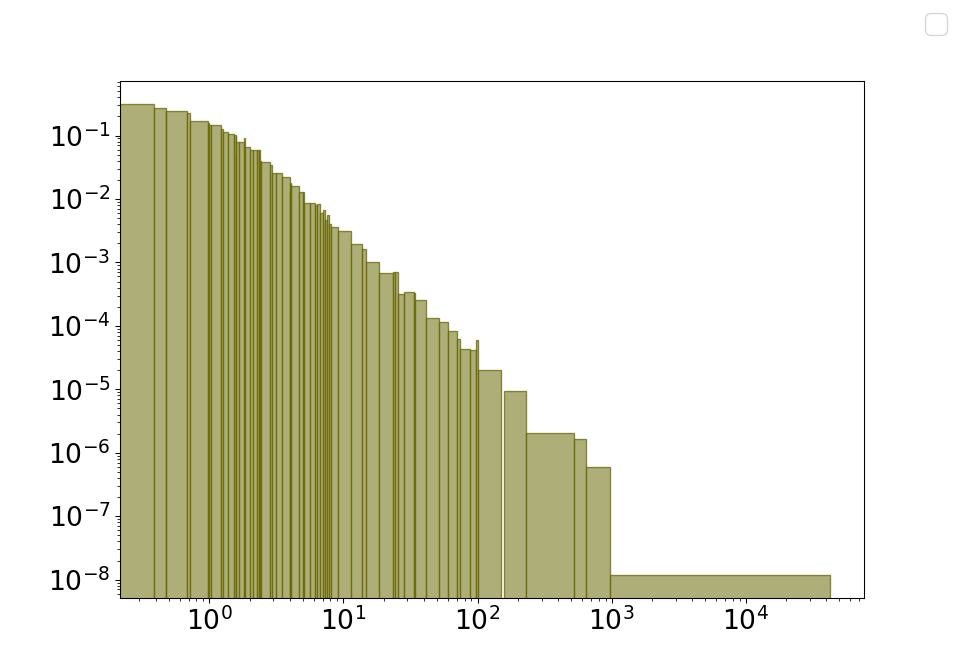}}%
	\hfill
	\subfloat[RMG ($K^*=31$)]
  {\includegraphics{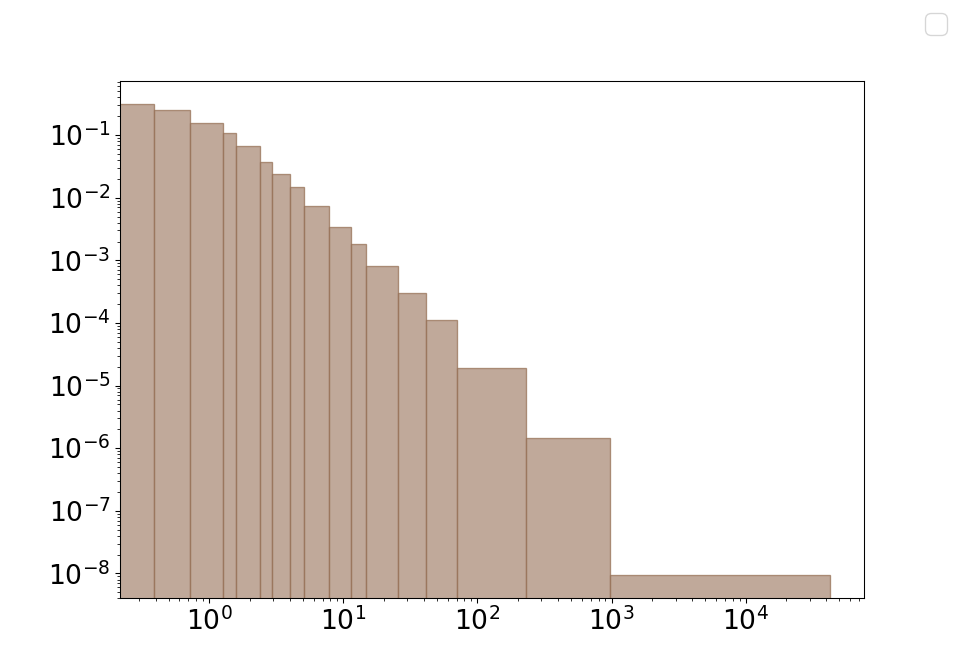}}%
	\hfill
	\subfloat[Sturges ($K^*=15$)]
  {\includegraphics{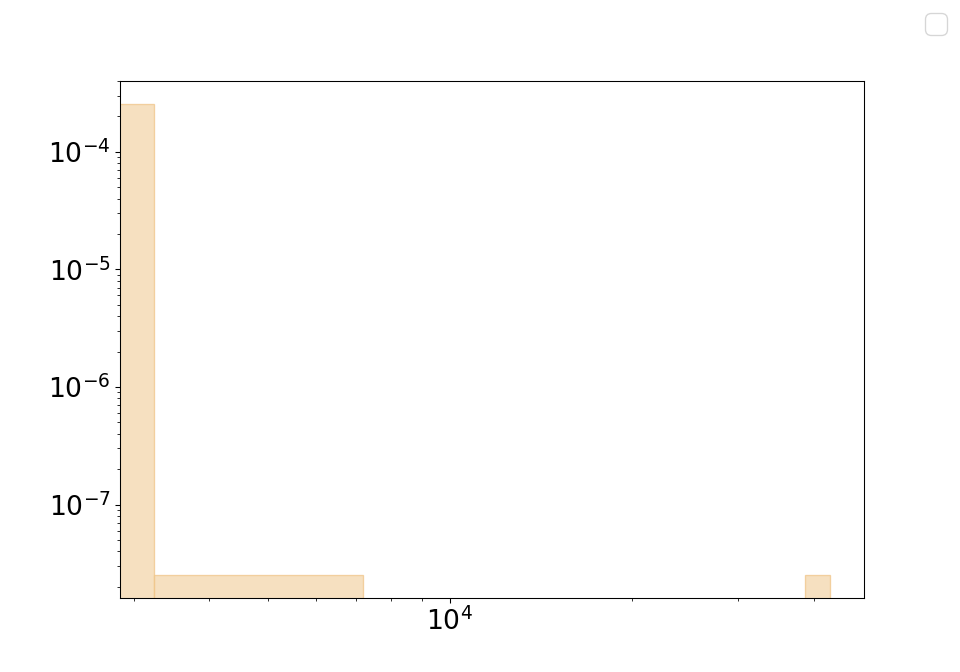}}%
	\hfill
	\subfloat[Freedman-Diaconis ($K^*= 319399 $)]
  {\includegraphics{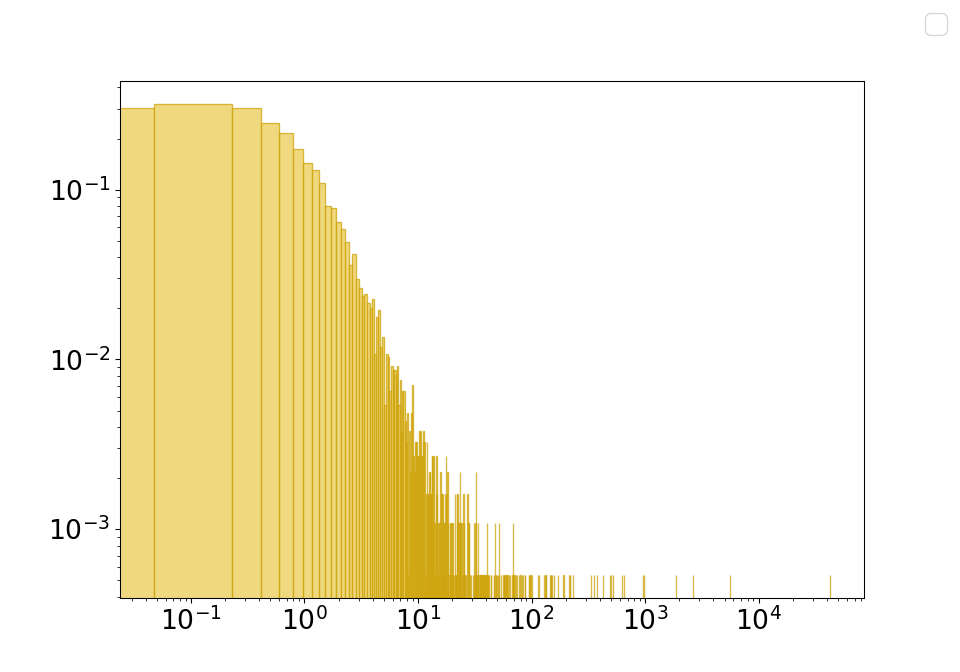}}
\end{figure}

\begin{figure}
\centering
\caption{Comparison with state-of-the-art methods over a Cauchy distribution of different sample size \label{fig:comparison-others-cauchy}}
\setkeys{Gin}{width=0.3\textwidth}
\subfloat[Number of intervals,
          \label{fig:intervals-others-cauchy}]{\includegraphics{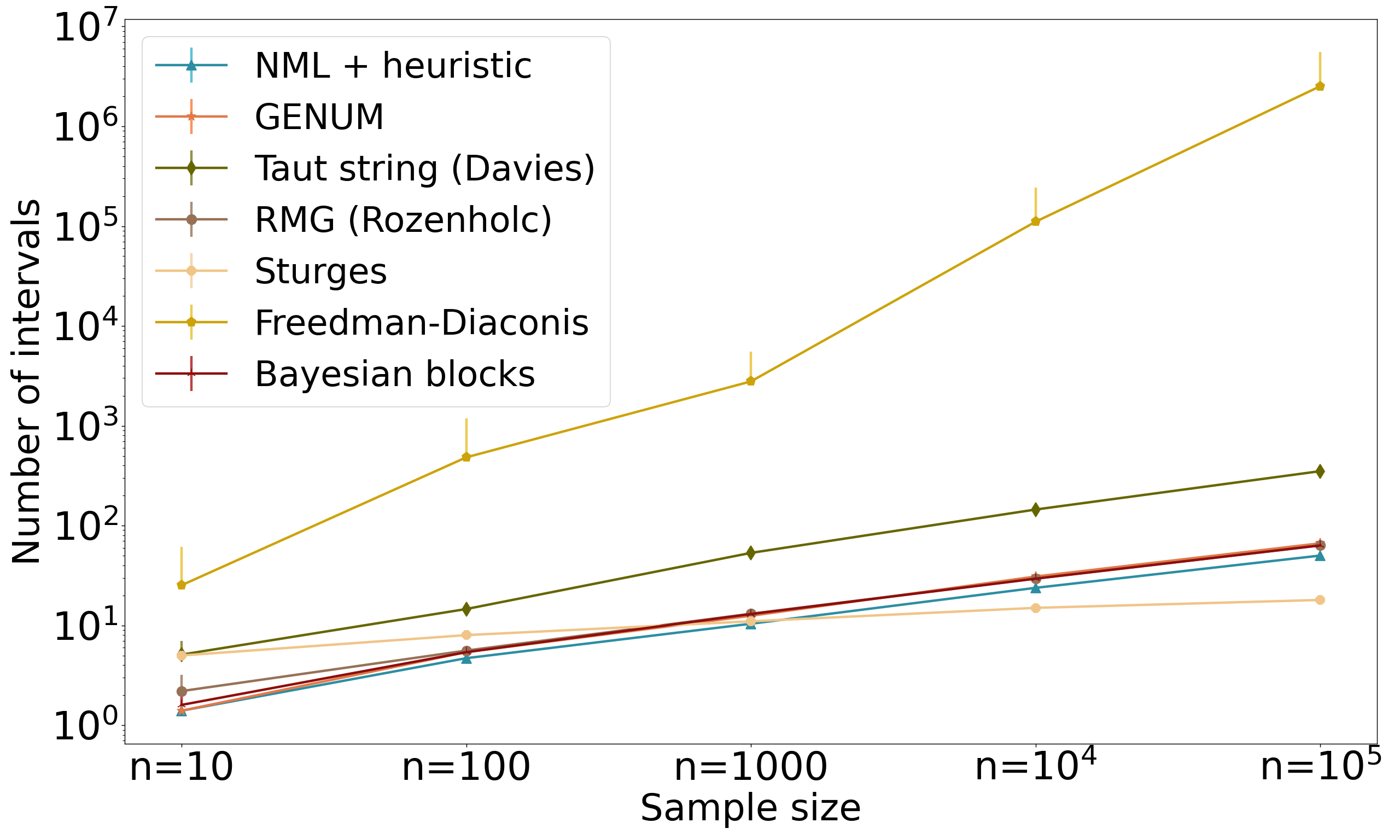}}
    \hfill
\subfloat[Computation time,
          \label{fig:time-others-cauchy}]{\includegraphics{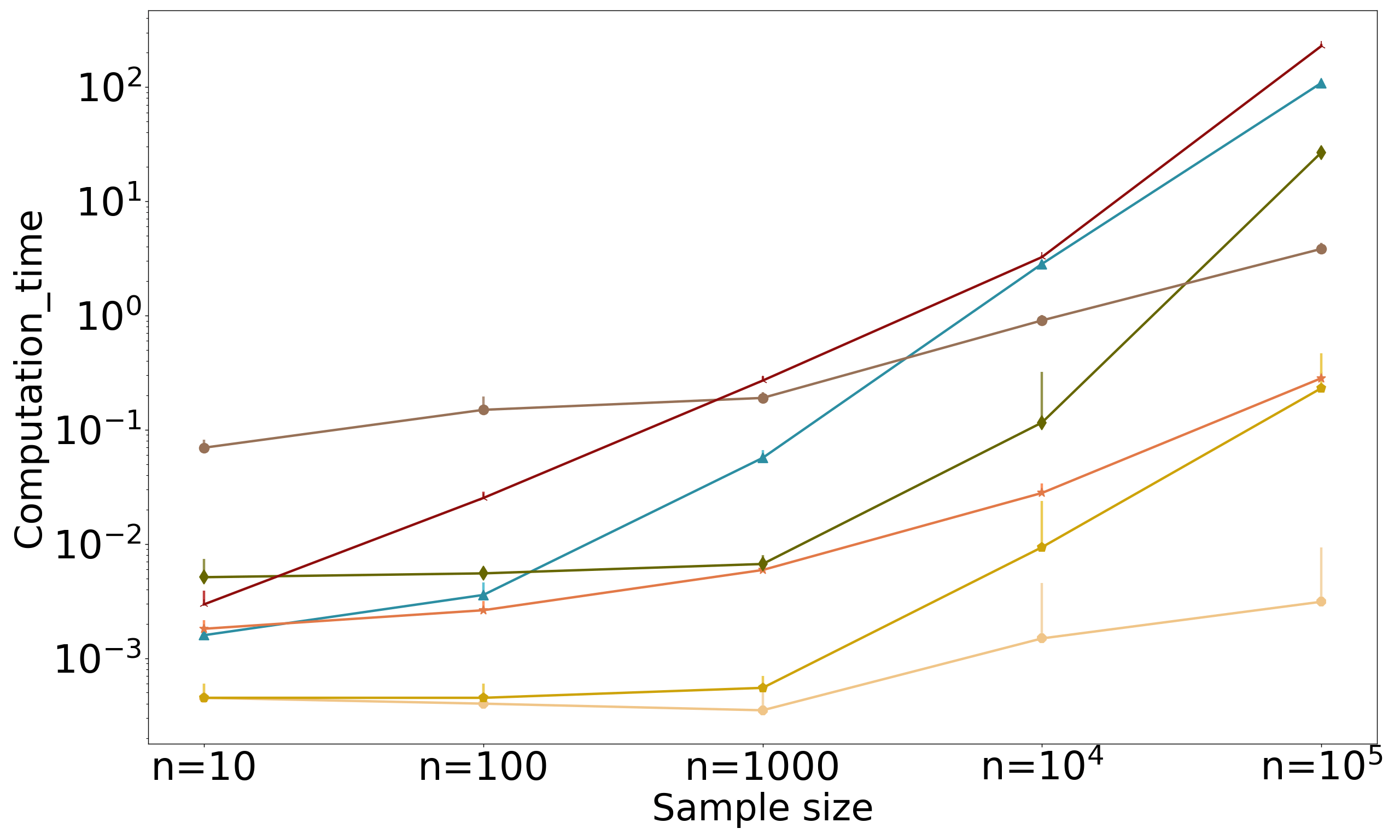}}
    \hfill
\subfloat[Hellinger distance,
          \label{fig:hd-others-cauchy}]{\includegraphics{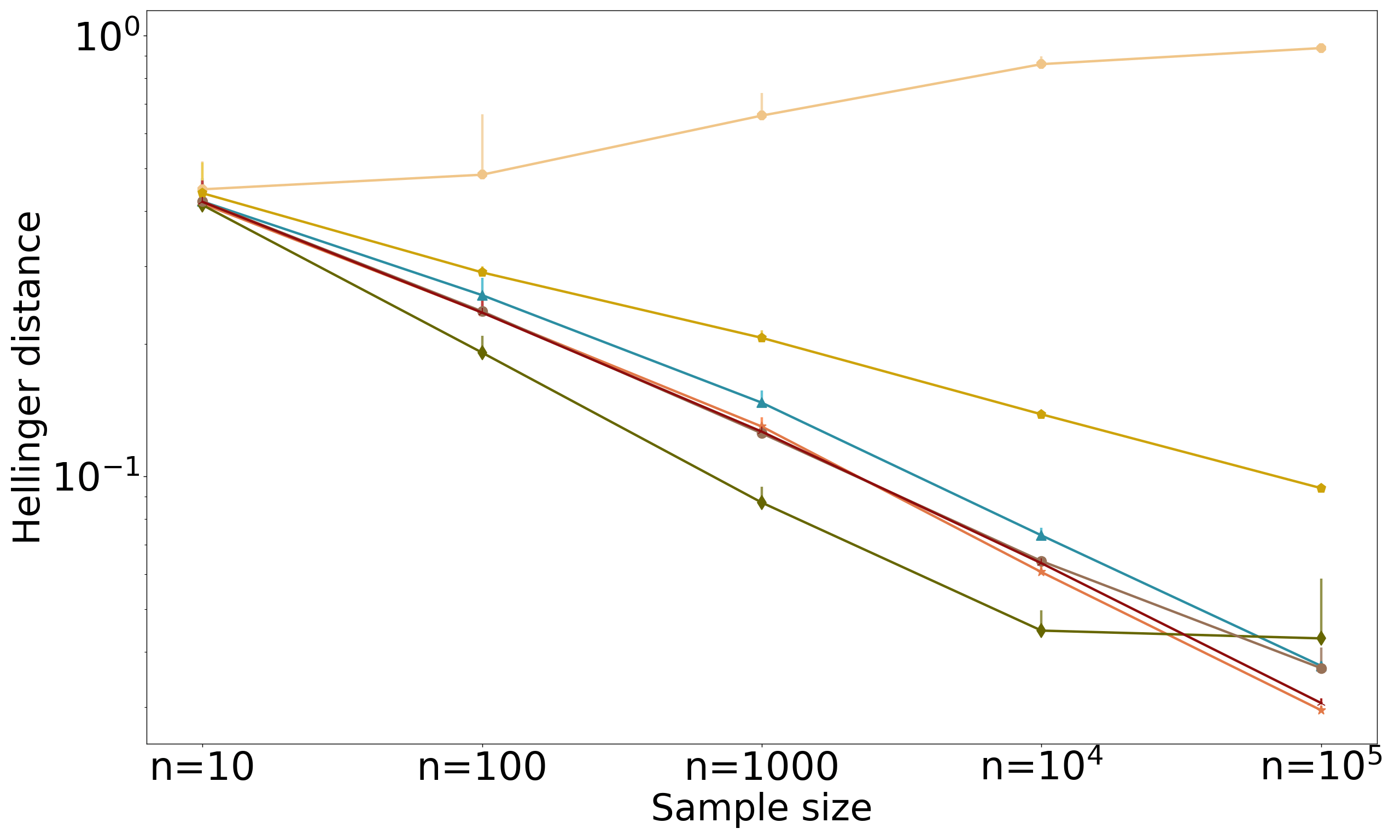}}

\end{figure}

\begin{figure}
  \centering  
  \caption{Different histograms obtained for a single uniform distribution of size $n=10^4$ \label{fig:uniform-hists}}
  \setkeys{Gin}{width=0.25\textwidth}
  \subfloat[A uniform distribution ]
  {\includegraphics{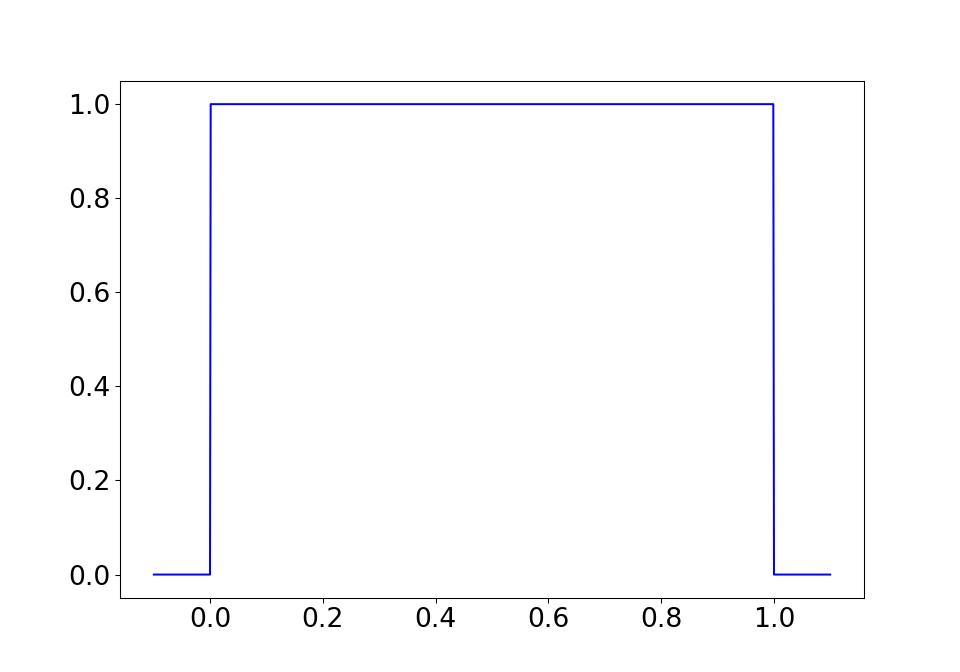}}%
	\hfill
	\subfloat[NML + heuristic ($K^*=3$)]
  {\includegraphics{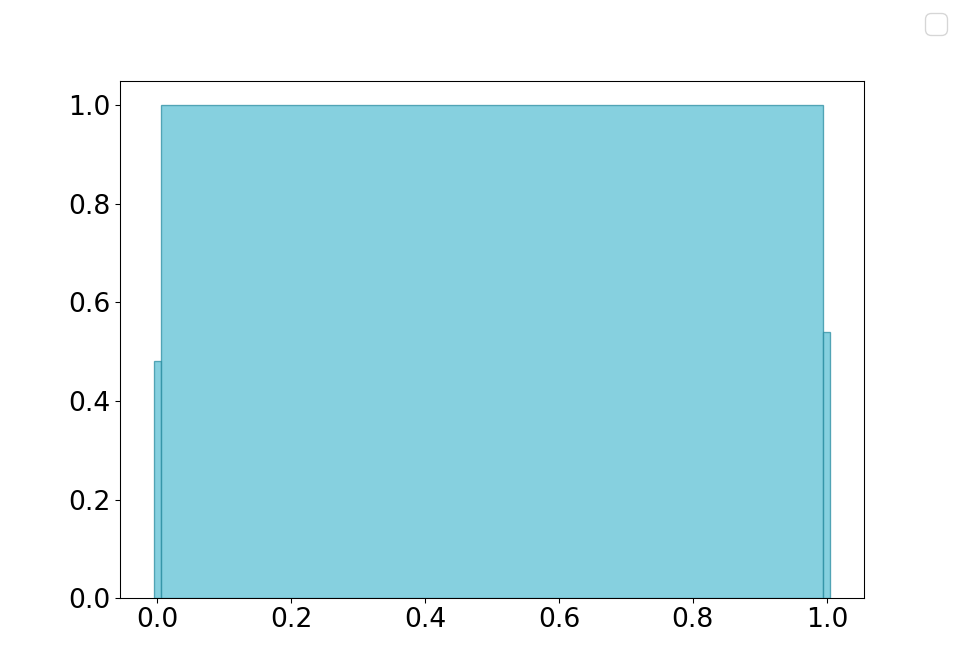}}%
	\hfill
	\subfloat[\GENUMname ($K^*=1$)]
  {\includegraphics{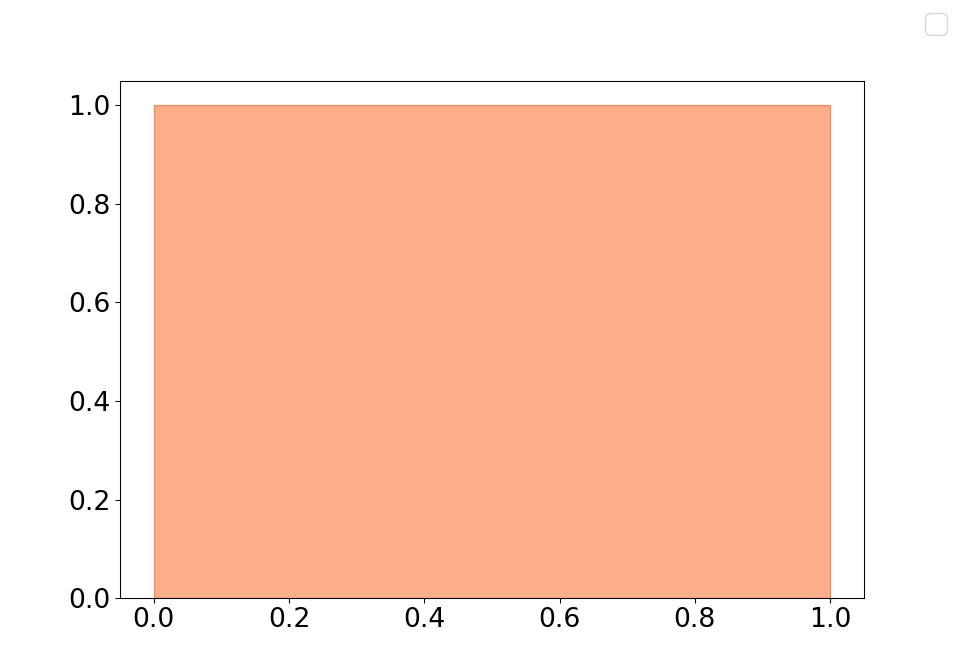}}%
	\hfill
	\subfloat[Bayesian blocks ($K^*=1$)]
  {\includegraphics{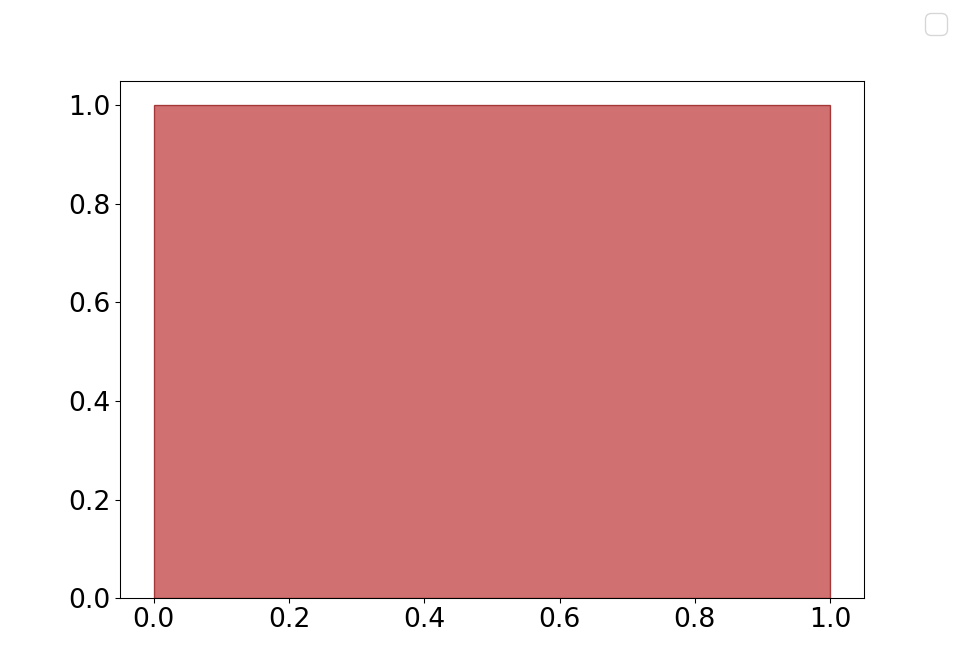}}%
	\hfill
	\subfloat[Taut string ($K^*= 1$)]
  {\includegraphics{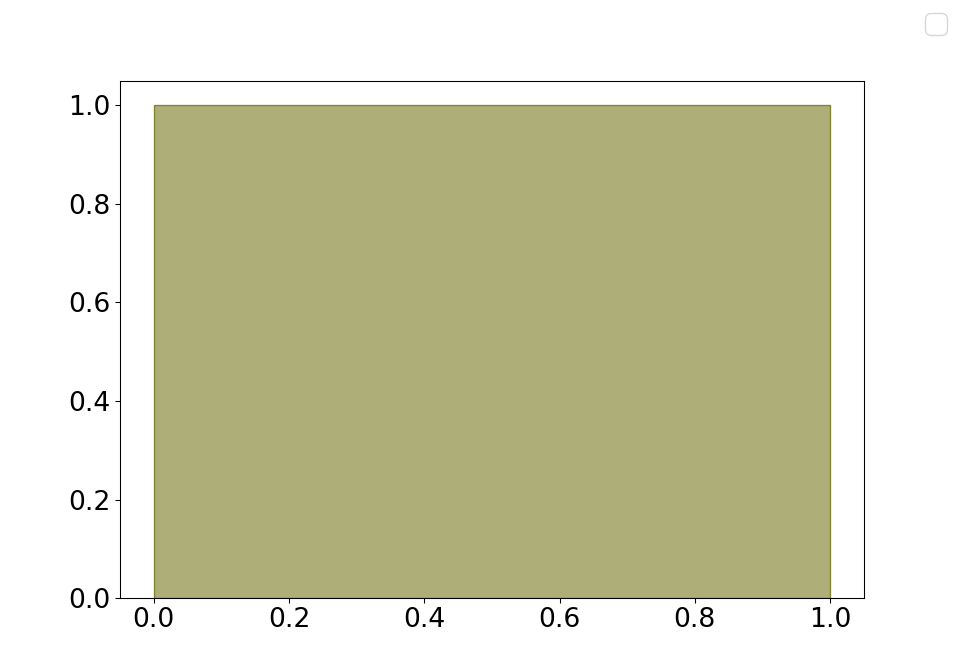}}%
	\hfill
	\subfloat[RMG ($K^*=1$)]
  {\includegraphics{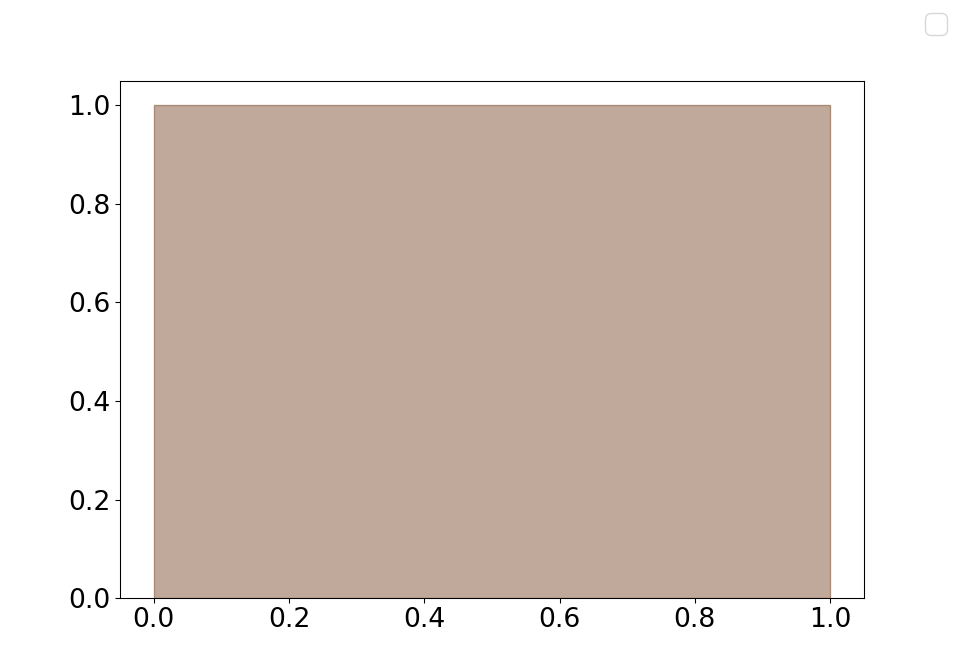}}%
	\hfill
	\subfloat[Sturges ($K^*=15$)]
  {\includegraphics{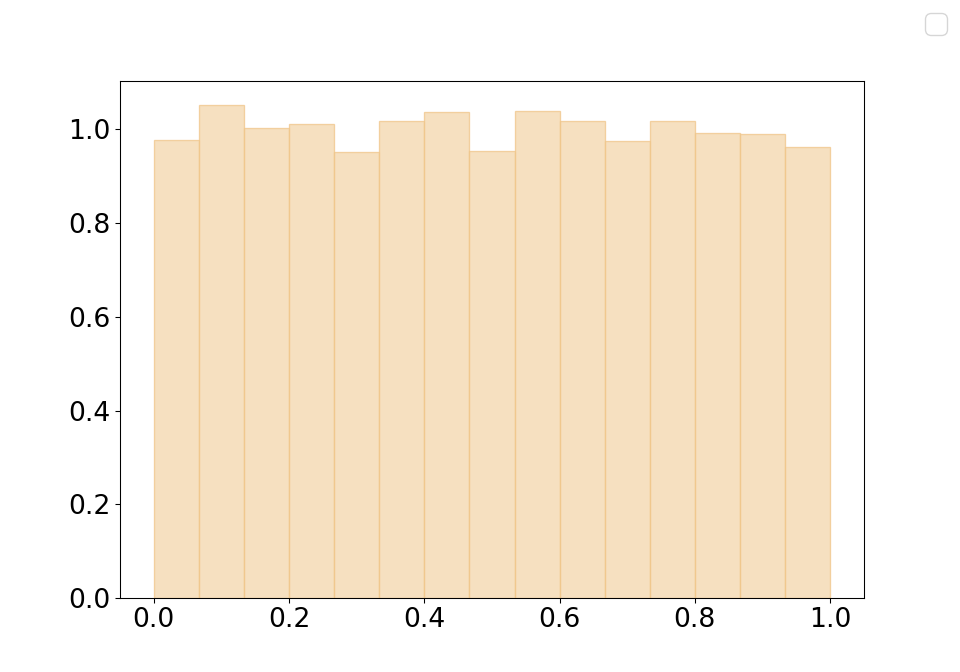}}%
	\hfill
	\subfloat[Freedman-Diaconis ($K^*= 22$)]
  {\includegraphics{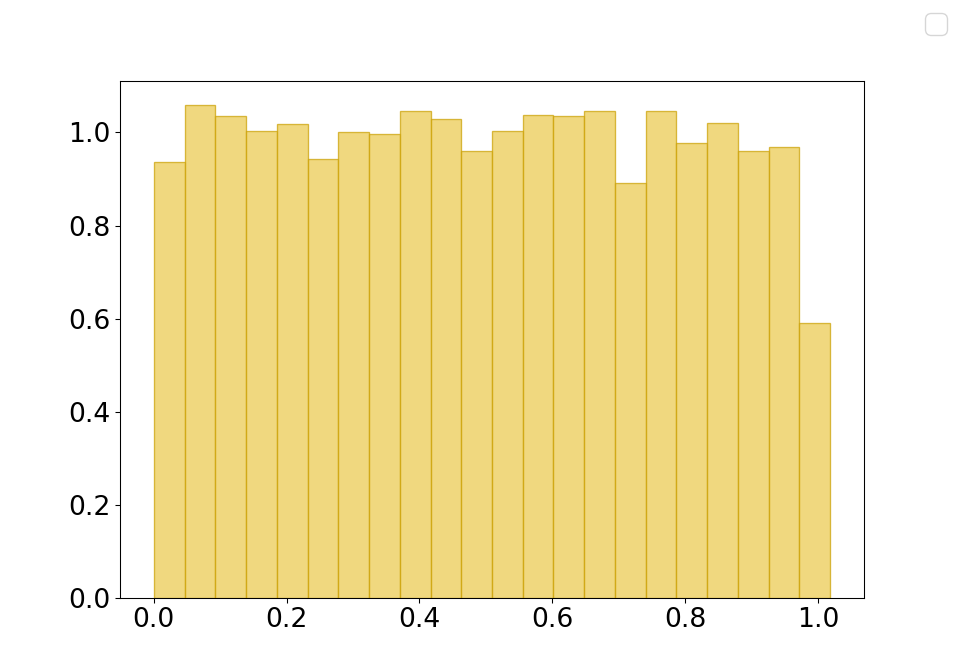}}
\end{figure}

\begin{figure}
\centering
\caption{Comparison with other methods over a uniform distribution of different sample size \label{fig:comparison-others-uniform}}
\setkeys{Gin}{width=0.3\textwidth}
\subfloat[Number of intervals,
          \label{fig:intervals-others-uniform}]{\includegraphics{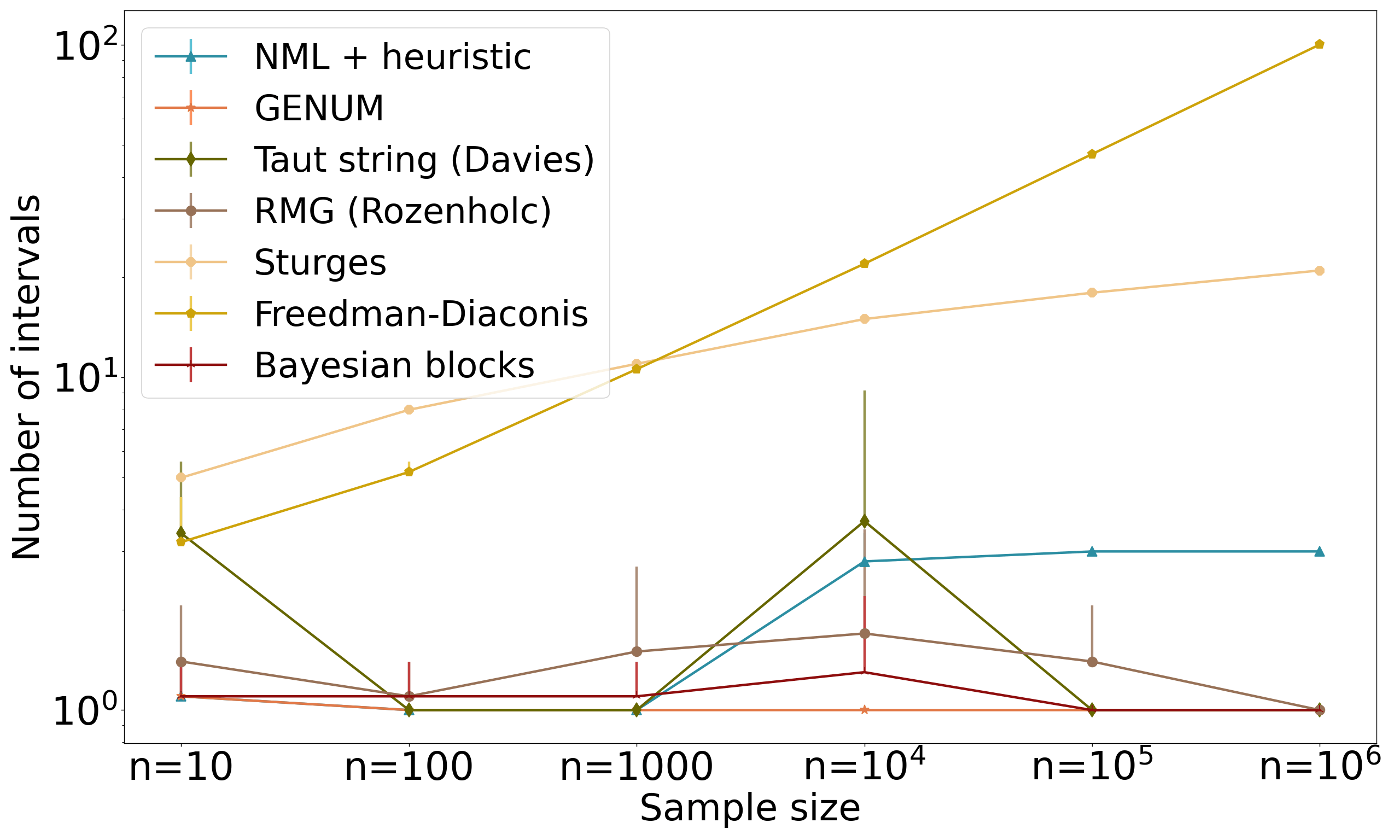}}
    \hfill
\subfloat[Computation time,
          \label{fig:time-others-uniform}]{\includegraphics{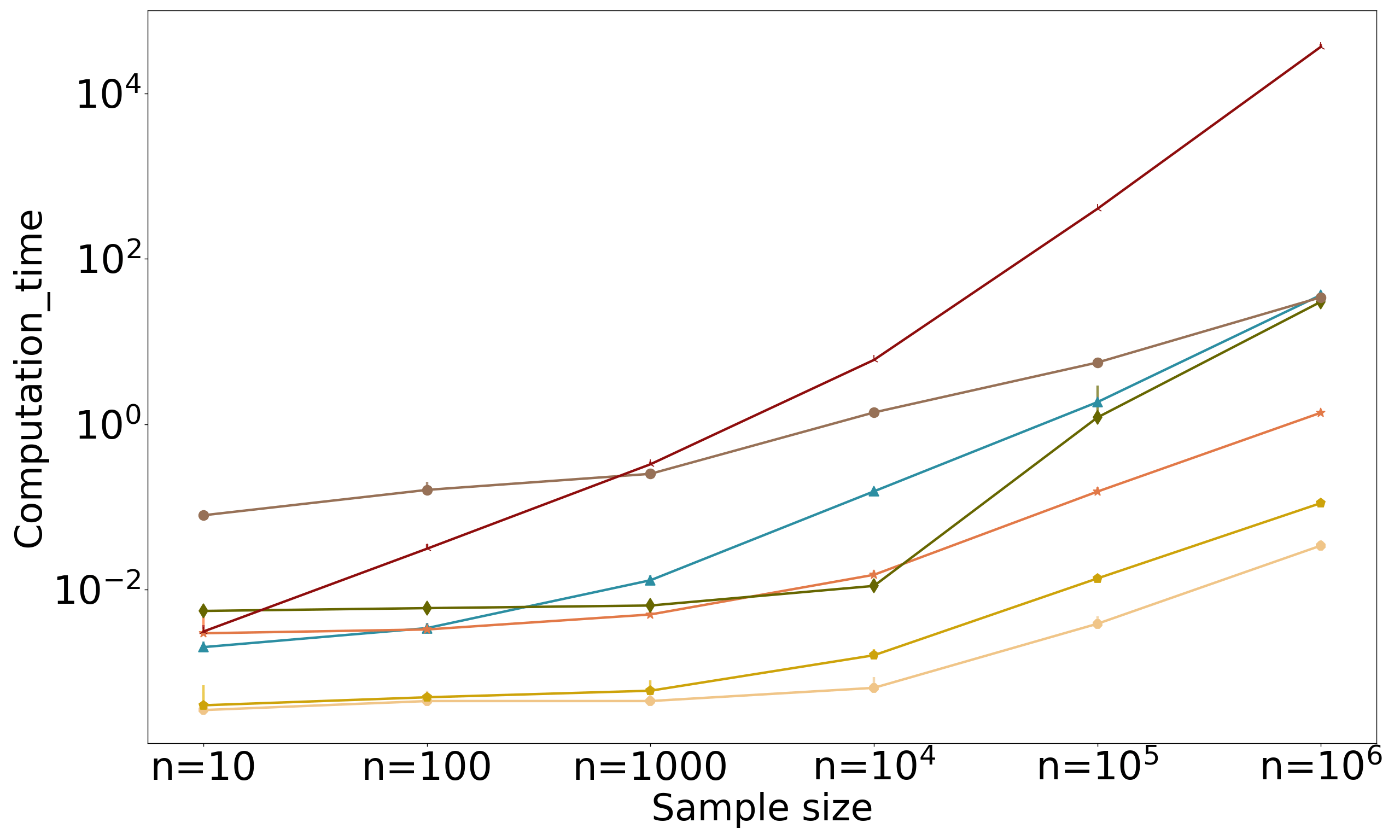}}
    \hfill
\subfloat[Hellinger distance,
          \label{fig:hd-others-uniform}]{\includegraphics{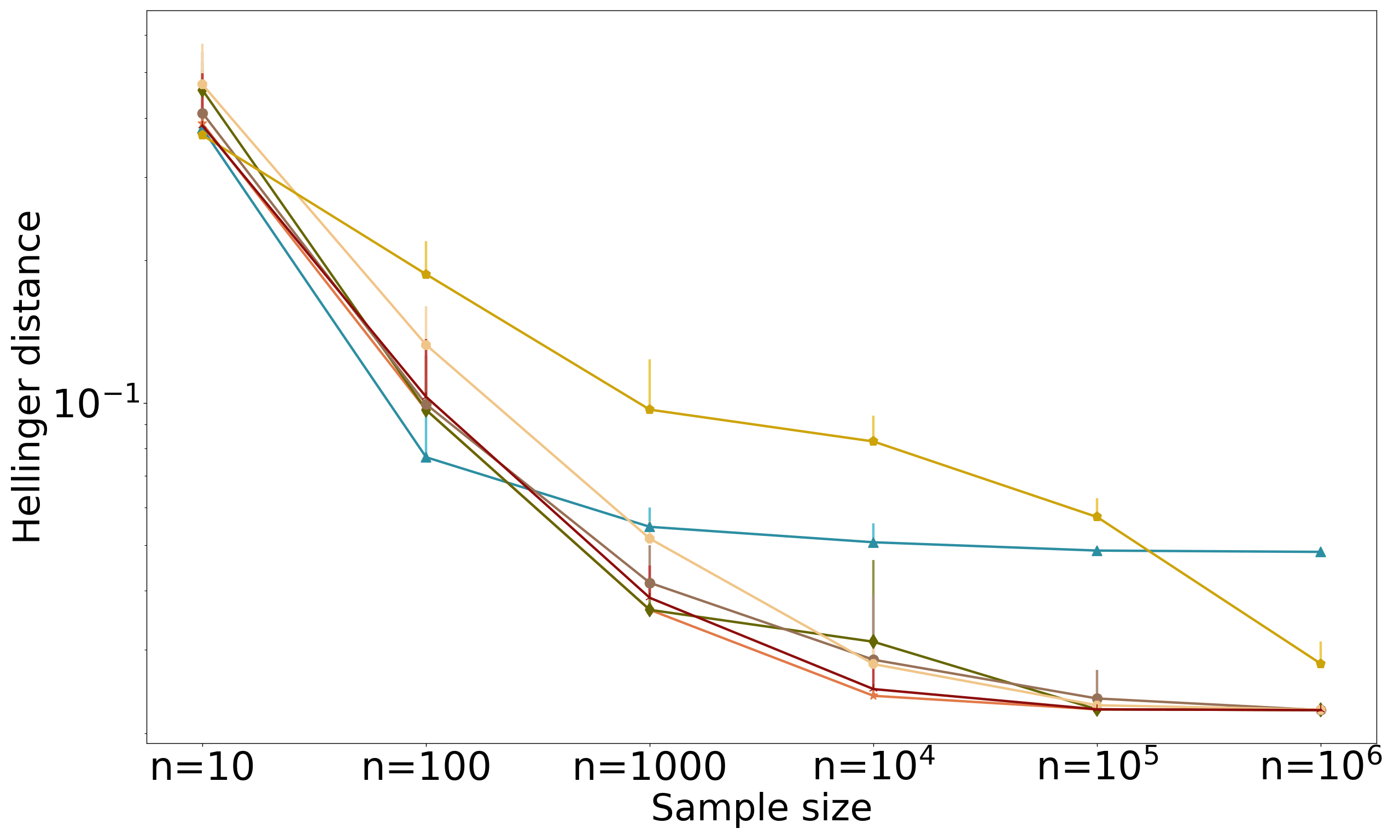}}
\end{figure}

\begin{figure}
  \centering  
  \caption{Different histograms obtained for a single Triangle distribution of size $n=10^4$ \label{fig:triangle-hists}}
  \setkeys{Gin}{width=0.25\textwidth}
  \subfloat[A triangle distribution ]
  {\includegraphics{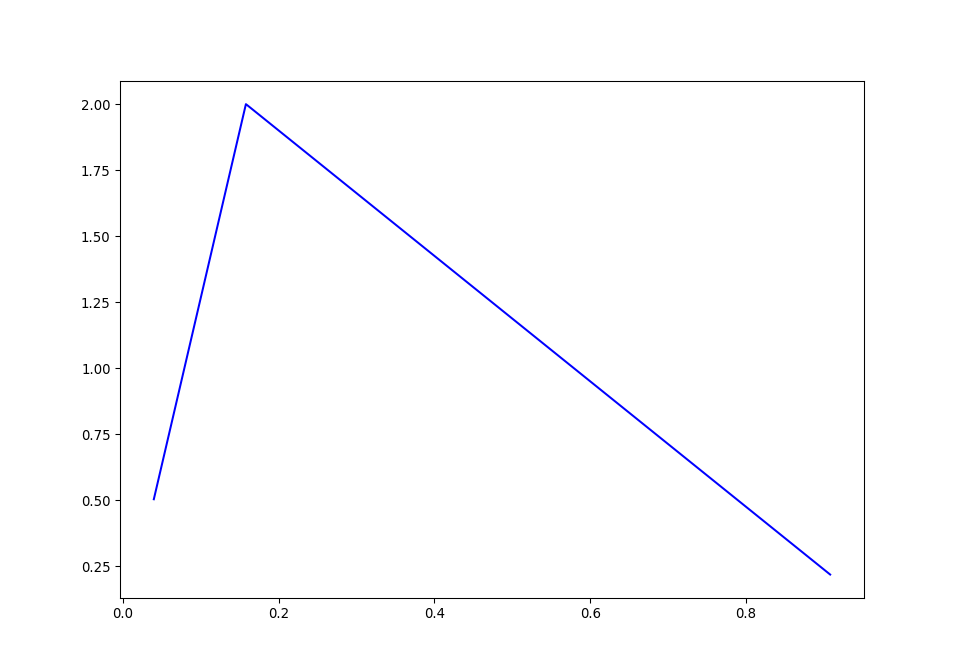}}%
	\hfill
	\subfloat[NML + heuristic ($K^*=13$)]
  {\includegraphics{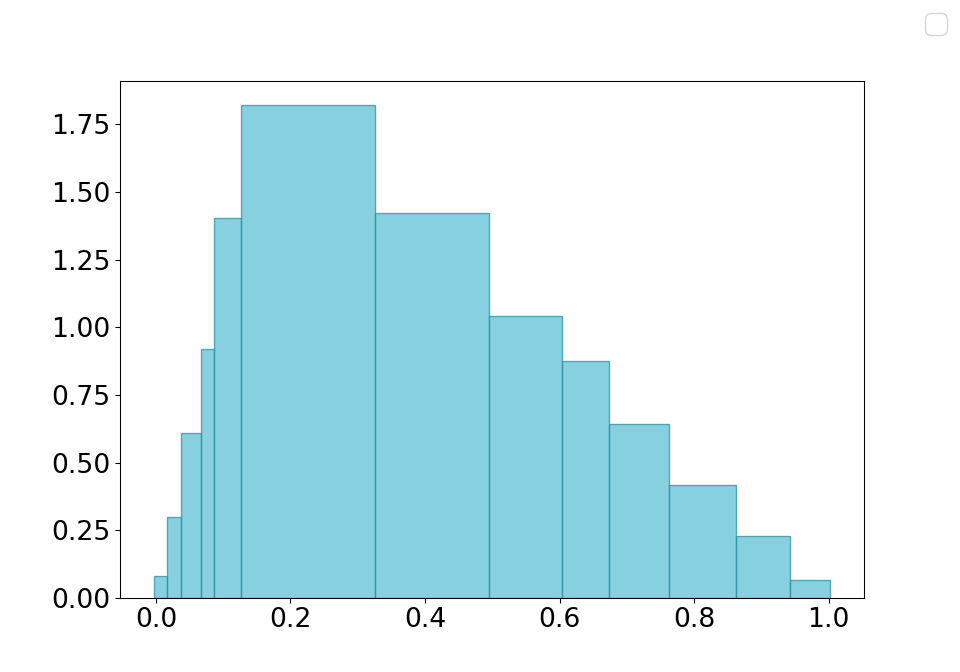}}%
	\hfill
	\subfloat[\GENUMname ($K^*=11$)]
  {\includegraphics{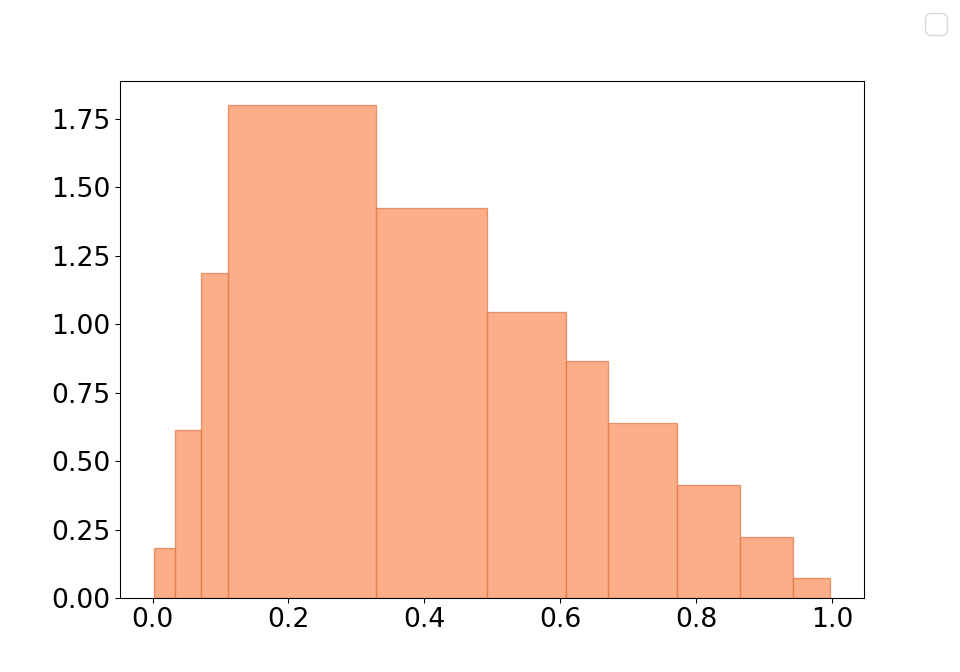}}%
	\hfill
	\subfloat[Bayesian blocks ($K^*=11$)]
  {\includegraphics{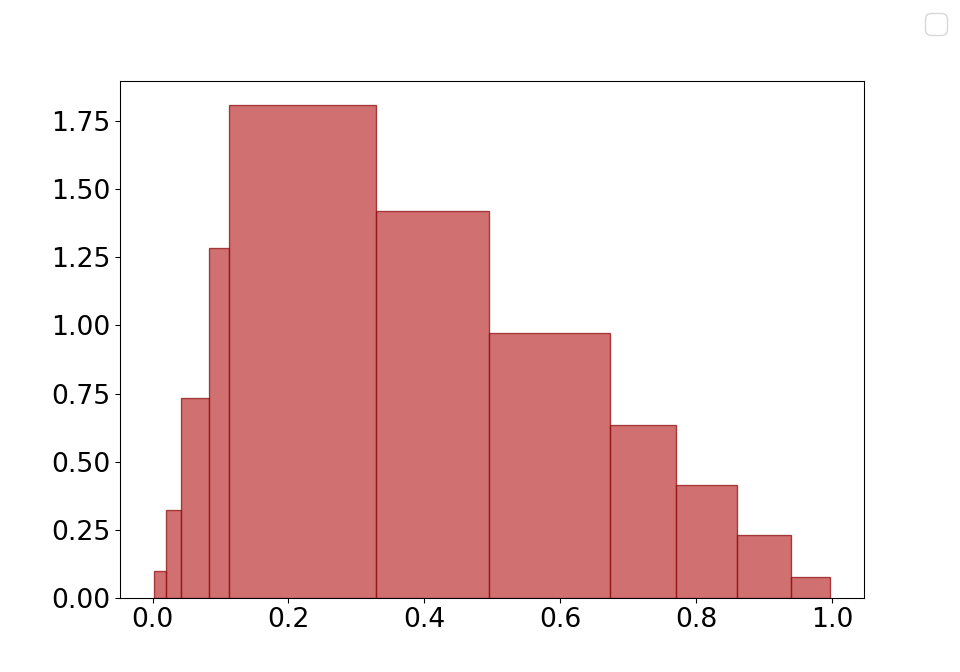}}%
	\hfill
	\subfloat[Taut string ($K^*= 43$)]
  {\includegraphics{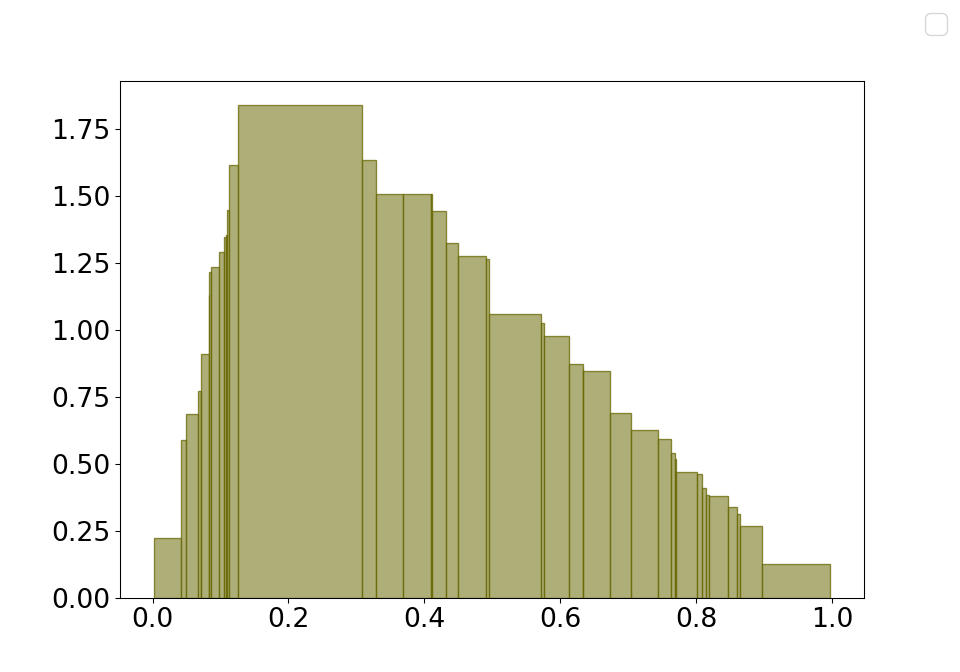}}%
	\hfill
	\subfloat[RMG ($K^*=25$)]
  {\includegraphics{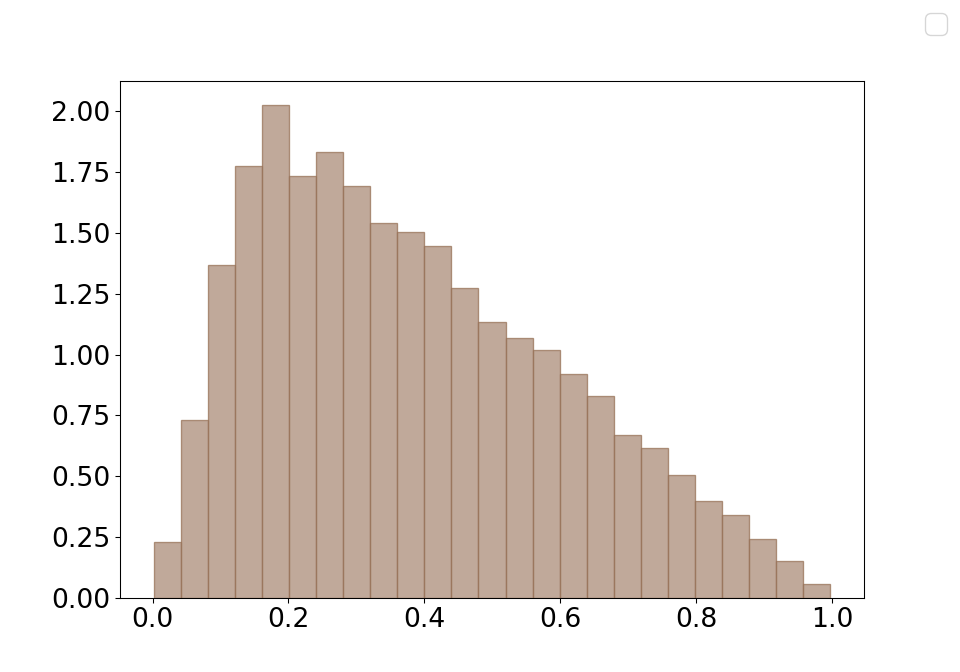}}%
	\hfill
	\subfloat[Sturges ($K^*=15$)]
  {\includegraphics{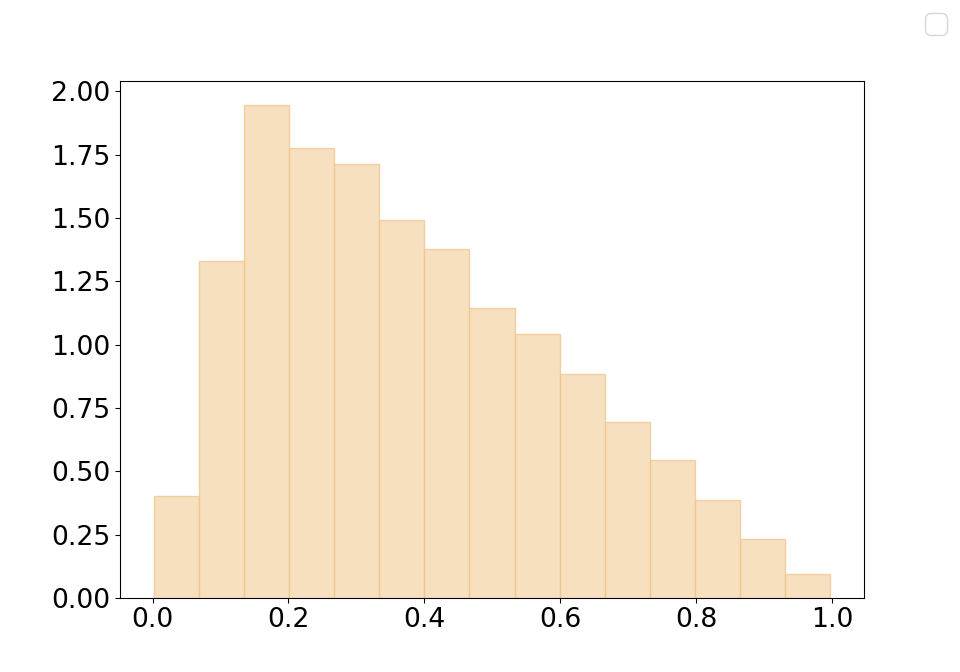}}%
	\hfill
	\subfloat[Freedman-Diaconis ($K^*= 33$)]
  {\includegraphics{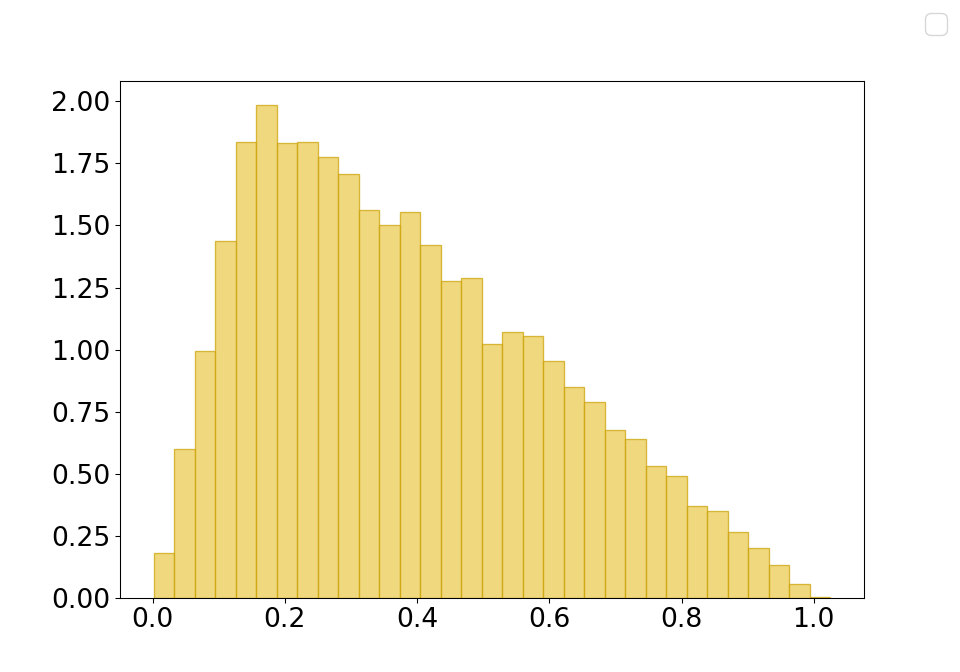}}
\end{figure}

\begin{figure}
\centering
\caption{Comparison with state-of-the-art methods over a Triangle distribution of different sample sizes \label{fig:comparison-others-triangle}}
\setkeys{Gin}{width=0.3\textwidth}
\subfloat[Number of intervals,
          \label{fig:intervals-others-triangle}]{\includegraphics{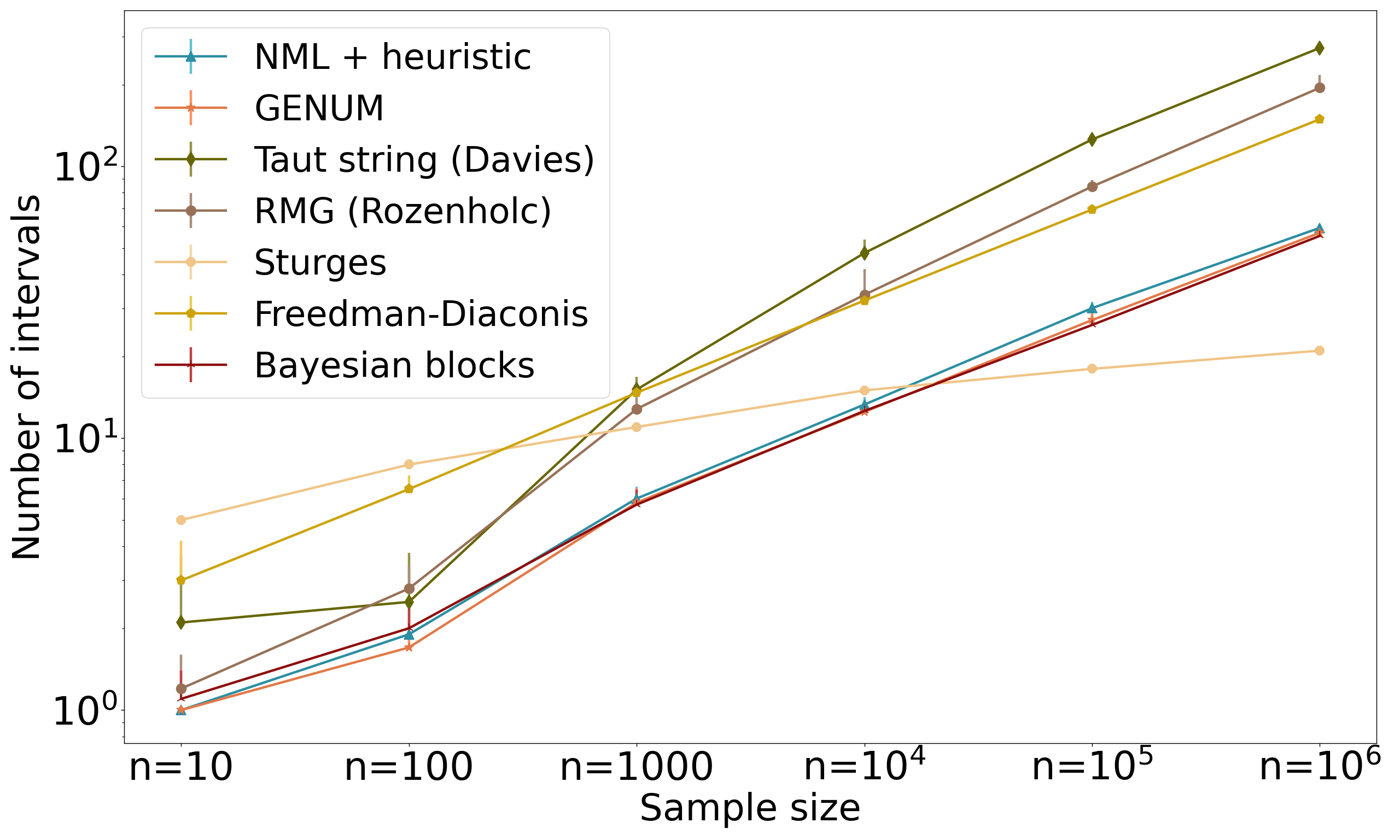}}
    \hfill
\subfloat[Computation time,
          \label{fig:time-others-triangle}]{\includegraphics{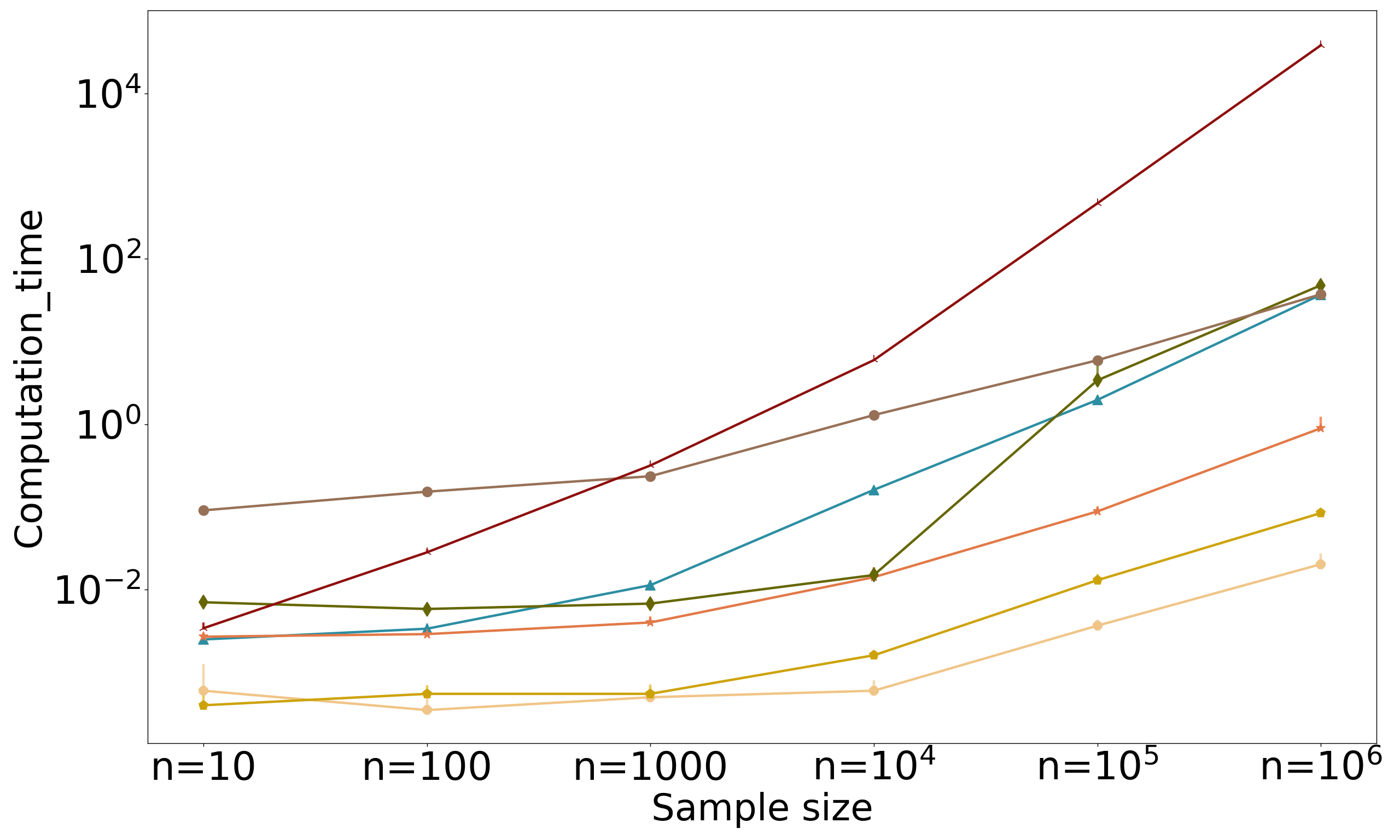}}
    \hfill
\subfloat[Hellinger distance,
          \label{fig:hd-others-triangle}]{\includegraphics{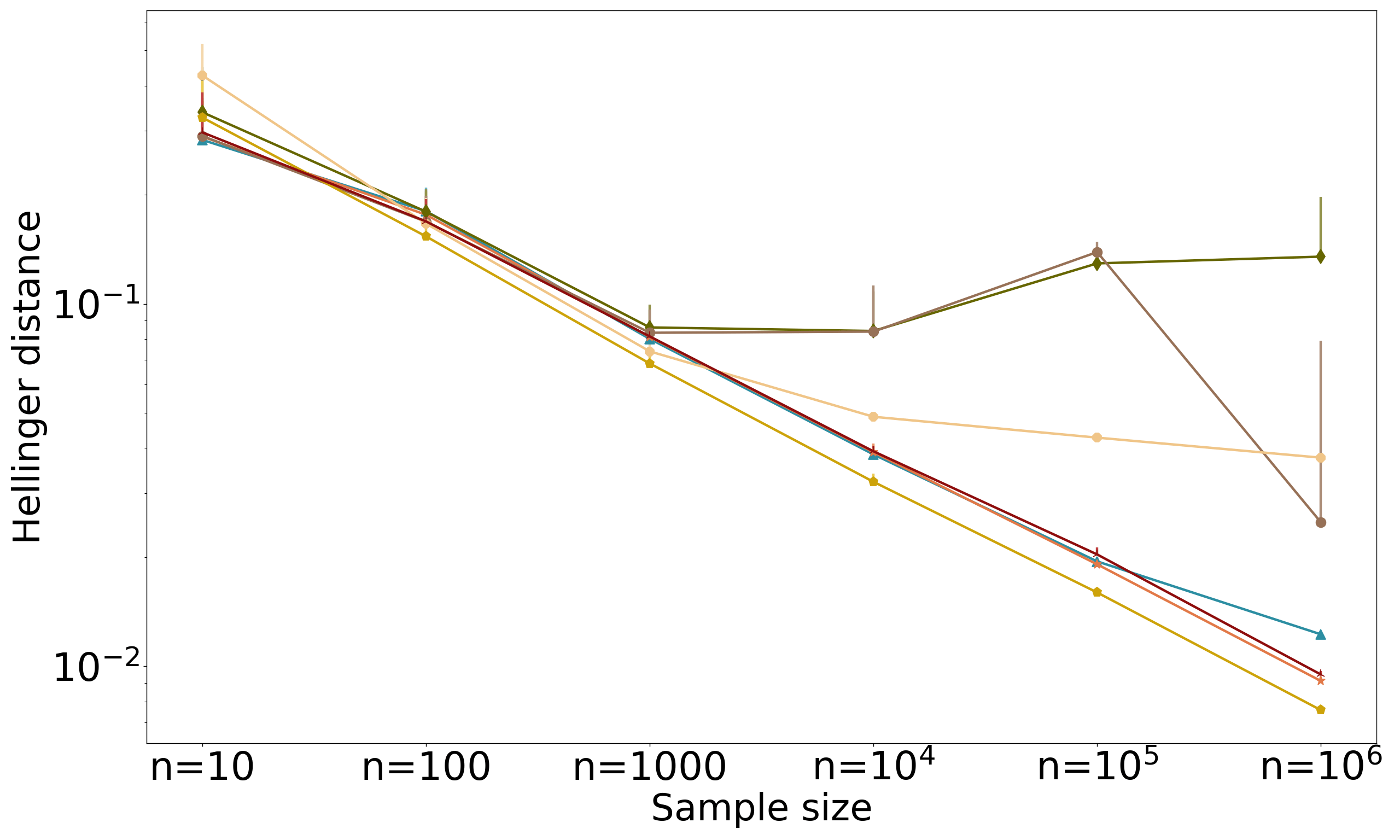}}

\end{figure}

\begin{figure}
  \centering  
  \caption{Different histograms obtained for a single mixture of 4 triangle distributions, of size $n=10^4$ \label{fig:tmix-hists}}
  \setkeys{Gin}{width=0.25\textwidth}
  \subfloat[A triangle mixture ]
  {\includegraphics{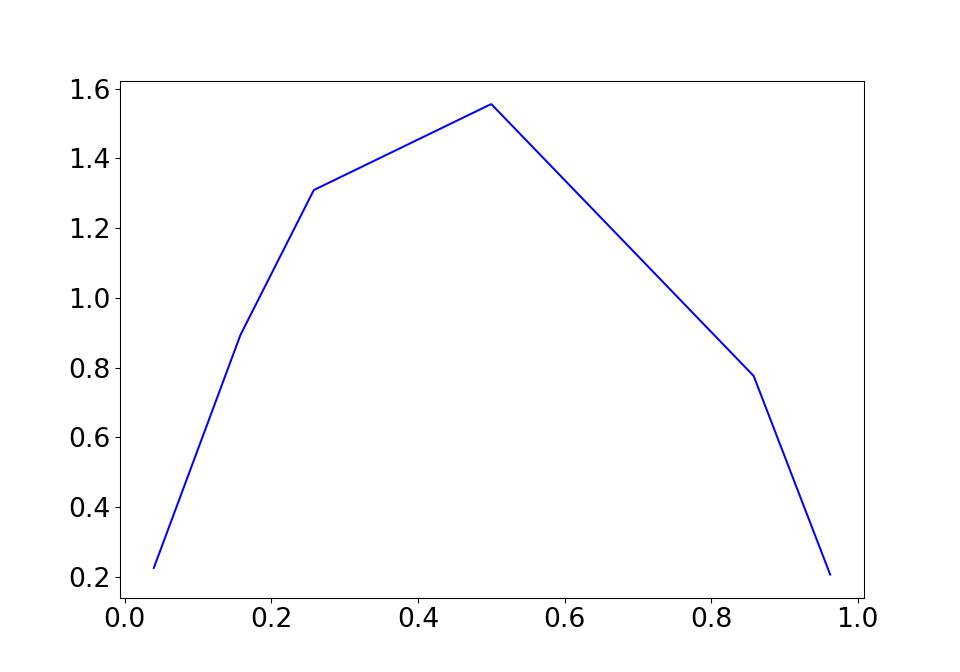}}%
	\hfill
	\subfloat[NML + heuristic ($K^*=12$)]
  {\includegraphics{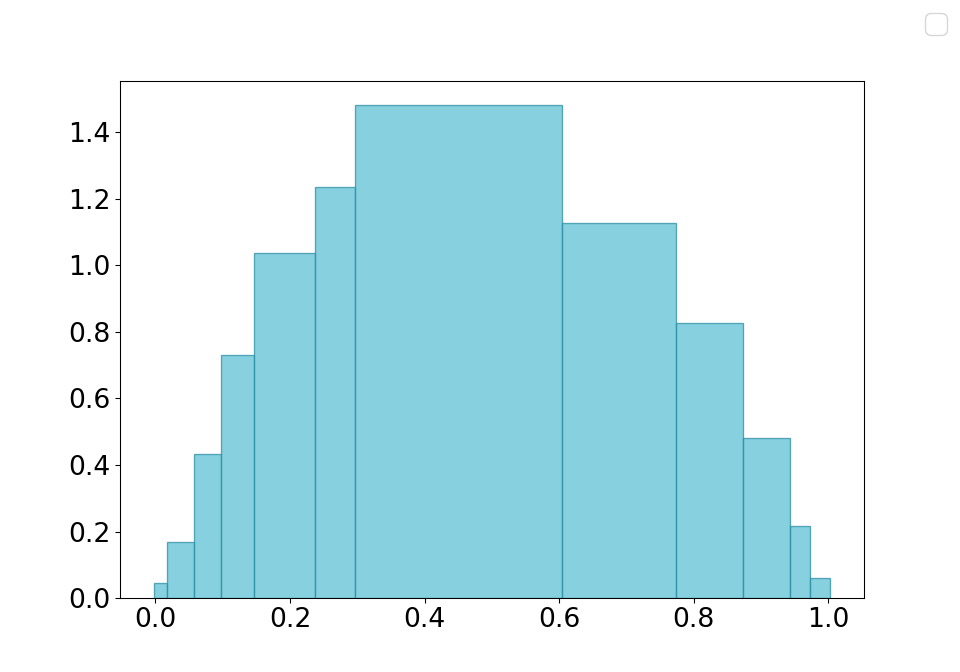}}%
	\hfill
	\subfloat[\GENUMname ($K^*=12$)]
  {\includegraphics{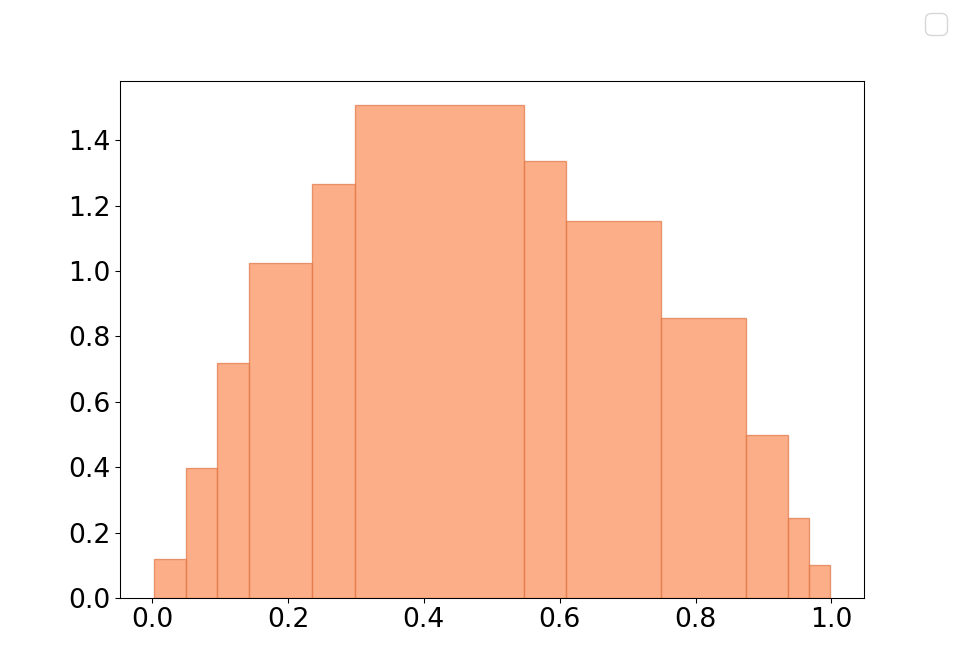}}%
	\hfill
	\subfloat[Bayesian blocks ($K^*=11$)]
  {\includegraphics{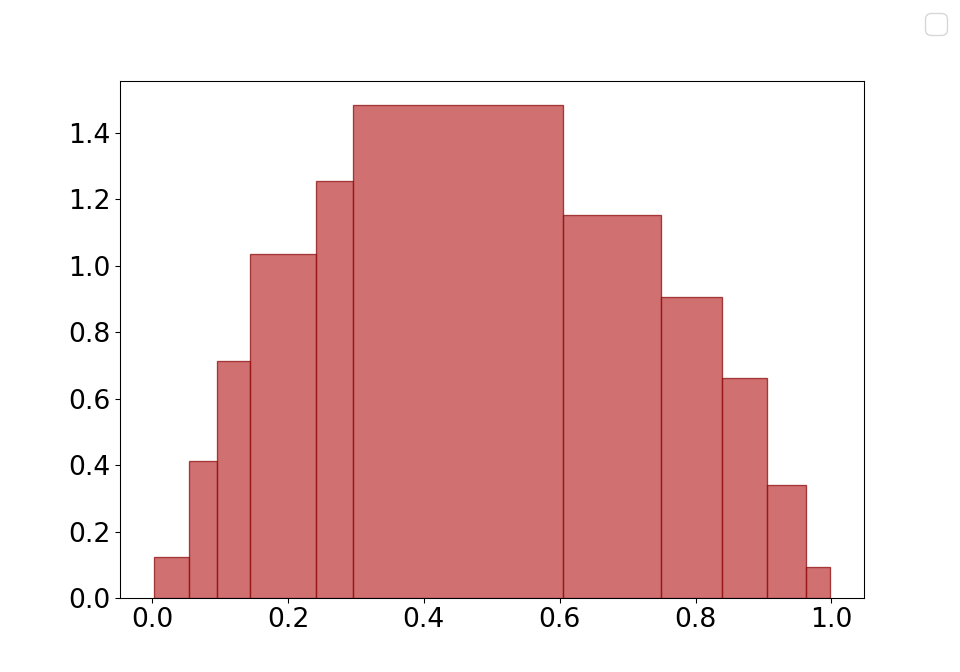}}%
	\hfill
	\subfloat[Taut string ($K^*= 53$)]
  {\includegraphics{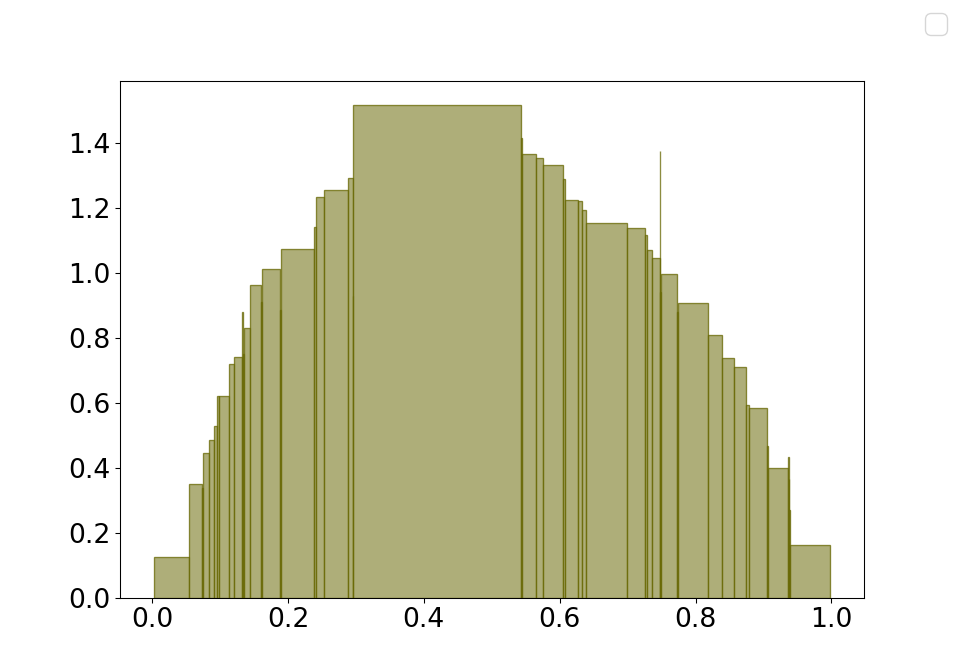}\label{fig:tmix-TS}}%
	\hfill
	\subfloat[RMG ($K^*=38$)]
  {\includegraphics{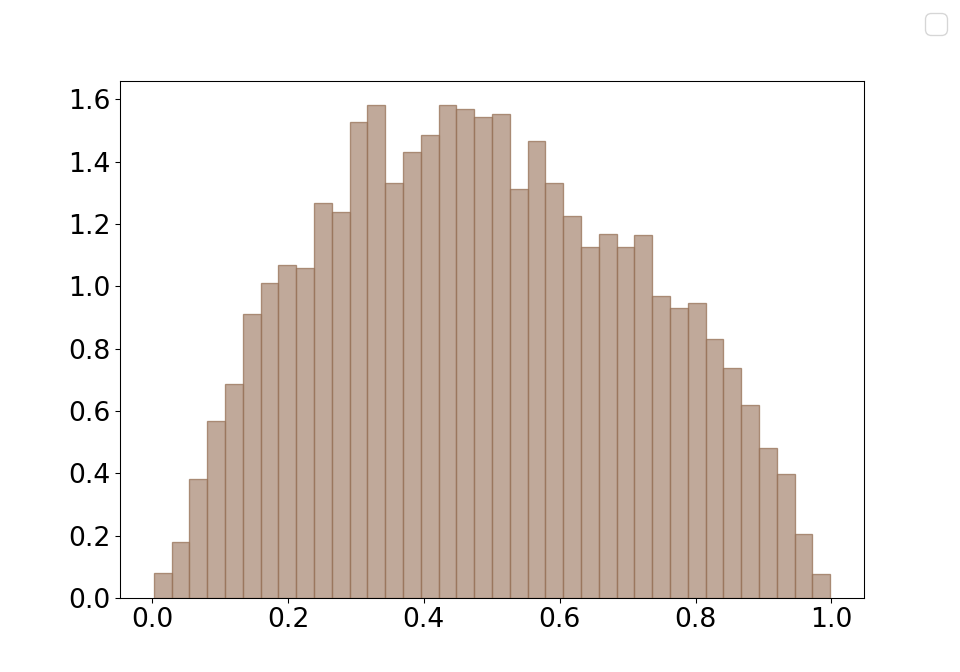}}%
	\hfill
	\subfloat[Sturges ($K^*=15$)]
  {\includegraphics{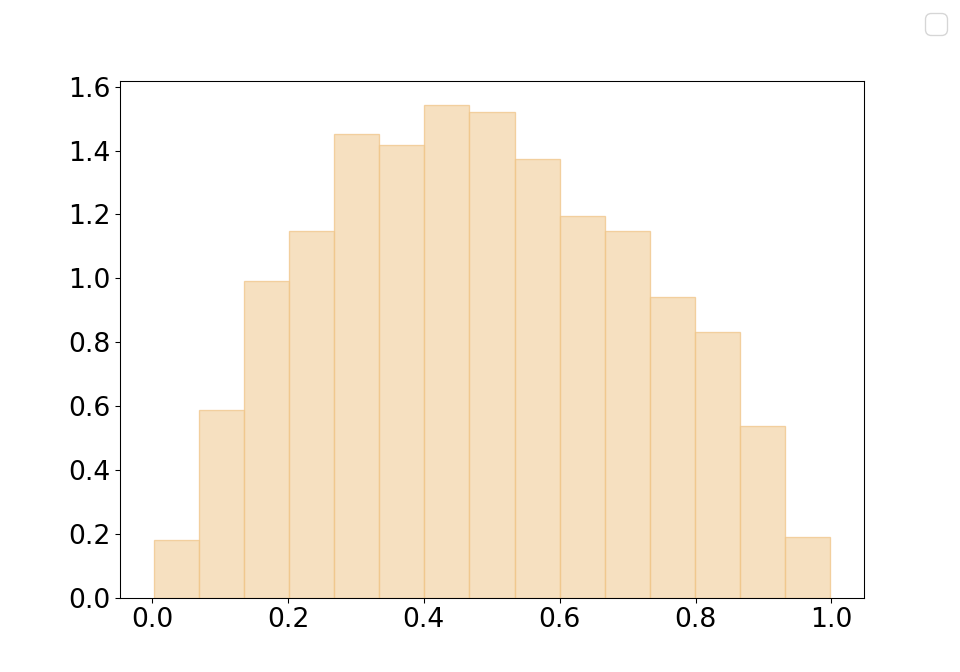}}%
	\hfill
	\subfloat[Freedman-Diaconis ($K^*= 31$)]
  {\includegraphics{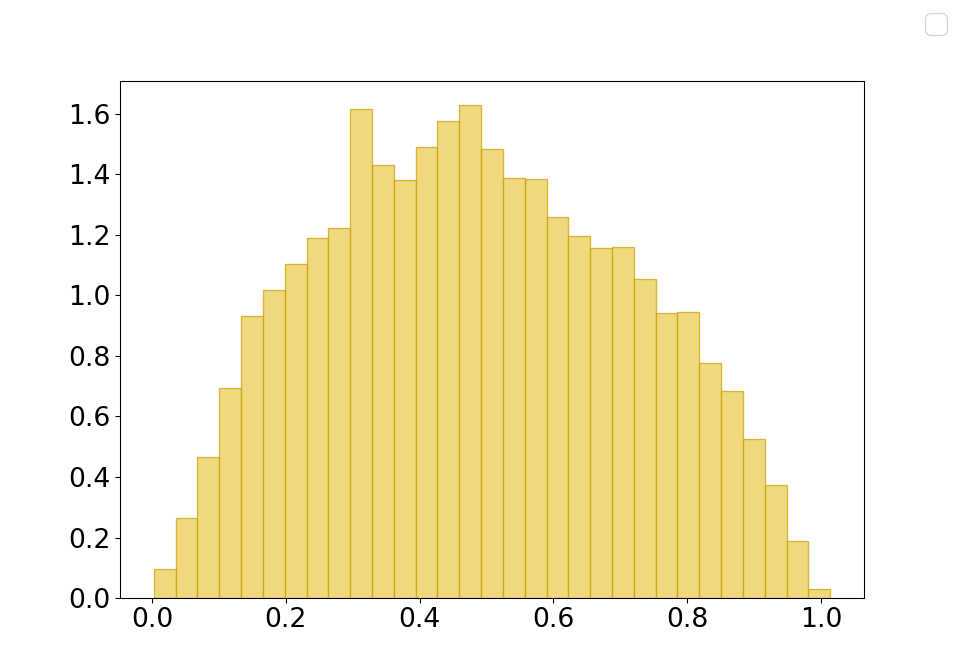}}
\end{figure}

\begin{figure}
\centering
\caption{Comparison with state-of-the-art methods over a mixture of 4 Triangle distributions, of different sample sizes \label{fig:comparison-others-tmix}}
\setkeys{Gin}{width=0.3\textwidth}
\subfloat[Number of intervals,
          \label{fig:intervals-others-tmix}]{\includegraphics{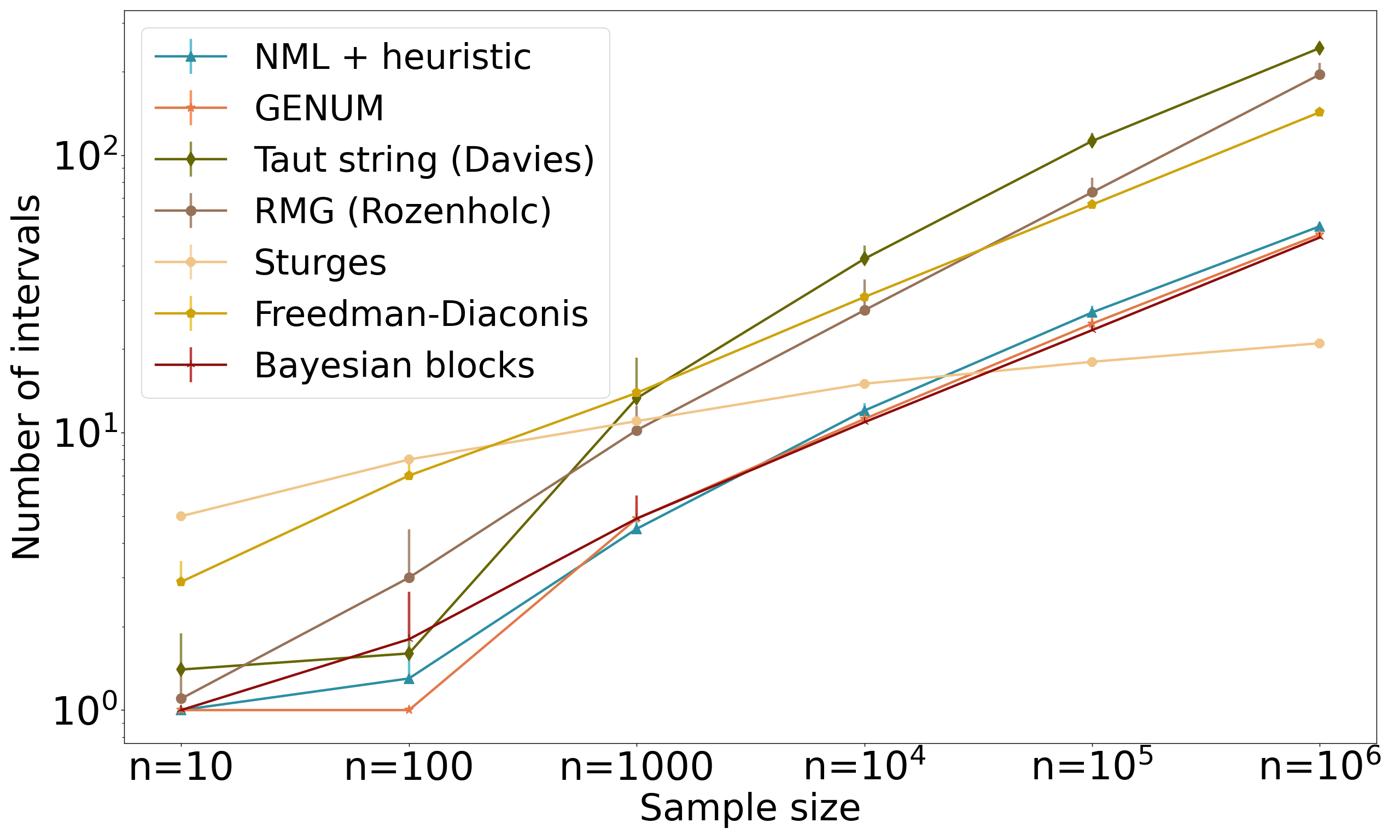}}
    \hfill
\subfloat[Computation time,
          \label{fig:time-others-tmix}]{\includegraphics{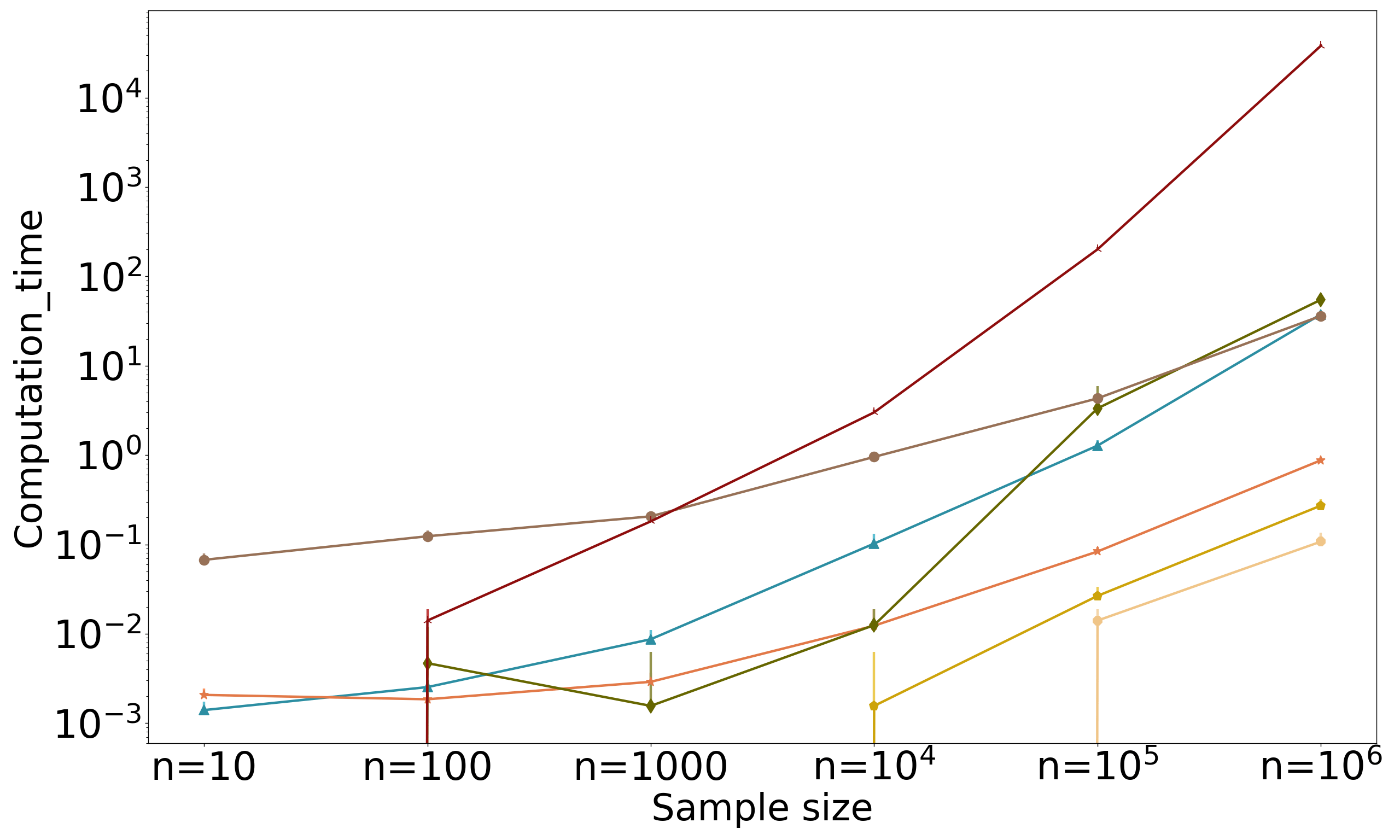}}
    \hfill
\subfloat[Hellinger distance,
          \label{fig:hd-others-tmix}]{\includegraphics{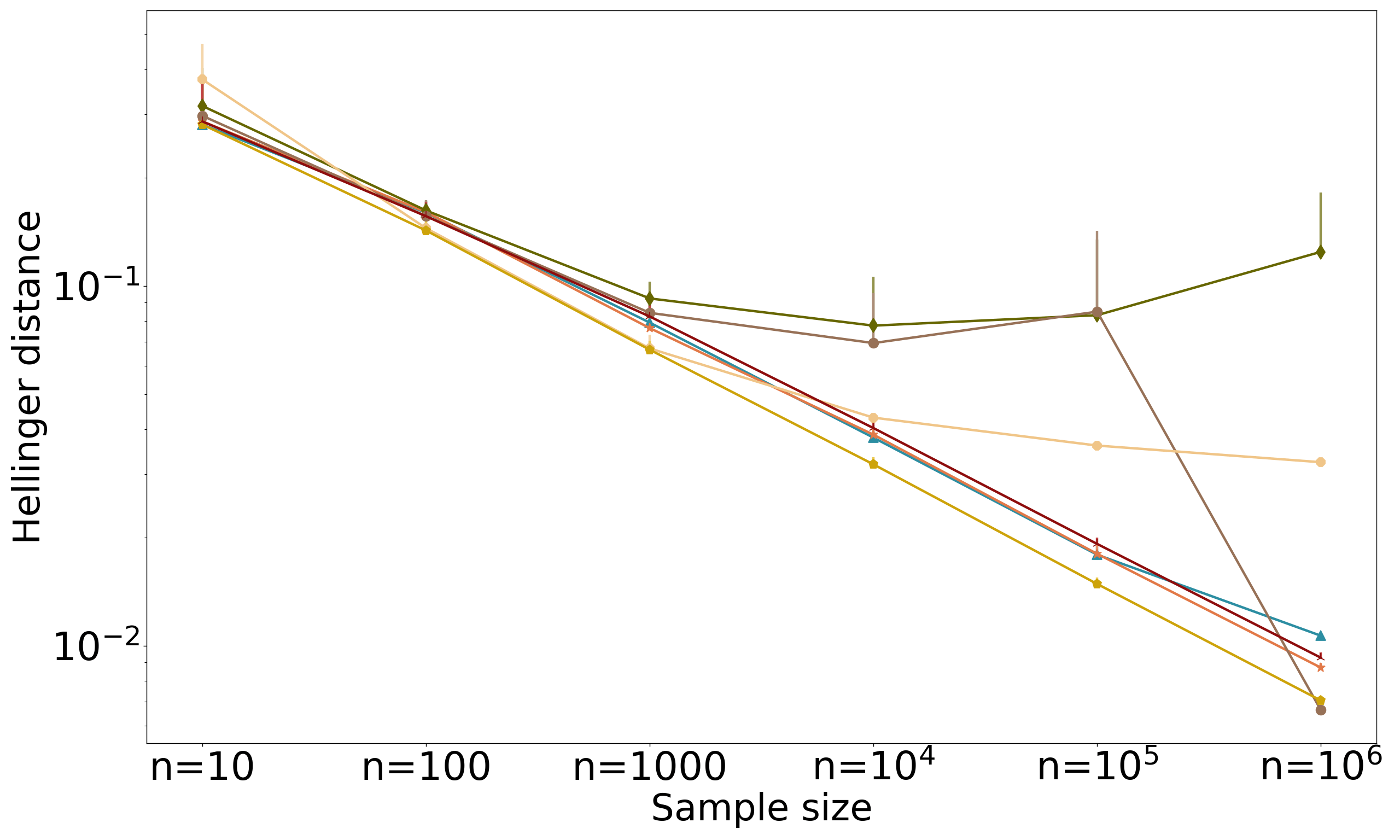}}

\end{figure}

 \begin{figure}
  \centering  
  \caption{Different histograms obtained for a single mixture of 4 triangle distributions, of size $n=10^4$ \label{fig:gmix-claw-hists}}
  \setkeys{Gin}{width=0.25\textwidth}
  \subfloat[Claw: a Gaussian mixture ]
  {\includegraphics{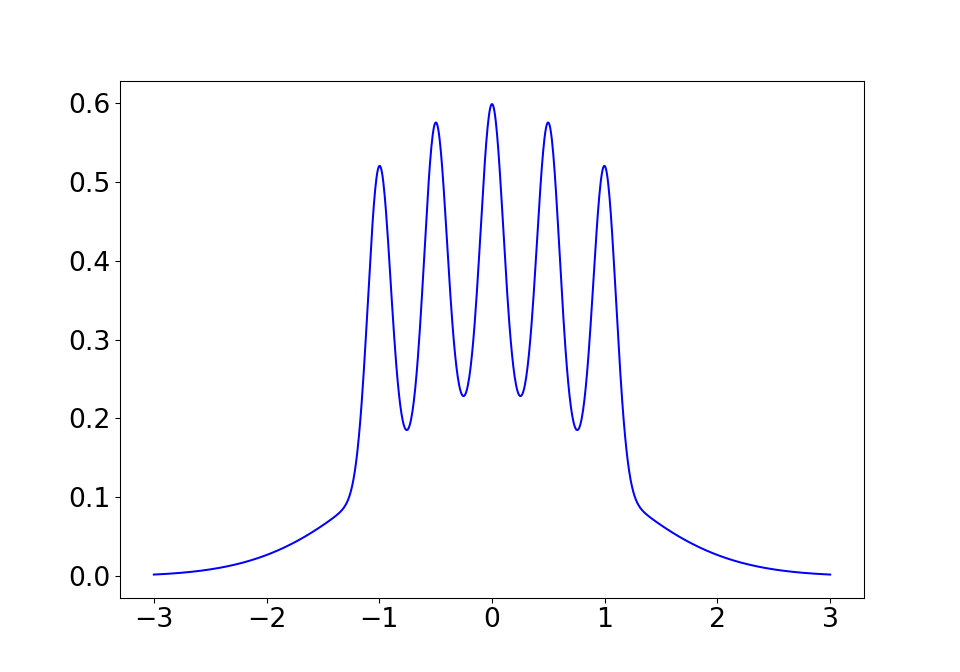}}%
	\hfill
	\subfloat[NML + heuristic ($K^*=44$)]
  {\includegraphics{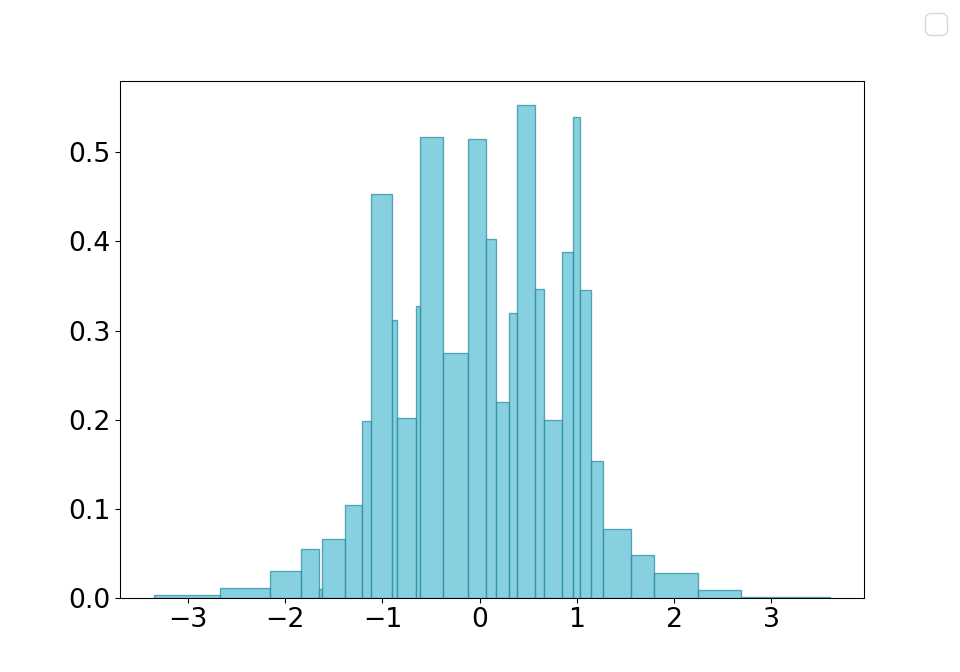}}%
	\hfill
	\subfloat[\GENUMname ($K^*=43$)]
  {\includegraphics{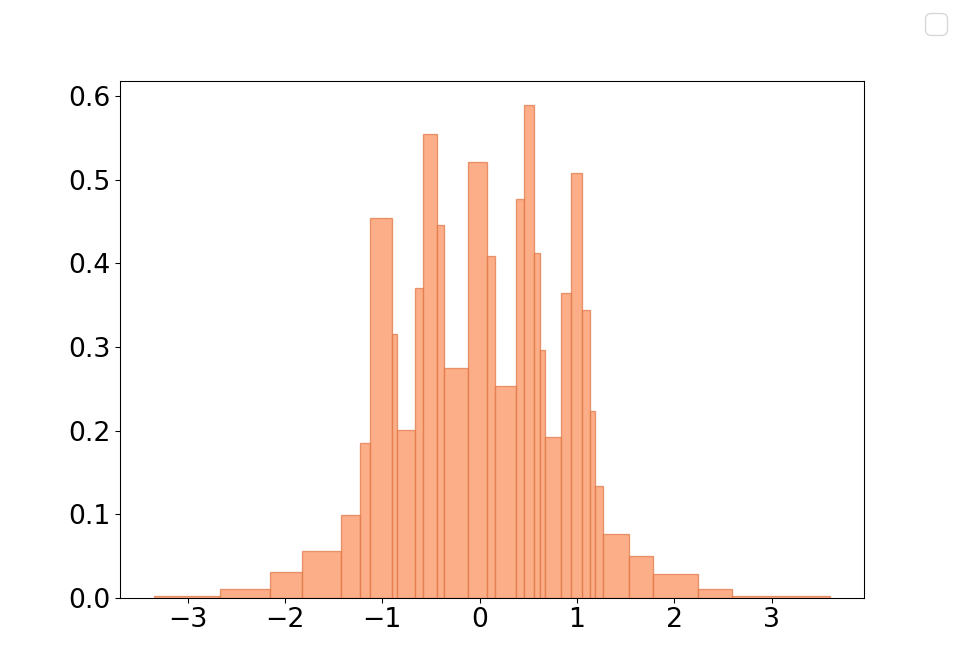}}%
	\hfill
	\subfloat[Bayesian blocks ($K^*=46$)]
  {\includegraphics{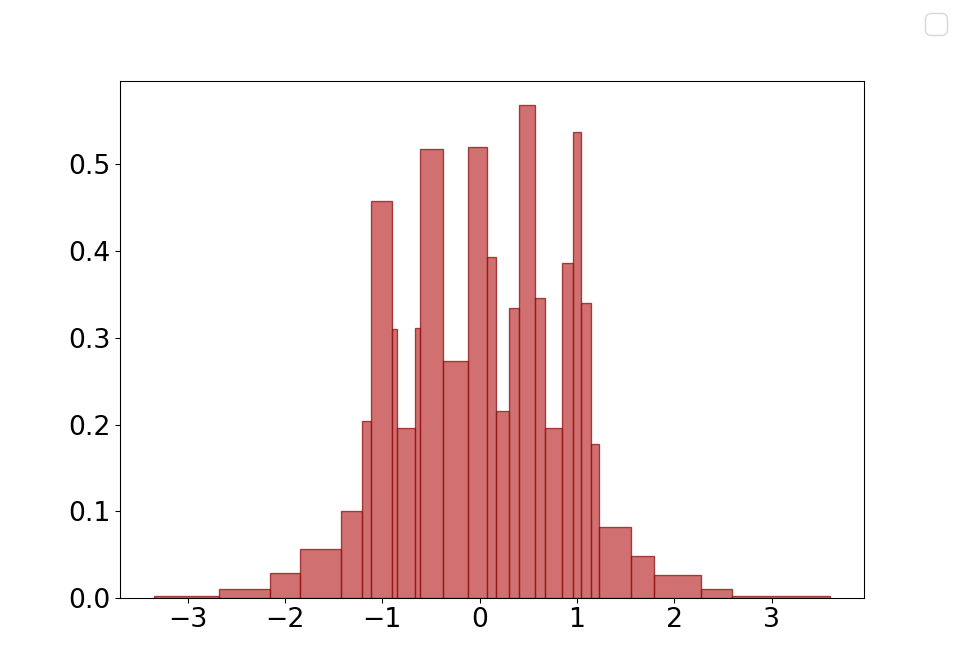}}%
	\hfill
	\subfloat[Taut string ($K^*= 205$)]
  {\includegraphics{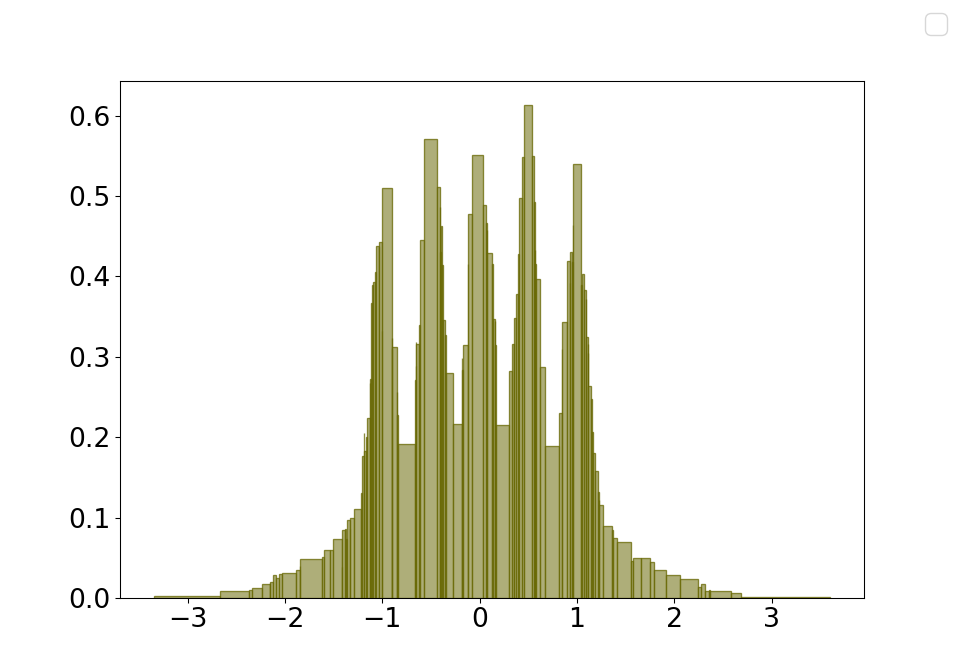}}%
	\hfill
	\subfloat[RMG ($K^*=50$)]
  {\includegraphics{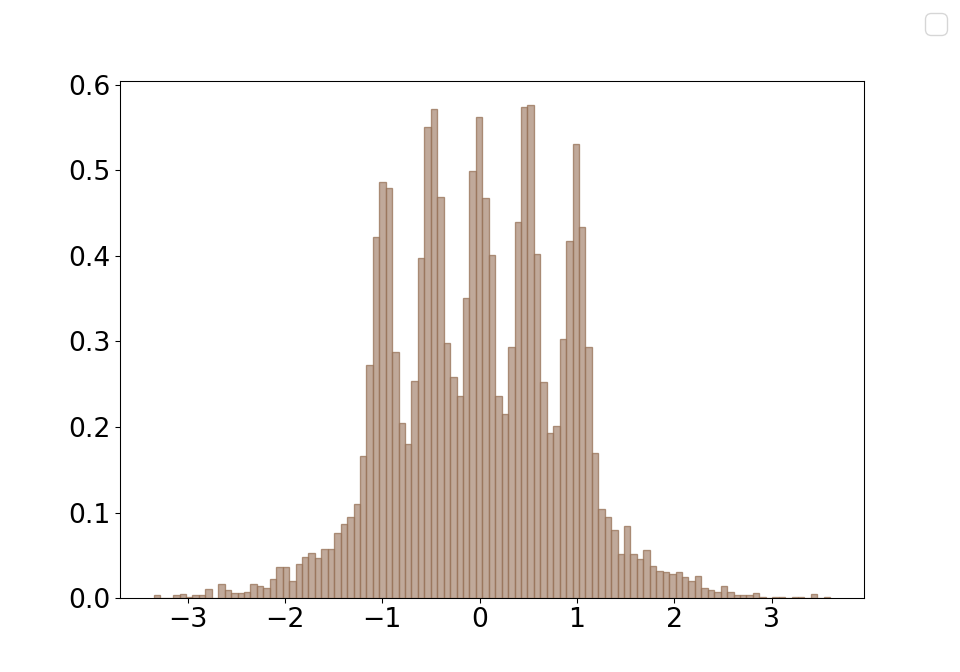}}%
	\hfill
	\subfloat[Sturges ($K^*=15$)]
  {\includegraphics{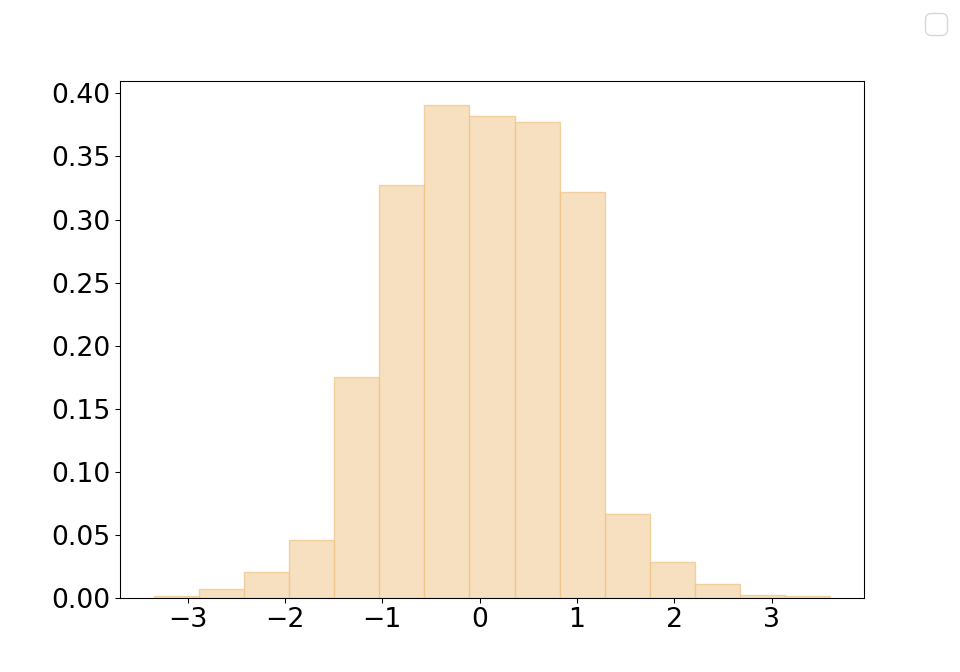}}%
	\hfill
	\subfloat[Freedman-Diaconis ($K^*= 24$)]
  {\includegraphics{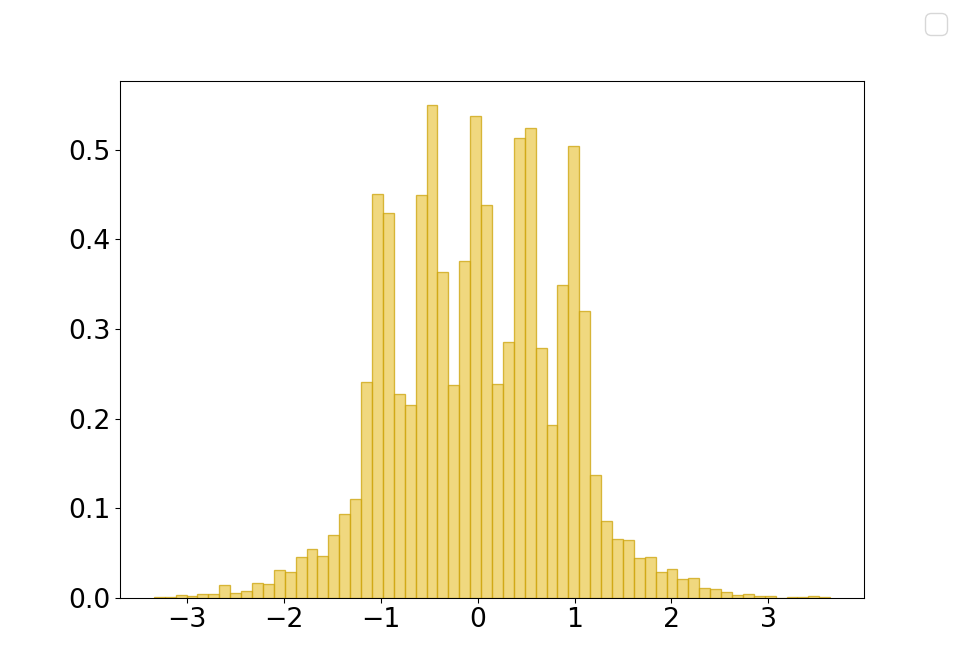}}
\end{figure}

\begin{figure}
\centering
\caption{Comparison with state-of-the-art methods over the Claw Gaussian mixture of different sample sizes \label{fig:comparison-others-claw}}
\setkeys{Gin}{width=0.3\textwidth}
\subfloat[Number of intervals,
          \label{fig:intervals-others-claw}]{\includegraphics{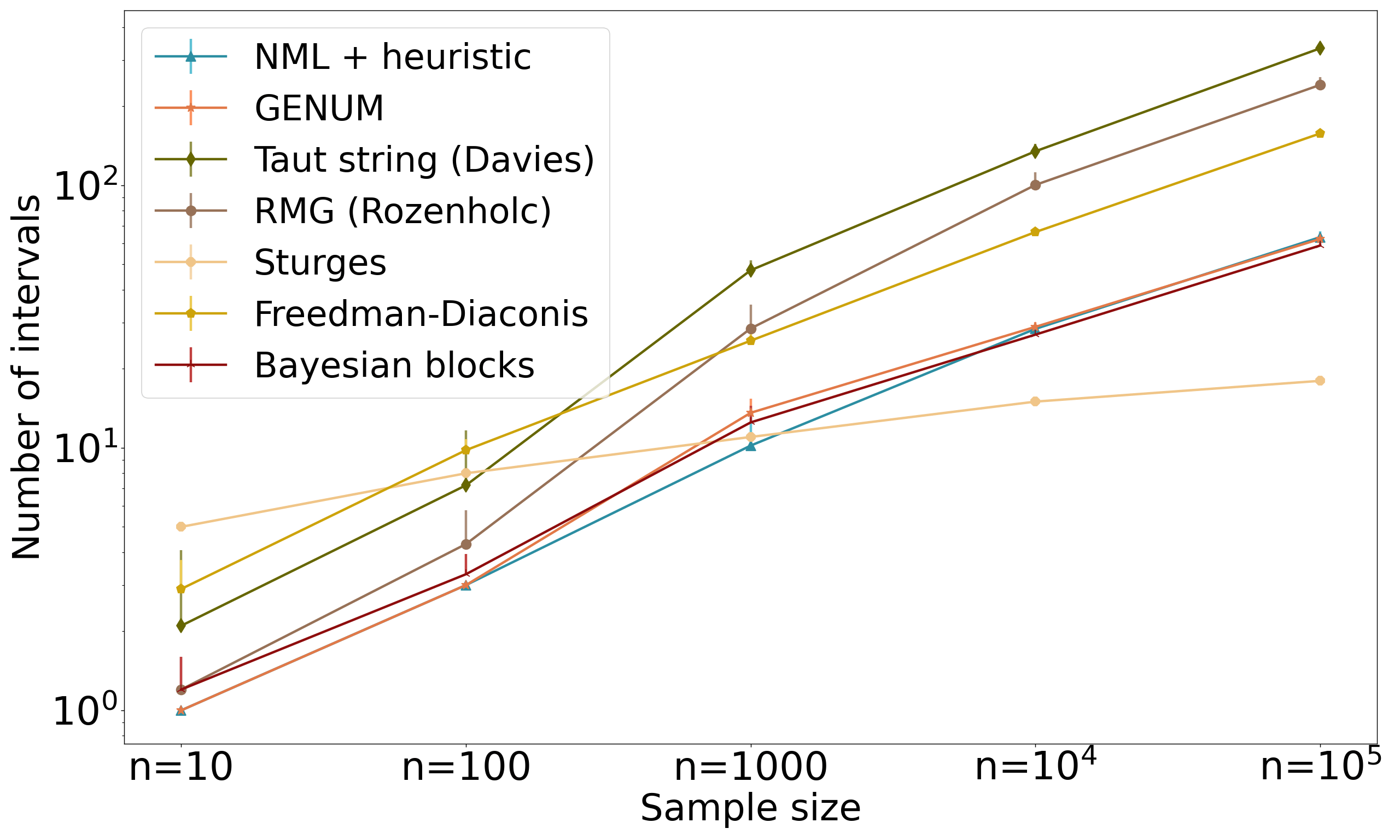}}
    \hfill
\subfloat[Computation time,
          \label{fig:time-others-claw}]{\includegraphics{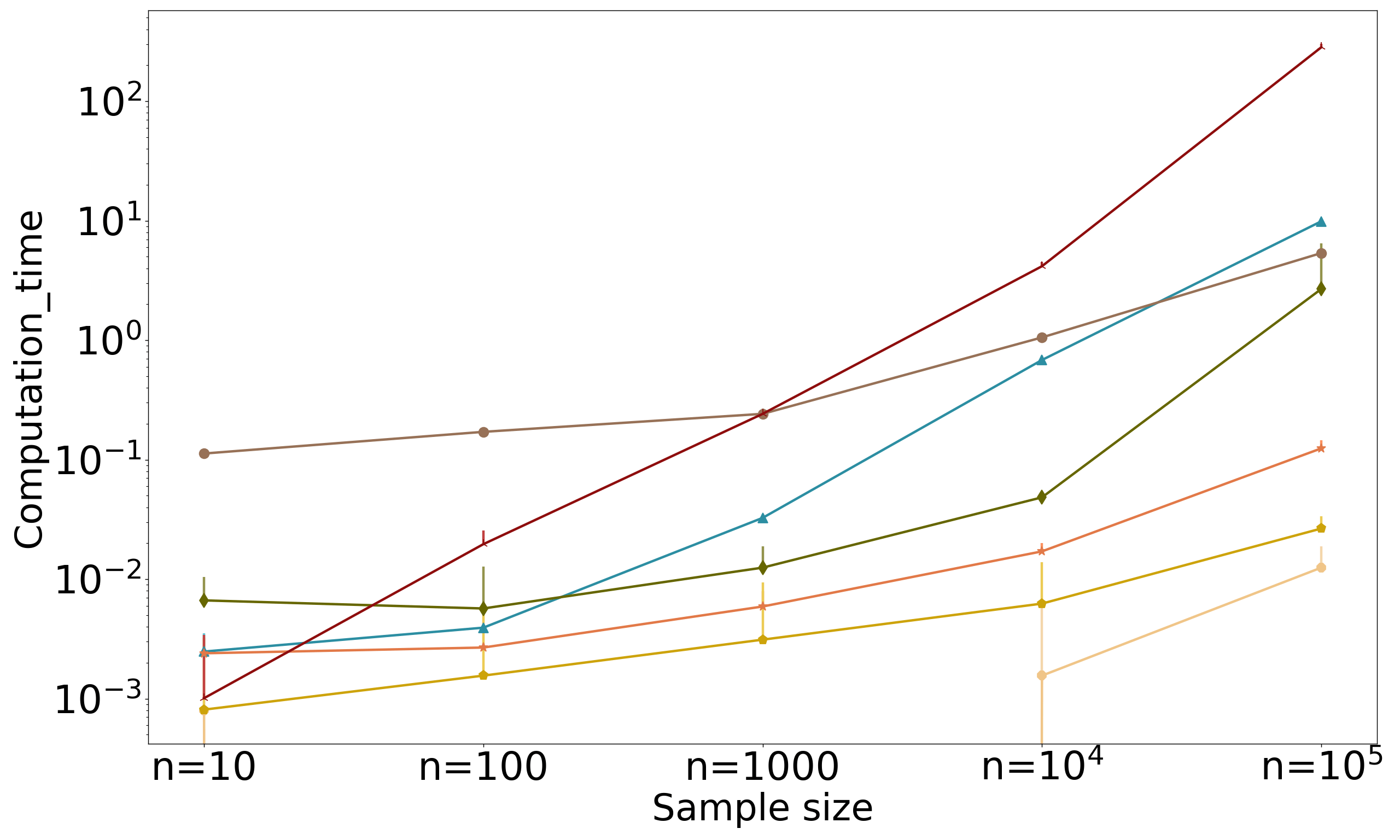}}
    \hfill
\subfloat[Hellinger distance,
\label{fig:hd-others-claw}]{\includegraphics{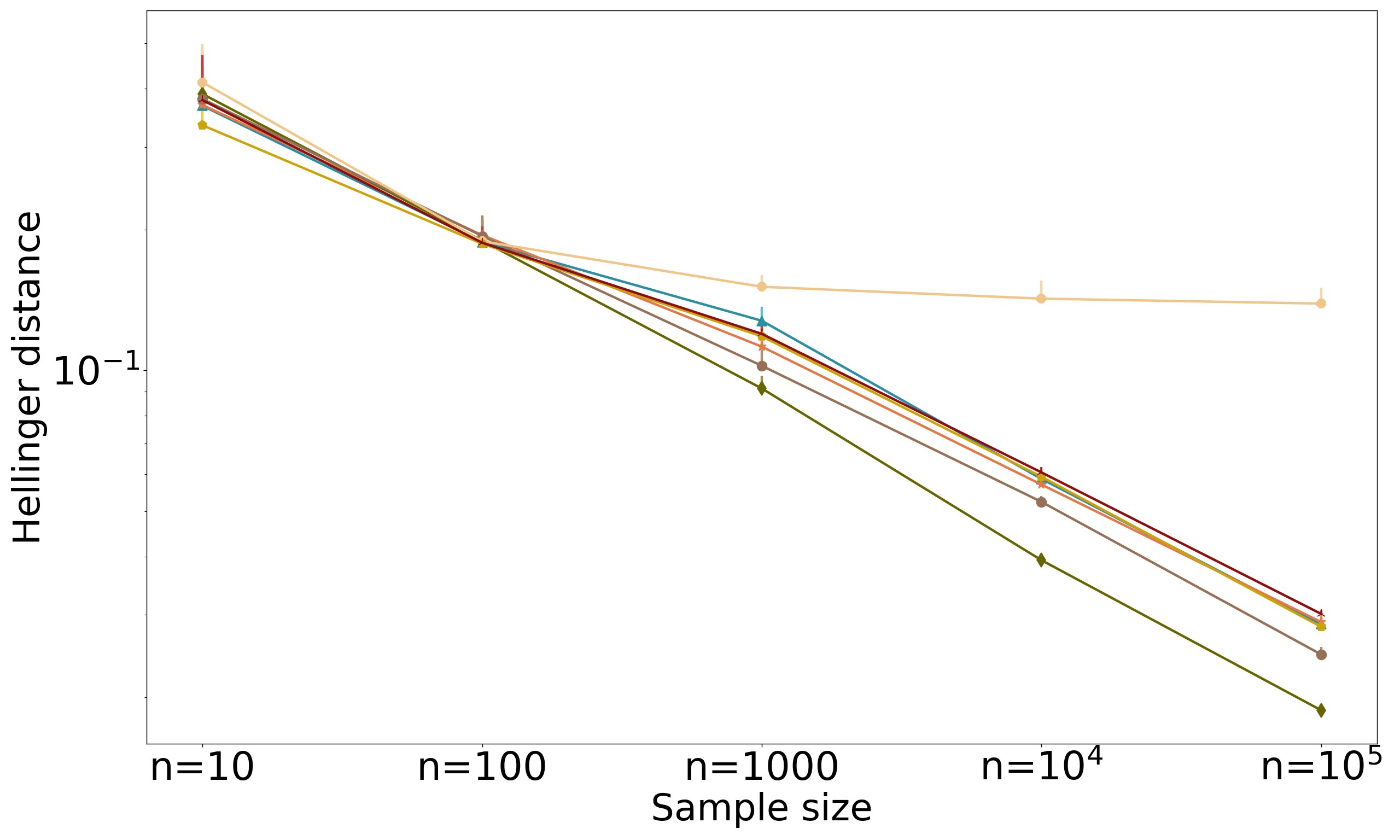}}
\end{figure}

\FloatBarrier
\mbox{}
\clearpage
        
\end{document}